\pdfoutput=1
\documentclass{article} %
\usepackage[utf8]{inputenc}
\usepackage[T1]{fontenc}
\usepackage{iclr2027_conference,times}

\usepackage{amsmath,amssymb,amsthm,mathtools}
\usepackage{mathptmx}
\usepackage{bm}
\usepackage{hyperref}
\hypersetup{hypertexnames=false,colorlinks=true,
  linkcolor=[rgb]{0.55,0.00,0.00},   %
  citecolor=[rgb]{0.00,0.00,0.75},   %
  urlcolor=[rgb]{0.00,0.00,0.75}}
\hypersetup{pdftitle={Continuous-Time Motion Planning and Steering with Global Convergence},pdfauthor={Mohsen Amiri, Sima Didari, Davoud Ataee Tarzanagh, Hankyu Moon, Yeongdae Kwon}}
\usepackage{url}
\usepackage{xcolor}
\usepackage[most]{tcolorbox}
\tcbset{highlight math style={enhanced,colback=blue!3!white,colframe=blue!50!black,
  boxrule=0.6pt,arc=2pt,left=6pt,right=6pt,top=3pt,bottom=3pt}}
\usepackage{colortbl}
\usepackage{graphicx}
\usepackage{algorithm}
\usepackage{algpseudocode}
\algrenewcommand{\algorithmiccomment}[1]{\hfill{\footnotesize\itshape\color{black!60}#1}}
\usepackage{booktabs}
\usepackage{placeins}
\usepackage{enumitem}
\usepackage{float}
\usepackage{wrapfig}
\definecolor{todocol}{RGB}{180,60,40}
\definecolor{rowsmooth}{RGB}{30,110,200}\definecolor{rowcoll}{RGB}{192,57,43}
\definecolor{rowsteer}{RGB}{192,122,42}\definecolor{rowbound}{RGB}{63,138,104}

\newcommand{\m}[1]{\mathbf{#1}}          %
\newcommand{\g}[1]{\boldsymbol{#1}}      %

\newcommand{\R}{\mathbb{R}}

\newcommand{\tr}{\operatorname{tr}}
\newcommand{\Gm}{\g{\Gamma}}
\newcommand{\dist}{\operatorname{dist}}

\newcommand{\free}{\mathrm{free}}
\newcommand{\trajspace}{\mathcal{H}}
\newcommand{\feasset}{\mathcal{F}}
\newcommand{\affset}{\mathcal{A}_{\mathrm{bc}}}

\newtheorem{theorem}{Theorem}
\newtheorem{proposition}[theorem]{Proposition}
\newtheorem{lemma}[theorem]{Lemma}
\newtheorem{corollary}[theorem]{Corollary}

\newtheorem{assumption}{Assumption}

\theoremstyle{definition}

\theoremstyle{definition}\newtheorem{problem}{Problem}

\theoremstyle{plain}
\newtheorem{maintheorem}[theorem]{Theorem}

\title{Continuous-Time Motion Planning\\ and Steering with Global Convergence}

\author{Mohsen Amiri$^{\dagger}$, Sima Didari$^{\dagger}$, Davoud Ataee Tarzanagh$^{\dagger}$, Hankyu Moon, Yeongdae Kwon \\
{\normalfont Samsung SDS Research America, Mountain View, CA, USA}\hspace{7.5em}%
\raisebox{-0.85cm}[0pt][0pt]{\href{https://globalmotionplanning.github.io}{\textbf{Project Page}}}}

\iclrfinalcopy %
\makeatletter
\renewcommand*\l@subsection{\@dottedtocline{2}{1.5em}{3.0em}}
\renewcommand*\l@subsubsection{\@dottedtocline{3}{4.5em}{3.6em}}
\makeatother

\begin{document}
\addtocontents{toc}{\protect\setcounter{tocdepth}{-1}}%

\maketitle
{\renewcommand{\thefootnote}{$\dagger$}\footnotetext{Equal contribution.}}
\fancyhead{}\renewcommand{\headrulewidth}{0pt}%
\FloatBarrier

\begin{abstract}
Motion planning is central to robotics, and it is increasingly used to
correct the motions proposed by learned planners and policies. In both uses, constraints must hold at every
instant of the motion, not only at sampled times. To our knowledge, no
algorithm provably converges to a globally optimal motion once a
smoothness cost is combined with such constraints. We cast planning and
steering as a single continuous-time program: robots avoid one another and
static or moving obstacles, and a learned proposal is steered by keeping the
motion close to it. The program is elementary to state, yet nonconvex. We
propose GlobMP (Global Motion Planning), a gradient projection method over
piecewise-polynomial curves, on which each constraint is one polynomial per piece that must
stay nonnegative over the whole motion; it alternates cost reduction with
feasibility restoration. Under
mild assumptions, we show that, from a fine grid start, its cost approaches the global optimum as the number of polynomial pieces grows, with every constraint met to any chosen tolerance, or exactly on stable instances. We
demonstrate its superior performance over strong baselines in real and
synthetic experiments, including teams of up to $200$ robots, multi-robot
planning on the Robotarium testbed and an extended benchmark, maze
navigation, and steering tactile-reactive policies on a bimanual humanoid.
\end{abstract}

\section{Introduction}
\label{sec:intro}
\suppressfloats[t]

\begin{figure}[t]
\centering
\includegraphics[width=\linewidth]{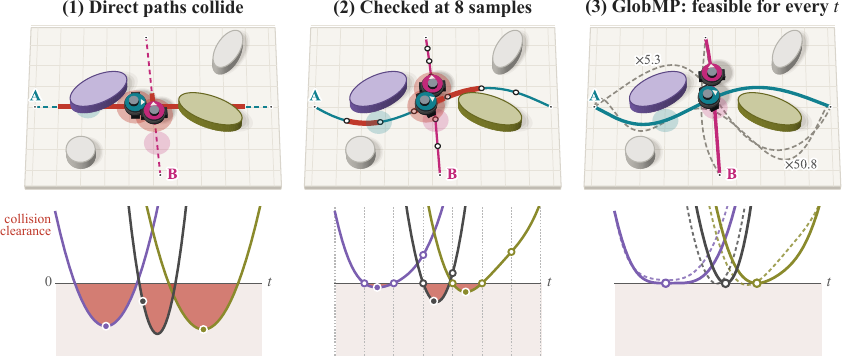}
\vspace{-.5cm}
\caption{Two robots, A and B, planned jointly among ellipsoidal obstacles;
each obstacle and the other robot enter the program as a
\textcolor{rowcoll}{collision} constraint (ellipses through the weighted norm
of Appendix~\ref{app:variants}). Top: each robot's path in its colour
(A teal, B magenta), with both robots drawn at the time of the dark-grey
A$-$B mark and a faded disc where each was earlier; red marks a collision.
Bottom: robot A's clearance to each coloured obstacle and the A$-$B
clearance (dark grey) over time; below zero is a collision. (1) The direct
(dashed) paths collide. (2) Checking the constraints at $8$ sample times (open
dots) misses collisions between them. (3) The final solid curves keep every
clearance nonnegative on $[0,1]$; coloured dashed: an unrestored cost step;
grey dashed: numerically refined alternatives, with cost relative to GlobMP's.}
\label{fig:pipeline}
\end{figure}

Motion planning is a fundamental problem in robotics: it asks for a trajectory that connects a given initial state to a goal while satisfying motion constraints, such as avoiding collisions \citep{lavalle1998rapidly}. It underlies navigation, manipulation
\citep{sundaralingam2019relaxed,gao2024confined,amice2024certifying},
multi-robot coordination \citep{li2021eecbs,tang2025stgcs,shaoul2025mmd},
and, increasingly, the \emph{steering} of learned planners and policies
\citep{janner2022diffuser,carvalho2023mpd,chi2023diffusionpolicy,direct2026}. 

Motion planners approach this problem in several ways, each with a gap.
Sampling-based planners explore many routes \citep{karaman2011sampling}, yet
smoothing their waypoints can introduce collisions; optimization-based
planners improve a motion locally
\citep{zucker2013chomp,drake}, with TrajOpt \citep{schulman2014trajopt} also supporting
continuous collision checks. Constraints enforced
at sampled times do not ensure feasibility between them
(\citealp{hargraves1987direct}; Figure~\ref{fig:pipeline}). Convex-region planners search a decomposition
of free space \citep{deits2015computing,marcucci2023gcs}, and semidefinite
relaxations certify a polynomial motion as globally optimal only when the
relaxation is tight \citep{graesdal2026sdp}. Even shortest-path complexity depends on
the geometry: polynomial-time in the real-RAM model among planar discs
\citep{chang2005shortest}, NP-hard among polyhedra
in $\R^3$ \citep{canny1987lower}, and still open in $\R^3$ even among disjoint unit discs or axis-aligned unit cubes
\citep{mitchell2017shortest}.

Our setting is a point robot in $\R^n$ that travels between prescribed
boundary states around obstacles while minimizing a smoothness cost; steering a learned planner or policy is the
same problem with one added constraint that keeps the robot near its proposal. Although
elementary to state, this setting already contains the core difficulty of
avoiding collisions; to
our knowledge, no algorithm converges to its global optimum once
a derivative cost such as acceleration or snap is combined with collision
constraints on the whole time interval. We introduce GlobMP and show that, from a fine grid start, its cost approaches the
global optimum as the number of pieces grows, within any constraint tolerance or exactly on stable instances.

\begin{samepage}
GlobMP combines exact whole-interval constraints, gradient-based refinement
and a systematic search over initial curves. We consider
\begin{equation}
\tag{$\mathrm{P}_\infty$}
\label{eq:Pinf}
\begin{alignedat}{3}
\min_{\g\gamma\in\trajspace^k}\quad
& \int_0^1\big\|\g\gamma^{(k)}(t)\big\|^2\,dt
&& &\qquad& \textcolor{rowsmooth}{\text{smoothness}}\\[2pt]
\text{subject to}\quad
& \big\|\g\gamma(t)-\m c_j(t)\big\|\ge r_j, &\quad& \forall t\in[0,1],\ j=1,\ldots,p, && \textcolor{rowcoll}{\text{collision}}\\
& \big\|\g\gamma(t)-\g\rho(t)\big\|\le\bar r, && \forall t\in[0,1], && \textcolor{rowsteer}{\text{steering}}\\
& \g\gamma^{(s)}(0)=\g\gamma_0^{(s)},\ \g\gamma^{(s)}(1)=\g\gamma_1^{(s)}, && s=0,\ldots,\ell, && \textcolor{rowbound}{\text{boundary}}
\end{alignedat}
\end{equation}
\end{samepage}%
Here $\g\gamma:[0,1]\to\R^n$ is the robot's trajectory, with time
normalized to $[0,1]$, and $\trajspace^k=\trajspace^k([0,1];\R^n)$ denotes
the Sobolev space of curves whose derivatives through order $k$ are
square-integrable: the objective is finite exactly on $\trajspace^k$, and its
curves have $k-1$ continuous derivatives, so the constraints and boundary
values are defined at every instant. The objective
measures smoothness through the $k$th derivative: $k=2$ penalizes
acceleration and $k=4$ snap \citep{mellinger2011minimum}. The collision
constraints keep the robot at distance at least $r_j>0$ from each of $p$
obstacle centres $\m c_j(t)$, which may move, and the steering constraint
keeps it within $\bar r>0$ of a reference $\g\rho(t)$, both at every instant.
The boundary conditions fix the position and its derivatives through order
$\ell$ at both ends. Unlike \citet{graesdal2026sdp}, whose spheres are
static, we let the centres move in time; this single extension covers
several robots, each a moving centre for the others, and steering, where the
robot must stay near the moving centre $\g\rho(t)$.

Several robots are planned
by stacking their curves into one $\g\gamma$ and keeping robots $a$
and $b$ (radii $r_a,r_b$) apart: $\|\g\gamma_a(t)-\g\gamma_b(t)\|\ge r_a+r_b$
\citep{shaoul2025mmd}. The reference $\g\rho$ can be the waypoints of a
learned planner \citep{janner2022diffuser} or the action chunk of a policy
\citep{chi2023diffusionpolicy}, which the steering constraint corrects rather
than replaces; with $\bar r=\infty$ it is dropped
and \eqref{eq:Pinf} is plain motion planning. Boundary orders $s\ge k$ bind only our polynomial curves
(Assumption~\ref{assum:data}). Ellipsoids, walls,
corridors and other variants are in Appendix~\ref{app:variants}.

Our contributions address two difficulties of \eqref{eq:Pinf}: constraints must hold at every
instant, and improving one feasible curve need not find the best route.

\begin{enumerate}[leftmargin=2em,label=\textbf{C\arabic*.},labelsep=0.35em,itemsep=1pt,topsep=1pt]
\item We enforce collision constraints \emph{exactly in continuous time}. At a
  fixed polynomial degree, each constraint becomes a univariate polynomial
  that must be nonnegative on $[0,1]$ \citep{powers2000positive}, so it holds
  between sample times too; collision, robot separation and steering share
  one block formulation (\S\ref{sec:setup}).
\item We propose \emph{GlobMP}, a gradient projection algorithm that alternates
  a cost step with feasibility restoration (\S\ref{sec:gradient-projection}).
  Cone projections give coefficient corrections for violated blocks, which
  Jacobians turn into control-point updates, block by block and with the
  boundary states kept.
\item We prove \emph{global convergence} (Theorem~\ref{thm:globalinf}):
  GlobMP keeps every curve feasible without increasing its
  cost and, from a fine grid start with $N$ polynomial pieces, approaches the
  best continuous-time motion as $N$ grows, within any constraint tolerance
  and exactly on stable instances.
\item We \emph{deploy} GlobMP on real robots and show superior performance
  over strong baselines: multi-robot teams on the Robotarium testbed
  (\S\ref{sec:exp-mrmp}), and steering state-of-the-art expert policies on a
  bimanual humanoid with dexterous hands in tactile manipulation tasks
  (\S\ref{sec:exp-tactile}). The appendix adds an extended multi-robot
  benchmark, maze navigation and end-effector obstacle
  avoidance on D3IL (Appendix~\ref{app:extexp}).
\end{enumerate}

\section{Polynomial formulation and gradient projection}
\label{sec:setup}\label{sec:blocks}
We optimize the control points of a polynomial curve, alternating a
gradient step on the cost with restoration of feasibility; because the
curve is a polynomial, each constraint is checked over the whole motion.
We represent the curve by a polynomial of degree $d$ (piecewise when the
horizon is split) in the \emph{Bernstein basis}
$[\m b_d(t)]_i=\binom{d}{i}t^i(1-t)^{d-i}$, $i=0,\ldots,d$, whose weights are nonnegative
and sum to one on $[0,1]$:
\begin{subequations}
\begin{equation}
\label{eq:bez}
  \g{\gamma}(t)=\Gm\,\m{b}_d(t),\qquad \Gm\in\R^{n\times(d+1)}.
\end{equation}
The columns of $\Gm$ are the \emph{control points}, whose convex hull
contains the curve. Substituting this representation into the objective gives
\begin{equation}
\label{eq:cost}
\mathrm{cost}(\Gm)
:=\int_0^1\big\|\Gm\m b_d^{(k)}(t)\big\|^2\,dt
=\tr(\m Q\Gm^\top\Gm),
\end{equation}
where $\tr$ is the trace and $\m Q:=\int_0^1\m b_d^{(k)}\m b_d^{(k)\top}\,dt$ is precomputed
exactly, so each cost step uses a matrix gradient without reevaluating the
integral. Since $\m Q\succeq0$ (positive semidefinite), the
objective is convex; the constraints carry the nonconvexity.

Let constraint $j$ be $g_j(\g\gamma(t),t)\ge 0$, where
$g_j(\m x,t):=\|\m x-\m c_j(t)\|^2-r_j^2$ for the collision constraints
$j\le p$, and $g_{p+1}(\m x,t):=\bar r^2-\|\m x-\g\rho(t)\|^2$ for the steering
constraint when $\bar r<\infty$. With several robots, each robot's own position $\m x_a$ replaces $\m x$, and each pair of robots forms an additional collision block, with $\m x_a-\m x_b$ and $r_a+r_b$ in place of $\m x-\m c_j$ and $r_j$; from here on, $p$ counts both the obstacle blocks and the pair blocks. For polynomial obstacle centres and a
polynomial reference, the value along the curve is a polynomial in time,
\begin{equation}
\label{eq:generalblocks}
g_j^{\Gm}(t):=g_j\big(\Gm\m{b}_d(t),t\big)
 =\g{\pi}_j(\Gm)^{\!\top}\m{b}_{m_j}(t),
\end{equation}
with degree bound $m_j=2d$, or $2\max(d,d_j)$ for a centre moving on a
polynomial path of degree $d_j$ (the reference plays the role of a centre
for $j=p+1$); piecewise-polynomial data are handled interval by interval
(Appendix~\ref{app:scope}). Its coefficient vector $\g\pi_j(\Gm)$
is quadratic in $\Gm$. The constraint holds at every instant exactly when
this vector belongs to the closed convex cone
\begin{equation}
\label{eq:Kdef}
\mathcal K_j:=\{\g\zeta\in\R^{m_j+1}:\g\zeta^{\!\top}\m{b}_{m_j}(t)\ge0
 \text{ for every }t\in[0,1]\}.
\end{equation}
This tests the entire polynomial: $(2t-1)^2$ has Bernstein coefficients
$(1,-1,1)$ yet is nonnegative, so a coefficient-sign test would wrongly
reject it. The coefficients are quadratic in $\Gm$, so the set of control
points that satisfy a block need not be convex.
Integral budgets, an optional extension beyond \eqref{eq:Pinf},
fit the same notation as scalar blocks,
\begin{equation}
\label{eq:intblock}
\pi_j(\Gm)
:=\int_0^1 \psi_j\big(\Gm\m b_d(t),t\big)\,dt
\in\mathcal K_j:=\R_+,
\end{equation}
where $\psi_j$ encodes the signed budget margin; for uniform notation we also write $\g\pi_j:=\pi_j$ for budgets. Let $\bar p$ denote the
number of pointwise blocks ($\bar p=p+1$ with the steering block, $\bar p=p$
without); the $q$ budgets are blocks $j=\bar p+1,\ldots,\bar p+q$. We collect the constraints as
\begin{equation}
\label{eq:stack}
\g\pi(\Gm):=\big(\g\pi_1(\Gm),\ldots,\g\pi_{\bar p}(\Gm),
 \pi_{\bar p+1}(\Gm),\ldots,\pi_{\bar p+q}(\Gm)\big),\qquad
\mathcal K:=\mathcal K_1\times\cdots\times\mathcal K_{\bar p}\times\R_+^{\,q}.
\end{equation}
Now, substituting \eqref{eq:bez} into the objective and the constraints of \eqref{eq:Pinf} gives \eqref{eq:cost} and \eqref{eq:generalblocks}--\eqref{eq:stack}; with the boundary conditions written as linear equations $\Gm\m A=\m B$ in the control points, we obtain the degree-$d$ reformulation of \eqref{eq:Pinf} as
\begin{equation}
\label{eq:Pd}
\inf_{\Gm\in\R^{n\times(d+1)}}\ \mathrm{cost}(\Gm)
\quad\text{s.t.}\quad
\g\pi(\Gm)\in\mathcal K,\qquad \Gm\m A=\m B.
\end{equation}
\end{subequations}

\label{sec:feasset}
Let $\affset:=\{\Gm:\Gm\m A=\m B\}$ denote the control points that
satisfy them. We fix one such point $\Gm_{\mathrm{aff}}$ and allow only
changes that keep the equations. When $2(\ell+1)\le d$, every
$\Gm\in\affset$ is $\Gm=\Gm_{\mathrm{aff}}+\Gm_\free\m N^\top$, where the
columns of $\m N$ are an orthonormal basis of the null space of
$\m A^\top$ and $\Gm_\free\in\R^{n\times d_f}$, $d_f=d+1-2(\ell+1)$,
collects the free coordinates (Lemma~\ref{lem:affchart}, Appendix~\ref{app:prelim}); distances
between free and full control points agree. In these free coordinates
\begin{equation}
\tag{P}
\label{eq:P}
\begin{aligned}
\inf_{\Gm_\free}\quad & \mathrm{cost}\big(\Gm_{\mathrm{aff}}+\Gm_\free\m N^\top\big)\\[2pt]
\text{subject to}\quad & \Gm_\free\in\feasset:=\Big\{\Gm_\free:\ \g\pi\big(\Gm_{\mathrm{aff}}+\Gm_\free\m N^\top\big)\in\mathcal K\Big\}.
\end{aligned}
\end{equation}
Every reconstructed curve meets the boundary conditions; for piecewise
curves the junction conditions join the linear equations. Fixing the
degree restricts the curves but neither discretizes time nor weakens the
whole-interval constraints.

\label{sec:gradient-projection}\label{sec:alg1}
Write $\g\pi_j(\Gm_\free)$ for
$\g\pi_j(\Gm_{\mathrm{aff}}+\Gm_\free\m N^\top)$. Restoration,
Algorithm~\ref{alg:proj}, takes gradient steps on the summed squared distances of the
blocks to their cones, where $P_{\mathcal K_j}$ maps a vector to its nearest
point in $\mathcal K_j$,
\begin{equation}
\label{eq:restoration-merit}
V(\Gm_\free):=\sum_{j=1}^{\bar p+q}
\big\|\g\pi_j(\Gm_\free)-P_{\mathcal K_j}\big(\g\pi_j(\Gm_\free)\big)\big\|^2.
\end{equation}
The squared distance to a closed convex cone is differentiable, with
gradient twice the residual $\m e_j=\g\pi_j-P_{\mathcal K_j}(\g\pi_j)$,
which vanishes exactly when block $j$ is satisfied. The chain
rule gives $2D\g\pi_j^{\!\top}\m e_j$, where $D\g\pi_j$ is the Jacobian (matrix of first derivatives) in
free coordinates and the gradient is reshaped as a control-point matrix.
Only violated blocks contribute
(lines~\ref{alg1:viol}--\ref{alg1:grad}). GlobMP (Algorithm~\ref{alg:pgd})
alternates a cost step along $-2\Gm\m Q\m N$ with restoration
(lines~\ref{alg2:grad}--\ref{alg2:repair}), and rebuilding $\Gm$ from the
free coordinates preserves the boundary states. $P_\feasset$ denotes
Algorithm~\ref{alg:proj}, which need not return an exactly feasible or
nearest curve in finitely many steps; the theory analyses the repair
below.

\label{sec:distcone}\label{sec:cone-correction}
Each cone projection is a small convex problem. Evaluating a polynomial at
time $t$ is an inner product of its coefficients with $\m b_{m_j}(t)$, so the
vectors $\m b_{m_j}(t)$, $t\in[0,1]$, generate the dual cone $\mathcal K_j^*$, and by
Moreau's decomposition \citep{moreau1962decomposition} the projection adds to $\g\pi_j$ the point of
$\mathcal K_j^*$ closest to $-\g\pi_j$. We build $\g\omega^\star$ greedily as a nonnegative combination $\g\omega$ of
$\m b_{m_j}(t)$ over a growing set of selected times, with $\g\xi$ the part of $-\g\pi_j$ still to be corrected:
\begin{equation}
\label{eq:main-cone-oracle}
\begin{aligned}
P_{\mathcal K_j}(\g\pi_j)&=\g\pi_j+\g\omega^\star,\qquad
\g\omega^\star=\operatorname*{arg\,min}_{\g\omega\in\mathcal K_j^*}\|\g\omega+\g\pi_j\|,\\
t_*&\in\operatorname*{arg\,max}_{t\in[0,1]}
\frac{\langle\g\xi,\m b_{m_j}(t)\rangle}{\|\m b_{m_j}(t)\|},\qquad \g\xi:=-\g\pi_j-\g\omega.
\end{aligned}
\end{equation}
Each step adds the time $t_*$, where the corrected polynomial
$\g\pi_j+\g\omega$ is most negative relative to its scale, and refits
nonnegative weights on all selected times (Algorithm~\ref{alg:pk};
\citealp{locatello2017cone}), with error bounds in
Appendix~\ref{app:innerrate}; scalar budgets use $u\mapsto\max\{u,0\}$. The
selected times locate the violation rather than fix a list of checks, and
blocks are projected independently, so the work per iteration grows
linearly with the number of blocks at fixed degree and cone-solve iterations.

\begin{figure}[t]
\begin{minipage}[t]{0.485\linewidth}
\begin{algorithm}[H]
\caption{GlobMP}
\label{alg:pgd}
\begin{minipage}[t][10.5\baselineskip][t]{\linewidth}
\begin{algorithmic}[1]
\Require anchor $\Gm_{\mathrm{aff}}$ and basis $\m N$,
  start $\Gm_\free$, step $\eta$, iterations $T$
\State $\Gm_\free \gets P_\feasset(\Gm_\free)$\label{alg2:init} \Comment{restore start}
\For{$\nu = 1,\dots,T$}
  \State $\Gm\gets\Gm_{\mathrm{aff}}+\Gm_\free\m{N}^{\!\top}$\label{alg2:full}
    \Comment{reconstruct curve}
  \State $\g\Delta \gets 2\,\Gm\m Q\m{N}$\label{alg2:grad} \Comment{cost gradient}
  \State $\m Y\gets\Gm_\free - \eta\, \g\Delta$\label{alg2:step} \Comment{cost step}
  \State $\Gm_\free \gets P_\feasset(\m Y)$\label{alg2:repair} \Comment{restore, Alg.~\ref{alg:proj}}
\EndFor
\State \Return $\Gm_\free$ \Comment{candidate}
\end{algorithmic}
\end{minipage}
\end{algorithm}
\end{minipage}\hfill
\begin{minipage}[t]{0.485\linewidth}
\begin{algorithm}[H]
\caption{Feasibility restoration for \eqref{eq:P}}
\label{alg:proj}
\begin{minipage}[t][10.5\baselineskip][t]{\linewidth}
\begin{algorithmic}[1]
\Require free coordinates $\Gm_\free$, blocks $\g\pi_j$, tolerance
  $\mathrm{tol}$, step $\eta_{\mathrm{proj}}$, iterations $T_{\mathrm{proj}}$
\For{$\nu = 1,\dots,T_{\mathrm{proj}}$}
  \State $\m e_j\gets(\m I-P_{\mathcal K_j})\,\g\pi_j(\Gm_\free)$\label{alg1:res}
    \Comment{all blocks}
  \State \textbf{if} $\max_j\|\m e_j\|^2\le\mathrm{tol}$ \textbf{then break}\label{alg1:stop}
  \State $\mathcal V\gets\{j:\m e_j\ne\m 0\}$\label{alg1:viol} \Comment{violated blocks}
  \State $\g\Delta\gets2\sum_{j\in\mathcal V}D\g\pi_j^{\!\top}\m e_j$\label{alg1:grad}
    \Comment{merit gradient}
  \State $\Gm_\free \gets \Gm_\free - \eta_{\mathrm{proj}}\, \g\Delta$\label{alg1:step}
    \Comment{restoration step}
\EndFor
\State \Return $\Gm_\free$ \Comment{candidate}
\end{algorithmic}
\end{minipage}
\end{algorithm}
\end{minipage}
\end{figure}

For the analysis we use a feasible version of the repair step. Let $\m z$
denote the free coordinates $\Gm_\free$ stacked into one vector and
$f(\m z)$ the cost; we keep the notation $\g\pi_j(\m z)$, $D\g\pi_j(\m z)$ and $\feasset$ for the
blocks, their Jacobians and the feasible set.
Each collision block is convex in the control points, so its linearization
at any $\m z$ lies below it at every instant. For a sphere the linearized
constraint is a half-space that excludes the whole obstacle, so each step
keeps every robot behind a moving plane that separates it from each
obstacle and from each other robot. Let $\mathcal D$ denote the points that satisfy the steering block exactly,
a speed limit (Appendix~\ref{app:variants}) when present, and the optional budgets; steering and speed sets are convex, since balls are convex
and the curve depends affinely on the free control points; the budgets are assumed convex
($\mathcal D$ is all points when none is present). Let
\begin{equation}
\label{eq:innerset}
\mathcal C(\m z):=\big\{\m w\in\mathcal D:\
\g\pi_j(\m z)+D\g\pi_j(\m z)(\m w-\m z)\in\mathcal K_j,\quad j=1,\ldots,p\big\}.
\end{equation}
Then $\mathcal C(\m z)$ is convex and contained in $\feasset$, and it contains
$\m z$ when $\m z$ is feasible. Line~\ref{alg2:repair} of Algorithm~\ref{alg:pgd}
becomes the convex solve
\begin{equation}
\label{eq:inner-step}
\m z^+=\operatorname*{arg\,min}_{\m w\in\mathcal C(\m z)}
\tfrac12\|\m w-(\m z-\eta\nabla f(\m z))\|^2.
\end{equation}
Every iterate is feasible, and $\mathcal C(\m z)$ is rebuilt after each step,
so this is not a convex reformulation of \eqref{eq:P}; it needs a feasible
start and an exact solve. In the analysis, GlobMP starts from the cheapest feasible point of a grid of free
control points, each tested exactly over the whole interval. For \eqref{eq:Pinf} it adds
pieces (Theorem~\ref{thm:globalinf}) or raises the degree
(Algorithm~\ref{alg:refinement}, under a condition on configurations).

\section{Theory}
\label{sec:theory}
We state one theorem, for \eqref{eq:Pinf}; its fixed-degree counterpart and
all proofs are in Appendix~\ref{app:theory}. Let $E^\star_\infty$ denote
the optimal value of \eqref{eq:Pinf}.

\begin{assumption}[problem data]
\label{assum:data}
The cost order is $k\ge1$ and the endpoints fix derivatives of orders
$0,\dots,\ell$ (if $\ell\ge k$, \eqref{eq:Pinf} keeps those of order below $k$); the obstacle centres and the reference are continuous and piecewise polynomial in $t$; the
endpoints are strictly feasible; and \eqref{eq:Pinf} is feasible.
\end{assumption}

\begin{assumption}[enough boundary conditions]
\label{assum:curv}
The two endpoints fix at least as many values per coordinate as the order
of the cost: $2(\ell+1)\ge k$.
\end{assumption}

A change of the curve that keeps the boundary values and has zero
$k$th-derivative cost is a polynomial of degree
below $k$ with a root of order $\ell+1$ at each endpoint, which is zero
under Assumption~\ref{assum:curv}. Hence $f$ is strongly
convex, with curvature at most
$\beta=2\lambda_{\max}(\m H)$, $\m H=\m N^\top\m Q\m N$.
For $|\tilde{r}|\le\tilde{r}_{\max}$, $\tilde{r}_{\max}>0$ below every radius and half the endpoint
clearance~(Appendix~\ref{app:thm1}), let $(\mathrm P^{\tilde{r}}_\infty)$ denote \eqref{eq:Pinf} with every
obstacle radius raised by $\tilde{r}$, every pair separation by $2\tilde{r}$ and $\bar r$
lowered by $\tilde{r}$ ($\tilde{r}<0$, a tolerance, loosens them), and
$E^\star_\infty(\tilde{r})$ its optimal value.

Let $\mathcal S_N$ denote the curves of $N$ polynomial pieces of a fixed degree
$d\ge\max\{2k-1,\ell\}$ on equal intervals, joined with $k-1$ continuous derivatives. Let
$\m z_h$ denote a cheapest point of the grid of spacing $h$ in their
free control points that is feasible for the problem being solved; grid points can be checked exactly by root
isolation for rational data.

\begin{maintheorem}[global convergence for \eqref{eq:Pinf}]
\label{thm:globalinf}
Under Assumptions~\ref{assum:data} and~\ref{assum:curv}, run GlobMP (Algorithm~\ref{alg:pgd}, line~\ref{alg2:repair} solved exactly by \eqref{eq:inner-step}) over
$\mathcal S_N$ on $(\mathrm P^{\tilde{r}}_\infty)$ with a step $0<\eta<2/\beta$ ($\beta$ computed for $\mathcal S_N$).
\begin{enumerate}[leftmargin=1.6em,label=\textup{\textbf{(\roman*)}},itemsep=1pt,topsep=2pt]
\item From a feasible start, every iterate is feasible for $(\mathrm P^{\tilde{r}}_\infty)$
at every instant, with nonincreasing cost.
\item If a curve of cost at most $U$ is feasible for $(\mathrm P^{\tilde{r}+\mu}_\infty)$,
$\mu>0$, $\tilde{r}+\mu\le\tilde{r}_{\max}$, and $\delta_N:=C\,N^{-r}\le\mu$, $r=\min\{k-1/2,1\}$, with $C$
depending on $U$, $\tilde{r}_{\max}$ and the data, then for small $h$, $\m z_h$ exists and
every cost from it lies in $[E^\star_\infty(\tilde{r}),E^\star_\infty(\tilde{r}+\delta_N)+\delta_N]$.
\item For $\tilde{r}=-\mu\ge-\tilde{r}_{\max}$ and $U\ge E^\star_\infty$, these curves violate no constraint of
\eqref{eq:Pinf} by more than $\mu$ ($2\mu$ for a separation) and cost at
most $E^\star_\infty+\delta_N$; as $\mu\to0$ (so $N\to\infty$), the costs tend to
$E^\star_\infty$ and the curves have subsequences converging uniformly to global
minimizers of \eqref{eq:Pinf}.
\end{enumerate}
\end{maintheorem}

At $\tilde{r}=0$ the costs converge to $E^\star_\infty$ when the optimal value is
stable, $E^\star_\infty(\tilde{r})\to E^\star_\infty$ as $\tilde{r}\downarrow0$, that is, when
the robust-optimum analogue of sampling-based planners \citep{solovey2020revisiting} is
the optimum; this holds at all but countably many $\tilde{r}$ in place of $0$;
if $E^\star_\infty(\tilde{r})\le E^\star_\infty+c\,\tilde{r}$, the rate is $(c+1)\delta_N$. Methods that keep a
positive clearance never reach $E^\star_\infty$ on unstable instances
(Appendix~\ref{app:scope}).

The proof (Appendix~\ref{app:gpa-global}) has three steps. First,
$\mathcal C(\m z)$ in \eqref{eq:innerset} is convex and, collision constraints being convex in position,
inside $\feasset$; the repair
\eqref{eq:inner-step} lowers $f$. Second, interpolating a curve
of $(\mathrm P^{\tilde{r}+\delta_N}_\infty)$ with $N$ pieces never raises its cost and
moves it by less than $\delta_N$, so a ball around it, which the grid hits, is
feasible for $(\mathrm P^{\tilde{r}}_\infty)$; if $\ell\ge k$, an end correction costs
$O(N^{-1})$. Third, curves of bounded cost have uniform limits, which are
optimal as $\mu\to0$.

\begin{corollary}
\label{cor:eps}
Under Assumptions~\ref{assum:data} and~\ref{assum:curv}, \textup{\textbf{(i)}} if
$\ell<k$, Theorem~\ref{thm:globalinf}(ii) holds with $r=k-1/2$ and
$C=1+2\sqrt U$; \textup{\textbf{(ii)}} for either $(r,C)$, $0<\mu\le\tilde{r}_{\max}$,
$\varepsilon>0$ and $N\ge(C/\min\{\varepsilon,\mu\})^{1/r}$, the curves of
Theorem~\ref{thm:globalinf}(iii), within tolerance $\mu$, cost at most $E^\star_\infty+\varepsilon$.
\end{corollary}

In minimum-snap planning ($k=4$) with position through jerk fixed at both
ends ($\ell=3$), accuracy $\varepsilon=\mu$ takes order $\varepsilon^{-2/7}$ pieces
($\varepsilon^{-2}$ and $\varepsilon^{-2/3}$ for $k=1,2$, $\ell<k$). Fixing the snap too ($\ell=k$) adds
cost of order $1/N$ at the end pieces, hence order $\varepsilon^{-1}$ pieces.

\section{Experiments}
\label{sec:experiments}
We evaluate GlobMP on synthetic teams of $2$ to $200$ robots among spheres
(\S\ref{sec:exp-synthetic}), a multi-robot benchmark on the Robotarium
testbed (\S\ref{sec:exp-mrmp}) and tactile-reactive manipulation
(\S\ref{sec:exp-tactile}). Algorithm~\ref{alg:proj} restores from
heuristic seeds (route classes, prioritized plans, the executing curve).
\subsection{Synthetic experiments}
\label{sec:exp-synthetic}

The synthetic study (Problem~\ref{prob:B}, Appendix~\ref{app:exp-synthetic-details}) has robots exchange places in
pairs past static spheres, under the minimum-snap cost with piecewise
degree-$9$ polynomials; each pair or sphere can be passed on either side, so
an instance has many route classes. A curve is \emph{feasible} if it
makes no contact at any $t$; GlobMP keeps a $0.5$\,cm margin. GlobMP runs
Algorithms~\ref{alg:pgd}--\ref{alg:proj} from one waypoint path per route
class and keeps the cheapest feasible result. The baselines are RRT-Connect
\citep{kuffner2000rrt}, RRT$^*$ \citep{karaman2011sampling} and GCS
\citep{marcucci2023gcs} on IRIS regions \citep{deits2015computing} in the
joint configuration space, whose paths are fitted by the cheapest
collision-free (else least-colliding) curve in GlobMP's class \citep{richter2016polynomial},
and SDP-MP of \citet{graesdal2026sdp} on the same program
(Appendix~\ref{app:exp-synthetic-details}).

Each baseline imposes the constraints off the curve it returns (along a
sampled path, inside convex regions, or on a relaxed variable), keeps
centimetres to decimetres of slack, and fails more often as robots are added;
SDP-MP collides or times out on every instance
(Figure~\ref{fig:synthetic}a,b). GlobMP is feasible at every $t$ on every
instance and ends at the margin. Feasible baseline curves cost from
about twice to hundreds of times more, with a gap that mostly grows with team
size (Figure~\ref{fig:synthetic}c); where SDP-MP returns a lower bound,
GlobMP is within $6\%$ of it at the median. Theorem~\ref{thm:globalinf} covers this curve class ($\ell=k$, rate
$N^{-1}$) but not its seeds, so the lowest cost found is not a
proven optimum. From $2$ to $200$ robots GlobMP is feasible at every $t$ on
every seed, at $200$ robots ($82{,}000$ blocks) in two minutes on CPU
and $28$\,s on one GPU from straight lines (Figure~\ref{fig:prob-b}a,b), with
time per iteration at most linear in the blocks
(Figure~\ref{fig:prob-b}c); the baselines are feasible on at most $2$ of $5$
seeds at eight robots and on none from $16$.

To check the rates of Theorem~\ref{thm:globalinf} and
Corollary~\ref{cor:eps}, we solve one two-robot instance again and again with
more pieces $N$ and track how far GlobMP's best cost stays above the limit
(Figure~\ref{fig:rate}). When the ends fix every derivative up to the one in
the cost, the extra cost halves each time $N$ doubles, for velocity, acceleration,
jerk and snap: this is the cost of the end pieces in the discussion after
Corollary~\ref{cor:eps}, and the bound is tight for $k\ge2$. With one derivative
fewer fixed and $k\ge2$, the extra cost falls much faster than the corollary's bound and
is below the numerical precision by $N=8$ for jerk, $N=10$ for snap and
$N=20$ for acceleration, so the bound holds but is loose here. For $k\ge2$ the cost grows linearly with the obstacle margin, so the rate $(c+1)\delta_N$ after
Theorem~\ref{thm:globalinf} fits.

\begin{figure}[t]
\centering
\setlength{\abovecaptionskip}{3pt}
\includegraphics[width=\linewidth]{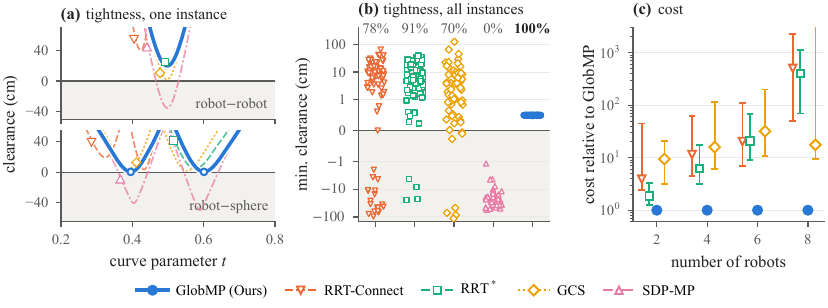}
\vspace{-8pt}
\caption{Tightness and cost, $20$ instances each with $2$, $4$, $6$ and $8$
robots. (a) Clearance over time on one $2$-robot instance, between the robots
(top) and to the spheres (bottom). (b) Minimum clearance of each returned
curve, with the share of runs feasible at every $t$. (c) Cost of feasible
curves relative to GlobMP's (median and interquartile range).}
\label{fig:tightness}\label{fig:synthetic}
\end{figure}

\begin{figure}[t]
\centering
\setlength{\abovecaptionskip}{3pt}
\includegraphics[width=\linewidth]{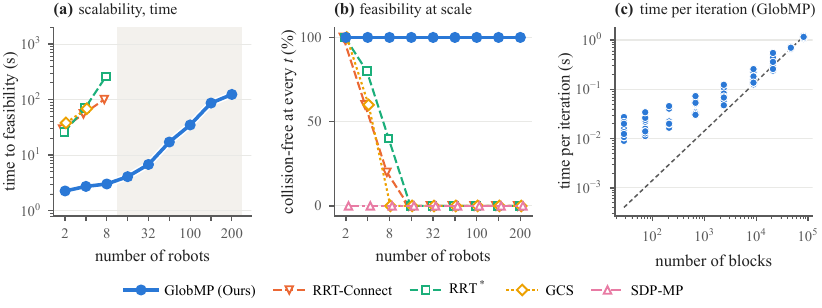}
\vspace{-8pt}
\caption{Scalability from $2$ to $200$ robots ($10$ runs per size for GlobMP,
$5$ per baseline). (a) Median CPU time to the first collision-free trajectory, setup
included; where shaded, no baseline finds one. (b) Share of runs
collision-free at every instant. (c) Algorithm~\ref{alg:proj}'s time per iteration against the
number of blocks; dashed: linear growth.}
\label{fig:prob-b}
\end{figure}

\begin{figure}[t]
\centering
\setlength{\abovecaptionskip}{3pt}
\includegraphics[width=\linewidth]{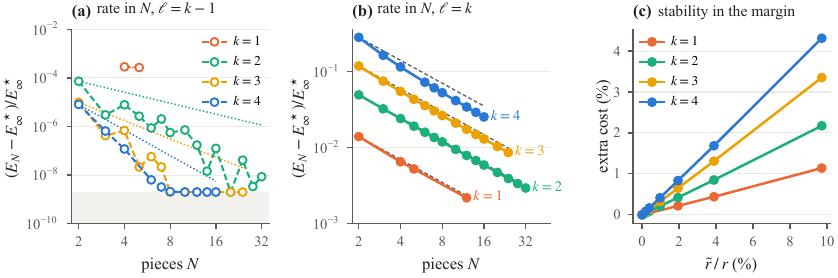}
\vspace{-8pt}
\caption{Relative extra cost over the limit as the number of pieces grows (one
two-robot instance), for velocity to snap ($k=1,\dots,4$). (a) One derivative fewer fixed at the ends; dotted: rates of
Corollary~\ref{cor:eps}; shaded: numerical precision. (b) The cost's own
derivative also fixed; dashed: halving per doubling. (c) Extra cost against
the margin $\tilde r/r$, with $r$ the smaller sphere radius.}
\label{fig:rate}
\end{figure}

\subsection{Multi-robot motion planning}
\label{sec:exp-mrmp}

Multi-robot motion planning coordinates the simultaneous motion of several
robots toward assigned goals while avoiding collisions with obstacles and with
one another. We evaluate GlobMP on five multi-robot environments, \texttt{Empty}, \texttt{Conveyor}, \texttt{Cross-Room},
\texttt{Drop-Region} and \texttt{Corridor-Bay}, from \citet{shaoul2025mmd}
(\texttt{Cross-Room} and \texttt{Corridor-Bay} reconstructed from \citealt{liang2025smd}), following the test-time velocity
constraint and collision predicate of \citet{direct2026} on the Georgia Tech
Robotarium~\citep{wilson2020robotarium}, with teams of $N_a\in\{3,6,9,12\}$ robots. The environments are fitted to the
$3.2\times2.0$\,m Robotarium arena, with the common velocity limit selected
for the Robotarium setting and inflated robot radius $r=0.13$\,m. For each (map, team-size) pair, we use ten
start--goal instances: nine sampled from the benchmark distribution and one
``ring'' instance across the workspace (\texttt{Conveyor} below $N_a=12$: ten sampled). Hardware details of the Robotarium setup are provided in
Appendix~\ref{app:robot-system}.

GlobMP needs no training: at test time it solves each instance directly as
one \eqref{eq:P}, with the robots' curves stacked, a collision constraint
at radius $2r$ for every pair and affine blocks for boxes
(Appendix~\ref{app:affine}), so its plan is collision-free at every $t$.
The baselines are CD, PCD-LB and PCD-SHD \citep{luan2026pcd}, Final~Proj.\ and
DiRecT \citep{direct2026} and MMD-CBS \citep{shaoul2025mmd}. Each method is evaluated on the same start--goal
instances. The sampling baselines draw $64$ candidates, whereas GlobMP returns one
trajectory.

Figure~\ref{fig:robotarium-sr-one} reports the success of a single returned
plan across the five environments and team sizes. On \texttt{Conveyor} and
\texttt{Cross-Room}, baseline performance drops rapidly as the number
of robots increases: every baseline reaches zero on \texttt{Conveyor} at
$N_a=12$ and stays at or below $0.50$ on \texttt{Cross-Room} from $N_a=9$,
whereas GlobMP's single plan succeeds on every instance except two of the ten on \texttt{Conveyor} at $N_a=12$. The stronger
projection baselines remain more reliable on \texttt{Empty},
\texttt{Drop-Region}, and \texttt{Corridor-Bay}, but most also degrade with team size. Figure~\ref{fig:robotarium-gallery} presents
representative $12$-robot ring trajectories on \texttt{Cross-Room} and
\texttt{Conveyor}. Appendix~\ref{app:exp-multirobot} extends the evaluation
to eight environments, teams of up to $21$ robots, three velocity caps, and
a larger set of start--goal instances.

\subsection{Steering tactile-reactive dexterous policies}
\label{sec:exp-tactile}

A Dexmate Vega-1 with two XHAND1 hands with fingertip force sensing
(Appendix~\ref{app:robot-system}) performs four bimanual tasks learned from
teleoperation (Figure~\ref{fig:steer}a). The expert policies are two state-of-the-art
vision-language-action models, GR00T N1.7~\citep{nvidia2026groot17}, built on the GR00T N1 architecture \citep{bjorck2025groot}, and
$\pi_{0.5}$~\citep{black2025pi05}, each with and without tactile input. Both
command absolute joint targets for both arms and hands ($n=38$) at $30$\,Hz
in chunks of $16$ executed steps.

Deployed, the policies fail geometrically (they replay memorized targets),
temporally (consecutive chunks disagree by up to $0.5$\,rad) and at contact,
mostly because of observations outside the training distribution, which no
motion constraint can observe; so we separate \emph{deciding} which
constraints apply from \emph{enforcing} them (Figure~\ref{fig:steer}b). A
\emph{meta-reasoner}, the vision-language model Claude
Opus~5.5~\citep{anthropic2026opus}, defines the constraints once, at reset, from
the images, the task and the measured state, for example the keep-out, keep-in
and grip of Figure~\ref{fig:steer}b. A \emph{certifier} then solves
\eqref{eq:P} per chunk, with the policy's chunk as the reference $\g\rho$, so
the executed curve meets every active constraint at every $t$. Keep-out and
keep-in blocks act on hand positions, which depend on the joints through forward
kinematics; the certifier enforces them by linearizing the kinematics at the
current curve and re-solving, so for these blocks feasibility at every $t$
holds up to that linearization, while blocks on the joints
themselves, such as the grip, are exact. The
meta-reasoner only chooses constraints and never solves them, so a wrong
choice changes \emph{which} feasible curve is returned, never \emph{whether} it
is feasible, unlike model-written programs~\citep{huang2024rekep}. If the
policy slows, a demonstrated time law rescales the reference
\citep{kant1986toward}.

\begin{figure}[t]
\centering
\includegraphics[width=\linewidth]{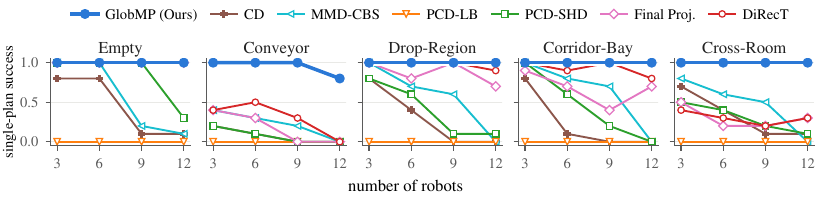}
\vspace{-8pt}
\caption{Success of a \emph{single} returned plan on the Robotarium arena for $N_a\in\{3,6,9,12\}$ robots,
one panel per map; for GlobMP, the one plan it returns, and for a sampling baseline, the first of its $64$ candidates. PCD-LB is at zero throughout: its log-barrier pushes agents into the arena boundary. MMD-CBS returns a single plan, with $60$\,s per instance.}
\label{fig:robotarium-sr-one}

\vspace{6pt}
\centering
\includegraphics[width=\linewidth]{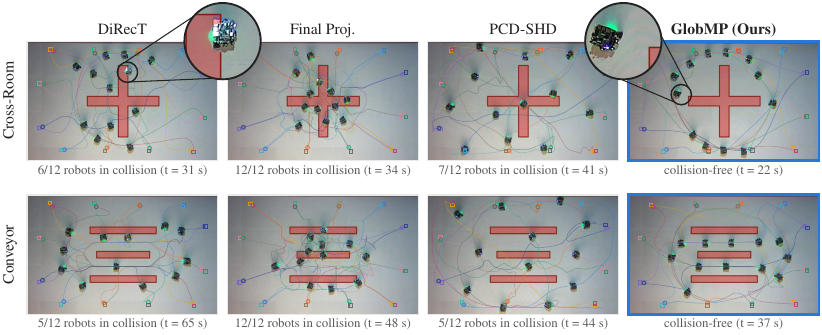}
\vspace{-8pt}
\caption{Executed Robotarium trajectories for the $12$-robot ring instance on \texttt{Cross-Room}
and \texttt{Conveyor}. Obstacle-hit counts span the whole run; circles magnify highlighted regions.}
\label{fig:robotarium-gallery}
\end{figure}

\begin{figure}[t]
\centering
\setlength{\abovecaptionskip}{2pt}
\includegraphics[width=\linewidth]{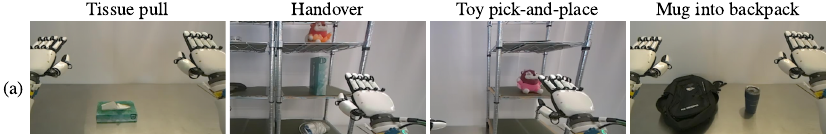}\\[2pt]
\includegraphics[width=\linewidth]{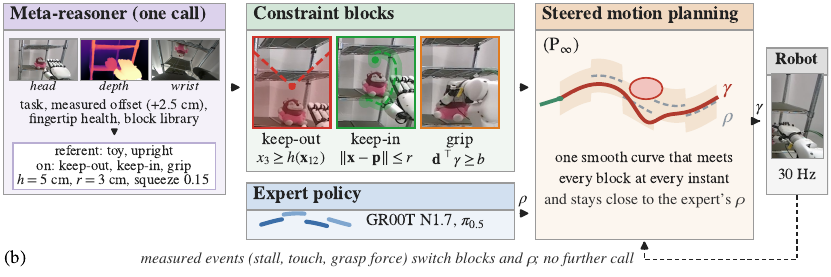}\\[8pt]
\includegraphics[width=\linewidth]{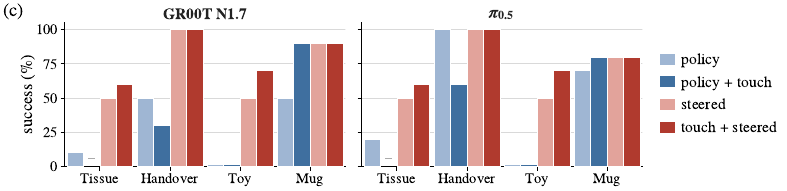}
\vspace{-5pt}
\caption{(a) The initial head-camera frame of each task. (b) Steered motion planning on one real toy trial: one meta-reasoner call at reset sets the active blocks and their
parameters; \eqref{eq:P} takes the policy's chunk as $\g\rho$ and returns a curve that is smooth, collision-free and
within $\bar r$ of $\g\rho$ at every $t$; measured events then switch blocks. (c) Real-robot success of GR00T and $\pi_{0.5}$ per task ($9$ or $10$ trials per bar): the policy without and with touch input, and steered without and with touch; dash: not run.}
\label{fig:steer}
\end{figure}

The baselines are the expert policies alone, with and without touch
(none with touch on the tissue), and every condition run has $9$ or $10$ real trials per task.
Steering matches or raises every task (Figure~\ref{fig:steer}c): it turns
the tissue, which the policies alone rarely complete, into a mostly
successful task, and completes the handover in every trial. Touch pays off
where the hand must regulate contact, as in placing the toy, pulling the
tissue and stowing the mug, and adds little where the task is guided by
vision, as in the handover, where it even hurts the policies alone. In a
closed-loop simulator replaying failed real tissue deployments, the pinch
lands where and when demonstrated only with steering.

\raggedbottom
\section{Conclusion}
\label{sec:conclusion}
We developed GlobMP, a gradient projection method with convex feasibility
restoration for continuous-time motion planning over polynomial curves; it
enforces every collision constraint at every instant, with planning and
steering in one program. Under mild assumptions and from a fine grid start,
its cost approaches the continuous-time optimum as the number of pieces
grows, within any constraint tolerance and exactly on stable instances. In
real and synthetic experiments, from teams of up to $200$ robots to the
Robotarium testbed and tactile manipulation on a bimanual humanoid, it
outperforms strong baselines. Extending the guarantee to speed limits and
corridor constraints is a natural next step.

\FloatBarrier
\section*{AI use statement}
We disclose the use of large language models
(LLMs) to aid and polish the writing, editing the text for
clarity and length; for retrieval and discovery, helping find related work
and check citations; and to generate the synthetic datasets, whose
instances are produced by LLM-assisted scripts from fixed seeds
(Appendix~\ref{app:exp-synthetic-details}). The authors take full
responsibility for all content.

\section*{Reproducibility statement}
We will release all code: the GlobMP solver (restoration and projected
gradient steps), the scripts and configurations for every experiment and
figure, and the steering and deployment code for learned policies on the
Vega-1 with XHAND1 hands (certifier, grounding, time law, policy servers),
together with the teleoperation demonstrations and fine-tuned policy
checkpoints of \S\ref{sec:exp-tactile}. The project page, \url{https://GlobMP.github.io/}, collects the videos and results.

\clearpage
\bibliography{references}
\bibliographystyle{iclr2027_conference}

\clearpage
\appendix
\makeatletter
\setlength{\@fptop}{0pt}\setlength{\@fpsep}{12pt plus 2fil}\setlength{\@fpbot}{0pt plus 1fil}
\makeatother
\raggedbottom

\renewcommand{\thesection}{A\arabic{section}}
\setcounter{secnumdepth}{3}
\addtocontents{toc}{\protect\setcounter{tocdepth}{3}}
{\noindent\Huge\bfseries Appendix\par}
\vspace{1.2em}
\tableofcontents
\vspace{1.5em}

\section{Addendum to Section~\ref{sec:intro}: related work and Figure~\ref{fig:pipeline}}

\subsection{Extended related work}
\label{sec:related}

\noindent\textbf{Relation to existing formulations.}
Collocation replaces each block by inequalities at sampled times
\citep{hargraves1987direct}; the control-polygon rule asks every
coefficient to be nonnegative \citep{gao2018online}, which is sufficient
for $\mathcal K_j$ but conservative; convex decompositions
\citep{deits2015computing,marcucci2023gcs} are exact for the decomposition
rather than for the free space; semidefinite relaxations
\citep{lasserre2001global,parrilo2003semidefinite,graesdal2026sdp} certify
a trajectory only when tight. We keep the original variables and free
space, test the exact condition on each block, and approach the optimum of \eqref{eq:Pinf} as the class is refined, within any constraint tolerance and exactly on stable instances (Theorem~\ref{thm:globalinf}); under a condition on configurations (Assumption~\ref{assum:mfcq}) we also reach global minimizers over each fine enough spline class (Proposition~\ref{prop:globalC}).

\noindent\textbf{Sampling-based planning and trajectory optimization.}
The two classical families both act on the nonconvex problem itself, and
neither certifies global optimality. Sampling-based planners
\citep{lavalle1998rapidly,karaman2011sampling} are probabilistically
complete and, for RRT*, asymptotically optimal, but in probability and only to the robust optimum \citep{solovey2020revisiting}, the least cost approached with a positive clearance, which can exceed the optimum. Trajectory optimizers instead discretize and hand the
result to local nonlinear programming: direct collocation
\citep{hargraves1987direct}, covariant gradient descent
\citep{zucker2013chomp}, sequential convexification
\citep{schulman2014trajopt}, and differentiable collision geometry
\citep{tracy2023differentiable}; this pipeline is fast and widely deployed but
converges to a stationary point with no test for whether it is global. We
keep this setting and add a proof of global convergence.

\noindent\textbf{Convex decomposition and graphs of convex sets.} A second
line makes the problem tractable by convexifying the \emph{domain}. IRIS
inflates large obstacle-free convex regions
\citep{deits2015computing}, then used for mixed-integer planning \citep{deits2015efficient}; shortest-paths-in-graphs-of-convex-sets
then plans over that decomposition as a mixed-integer convex program, solved through its convex relaxation and rounding
\citep{marcucci2024shortest,marcucci2023gcs}, with subsequent work on scalability
\citep{marcucci2025biconvex,werner2024faster,werner2025superfast},
on space-time variants for multiple robots \citep{tang2025stgcs}, and on
using the convex solution to warm-start nonconvex trajectory optimization
\citep{vonwrangel2024gcs}. These methods bound their optimality gap
\emph{for the decomposition}: the regions are a modelling choice made
before optimizing, and a curve is admissible only if it respects them. Our
guarantee is stated on the original free space, with no decomposition.

\noindent\textbf{Learning-based planners.} Diffusion models have recently
been trained to sample trajectories directly, for single
\citep{carvalho2023mpd} and multiple \citep{shaoul2025mmd} robots, with
constraints imposed during sampling by projection
\citep{luan2026pcd} or by receding-horizon denoising
\citep{direct2026}. These amortize planning into fast learned sampling and
represent several routes, but they enforce feasibility at sampled waypoints
at best and produce no optimality certificate. We compare against
them in \S\ref{sec:experiments}; our contribution is a guarantee rather
than a speedup, and the two can be combined (Appendix~\ref{sec:steering}).

\noindent\textbf{Convex relaxations of planning and control.} Beyond \citet{graesdal2026sdp}, several works apply
semidefinite or moment-hierarchy relaxations
\citep{shor1987quadratic,lasserre2001global,parrilo2003semidefinite} to
planning and control problems and mostly report tight solutions empirically,
rarely analysing when or why tightness holds. \citet{teng2023convex}
formulate kinodynamic motion planning on Lie groups as polynomial
optimization and apply the moment hierarchy; \citet{haring2024trajectory}
study minimum-energy unicycle control with a semidefinite relaxation;
\citet{vega2025convex} apply the Shor relaxation to spacecraft manoeuvre
planning with an ellipsoidal, probability-of-collision keep-out constraint; \citet{mahajan2026tinysdp} use
it for real-time predictive control around ball-shaped obstacles on
embedded hardware. Control problems with free timing
\citep{yang2025new,dong2026fast} and contact-rich
planning \citep{graesdal2024towards,kang2025global} have been
relaxed in the same way, and a sum-of-squares surrogate guides contact-rich
search \citep{wei2026global}. Outside planning, such relaxations underlie
certifiable estimation, for example registration and SLAM
\citep{yang2020teaser,rosen2019se,papalia2024certifiably}, which we do not
review. In contrast to the planning and control work
above, which mostly observes tightness empirically, we prove that a feasible
descent method on the original nonconvex program, started from a fine
grid, approaches the optimal cost of the original program as the class is
refined, within any constraint tolerance (Theorem~\ref{thm:globalinf}).

\subsection{Inexact sampled checks and the data for Figure~\ref{fig:pipeline}}
\label{app:examples}
\label{app:proofs-examples}

\paragraph{Sampling and control-polygon tests are not exact.}
\label{sec:whynotelementary}

\emph{Sampling} $t$ at a finite set is a strict \emph{relaxation}. For any
finite $\mathcal T=\{t_1<\dots<t_M\}\subset[0,1]$, choosing $a<b$ inside any gap
between consecutive samples makes $g(t)=(t-a)(t-b)$ (of degree
$2\le2d$) nonnegative at every sample and negative in between, with
$\min_{[0,1]}g=-\big(\tfrac{b-a}{2}\big)^2$. So the sampled feasible set is
in general a proper superset of $\feasset$ at every $M$ and every degree; no finite
sample is ever exact. Structurally, each $\m{b}_{2d}(t_0)$ is an extreme ray of
the dual cone $\mathcal K_{2d}^{*}$ \eqref{eq:Kdual} of the cone \eqref{eq:Kdef} of degree-$2d$ nonnegative polynomials (indexed by degree, Appendix~\ref{app:coneprops}): if $\m{b}_{2d}(t_0)=\m x+\m y$ with
$\m x,\m y\in\mathcal K_{2d}^{*}$, pairing with $(t-t_0)^2$ forces $\m x$'s representing
measure onto $\{t_0\}$, so a continuum of extreme rays would be replaced
by finitely many. The direction of the error is the unsafe one for a
planner: the sampled problem gives a valid \emph{lower} bound on the
optimal cost of \eqref{eq:P} when solved globally, but the curve it returns may penetrate the obstacle, and the
margin violation is only $\Theta(1/M^2)$: on a uniform grid a
gap of width $\tau=1/(M-1)$ admits violations approaching
$\min_{[0,1]}g=-\tau^2/4$ for $g$ above (for a block the violation scales with $\|g''\|_\infty$, so $\Theta(1/M^2)$ holds at bounded speed and acceleration).

\emph{Requiring all Bernstein coefficients to be nonnegative} is a strict
\emph{restriction}, and degree elevation does not close it. If $g\ge0$
vanishes at some $t^\star\in(0,1)$, then since every entry of $\m b_D(t^\star)$ is strictly positive for every degree $D$,
nonnegative coefficients $\g\zeta$ of $g$ would force $\g\zeta^{\!\top}\m b_D(t^\star)=0$ with all terms
nonnegative, hence $\g\zeta=\m 0$ and $g\equiv0$: no elevation degree
ever works for a polynomial with an interior zero. This case arises here,
because whenever a minimizer of \eqref{eq:P} touches an obstacle at a
time inside a piece, its avoidance polynomial has an interior double root. The
smallest Bernstein-$D$ coefficient of such a polynomial stays negative and
decays only like $-t^\star(1-t^\star)g''(t^\star)/(2D)$: for
$g(t)=(t-\tfrac12)^2$ it is exactly $-1/\big(4(D-1)\big)$ for even $D$ and
$-1/(4D)$ for odd $D$, so no elevation is ever enough. Approaching tangency
the required degree likewise diverges: for
$g_\varepsilon(t)=(t-\tfrac12)^2+\varepsilon$, so that
$\min_{[0,1]}g_\varepsilon=\varepsilon$, the least $D$ with all
Bernstein-$D$ coefficients nonnegative satisfies
$D_{\min}=1/(4\varepsilon)+O(1)$; in
workspace terms
$D_{\min}\approx\|\g\gamma'\|^2t^\star(1-t^\star)/(2r_j\delta)$ for
clearance $\delta$. The avoidance condition is therefore carried exactly, as written in
\eqref{eq:Kdef}.

\paragraph{Data for Figure~\ref{fig:pipeline}.}
\label{app:figure-pipeline}
The figure shows two robots, A (left to right) and B (bottom to top),
planned jointly with one degree-$6$ Bernstein segment each; each robot is a
moving centre for the other (\S\ref{sec:intro}). The obstacles are
three ellipses and a circle, each replaced by the smallest scaled ellipse containing its offset by the robot radius $0.055$ (semi-axes raised by $0.055$, then by up to $2\%$), so
that the constraint on the robot's centre keeps the whole robot clear, and
the robots keep their centres at least $0.11$ apart. The cost is the
smoothness objective of \eqref{eq:Pinf} with $k=2$ (minimum acceleration),
positions fixed at both ends ($\ell=0$), and no steering constraint. Panel (2)
imposes the constraints only at $8$ sample times, ends included; panel (3) runs GlobMP with
every constraint on all of $[0,1]$. The instance has eight route classes (A
passes each of the two coloured ellipses on either side, and either robot crosses first);
one opening curve per class, restored and then refined by
Algorithm~\ref{alg:pgd}, gives feasible KKT curves in five of them,
with costs $1$, $1.2$, $5.3$, $5.6$ and $50.8$ times the cheapest; panel (3) shows
the cheapest and, in grey, the routes that take a different path in space.

\section{Addendum to Section~\ref{sec:setup}: formulation and algorithm}

\subsection{Variants of the formulation}
\label{app:variants}

The experiments use four variants of the blocks of \eqref{eq:Pinf}. Each is
convex in the position or has a convex feasible set imposed exactly, so
Proposition~\ref{prop:fixeddegree} covers all four. Theorem~\ref{thm:globalinf}
covers the position blocks (ellipsoids, robot pairs, boxes and walls); the
speed limit constrains $\g\gamma'$ and lies outside it.

\paragraph{Ellipsoidal obstacles.} A weighted norm
$\|\m x\|_{\g\Sigma}=(\m x^{\!\top}\g\Sigma\m x)^{1/2}$, $\g\Sigma\succeq0$, in the
collision constraint, $\|\g\gamma(t)-\m c_j(t)\|_{\g\Sigma_j}\ge r_j$, keeps every
block quadratic; with $\g\Sigma_j\succ0$ it describes an ellipsoid.

\paragraph{Several robots.} Several robots are planned jointly by stacking
their curves into one $\g\gamma$; let $\m E_a$ select robot $a$'s rows, so
$\g\gamma_a=\m E_a\g\gamma$. For a robot of radius $r_a$, a sphere of
radius $r$ and centre $\m m_j$ is $\g\Sigma_j=\m E_a^{\!\top}\m E_a$,
$\m E_a\m c_j=\m m_j$, $r_j=r+r_a$; two robots $a,b$ are kept apart by
$\g\Sigma_j=\m E_{ab}^{\!\top}\m E_{ab}$ with $\m E_{ab}=\m E_a-\m E_b$ and
$r_j=r_a+r_b$, so that $\|\g\gamma\|_{\g\Sigma_j}=\|\g\gamma_a-\g\gamma_b\|$. The
multi-robot experiments of \S\ref{sec:exp-multirobot} use this separation.

\paragraph{Workspace box and speed limit.} Box constraints use
$x_i-\underline x_i\ge0$ and $\overline x_i-x_i\ge0$ (Appendix~\ref{app:box}); a speed
limit uses $(T_{\mathrm{dur}}v_{\max})^2-\|\g\gamma'(t)\|^2\ge0$, where $T_{\mathrm{dur}}$ is the
trajectory duration, so the physical velocity is $\g\gamma'/T_{\mathrm{dur}}$. Both are
affine or concave quadratic, so their feasible sets are convex.

\subsubsection{Half-space and corridor constraints}
\label{app:affine}

Confining the curve to one side of a wall (a workspace bound, or a face
of a corridor) is the constraint $\m{a}_w^{\!\top}\g\gamma(t)-\xi_w\ge0$ for all
$t$. The polynomial $\m{a}_w^{\!\top}\g\gamma(t)-\xi_w$ has degree $d$, so its
Bernstein coefficient vector $\Gm^{\!\top}\m{a}_w-\xi_w\m{u}$ ($\m u$ the vector of ones) is raised to degree
$2d$ by the elevation map $\m{W}\in\R^{(2d+1)\times(d+1)}$,
$W_{si}=\binom{d}{i}\binom{d}{s-i}\big/\binom{2d}{s}$:
\begin{equation}
\label{eq:wallblock}
\g{\pi}_w(\Gm) \;:=\; \m{W}\big(\Gm^{\!\top}\m{a}_w-\xi_w\m{u}\big)\;\in\;\mathcal K_{2d} .
\end{equation}
Elevation rewrites one polynomial in a higher basis and changes nothing
else, so \eqref{eq:wallblock} holds exactly when the wall is respected at
every $t$, and the block enters the same cone $\mathcal K_{2d}$ as the collision
blocks. It is affine in $\Gm$, so its linearization is exact. An obstacle
block is a squared length and a wall block a length, so the implementation
divides each obstacle block by $r_j^2$ and each wall block by the workspace
scale; since $\mathcal K$ is a cone, this changes the numerics but not the
feasible set.

\subsubsection{Reference motions from learned policies}
\label{app:policyref}

The reference $\g\rho$ of \eqref{eq:Pinf} is a motion supplied by a learned
planner or policy (one fixed proposal per solve), a demonstration or a
classical planner. For positions $\m y_0,\ldots,\m y_L\in\R^n$ at normalized times
$0=t_0<\cdots<t_L=1$, we use the piecewise-linear interpolant
\begin{equation}
\label{eq:policyref}
\g\rho(t)
=\frac{t_{i+1}-t}{t_{i+1}-t_i}\,\m y_i
 +\frac{t-t_i}{t_{i+1}-t_i}\,\m y_{i+1},
\qquad t\in[t_i,t_{i+1}],\quad i=0,\ldots,L-1.
\end{equation}
It need not have the derivatives required of $\g\gamma$, and it enters
the steering constraint (Appendix~\ref{sec:steering}). Relative action commands
are first composed from the current configuration into absolute waypoints.

\subsection{Steering a generative planner}
\label{sec:steering}

Learned planners propose motion as waypoints. Diffusion and flow-matching
planners \citep{janner2022diffuser,carvalho2023mpd,shaoul2025mmd,direct2026}
denoise an array $\m y=(\m y_0,\ldots,\m y_L)$ of positions at fixed times, one
model per map, with start and goal written into the array at every step.
Constraints enter either as gradient guidance, which does not enforce them,
or as a projection of the waypoints, which enforces them at the $L+1$
sampled times only; Figure~\ref{fig:pipeline} and Appendix~\ref{app:examples} show what happens
between samples. We use \eqref{eq:P} as that projection, so the object the
sampler carries is a polynomial trajectory whose every block is tested at
every $t$, and the generator is left untouched.

Given waypoints $\m y$, the reference $\g\rho$ is their piecewise-linear
interpolant \eqref{eq:policyref}, and the steering constraint of
\eqref{eq:Pinf}, $\|\g\gamma(t)-\g\rho(t)\|\le\bar r$ for every $t$, keeps the
curve near the proposal (in the Robotarium runs, a tracking cost to $\g\rho$ replaces this constraint). Otherwise Algorithms~\ref{alg:pgd}--\ref{alg:proj}
run as in \S\ref{sec:setup}.

\section{Addendum to Section~\ref{sec:theory}: theory and proofs}
\label{app:theory}
Appendix~\ref{app:prelim} sets up the chart, the matrices, the compactness
facts, Assumption~\ref{assum:structural}, and the step-size and cone-solve bounds (Algorithms~\ref{alg:proj} and~\ref{alg:pk}). Appendix~\ref{app:gpa-global} proves
the fixed-degree result (Proposition~\ref{prop:fixeddegree}) and Theorem~\ref{thm:globalinf}, and, under Assumption~\ref{assum:mfcq} stated there, its counterpart at a fixed number of pieces (Proposition~\ref{prop:globalC}) and the degree refinement (Proposition~\ref{prop:degree}); Appendix~\ref{app:scope} discusses the assumptions and the stability of the optimal value.

\subsection{Preliminaries}
\label{app:prelim}

\subsubsection{Boundary chart}
\begin{lemma}
\label{lem:affchart}
If $\m A$ has full column rank, then
$\Gm_{\mathrm{aff}}=\m B(\m A^\top\m A)^{-1}\m A^\top$
belongs to $\affset$. If $\m N$ is an orthonormal basis for
$\ker\m A^\top$, then
$
\m Z\longmapsto\Gm_{\mathrm{aff}}+\m Z\m N^\top$ ($\m Z=\Gm_\free$) is an affine isometry from $\R^{n\times d_f}$ onto $\affset$,
with inverse $\Gm\longmapsto(\Gm-\Gm_{\mathrm{aff}})\m N$.
\end{lemma}
\begin{proof}[Proof of Lemma~\ref{lem:affchart}]
Full column rank makes $\m A^\top\m A$ invertible, and multiplying
the displayed anchor by $\m A$ gives $\m B$. Any difference of two
solutions has rows in $\ker\m A^\top$ and thus uniquely equals
$\m Z\m N^\top$. The inverse formula follows from $\m N^\top\m N=\m I$.
Finally, with $\|\cdot\|_F$ the Frobenius norm, $\|\m Z\m N^\top\|_F^2=\tr(\m Z^\top\m Z)=\|\m Z\|_F^2$.
The map and its inverse are continuous, so relative local minima and
set membership transfer between the two coordinates.
\end{proof}

\subsubsection{Matrices, cone and bounds}
\label{app:proofs-prelim}

\paragraph{Polynomial cost, boundary, and constraint matrices.}
\label{app:polynomial-details}
The costs we use are all of the quadratic form
\begin{equation}
\label{eq:generalcost}
\mathrm{cost}(\Gm)=\tr(\m Q\Gm^\top\Gm)
-2\tr(\m L^\top\Gm)+c_0,
\end{equation}
with $\m Q$ symmetric. The main-text cost is the case $\m L=0$, $c_0=0$; the linear and constant
terms cover tracking costs, and the gradient
is $2(\Gm\m Q-\m L)\m N$; the constraint blocks do not depend on the cost.
For reference fitting, define
$\m{G}_d=\int_0^1\m{b}_d(t)\m{b}_d(t)^{\!\top}dt$ and
$\m{R}_d=\int_0^1\g\rho(t)\m{b}_d(t)^{\!\top}dt$.
Then \eqref{eq:generalcost} uses $\m Q=\tfrac12\m G_d$,
$\m L=\tfrac12\m R_d$, and
$c_0=\tfrac12\int_0^1\|\g\rho(t)\|^2dt$.
The reference need only be square-integrable. For the integrated squared-derivative
objective in \eqref{eq:Pinf}, let $\m D_k$ be the matrix with $\m b_d^{(k)}=\m D_k\m b_{d-k}$ and use $\m Q=\m G^{(k)}:=\m D_k\m G_{d-k}\m D_k^{\!\top}$, $\m L=0$, $c_0=0$.
The endpoint evaluation matrix is
\begin{equation}
\label{eq:bc}
\Gm\m{A}=\m{B},\qquad
\m{A}=[\,\m{b}_d(0)\;\m{b}_d(1)\;\cdots\;
\m{b}_d^{(\ell)}(0)\;\m{b}_d^{(\ell)}(1)\,]
\in\R^{(d+1)\times2(\ell+1)}.
\end{equation}
The columns of $\m{B}$ give the prescribed values in the same order.

For the signed quadratic derivative blocks used in the examples, let
$0\le i_j\le k$, $\g\Sigma_j\succeq0$, $\bar d_j=d-i_j$, and
$\m{D}_0=\m{I}$. Then
\begin{equation}
\label{eq:MjY}
\m{M}_j(\Gm)=\sigma_j\big(\Gm_j^{\!\top}\g\Sigma_j\Gm_j
-r_j^2\m{u}\m{u}^{\!\top}\big),\qquad
\Gm_j=\Gm\m{D}_{i_j}-\m{C}_j,\quad \sigma_j\in\{-1,+1\},
\end{equation}
where $\m{u}$ is the vector of ones and the columns of $\m C_j$ are the Bernstein coefficients of the centre $\m c_j$. Let $S_e$ denote the linear map that sends a
symmetric $(e+1)\times(e+1)$ matrix $\m M$ to the degree-$2e$ Bernstein
coefficients of $\m b_e(t)^{\!\top}\m M\m b_e(t)$. Then
\begin{equation}
\label{eq:pj}
g_j^{\Gm}(t)=\m{b}_{\bar d_j}(t)^{\!\top}\m{M}_j(\Gm)\m{b}_{\bar d_j}(t),
\qquad \g{\pi}_j(\Gm)=S_{\bar d_j}[\m{M}_j(\Gm)].
\end{equation}
Here $m_j=2\bar d_j$. If the centre has higher degree, first elevate both
coefficient representations to a common degree and use the corresponding
linear derivative/elevation map in place of $\m{D}_{i_j}$.
Affine box blocks have degree $d$, are raised to degree $2d$ (Appendix~\ref{app:box}), and have affine coefficient maps.

\paragraph{Properties of the cone.}
\label{app:coneprops}

Write $\mathcal K_m$ for the cone \eqref{eq:Kdef} of degree $m$, so $\mathcal K_j=\mathcal K_{m_j}$, and $p_{\g\zeta}(t):=\g\zeta^{\!\top}\m b_m(t)$ for the polynomial with coefficients $\g\zeta$. All four properties of $\mathcal K_m$ are immediate
from \eqref{eq:Kdef}. It is an intersection of closed homogeneous
half-spaces, hence a \emph{closed convex cone}. It is \emph{pointed}, since
$\g{\zeta},-\g{\zeta}\in\mathcal K_m$ forces the polynomial to vanish, hence $\g\zeta=\m 0$ as $\m b_m$ is a basis. It has \emph{nonempty interior},
since $\m b_m$ is a nonnegative partition of unity, so
$|p_{\g{\zeta}}(t)|\le\|\g{\zeta}\|_\infty\le\|\g{\zeta}\|$ and any $\g{\zeta}$ with
$\min_{[0,1]}p_{\g{\zeta}}=\mu>0$ carries the open $\mu$-ball. And it is
\emph{semialgebraic}, by Tarski--Seidenberg applied to
$\forall t\,(0\le t\le1\to p_{\g{\zeta}}(t)\ge\nobreak0)$.

The \emph{dual cone}, the set of vectors with nonnegative inner product with every element of $\mathcal K_m$, is
\begin{equation}
\label{eq:Kdual}
\mathcal K_m^*=\mathrm{cone}\{\m b_m(t):t\in[0,1]\},
\end{equation}
the nonnegative combinations of the evaluation vectors used in
\S\ref{sec:distcone}. This holds because $\mathcal K_m$ is by
definition the dual of that conic hull, the hull is closed (each element
needs at most $m+1$ generators by Carath\'eodory, and
$\m{u}^{\!\top}\m b_m(t)\equiv1$ bounds the total weight, so the
generating data lie in a compact set), and the bipolar theorem (a closed
convex cone is the dual of its own dual) applies.

\paragraph{Lipschitz bound for the restoration merit.}
\label{app:proof-lipJ}

For general $C^2$ coefficient maps on a compact convex region $\mathcal R$, set
$\m e_j(\m z)=\g{\pi}_j(\m z)-P_{\mathcal K_j}(\g{\pi}_j(\m z))$.
Projection residuals are nonexpansive and $0\in\mathcal K_j$, hence, with
$B_{\pi,j}$ and $B_{J,j}$ bounds on $\|\g\pi_j\|$ and $\|D\g\pi_j\|_{\mathrm{op}}$ over $\mathcal R$ and $L_{J,j}$ a Lipschitz constant of $D\g\pi_j$ (computed below for spherical exclusions),
$\|\m e_j(\m z)\|\le B_{\pi,j}$ and
$\|\m e_j(\m z)-\m e_j(\m y)\|\le B_{J,j}\|\m z-\m y\|$.
The gradient $2D\g\pi_j(\m z)^{\!\top}\m e_j(\m z)$ therefore has Lipschitz constant
at most $2B_{J,j}^2+2L_{J,j}B_{\pi,j}$, and summing over blocks bounds the
full merit; the step scale of Appendix~\ref{app:stepdisc} uses the larger,
also valid, constant with $4B_{J,j}^2$.
The remaining calculation specializes to spherical position exclusions ($\g\Sigma_j=\m I$).

We now show that $L_{J,j}$, the Lipschitz constant of the Jacobian
$D\g\pi_j$ of \S\ref{sec:alg1} as a function of the base point, is a fixed
number, independent of $\mathcal{R}$. By the exact quadratic expansion of the
coefficient map, the Jacobian of $\g{\pi}_j$ at a base point $\Gm$, applied to a
direction $\g{\Delta}$, is
\begin{equation}\label{eq:jacpi}
D\g\pi_j(\Gm)[\g{\Delta}]=S_d\big(\g{\Delta}^{\!\top}\Gm_j+\Gm_j^{\!\top}\g{\Delta}\big).
\end{equation}

Take two base points $\Gm^{(1)},\Gm^{(2)}\in\affset$ and write
$\m{F}:=\Gm^{(1)}-\Gm^{(2)}$. Then $\Gm_j^{(1)}-\Gm_j^{(2)}=\m{F}$ as well, since the
obstacle-centring shift is a constant offset. Subtracting \eqref{eq:jacpi} at the
two base points,
\[
\big(D\g\pi_j(\Gm^{(1)})-D\g\pi_j(\Gm^{(2)})\big)[\g{\Delta}] \;=\; S_d\big(\g{\Delta}^{\!\top}\m{F}+\m{F}^{\!\top}\g{\Delta}\big),
\]
which is linear in $\g{\Delta}$ and in $\m{F}$. Since $S_d$ is a fixed linear map,
\[
\|D\g\pi_j(\Gm^{(1)})-D\g\pi_j(\Gm^{(2)})\|_{\mathrm{op}}\le2\|S_d\|_{\mathrm{op}}\|\m{F}\|_F ,
\]
so $D\g\pi_j$ is Lipschitz on all of $\affset$ with constant
$2\|S_d\|_{\mathrm{op}}$, which depends only on $d$. For a robot pair, both
$\m F$ and $\g\Delta$ enter through the difference of the two robots' rows, each
at most $\sqrt2$ times longer, so the constant gains a factor $2$. Finally,
restricting to the free coordinates through the chart
$\m{Z}\mapsto\Gm_{\mathrm{aff}}+\m{Z}\m{N}^{\!\top}$ of Lemma~\ref{lem:affchart} composes this
bound with the fixed linear map $\m{N}$, and $\|\m N\|_{\mathrm{op}}=1$. Hence
$L_{J,j}:=2\|S_d\|_{\mathrm{op}}$ ($4\|S_d\|_{\mathrm{op}}$ for a robot pair), an explicit constant that depends only on $d$
and not on any compact set $\mathcal{R}$.

\paragraph{Quadratic structure and curvature.}
\label{sec:quadstructure}

For matrices $\m{X},\m{Y}$ of the same shape,
$\langle \m{X},\m{Y}\rangle := \tr(\m{X}^{\!\top}\m{Y})$ is the Frobenius inner
product.

\begin{lemma}
\label{lem:quadexact}
Fix the chart anchor $\Gm_{\mathrm{aff}}\in\affset$ and write $\Gm=\Gm_{\mathrm{aff}}+\m{Z}\m{N}^{\!\top}$
(Lemma~\ref{lem:affchart}). Then, exactly, with no remainder,
\[
\mathrm{cost}(\Gm_{\mathrm{aff}}+\m{Z}\m{N}^{\!\top}) \;=\; \mathrm{cost}(\Gm_{\mathrm{aff}})
\;+\; 2\big\langle (\Gm_{\mathrm{aff}}\m Q-\m L)\m{N},\,\m{Z}\big\rangle
\;+\; \tr\big(\m{H}\m{Z}^{\!\top}\m{Z}\big), \qquad \m{H} \;:=\; \m{N}^{\!\top}\m Q\m{N}.
\]
\end{lemma}
\begin{proof}[Proof of Lemma~\ref{lem:quadexact}]
Let $\m W:=\m Z\m N^{\!\top}$, so $\Gm=\Gm_{\mathrm{aff}}+\m W$; $\m Q$ is symmetric. Then
\[
\tr(\m Q\Gm^{\!\top}\Gm)=\tr(\m Q\Gm_{\mathrm{aff}}^{\!\top}\Gm_{\mathrm{aff}})
+2\tr(\m Q\Gm_{\mathrm{aff}}^{\!\top}\m W)+\tr(\m Q\m W^{\!\top}\m W),
\]
since $\tr(\m Q\m W^{\!\top}\Gm_{\mathrm{aff}})=\tr(\Gm_{\mathrm{aff}}^{\!\top}\m W\m Q)=\tr(\m Q\Gm_{\mathrm{aff}}^{\!\top}\m W)$. By cyclicity of the trace,
\[
\tr(\m Q\Gm_{\mathrm{aff}}^{\!\top}\m Z\m N^{\!\top})=\langle\Gm_{\mathrm{aff}}\m Q\m N,\m Z\rangle,\qquad
\tr(\m Q\m N\m Z^{\!\top}\m Z\m N^{\!\top})=\tr(\m H\m Z^{\!\top}\m Z),\qquad
\tr(\m L^{\!\top}\m W)=\langle\m L\m N,\m Z\rangle .
\]
Substituting into \eqref{eq:generalcost} and collecting the terms without $\m Z$
into $\mathrm{cost}(\Gm_{\mathrm{aff}})$ gives the identity. Finally, with $\m Z_i$
the rows of $\m Z$,
\[
\lambda_{\min}(\m H)\|\m Z\|_F^2\le\tr(\m H\m Z^{\!\top}\m Z)=\textstyle\sum_{i}\m Z_i\m H\m Z_i^{\!\top}\le\lambda_{\max}(\m H)\|\m Z\|_F^2 . \qedhere
\]
\end{proof}

\begin{lemma}
\label{lem:hpos}
For $\m Q\succeq0$ and a full-column-rank $\m N$,
$\m H=\m N^\top\m Q\m N\succ0$ if and only if
$\ker\m Q\cap\operatorname{range}(\m N)=\{0\}$.
For the derivative objective in \eqref{eq:Pinf} with the endpoint data
\eqref{eq:bc} and $\m N$ the chart basis of Lemma~\ref{lem:affchart}, $\m H\succ0$ for $d\ge\max\{k,2(\ell+1)\}$ under Assumption~\ref{assum:curv}.
\end{lemma}
\begin{proof}[Proof of Lemma~\ref{lem:hpos}]
For $\m w:=\m N\m v$,
\begin{equation}\label{eq:l5-quad}
\m v^{\!\top}\m H\m v=\m w^{\!\top}\m Q\m w=\|\m Q^{1/2}\m w\|^2\ge0 ,
\end{equation}
which vanishes if and only if $\m Q\m w=\m 0$. If $\ker\m Q\cap\operatorname{range}(\m N)=\{0\}$,
then \eqref{eq:l5-quad} vanishes only for $\m N\m v=\m 0$, that is, $\m v=\m 0$ ($\m N$
has full column rank). Conversely, $\m 0\ne\m N\m v\in\ker\m Q$ gives
$\m v^{\!\top}\m H\m v=0$ with $\m v\ne\m 0$.

For tracking, $\m Q=\frac12\m G_d$ and $\m w^{\!\top}\m G_d\m w=\int_0^1p_{\m w}^2\,dt>0$ for
$\m w\ne\m 0$, with $p_{\m w}$ the polynomial with coefficients $\m w$. For
planning, $\m Q=\m G^{(k)}$; let $\m w\in\ker\m G^{(k)}\cap\ker\m A^{\!\top}$. By
Lemma~\ref{lem:kernels}, $\deg p_{\m w}<k$ and $p_{\m w}=[t(1-t)]^{\ell+1}q$. If
$q\ne0$, then $\deg p_{\m w}\ge2(\ell+1)\ge k$ by Assumption~\ref{assum:curv}, a
contradiction; so $\m w=\m 0$.
\end{proof}

Hence the constants
\begin{equation}
\label{eq:albeta}
\alpha \;:=\; 2\lambda_{\min}(\m{H}) \;>\;0, \qquad
\beta \;:=\; 2\lambda_{\max}(\m{H}) \;<\;\infty
\end{equation}
are well defined, and by Lemma~\ref{lem:quadexact} they bound the
curvature on all of $\affset$, with no remainder term. For tracking with orthonormal $\m{N}$, $\lambda_{\min}(\m{H})\ge\frac12\lambda_{\min}(\m{G}_d)>0$.

\paragraph{Compactness and existence on each component.}
\label{sec:compactexist}
Lemma~\ref{lem:ellipsoid} gives the compact sublevel sets used in the proof of
Proposition~\ref{prop:fixeddegree}. Lemmas~\ref{lem:components}
and~\ref{lem:compexist} are not needed for the main results; they record that the
feasible set splits into finitely many connected components, each with its own
cheapest curve, which is why GlobMP is run from several seeds in
\S\ref{sec:experiments}.

\begin{lemma}
\label{lem:ellipsoid}
For any $\bar{\Gm}\in\affset$, writing $\Gm=\Gm_{\mathrm{aff}}+\m{Z}\m{N}^{\!\top}$ in the
chart of Lemma~\ref{lem:affchart},
$\mathcal E(\bar{\Gm}) := \{\Gm\in\affset : \mathrm{cost}(\Gm)\le\mathrm{cost}(\bar{\Gm})\}$
is, under $\m{H}\succ0$ (Lemma~\ref{lem:hpos}), an ellipsoid in
$\m{Z}$-coordinates centred at the unconstrained minimizer
$\m{Z}^\star=-(\Gm_{\mathrm{aff}}\m Q-\m L)\m{N}\m{H}^{-1}$ of $\mathrm{cost}$ on
$\affset$, hence a compact convex subset of $\affset$.
\end{lemma}
\begin{proof}[Proof of Lemma~\ref{lem:ellipsoid}]
Let $c_{\mathrm{a}}:=\mathrm{cost}(\Gm_{\mathrm{aff}})$ and $\m V:=(\Gm_{\mathrm{aff}}\m Q-\m L)\m N$, so
$\m Z^\star=-\m V\m H^{-1}$ and $\m H\m Z^{\star\top}=-\m V^{\!\top}$. Completing the square in
Lemma~\ref{lem:quadexact},
\begin{equation}\label{eq:l6-square}
\mathrm{cost}(\Gm_{\mathrm{aff}}+\m{Z}\m{N}^{\!\top})=c_{\mathrm{a}}
-\tr(\m H\m Z^{\star\top}\m Z^\star)+\tr\big(\m{H}\m Y^{\!\top}\m Y\big),\qquad \m Y:=\m Z-\m Z^\star,
\end{equation}
because $\tr(\m H\m Y^{\!\top}\m Y)=\tr(\m H\m Z^{\!\top}\m Z)+2\langle\m V,\m Z\rangle+\tr(\m H\m Z^{\star\top}\m Z^\star)$.
So $\m Z^\star$ is the unconstrained minimizer, and
\[
\mathcal E(\bar\Gm)=\big\{\Gm_{\mathrm{aff}}+(\m Z^\star+\m Y)\m N^{\!\top}:\ \tr(\m H\m Y^{\!\top}\m Y)\le\varsigma\big\},\qquad
\varsigma:=\tr\big(\m H(\bar{\m Z}-\m Z^\star)^{\!\top}(\bar{\m Z}-\m Z^\star)\big)\ge0,
\]
with $\bar{\m Z}$ the coordinates of $\bar\Gm$. As $\m H\succ0$, the set of such
$\m Y$ is an ellipsoid: closed, convex, and bounded by
$\|\m Y\|_F^2\le\varsigma/\lambda_{\min}(\m H)$ (Lemma~\ref{lem:quadexact}). The
chart is an affine isometry (Lemma~\ref{lem:affchart}), which preserves these
properties.
\end{proof}

\begin{lemma}
\label{lem:components}
Suppose every coefficient map is continuous and semialgebraic, that is, described
by finitely many polynomial equations and inequalities (polynomial maps qualify).
Consider the set of feasible curves,
$\affset\cap\bigcap_j\g{\pi}_j^{-1}(\mathcal K_j)$, identified with $\feasset$
through the chart of Lemma~\ref{lem:affchart}. Then:
\begin{enumerate}[label=(\roman*),nosep]
\item \emph{Finitely many pieces.} $\feasset$ splits into finitely many separate
pieces $\feasset_1,\dots,\feasset_{\bar c}$, its connected components: two
feasible curves lie in the same piece exactly when one can be deformed
continuously into the other without leaving $\feasset$.
\item \emph{Each piece is closed.} Each $\feasset_c$ is closed in
$\R^{n\times(d+1)}$, and also open relative to $\feasset$ (clopen in
$\feasset$).
\item \emph{The pieces are a positive distance apart.} If $\mathcal U$ is a nonempty
compact subset of one piece $\feasset_c$, then every feasible curve outside
that piece stays at distance at least
$d_0(\mathcal U):=\dist(\mathcal U,\feasset\setminus\feasset_c)>0$ from
$\mathcal U$; we set $d_0(\mathcal U)=+\infty$ when $\feasset$ has only the one
piece.
\end{enumerate}
\end{lemma}
\begin{proof}[Proof of Lemma~\ref{lem:components}]
$\feasset=\affset\cap\bigcap_j\g\pi_j^{-1}(\mathcal K_j)$ is closed, as a finite
intersection of preimages of closed sets under continuous maps. It is
semialgebraic: each $\g\pi_j$ and each $\mathcal K_j$ is (the quantifier
``for every $t\in[0,1]$'' is eliminated by the Tarski--Seidenberg theorem), and
preimages and finite intersections preserve this \citep{bochnak1998real}.

(i) A semialgebraic set has finitely many connected components, each path
connected \citep{bochnak1998real}.

(ii) A connected component is closed in its set, hence in $\R^{n\times(d+1)}$
since $\feasset$ is closed. Its complement in $\feasset$ is a finite union of
closed components, so $\feasset_c$ is also open in $\feasset$.

(iii) $\mathcal W:=\feasset\setminus\feasset_c$ is closed by (ii) and disjoint from
$\mathcal U$. If $\mathcal W\ne\emptyset$, then $x\mapsto\dist(x,\mathcal W)$ is
continuous and positive on the compact set $\mathcal U$, so its minimum
$d_0(\mathcal U)$ is positive.
\end{proof}

\begin{lemma}
\label{lem:compexist}
Assume $\m H\succ0$ and continuous coefficient maps. Then every connected
component $\feasset_c$ of the feasible set contains a cheapest curve: there is a
$\Gm^\star_c\in\feasset_c$ with $\mathrm{cost}(\Gm^\star_c)\le\mathrm{cost}(\Gm)$
for every $\Gm\in\feasset_c$ (it need not be unique; we fix one).
\end{lemma}
\begin{proof}[Proof of Lemma~\ref{lem:compexist}]
Fix $\bar\Gm\in\feasset_c$ and let $\mathcal Y:=\feasset_c\cap\mathcal E(\bar\Gm)$, with
$\mathcal E(\bar\Gm)$ from Lemma~\ref{lem:ellipsoid}. $\mathcal Y$ contains $\bar\Gm$; it is
closed ($\feasset$ is closed and so is each of its components) and bounded
($\mathcal E(\bar\Gm)$ is, as $\m H\succ0$). So the continuous cost attains its minimum on
$\mathcal Y$ at some $\Gm^\star_c$, and every $\Gm\in\feasset_c\setminus\mathcal Y$ has
$\mathrm{cost}(\Gm)>\mathrm{cost}(\bar\Gm)\ge\mathrm{cost}(\Gm^\star_c)$.
\end{proof}

\paragraph{Cumulative constraints and derivative-cost curvature.}
\label{app:cumulative-theory}
The finite-dimensional reduction of \eqref{eq:Pinf} also admits scalar budget blocks. Write $\m z=\operatorname{vec}(\Gm_\free)$
and let $\pi_j(\m z)$ denote a budget block \eqref{eq:intblock}. Its cone is $\R_+$,
so its contribution to the restoration merit and gradient is
\[
V_j(\m z)=\min\{\pi_j(\m z),0\}^2,\qquad
\nabla V_j(\m z)=2\min\{\pi_j(\m z),0\}\nabla\pi_j(\m z).
\]
These formulas do not require the integrand to be polynomial in time.
For a general integrand, differentiation under the integral is justified,
for example, if it is measurable in time and, for almost every $t$, $C^2$ in the trajectory
arguments, with its value and first two composed
$\m z$-derivatives dominated by integrable functions on each compact set
of free coordinates.

For a single segment of degree $d\ge k$ with only start and goal fixed, the objective of
\eqref{eq:Pinf} has $\m H\succ0$ if and only if $k\le2$ (the case $\ell=0$ of
Assumption~\ref{assum:curv}): a free perturbation has zero derivative cost exactly when its degree is
below $k$, and among those vanishing at $t=0$ and $t=1$ the first nonzero one is $t(1-t)$, of degree $2$. When the cost has free zero-curvature directions, compactness must be checked separately. Additional endpoint derivatives
can remove these directions as in Lemma~\ref{lem:kernels}; adding them
changes the boundary data of \eqref{eq:Pinf}.

\subsubsection{The structural assumption}
\label{app:proofs-assumptions}
\setcounter{assumption}{3}%
\begin{assumption}[finite-dimensional representation]
\label{assum:structural}
$d\ge k$ and $2(\ell+1)\le d$, and every pointwise coefficient map and
scalar cumulative map is $C^2$ on an open neighbourhood of the regions used
in the analysis.
\end{assumption}
The quadratic obstacle and separation blocks, and budget blocks whose integrand meets the
condition of Appendix~\ref{app:cumulative-theory}, satisfy the smoothness part.

The degree condition makes $\m A$ of full column rank,
since its columns are independent Hermite interpolation conditions
\citep{stoer2002introduction}; $\m N$ is an orthonormal basis of
$\ker\m A^\top$, and $\m Q\succeq0$ always. Under
Assumption~\ref{assum:structural}, Assumption~\ref{assum:curv} is
equivalent to $\m H\succ0$ for the derivative objective, with $\m H$ as in Lemma~\ref{lem:quadexact}
(Lemmas~\ref{lem:hpos} and~\ref{lem:kernels}).

Write $\mathcal P_r$ for scalar polynomials of degree at most $r$, with
$\mathcal P_{r}=\{0\}$ for $r<0$, and $\mathrm{coeff}(\cdot)$ for the Bernstein-$d$
coefficient vector of a polynomial.
\begin{lemma}[kernels]
\label{lem:kernels}
Let $0\le k\le d$, let $\m A$ be the endpoint matrix of \eqref{eq:bc}, and read
each coefficient vector as the polynomial of degree at most $d$ it represents.
The vectors with zero $k$th-derivative cost are exactly the polynomials of
degree below $k$, and the vectors with zero endpoint data are exactly the
multiples of $[t(1-t)]^{\ell+1}$:
\[
\ker\m G^{(k)}=\mathrm{coeff}(\mathcal P_{k-1}),\qquad
\ker\m A^\top=
\mathrm{coeff}\big([t(1-t)]^{\ell+1}\mathcal P_{d-2(\ell+1)}\big).
\]
Under Assumption~\ref{assum:structural}, the second space has dimension
$d_f=d+1-2(\ell+1)$.
\end{lemma}
\begin{proof}[Proof of Lemma~\ref{lem:kernels}]
Let $p_{\g\zeta}$ be the polynomial with coefficients $\g\zeta$. Since
$\g\zeta^\top\m G^{(k)}\g\zeta=\int_0^1|p_{\g\zeta}^{(k)}|^2dt$ and $\m G^{(k)}\succeq0$,
$\g\zeta\in\ker\m G^{(k)}$ if and only if $p_{\g\zeta}^{(k)}\equiv0$, that is,
$\deg p_{\g\zeta}<k$.

$\m A^\top\g\zeta=\m 0$ means $p_{\g\zeta}^{(s)}(0)=p_{\g\zeta}^{(s)}(1)=0$ for
$s\le\ell$. By Taylor's formula at $0$ and at $1$, $t^{\ell+1}$ and $(1-t)^{\ell+1}$
divide $p_{\g\zeta}$; they are coprime, so
\[
p_{\g\zeta}=[t(1-t)]^{\ell+1}q,\qquad \deg q\le d-2(\ell+1),
\]
and conversely every such product has these zero data. The map
$q\mapsto[t(1-t)]^{\ell+1}q$ is linear and injective on $\mathcal P_{d-2(\ell+1)}$,
whose dimension is $d+1-2(\ell+1)=d_f$.
\end{proof}

\subsubsection{Step size and the cone solve}
\label{app:proofs-convergence}
\paragraph{Step size.}
\label{app:stepdisc}
A valid conservative step scale, provided the step stays in a compact
convex region $\mathcal R$, is
\[
L_{\mathrm{proj}}=\sum_j(4B_{J,j}^2+2L_{J,j}B_{\pi,j}),
\qquad \eta_{\mathrm{proj}}=1/L_{\mathrm{proj}},
\]
where $B_{J,j}=\max_{\mathcal R}\|D\g\pi_j\|_{\mathrm{op}}$,
$B_{\pi,j}=\max_{\mathcal R}\|\g{\pi}_j\|$, and $L_{J,j}$ is the
Lipschitz constant of the Jacobian
$D\g\pi_j$ of \S\ref{sec:alg1}.

On a compact convex region containing both endpoints of a restoration step,
the Lipschitz bound of Appendix~\ref{app:proof-lipJ} gives
\[
V(\m z-\eta_{\mathrm{proj}}\nabla V(\m z))
\le V(\m z)-\frac{\eta_{\mathrm{proj}}}{2}\|\nabla V(\m z)\|^2,
\]
for $0<\eta_{\mathrm{proj}}\le1/L_{\mathrm{proj}}$.
The step must be kept in the region, which compactness does not do by
itself, and constants taken over a ball around the unknown minimizer are
not computable. Algorithm~\ref{alg:proj} keeps every nonzero residual, so
omitted blocks have zero gradient and every update is a step on the full
merit
$V(\m z)=\sum_j\dist^2(\g\pi_j(\m z),\mathcal K_j)$, the merit
\eqref{eq:restoration-merit} in the coordinates $\m z$.
With exact cone projections and the region condition above, telescoping
$T_{\mathrm{proj}}$ executed steps gives
\[
\min_{0\le\nu<T_{\mathrm{proj}}}\|\nabla V(\m z^\nu)\|^2
\le \frac{2V(\m z^0)}{\eta_{\mathrm{proj}} T_{\mathrm{proj}}}.
\]
The bound covers scalar cumulative blocks and does not need $\m H\succ0$.
It gives a small merit gradient, not feasibility: for the scalar block $\pi(z)=z^2-1\in\R_+$ of one coordinate $z$,
$V(z)=(1-z^2)^2$ near zero and $V'(0)=0$ although $V(0)=1$.

\paragraph{The cone solve and its accuracy.}
\label{app:innerrate}
Algorithm~\ref{alg:pk} approximates the projection onto $\mathcal K_m$
for any degree $m\ge0$, with $m=m_j$ for block $j$. By Moreau's identity it
projects $-\g\pi/\|\g\pi\|$ onto the dual cone $\mathcal K_m^*$ with the atoms
$\m b_m(t)$ and adds $\|\g\pi\|$ times the result to $\g\pi$;
$T_{\mathrm{cone}}$ is a numerical budget.

\begin{algorithm}[t]
\caption{Approximate projection onto $\mathcal K_m$ by fully corrective
nonnegative matching pursuit}
\label{alg:pk}
\begin{algorithmic}[1]
\Require $\g\pi\in\R^{m+1}$, degree $m\ge0$, threshold $\varepsilon\ge0$,
 iteration budget $T_{\mathrm{cone}}$
\State \textbf{if} $\g\pi=\m 0$ \textbf{then return} $\m 0$
\State $\m y\gets-\g\pi/\|\g\pi\|$, $\mathcal T\gets\emptyset$,
 $\g\omega\gets\m 0$, $\m v\gets\m y$
\For{$\nu=1,\ldots,T_{\mathrm{cone}}$}
 \State $t_\nu\gets\operatorname*{arg\,max}_{t\in[0,1]}
  \langle\m v,\m b_m(t)\rangle/\|\m b_m(t)\|$
 \State \textbf{if} $\langle\m v,\m b_m(t_\nu)\rangle/\|\m b_m(t_\nu)\|
   \le\varepsilon$ \textbf{then break}
 \State $\mathcal T\gets\mathcal T\cup\{t_\nu\}$
 \State $\m w\gets\operatorname*{arg\,min}_{w_t\ge0,\,t\in\mathcal T}
    \|\sum_{t\in\mathcal T}w_t\m b_m(t)-\m y\|^2$
 \State $\g\omega\gets\sum_{t\in\mathcal T}w_t\m b_m(t)$,
    $\m v\gets\m y-\g\omega$
\EndFor
\State \Return $\g\pi+\|\g\pi\|\g\omega$
\end{algorithmic}
\end{algorithm}

The normalized correlation gives the best decrease along a single
nonnegative ray. It is not a polynomial in $t$: if
$\mu(t)=\langle\m v,\m b_m(t)\rangle$ and
$\|\m b_m(t)\|^2>0$, its interior stationary points satisfy
$2\mu'\|\m b_m\|^2-\mu(\|\m b_m\|^2)'=0$. Evaluate the correlation at these roots and at the endpoints; an identically
zero derivative means the correlation is constant.

A direct rate argument specializes the conic pursuit framework of
\citet{locatello2017cone}. Let $\g\omega_\nu$ be an exact fully corrective
iterate, $\g\omega_* =P_{\mathcal K_m^*}(\m y)$ (so $\g\omega^\star=\|\g\pi\|\g\omega_*$ in \eqref{eq:main-cone-oracle}), and measure the gap
\[
\theta_\nu=\tfrac12\|\m y-\g\omega_\nu\|^2-\tfrac12\|\m y-\g\omega_*\|^2\ \ge0,
\]
which is nonnegative because $\g\omega_*$ is the nearest point of the cone.

For the normalized atoms $\m a_t=\m b_m(t)/\|\m b_m(t)\|$,
nonnegative entries summing to one give $\|\m b_m(t)\|\le1$, so $\m u^\top\m a_t\ge1$ for the vector of ones $\m u$.
Since $\|\g\omega_*\|\le\|\m y\|=1$, it has an atomic
representation of total weight at most $\sqrt{m+1}$.

Exact nonnegative least squares gives
$\langle\m y-\g\omega_\nu,\g\omega_\nu\rangle=0$. Convexity then bounds
$\theta_\nu$ by $\sqrt{m+1}$ times the largest positive atom correlation.
Adding that unit atom with its optimal nonnegative weight decreases the
objective by half the squared correlation; full correction does at least
as well. Hence
\[
\theta_{\nu+1}\le \theta_\nu-\frac{\theta_\nu^2}{2(m+1)},\qquad
\theta_\nu\le\frac{2(m+1)}{\nu+4(m+1)},
\]
using $\theta_0\le1/2$, for exact maximization and weight solves without a
stopping threshold; it gives neither a linear rate nor finite termination.

If the correlation test stops the loop at iteration $\nu$, then
$\theta_{\nu-1}\le\sqrt{m+1}\,\varepsilon$. Strong convexity of the squared-distance
objective gives projection error at most
$\|\g\pi\|\sqrt{2\sqrt{m+1}\,\varepsilon}$ for the returned vector.
Its polynomial values are at least $-\|\g\pi\|\varepsilon$, since
that vector is $-\|\g\pi\|\m v$.
For $\varepsilon>0$ these bound the error rather than give exact
membership, and stopping on the budget gives neither bound. Proposition~\ref{prop:fixeddegree}
and Theorem~\ref{thm:globalinf} assume an exact solve of the convex problem \eqref{eq:inner-step}, which is not
this cone projection; Algorithm~\ref{alg:pk} plays no role there.

\subsection{Proof of Theorem~\ref{thm:globalinf}}
\label{app:gpa-global}

Throughout this subsection, the blocks are polynomial in $\Gm$, as the
sphere, ellipsoid, pair and affine blocks, and budgets with polynomial integrands, are, and $\m z$ stands for $\Gm_\free$. It uses the boundary chart
(Lemma~\ref{lem:affchart}) and the compact cost sublevel sets
(Lemma~\ref{lem:ellipsoid}) of Appendix~\ref{app:prelim}; the push maps
that move a curve off the obstacles (Lemma~\ref{lem:push}) serve only
Appendix~\ref{app:globalC} and the stability results of Appendix~\ref{app:scope}.

\subsubsection{Fixed degree}

Let $E^\star$ denote the optimum of \eqref{eq:P}. Its collision blocks
($j\le p$) are polynomial and convex in the position, and its steering block,
speed limit and (assumed convex) budgets have convex feasible sets, all
imposed through $\mathcal D$.

GlobMP starts from the cheapest point of a grid. Take a box $\mathcal B$ with a
global minimizer of \eqref{eq:P} in its interior, and let
\[
\m z_h\in\operatorname*{arg\,min}\big\{f(\m z):\ \m z\in\feasset\cap h\mathbb Z^{nd_f}\cap\mathcal B\big\},
\]
in the rational chart of free control points described after the grid part
below. The box may be the whole chart: $f$ has compact sublevel sets with finitely many grid points, so a
cheapest feasible grid point exists whenever one grid point is feasible. The
start of Theorem~\ref{thm:globalinf} and Proposition~\ref{prop:globalC} is this
point for the class $\mathcal S_N$, with $\mathcal B$ the whole chart and the
blocks of the problem being solved; there $d_f$ counts the free control points
per coordinate left by the boundary and junction equations.

Proposition~\ref{prop:globalC} rests on the following result at fixed degree,
which uses a constraint qualification on the control points.
Theorem~\ref{thm:globalinf} uses only its inner-set and descent steps, which do
not need Assumption~\ref{assum:rob}.

\begin{assumption}[constraint qualification]
\label{assum:rob}
For some $U>E^\star$, every $\m z\in\feasset$ with $f(\m z)\le U$ has a $\m v$ with
$\g\pi(\m z)+D\g\pi(\m z)\,\m v\in\mathrm{int}\,\mathcal K$.
\end{assumption}

Assumption~\ref{assum:rob} is Robinson's condition
\citep{robinson1976stability}: some change of the control points makes
every block strictly feasible to first order. It is needed only on this
sublevel set, and on fine splines it follows from
Assumption~\ref{assum:mfcq} (Appendix~\ref{app:globalC}, Lemma~\ref{lem:mfcq-rob}).

\begin{proposition}[GlobMP at fixed degree]
\label{prop:fixeddegree}
Assume Assumptions~\ref{assum:structural} and~\ref{assum:curv} (so that
$\m H\succ0$) and Assumption~\ref{assum:rob}, and use a step
$0<\eta<2/\beta$. Run the repair \eqref{eq:inner-step} from a feasible start
of cost at most $U$. Then every iterate is feasible, each step lowers the
cost by at least $(1/\eta-\beta/2)\|\m z^{\nu+1}-\m z^\nu\|^2$, and every limit
point of the iterates is a KKT point of \eqref{eq:P} (defined in the grid part
of the proof). If the start is the cheapest feasible grid point $\m z_h$ and
the grid spacing $h$ is small enough, then moreover the costs decrease to the
optimum, $f(\m z^\nu)\downarrow E^\star$, and every limit point is a global
minimizer of \eqref{eq:P}.
\end{proposition}

The proof has a descent part and a grid part. With uniform feasible random
starts, a run with globally minimizing limit points is obtained with probability one after finitely many starts (Corollary~\ref{cor:random}).

\paragraph{Proof of Proposition~\ref{prop:fixeddegree}, descent part.}
We first show $\mathcal C(\m z)\subset\feasset$. Let $\m z$ be feasible, $\m w\in\mathcal C(\m z)$, with
curves $\g\gamma_z,\g\gamma_w$. For a collision block $j\le p$, convexity of
$g_j(\cdot,t)$ gives
\begin{equation}\label{eq:p10-tangent}
g_j(\g\gamma_w(t),t)\;\ge\; g_j(\g\gamma_z(t),t)+\nabla_{\m x}g_j(\g\gamma_z(t),t)^\top\big(\g\gamma_w(t)-\g\gamma_z(t)\big),\qquad t\in[0,1],
\end{equation}
and the right side has coefficients $\g\pi_j(\m z)+D\g\pi_j(\m z)(\m w-\m z)\in\mathcal K_j$
(for a sphere the gap in \eqref{eq:p10-tangent} is exactly $\|\g\gamma_w-\g\gamma_z\|^2$). So
$g_j(\g\gamma_w(t),t)\ge0$ for all $t$; the remaining blocks hold through $\mathcal D$.
$\mathcal C(\m z)$ is closed, convex and contains $\m z$, so \eqref{eq:inner-step} has
exactly one solution.

$\m z^{\nu+1}$ is the projection of $\m y:=\m z^\nu-\eta\nabla f(\m z^\nu)$
onto $\mathcal C(\m z^\nu)$, characterized by
\begin{equation}\label{eq:p10-proj}
\langle\m y-\m z^{\nu+1},\m w-\m z^{\nu+1}\rangle\le0\qquad\forall\,\m w\in\mathcal C(\m z^\nu).
\end{equation}
With $\m w=\m z^\nu$ and $\g\Delta:=\m z^{\nu+1}-\m z^\nu$ this gives
\begin{equation}\label{eq:p10-dir}
\langle\nabla f(\m z^\nu),\g\Delta\rangle\le-\tfrac{1}{\eta}\|\g\Delta\|^2 ,
\end{equation}
and since $\nabla f$ is $\beta$-Lipschitz,
\begin{equation}\label{eq:p10-descent}
f(\m z^{\nu+1})\le f(\m z^\nu)+\langle\nabla f(\m z^\nu),\g\Delta\rangle+\tfrac\beta2\|\g\Delta\|^2
\le f(\m z^\nu)-c_\eta\|\g\Delta\|^2,\qquad c_\eta:=\tfrac1\eta-\tfrac\beta2>0 .
\end{equation}
Summing over $\nu<T$,
\begin{equation}\label{eq:p10-sum}
c_\eta\textstyle\sum_{\nu<T}\|\m z^{\nu+1}-\m z^\nu\|^2\le f(\m z^0)-E^\star ,
\end{equation}
so the steps tend to zero and the iterates stay in a compact sublevel set
(Lemma~\ref{lem:ellipsoid}).

Let $\m z^{\nu_i}\to\bar{\m z}$; then $\m z^{\nu_i+1}\to\bar{\m z}$, and
$\bar{\m z}\in\feasset$. Take $\m v$ from Assumption~\ref{assum:rob} at $\bar{\m z}$ and
$\m w_0:=\bar{\m z}+\tau\m v$, $\tau\in(0,1]$. Its linearized blocks,
\begin{equation}\label{eq:p10-w0}
\g\pi(\bar{\m z})+D\g\pi(\bar{\m z})(\m w_0-\bar{\m z})
=(1-\tau)\,\g\pi(\bar{\m z})+\tau\big(\g\pi(\bar{\m z})+D\g\pi(\bar{\m z})\m v\big)\in\mathrm{int}\,\mathcal K,
\end{equation}
and, for small $\tau$, the blocks in $\mathcal D$ hold strictly at $\m w_0$. For
$\m w\in\mathcal C(\bar{\m z})$ and $\theta\in(0,1]$, $\m w_\theta:=(1-\theta)\m w+\theta\m w_0$
is strictly feasible for $\mathcal C(\bar{\m z})$, hence lies in $\mathcal C(\m z^{\nu_i})$
for large $i$. Using it in \eqref{eq:p10-proj} at step $\nu_i$ and letting
$i\to\infty$, then $\theta\to0$,
\begin{equation}\label{eq:p10-vi}
\langle\nabla f(\bar{\m z}),\m w-\bar{\m z}\rangle\ge0\qquad\forall\,\m w\in\mathcal C(\bar{\m z}),
\end{equation}
so $\bar{\m z}$ minimizes $f$ over $\mathcal C(\bar{\m z})$. This convex problem is strictly
feasible ($\m w_0$), so $\bar{\m z}$ has Lagrange multipliers; with
$\mathcal N_{\mathcal D}(\bar{\m z})=D\g\pi_{\mathcal D}(\bar{\m z})^{\!\top}\mathcal N_{\mathcal K_{\mathcal D}}(\g\pi_{\mathcal D}(\bar{\m z}))$
\citep[Thm.~6.14]{rockafellar1998variational} (by Assumption~\ref{assum:rob}), where $\g\pi_{\mathcal D}$ stacks the
blocks in $\mathcal D$ and $\mathcal K_{\mathcal D}$ is the product of their cones, and $\mathcal N_{\mathcal K}(\g\zeta)$, $\mathcal N_{\mathcal K_{\mathcal D}}(\g\zeta)$ denote the normal cones at $\g\zeta$, these are
the KKT conditions \eqref{eq:p10-kkt} of \eqref{eq:P} at $\bar{\m z}$.

\paragraph{Proof of Proposition~\ref{prop:fixeddegree}, grid part.}
(i) Let $\m z^\star$ be a global minimizer in the
interior of $\mathcal B$ and $\m v$ the vector of Assumption~\ref{assum:rob} at
$\m z^\star$. For $\tau\in(0,1)$,
\begin{equation}\label{eq:p10-taylor}
\g\pi(\m z^\star+\tau\m v)=(1-\tau)\,\g\pi(\m z^\star)+\tau\big(\g\pi(\m z^\star)+D\g\pi(\m z^\star)\m v\big)+o(\tau).
\end{equation}
The first term is in $\mathcal K$; the second is $\tau$ times a point of $\mathrm{int}\,\mathcal K$, so the ball of radius $\tau\rho_0$ around it lies in
$\mathcal K$ for some $\rho_0>0$; so for small $\tau$ the sum is in $\mathrm{int}\,\mathcal K$. By continuity, a
ball $\overline{\mathbb B}(\m z^\star+\tau\m v,\delta)\subset\feasset\cap\mathcal B$ has
$f<E^\star+\varepsilon$ for $\tau,\delta$ small, and contains a grid point for small $h$ (every point is within
$h\sqrt{nd_f}/2$ of the grid in the rational chart, an affine bijection). Hence
\begin{equation}\label{eq:p10-grid}
E^\star\le f(\m z_h)<E^\star+\varepsilon ,
\end{equation}
and every limit point of $\m z_h$ is feasible with cost $E^\star$.

(ii) The extended cost $\varphi:=f$ on
$\feasset$, $+\infty$ elsewhere, is lower semicontinuous, semialgebraic and continuous on its closed domain
($\g\pi$ is polynomial; the quantifier over $t$ is eliminated by
Tarski--Seidenberg). So its Clarke critical points take finitely many values
\citep[Cor.~9(ii)]{bolte2007clarke}. A \emph{KKT point} of \eqref{eq:P} is a feasible
$\m z$ with
\begin{equation}\label{eq:p10-kkt}
-\nabla f(\m z)\in D\g\pi(\m z)^\top\mathcal N_{\mathcal K}\big(\g\pi(\m z)\big),
\end{equation}
and it is Clarke critical, since the right side lies in the normal cone of
$\feasset$ \citep[Thm.~6.14]{rockafellar1998variational}. Let $\varepsilon_0>0$ be
below $U-E^\star$ and below the distance from $E^\star$ to the next critical
value, and, by \eqref{eq:p10-grid}, take $h_0$ with $f(\m z_h)<E^\star+\varepsilon_0$ for
$h\le h_0$. By the descent part, the costs decrease to a limit $\bar E$ attained
at KKT points, with
\[
E^\star\le\bar E\le f(\m z_h)<E^\star+\varepsilon_0 ;
\]
$\bar E$ is a critical value, so $\bar E=E^\star$, and every limit point is a global
minimizer. The iterates are bounded, so they approach the compact set of
global minimizers.

The chart of Lemma~\ref{lem:affchart} has irrational entries in general.
The control points that the boundary conditions (for a spline, also the junction conditions) leave free form a
rational affine chart of the curve class: these conditions fix the others in
terms of these and the data (for one segment, the data alone). For a single segment, a move of the free
coordinates moves $\Gm$ by the same amount in the Frobenius norm and the free
coordinates have norm at most $\|\Gm\|_F$, which is all
Proposition~\ref{prop:degree} uses of Lemma~\ref{lem:affchart}; for a spline
class such as $\mathcal S_N$ the chart is still an affine bijection, which is
all the grid part of Proposition~\ref{prop:fixeddegree} needs. In both cases
the membership test is exact in rational arithmetic for rational data and breakpoints, by the Sturm-sequence test stated before Proposition~\ref{prop:degree}.

\begin{corollary}[random starts]
\label{cor:random}
Assume the hypotheses of Proposition~\ref{prop:fixeddegree}, with $\mathcal B$
bounded. Draw starting points independently and uniformly at random from
$\feasset\cap\mathcal B$, run the exact repair \eqref{eq:inner-step} from each
(each run understood in its infinite-iteration limit), and keep the cheapest
limiting cost. There is a number $\vartheta>0$, depending on the instance, such
that after $n_s$ starts the kept cost is above $E^\star$ with probability at
most $(1-\vartheta)^{n_s}$. So, with probability one, some run reaches globally
minimizing limit points after finitely many starts.
\end{corollary}

\paragraph{Proof of Corollary~\ref{cor:random}.}
Part (i) of the grid part gives a ball in $\feasset\cap\mathcal B$ on which
$f<E^\star+\varepsilon_0$; let $\vartheta\in(0,1]$ be the ratio of its volume to that
of $\feasset\cap\mathcal B$. A start in the ball is feasible with cost below
$E^\star+\varepsilon_0$, which is all part (ii) uses, so its run reaches only global
minimizers. By independence,
\[
\Pr\{\text{none of the first $n_s$ starts lies in the ball}\}=(1-\vartheta)^{n_s}\to0 .
\]
Uniform starts are drawn by rejection from $\mathcal B$, since membership in
$\feasset$ is decided exactly (rounding each draw to rationals finely enough keeps a
positive probability of landing in the ball). The grid of
Proposition~\ref{prop:fixeddegree}, in contrast, has $O((\mathrm{diam}\,\mathcal B/h)^{nd_f})$
points: that proposition establishes global convergence, not an efficient search.

\subsubsection{Theorem~\ref{thm:globalinf}: the continuous-time problem}
\label{app:thm1}

Let $E(\g\gamma)=\int_0^1\|\g\gamma^{(k)}\|^2dt$ denote the cost of
\eqref{eq:Pinf} on $\trajspace^k$.
For a margin $\tilde{r}$, the problem
$(\mathrm P^{\tilde{r}}_\infty)$ keeps the boundary data of \eqref{eq:Pinf} and makes
the replacements:
\[
r_j\ \to\ r_j+\tilde{r},\qquad r_a+r_b\ \to\ r_a+r_b+2\tilde{r},\qquad \bar r\ \to\ \bar r-\tilde{r}.
\]
We write $E^\star_\infty(\tilde{r})$ for its optimal value, $+\infty$ when it is
infeasible. For the position blocks of Appendix~\ref{app:variants}, an
ellipsoid's radius is raised by $\tilde{r}$ in its own norm (a robot pair keeps
the pair rule even when written as an ellipsoid block), and a wall or box face,
written with a unit normal, is moved by $\tilde{r}$ towards its feasible side.

The \emph{clearance} $\sigma$ of the endpoints is the
largest $\tilde{r}$ for which the start and goal positions are feasible for
$(\mathrm P^{\tilde{r}}_\infty)$; it is positive by Assumption~\ref{assum:data}. We take
$\tilde{r}_{\max}$ below $\sigma/2$ and below every obstacle and robot radius. Then,
for $|\tilde{r}|\le\tilde{r}_{\max}$, the radii $r_j+\tilde{r}$, the separations and
$\bar r-\tilde{r}$ stay positive, so every block keeps the form of
\S\ref{sec:setup}: the collision and pair blocks are convex in the position and
the steering set is a ball. The endpoints keep a clearance of at least
$\sigma/2$ in $(\mathrm P^{\tilde{r}}_\infty)$.

The feasible sets shrink as $\tilde{r}$ grows, so
$E^\star_\infty(\cdot)$ is nondecreasing, and $(\mathrm P^{\tilde{r}}_\infty)$ is feasible for
$\tilde{r}\le0$. Whenever $(\mathrm P^{\tilde{r}}_\infty)$ is feasible it has a minimizer,
by the existence step of the proof of Proposition~\ref{prop:degree}, which uses
only a feasible curve and Assumption~\ref{assum:curv}. The chart, the
coercivity step and the end correction borrowed from
Appendix~\ref{app:globalC} do not use Assumption~\ref{assum:mfcq} either.

\begin{figure}[t]
\centering
\setlength{\abovecaptionskip}{3pt}
\includegraphics[width=\linewidth]{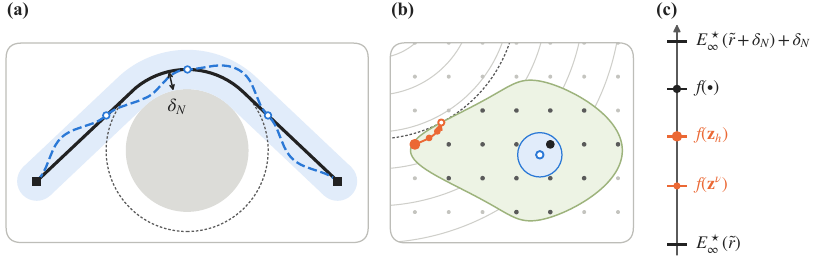}
\caption{Proof of Theorem~\ref{thm:globalinf}(ii) (drawn for $\ell<k$; otherwise Lemma~\ref{lem:margin}(c) adds an end correction). (a) The cheapest curve (black) with the margin raised by $\delta_N$ (dotted circle) and its $N$-piece interpolant (blue), which stays within $\delta_N$ of it (shaded) and so clears the obstacle grown by $\tilde r$ (grey). (b) In the free control points (green: feasible set; grey: cost levels and infeasible grid points), a feasible ball (blue) around the interpolant (blue open dot) contains a grid point (black), so the cheapest feasible grid point (orange) costs no more, and GlobMP descends from it (small dots: iterates; orange open dot: their limit, on the dashed cost level). (c) The resulting cost bounds, lowest at the bottom (dots as in (b)).}
\label{fig:proof}
\end{figure}

\begin{lemma}[interpolation keeps a margin]
\label{lem:margin}
Let $-\tilde{r}_{\max}\le \tilde{r}<\tilde{r}'\le\tilde{r}_{\max}$ and $U\ge0$, and let
$\g\gamma$ be a curve that is feasible for the stricter problem
$(\mathrm P^{\tilde{r}'}_\infty)$ and costs $E(\g\gamma)\le U$. Let $\m q$ be its
piecewise Hermite interpolant of degree $2k-1$ at the nodes $i/N$,
$i=0,\dots,N$, the piecewise polynomial that matches $\g\gamma$ and its first
$k-1$ derivatives at every node. Then:
\begin{enumerate}[label=(\alph*),nosep]
\item $\m q$ costs no more than $\g\gamma$, $E(\m q)\le E(\g\gamma)$, and stays
uniformly close to it, $\|\m q-\g\gamma\|_\infty\le N^{-(k-1/2)}\sqrt{E(\g\gamma)}$;
\item if the two margins are far enough apart,
$(1+2\sqrt U)N^{-(k-1/2)}\le\tilde{r}'-\tilde{r}$ (for an ellipsoid block, with the left side multiplied by $\lambda_{\max}(\g\Sigma_j)^{1/2}$), then $\m q$ satisfies every constraint of
the looser problem $(\mathrm P^{\tilde{r}}_\infty)$ strictly at every instant, and if
moreover $\ell<k$, then $\m q\in\mathcal S_N$;
\item if $\ell\ge k$ and $\g\gamma$ minimizes $(\mathrm P^{\tilde{r}'}_\infty)$, there are
$N_d$ and $C_d$, depending only on $U$ and the data, such that for $N\ge N_d$ a
correction $\m p$ supported on the first and last two pieces makes $\m q+\m p$ a
curve of $\mathcal S_N$ with all the boundary data and
$E(\m q+\m p)\le E(\g\gamma)+C_d/N$, along which, under the condition of (b),
every constraint of $(\mathrm P^{\tilde{r}}_\infty)$ holds strictly at every instant.
\end{enumerate}
\end{lemma}

\begin{proof}
(a) On a piece $[a,a+1/N]$, the error $\m e:=\g\gamma-\m q$ and its first $k-1$
derivatives vanish at both ends. For any polynomial $\m w$ of degree below $k$,
$k$ integrations by parts give $\int_a^{a+1/N}\m e^{(k)}\cdot\m w\,dt=\pm\int\m e\cdot\m w^{(k)}dt=0$.
So $\g\epsilon:=\m e^{(k)}$ is orthogonal to $\m q^{(k)}$ on every piece, and
\begin{equation}\label{eq:l12-pyth}
E(\m q)=E(\g\gamma)-\|\g\epsilon\|_{L^2}^2\le E(\g\gamma),\qquad \|\g\epsilon\|_{L^2}\le\sqrt{E(\g\gamma)} .
\end{equation}
Taylor's formula at the left node and the Cauchy--Schwarz inequality give
\begin{equation}\label{eq:l12-sup}
\m e(t)=\int_a^t\frac{(t-s)^{k-1}}{(k-1)!}\,\g\epsilon(s)\,ds,\qquad
\|\m e(t)\|\le\frac{(t-a)^{k-1/2}}{(k-1)!\sqrt{2k-1}}\,\|\g\epsilon\|_{L^2}
\le N^{-(k-1/2)}\sqrt{E(\g\gamma)} .
\end{equation}

(b) By (a) and the condition, $\|\m q-\g\gamma\|_\infty<\tilde r'-\tilde r$. The two problems
differ by exactly this gap: each obstacle distance $\ge r_j+\tilde r'$ along $\g\gamma$
stays $>r_j+\tilde r$ along $\m q$; a separation (each robot moves less than the
gap) $\ge r_a+r_b+2\tilde r'$ stays $>r_a+r_b+2\tilde r$; a reference distance
$\le\bar r-\tilde r'$ stays $<\bar r-\tilde r$; a wall offset $\ge\tilde r'$ stays $>\tilde r$.
An ellipsoid norm enlarges the move by at most $\lambda_{\max}(\g\Sigma_j)^{1/2}$. The
interpolant is $C^{k-1}$ of degree $2k-1\le d$ and matches $\g\gamma$ to order
$k-1$ at the endpoints, so for $\ell<k$ it lies in $\mathcal S_N$.

(c) Let $\ell\ge k$. Curves with the data of order below $k$ and cost $\le U$
are bounded in $\trajspace^k$ (coercivity step of
Proposition~\ref{prop:degree}), hence equicontinuous; with the endpoint
clearance $\sigma/2$, there is $\tau_0\in(0,1/4]$, depending only on $U$ and the
data, such that every such curve feasible for $(\mathrm P^{\tilde r'}_\infty)$ has
clearance $\ge\sigma/4$ (in margin units) on $[0,\tau_0]\cup[1-\tau_0,1]$. So
$\g\gamma+s\g\varphi$ stays feasible for small $s$ and smooth $\g\varphi$ supported in
$(0,\tau_0)$, and optimality gives $\int\g\gamma^{(k)}\cdot\g\varphi^{(k)}=0$: by
du Bois-Reymond, $\g\gamma$ is a polynomial of degree $\le2k-1$ on $(0,\tau_0)$,
with derivatives at $0$ bounded in terms of $U$, $\tau_0$ and the data
($\|\g\gamma^{(k)}\|_{L^2(0,\tau_0)}\le\sqrt U$; norms on polynomials of bounded
degree are equivalent). The same holds at $t=1$.

For $N\ge2/\tau_0$, Hermite interpolation reproduces $\g\gamma$ on the first and
last two pieces, and only the data of orders $k,\dots,\ell$ are missing:
$\m a_s:=\g\gamma_0^{(s)}-\g\gamma^{(s)}(0)$, with $\sum_s\|\m a_s\|\le\bar a$ independent of
$\tilde r,\tilde r',N$. The correction $\m p(t)=N^{-k}\m P(Nt)$ on $[0,2/N]$ of the
proof of Lemma~\ref{lem:rate} (which does not use Assumption~\ref{assum:mfcq}),
plus its mirror image at $t=1$, has
\[
\m p^{(s)}(0)=\m p^{(s)}(2/N)=\m 0\ \ (s<k),\qquad \m p^{(s)}(0)=\m a_s\ \ (k\le s\le\ell),
\]
so $\m q+\m p\in\mathcal S_N$ has all the data. $\m P$ and its derivatives up to order
$k$ are at most $K\bar a$ ($K$ depending on $k,\ell$), so $\m p$ moves the curve by
$\le N^{-k}K\bar a<\sigma/4$ for large $N$, only where the clearance is $\ge\sigma/4$.
The cost increase at each end, with $u=Nt$ and the Cauchy--Schwarz inequality,
\begin{equation}\label{eq:l12-corr}
\int_0^{2/N}|\m p^{(k)}|^2dt=\frac1N\int_0^2\|\m P^{(k)}\|^2du,\qquad
2\Big|\int_0^{2/N}\m q^{(k)}\cdot\m p^{(k)}dt\Big|\le
\frac{2\sqrt2}{N}\,\|\g\gamma^{(k)}\|_{L^\infty(0,\tau_0)}\|\m P^{(k)}\|_{L^2(0,2)},
\end{equation}
is $O(1/N)$, which gives $C_d$; $N_d$ is the least $N\ge\max\{4,2/\tau_0\}$ with
$N^{-k}K\bar a<\sigma/4$.
\end{proof}

\paragraph{Proof of Theorem~\ref{thm:globalinf}.}
(i) For $|\tilde{r}|\le\tilde{r}_{\max}$ the blocks of $(\mathrm P^{\tilde{r}}_\infty)$ keep the form
of \S\ref{sec:setup}, and $\m H\succ0$ and the affine chart of $\mathcal S_N$ hold
without Assumption~\ref{assum:mfcq} (proof of Proposition~\ref{prop:globalC}). So
the descent part of the proof of Proposition~\ref{prop:fixeddegree} applies:
every iterate is feasible and \eqref{eq:p10-descent} holds. $\m z_h$ exists once
one grid point is feasible, as $f$ has compact sublevel sets.

(ii) (See Figure~\ref{fig:proof}.) Let $r=\min\{k-\tfrac12,1\}$, take $N_d,C_d$ from
Lemma~\ref{lem:margin} ($C_d=0$ if $\ell<k$), and
\begin{equation}\label{eq:t1-C}
C=1+2\sqrt U\ \ (\ell<k),\qquad
C=\max\big\{1+2\sqrt U,\ 2C_d,\ 2\tilde{r}_{\max}\,N_d^{\,r}\big\}\ \ (\ell\ge k).
\end{equation}
If $\ell\ge k$, then $\delta_N=CN^{-r}\le\mu\le2\tilde r_{\max}$ gives $N\ge N_d$, and
$C_d/N\le\tfrac12CN^{-r}<\delta_N$. Let $\tilde r':=\tilde r+\delta_N\le\tilde r+\mu$. The given
curve is feasible for $(\mathrm P^{\tilde r'}_\infty)$, so it has a minimizer $\g\gamma$
with $E(\g\gamma)=E^\star_\infty(\tilde r')\le U$. As
$(1+2\sqrt U)N^{-(k-1/2)}\le\tilde r'-\tilde r$, Lemma~\ref{lem:margin} gives
$\tilde{\m q}\in\mathcal S_N$ with all the boundary data, strictly feasible for
$(\mathrm P^{\tilde r}_\infty)$, and $E(\tilde{\m q})<E^\star_\infty(\tilde r')+\delta_N$. By
continuity (the chart is an affine bijection), a closed ball around the free
coordinates of $\tilde{\m q}$ is feasible for $(\mathrm P^{\tilde r}_\infty)$ with cost below
$E^\star_\infty(\tilde r')+\delta_N$; for small $h$ it contains a grid point, so
\begin{equation}\label{eq:t1-upper}
E^\star_\infty(\tilde r)\le f(\m z^\nu)\le f(\m z_h)<E^\star_\infty(\tilde{r}+\delta_N)+\delta_N ,
\end{equation}
the lower bound because every iterate is feasible for $(\mathrm P^{\tilde r}_\infty)$
with the boundary data. The threshold on $h$ depends on $N$ and the unknown
minimizer, so Corollary~\ref{cor:eps} bounds $N$ but not $h$.

(iii) With $\tilde r=-\mu$, $(\mathrm P^{\tilde r+\mu}_\infty)$ is \eqref{eq:Pinf}, and
\eqref{eq:t1-upper} with monotonicity of $E^\star_\infty(\cdot)$ gives
\begin{equation}\label{eq:t1-iii}
f(\m z^\nu)<E^\star_\infty(-\mu+\delta_N)+\delta_N\le E^\star_\infty+\delta_N .
\end{equation}
A curve feasible for $(\mathrm P^{-\mu}_\infty)$ violates no constraint of
\eqref{eq:Pinf} by more than $\mu$ ($2\mu$ for a separation). Let $\mu_i\to0$ with
$N_i,h_i$ as in (ii), and let $\g\gamma_i$ be a curve from $\m z_{h_i}$ or a
minimizer of $(\mathrm P^{-\mu_i}_\infty)$. These cost at most $E^\star_\infty+\mu_i$, so they are bounded in $\trajspace^k$ (coercivity, Proposition~\ref{prop:degree}) and a
subsequence converges weakly in $\trajspace^k$, and uniformly with the
derivatives below $k$, to a curve $\g\gamma$ with the boundary data that satisfies
every constraint of \eqref{eq:Pinf} (limits of the loosened inequalities). By
weak lower semicontinuity,
\[
E(\g\gamma)\le\liminf_i E(\g\gamma_i)\le E^\star_\infty ,
\]
so $\g\gamma$ is a global minimizer. Applied to the minimizers, this gives
$E^\star_\infty(-\mu_i)\to E^\star_\infty$, and the costs from $\m z_{h_i}$, which lie in
$[E^\star_\infty(-\mu_i),E^\star_\infty+\delta_{N_i}]$, tend to $E^\star_\infty$. For the
position blocks of Appendix~\ref{app:variants} the proof is the same, with
$K\bar a$ (hence $N_d$) and $C$ multiplied by
$\lambda:=\max\{1,\max_j\lambda_{\max}(\g\Sigma_j)^{1/2}\}$ (Lemma~\ref{lem:margin}).

\paragraph{Proof of Corollary~\ref{cor:eps}.}
Item (i) is the case $\ell<k$ of the proof of Theorem~\ref{thm:globalinf}(ii),
which needs no end correction: the margin condition and the cost bound hold
with $\delta_N=(1+2\sqrt U)N^{-(k-1/2)}$. For item (ii), the choice of $N$
gives $\delta_N=CN^{-r}\le\min\{\varepsilon,\mu\}$, so
Theorem~\ref{thm:globalinf}(iii) applies: the curves violate no constraint of
\eqref{eq:Pinf} by more than $\mu$ ($2\mu$ for a separation) and cost at most
$E^\star_\infty+\delta_N\le E^\star_\infty+\varepsilon$.

\subsubsection{Global minimizers over \texorpdfstring{$\mathcal S_N$}{S\_N} under a condition on configurations}
\label{app:globalC}

Theorem~\ref{thm:globalinf} does not say that GlobMP reaches a global minimizer
over a fixed class $\mathcal S_N$. It does for \eqref{eq:Pinf} and large $N$ under the
following condition.

\setcounter{assumption}{2}%
\begin{assumption}[free configurations]
\label{assum:mfcq}
With $g_j$ as in \S\ref{sec:setup}, at every time and feasible position, some
$\m v$ increases every active constraint:
$\nabla_{\m x}g_j(\m x,t)^\top\m v>0$ whenever $g_j(\m x,t)=0$.
\end{assumption}
\setcounter{assumption}{5}%

Assumption~\ref{assum:mfcq} is the Mangasarian--Fromovitz condition on
configurations; it rules out jamming and is used only at feasible positions
within distance $1$ of the points visited by curves of cost at most
$E^\star_\infty+1$ (Lemma~\ref{lem:push} and the proof of Proposition~\ref{prop:globalC}). Under it, GlobMP from $\m z_h$
reaches global minimizers over every fine enough class $\mathcal S_N$; the proof
chains four results and is at the end of this subsection.

\begin{proposition}[global minimizers over $\mathcal S_N$]
\label{prop:globalC}
Under Assumptions~\ref{assum:data}, \ref{assum:curv}
and~\ref{assum:mfcq}:
\begin{enumerate}[leftmargin=1.6em,label=\textup{\textbf{(\roman*)}},itemsep=1pt,topsep=2pt]
\item There is $N_0$ such that, for every $N\ge N_0$, GlobMP over
$\mathcal S_N$ on \eqref{eq:Pinf} with the repair \eqref{eq:inner-step}, a step
$0<\eta<2/\beta$ ($\beta$ computed for $\mathcal S_N$) and the start $\m z_h$ with $h$ small enough produces curves
feasible at every instant, with nonincreasing cost, and all its
limit points are global minimizers over $\mathcal S_N$.
\item For $N\ge N_0$, the optimal cost $E^\star_N$ over $\mathcal S_N$ satisfies
$E^\star_\infty\le E^\star_N\le E^\star_\infty+C\,N^{-\min\{k-1/2,\,1\}}$ with $C$
independent of $N$, and as $N\to\infty$ subsequences of these minimizers
converge uniformly to global minimizers of \eqref{eq:Pinf}.
\end{enumerate}
\end{proposition}

\begin{lemma}[pushing into the interior]
\label{lem:push}
Fix a time and let the $g_j$ be $C^1$ functions of the position, so that
$\mathcal X:=\{\m x\in\R^n:g_j(\m x)\ge0\ \forall j\}$ is the set of feasible
positions. Assume Assumption~\ref{assum:mfcq} at this time: at every
$\m x\in\mathcal X$ some direction $\m v$ increases all active constraints,
$\nabla g_j(\m x)\cdot\m v>0$ whenever $g_j(\m x)=0$. Assume also that the
start and goal positions are strictly feasible. Then for every compact
$\mathcal Q\subset\R^n$ there are numbers $\mu_0,G>0$ and, for each
$0<\mu\le\mu_0$, a smooth map $\g\Phi_\mu:\R^n\to\R^n$, a \emph{push}, such that:
\begin{enumerate}[label=(\alph*),nosep]
\item $\g\Phi_\mu$ is the identity near the start and goal positions;
\item it raises every block value of a feasible point to at least $\mu$:
$g_j(\g\Phi_\mu(\m x))\ge\mu$ for all $j$ and all $\m x\in\mathcal X\cap\mathcal Q$;
\item it moves points only by order $\mu$:
$\|\g\Phi_\mu-\mathrm{id}\|_{C^i(\mathcal Q)}\le a_i\mu$ for every $i$, with $a_i$
independent of $\mu$;
\item a point $\m y$ within distance $1$ of $\mathcal Q$ with $g_j(\m y)\ge\mu$ for
all $j$ has the whole open ball $\mathbb B(\m y,\mu/G)$ inside $\mathcal X$.
\end{enumerate}
\end{lemma}
\begin{proof}
Let $\mathcal Q'$ be the closed $1$-neighbourhood of $\mathcal Q$. For
$\m x\in\mathcal X\cap\mathcal Q'$ with active set $\mathcal J(\m x)$, scale the vector of
Assumption~\ref{assum:mfcq} to get $\m v_{\m x}$ and an open neighbourhood
$\mathcal O_{\m x}$ on which $\nabla g_j\cdot\m v_{\m x}\ge2$ ($j\in\mathcal J(\m x)$) and
$g_j\ge\kappa_{\m x}>0$ ($j\notin\mathcal J(\m x)$). By compactness, finitely many
$\mathcal O_{\m x_1},\dots,\mathcal O_{\m x_L}$ cover $\mathcal X\cap\mathcal Q'$. Let $\kappa>0$
be below every $\kappa_{\m x_l}$ and below $\frac12\min_jg_j$ near the start and
goal, take a partition of unity $(\varpi_l)$ subordinate to the cover and a
smooth cutoff $\chi$ equal to $0$ near the start and goal and to $1$ where some
$g_j<\kappa$, and set
\begin{equation}\label{eq:l14-field}
\g\phi:=\chi\textstyle\sum_{l=1}^{L}\varpi_l\,\m v_{\m x_l}.
\end{equation}
If $g_j(\m y)<\kappa$ and $\varpi_l(\m y)>0$, then $j\in\mathcal J(\m x_l)$, so
\begin{equation}\label{eq:l14-rate}
\nabla g_j(\m y)\cdot\g\phi(\m y)\ge2\qquad(\m y\in\mathcal X\cap\mathcal Q',\ g_j(\m y)<\kappa).
\end{equation}
Let $\g\Phi_\mu$ be the time-$\mu$ flow of $\g\phi$, $M_\phi>0$ a bound on
$\|\g\phi\|$, $G$ a bound on $\|\nabla g_j\|$ on the $2$-neighbourhood of $\mathcal Q$,
and
\begin{equation}\label{eq:l14-mu0}
\mu\le\mu_0:=\min\Big\{\frac\kappa4,\ \frac{\kappa}{2GM_\phi},\ \frac1{M_\phi},\ G\Big\}.
\end{equation}
A trajectory from $\m x\in\mathcal X\cap\mathcal Q$ moves at most $M_\phi\mu\le1$; along it
$\tfrac{d}{ds}g_j\ge2$ while $g_j<\kappa$, and $g_j$ loses at most
$GM_\phi\mu\le\kappa/2$ otherwise. So no $g_j$ can first reach $0$, and it stays in $\mathcal X\cap\mathcal Q'$ and
$g_j(\g\Phi_\mu(\m x))\ge\min\{2\mu,\kappa/2\}\ge\mu$, which is (b); (a) holds as
$\g\phi=0$ near the endpoints. For (c),
\begin{equation}\label{eq:l14-int}
\g\Phi_\mu-\mathrm{id}=\int_0^\mu\g\phi\circ\g\Phi_s\,ds ,
\end{equation}
and the derivatives of $\g\Phi_s$ up to order $i$ are bounded on $\mathcal Q$
uniformly in $s\le\mu_0$. For (d), $\|\nabla g_j\|\le G$ on $\mathbb B(\m y,1)$, so
$g_j(\m y)\ge\mu$ gives $g_j>0$ on $\mathbb B(\m y,\mu/G)$. Only points of
$\mathcal X\cap\mathcal Q'$ were used.
\end{proof}

\begin{lemma}[rate in the number of pieces]
\label{lem:rate}
Write $E^\star_N$ for the optimum over $\mathcal S_N$, whose curves carry the
boundary data. Under the hypotheses of Proposition~\ref{prop:globalC}, the
best spline with $N$ pieces is almost as cheap as the best curve: there are
$C$ and $N_1$, depending on the instance, such that for all $N\ge N_1$
\[
E^\star_\infty\le E^\star_N\le E^\star_\infty+C\,N^{-(k-1/2)}\ \text{ if }\ell<k,
\qquad
E^\star_\infty\le E^\star_N\le E^\star_\infty+C\,N^{-\min\{k-1/2,\,1\}}\ \text{ if }\ell\ge k.
\]
\end{lemma}

\paragraph{Proof of Lemma~\ref{lem:rate}.}
$E^\star_N\ge E^\star_\infty$ because $\mathcal S_N\subset\trajspace^k$. For the upper
bound, push a minimizer $\g\gamma^\star$ of \eqref{eq:Pinf} (it exists by the
existence step of Proposition~\ref{prop:degree}) inward by $\mu$ and
interpolate. For moving centres and reference, build the field of
Lemma~\ref{lem:push} on the compact set
$\{(\m x,t)\in\mathcal Q'\times[0,1]:g_j(\m x,t)\ge0\ \forall j\}$ ($\mathcal Q'$ the
$1$-neighbourhood of a compact set containing $\g\gamma^\star([0,1])$), with a
partition of unity in $(\m x,t)$ and $\chi=0$ near $(\g\gamma_0,0)$ and
$(\g\gamma_1,1)$. The resulting maps $\g\Phi_\mu(\cdot,t)$ are the identity near the
endpoints, give $g_j\ge\mu$, and $\g\Phi_\mu-\mathrm{id}$ has derivatives in
$(\m x,t)$ up to order $k$ bounded by a constant times $\mu$ (as in
Lemma~\ref{lem:push}(c)). Let $\g\gamma_\mu(t):=\g\Phi_\mu(\g\gamma^\star(t),t)$. By the chain
rule, each term of $\g\gamma_\mu^{(k)}-\g\gamma^{\star(k)}$ has a factor $O(\mu)$ and at
most one factor $\g\gamma^{\star(k)}\in L^2$, so
\begin{equation}\label{eq:l15-push}
\|\g\gamma_\mu^{(k)}-\g\gamma^{\star(k)}\|_{L^2}\le c_1\mu,\qquad E(\g\gamma_\mu)\le E^\star_\infty+c_2\mu ,
\end{equation}
with $c_1,c_2$ independent of $\mu\le\mu_0$, and $\mathbb B(\g\gamma_\mu(t),\mu/G)$ is
feasible (Lemma~\ref{lem:push}(d)). Let $\m q$ be the Hermite interpolant of
$\g\gamma_\mu$. By \eqref{eq:l12-pyth} and \eqref{eq:l12-sup},
\begin{equation}\label{eq:l15-interp}
E(\m q)\le E(\g\gamma_\mu),\qquad
\|\g\gamma_\mu-\m q\|_\infty\le c_3N^{-(k-1/2)},\qquad c_3:=\sqrt{E^\star_\infty+c_2\mu_0}+1 .
\end{equation}
With
\begin{equation}\label{eq:l15-mu}
\mu=2Gc_3N^{-(k-1/2)}\ (\le\mu_0\text{ for large }N),
\end{equation}
$\m q$ stays within $\mu/(2G)$ of $\g\gamma_\mu$, hence is feasible, and costs at most
$E^\star_\infty+c_2\mu=E^\star_\infty+2c_2Gc_3N^{-(k-1/2)}$. For $\ell<k$, $\m q\in\mathcal S_N$,
which proves the claim with $C=2c_2Gc_3$.

If $\ell\ge k$, no constraint is active near the endpoints, so there
$\g\gamma_\mu=\g\gamma^\star$ and $\g\gamma^{\star(2k)}=0$ (first variation, as in Lemma~\ref{lem:margin}(c)) and, for large $N$, $\m q=\g\gamma^\star$ on the first and last
two pieces; the missing data $\m a_s=\g\gamma_0^{(s)}-\g\gamma^{\star(s)}(0)$,
$k\le s\le\ell$, do not depend on $N$. Let $\m P$ be the spline on $[0,2]$ with
two unit pieces of degree $d$, $C^{k-1}$, zero derivatives of order $<k$ at
$0$ and $2$, and $\m P^{(s)}(0)=N^{k-s}\m a_s$ for $k\le s\le\ell$ (Taylor
polynomial on $[0,1]$, $d\ge\ell$; Hermite join to $0$ on $[1,2]$,
$d\ge2k-1$); as $N^{k-s}\le1$, $\m P$ is bounded uniformly in $N$. Then
$\m p(t)=N^{-k}\m P(Nt)$ on $[0,2/N]$, with its mirror image at $t=1$, supplies
$\m p^{(s)}(0)=\m a_s$, moves the curve by $O(N^{-k})$ inside the endpoint
clearance, and, as in \eqref{eq:l12-corr}, raises the cost at each end by at most
\[
\frac1N\int_0^2\|\m P^{(k)}\|^2du\;+\;2\|\g\gamma^{\star(k)}\|_{L^\infty(0,2/N)}\Big(\frac2N\Big)^{1/2}\|\m p^{(k)}\|_{L^2}
\;=\;O(N^{-1}).
\]
Hence $E^\star_N\le E^\star_\infty+C N^{-\min\{k-1/2,1\}}$.

\begin{lemma}[Assumption~\ref{assum:mfcq} gives Assumption~\ref{assum:rob}]
\label{lem:mfcq-rob}
Under the hypotheses of Lemma~\ref{lem:rate}, fix any $U$. For all large
enough $N$, every feasible curve of $\mathcal S_N$ with cost at most $U$
satisfies Robinson's condition: some change $\m v$ of its control points makes
every block strictly feasible to first order, as in
Assumption~\ref{assum:rob}.
\end{lemma}

\paragraph{Proof of Lemma~\ref{lem:mfcq-rob}.}
Let $\g\phi(\m x,t)$ be the space-time field of the proof of Lemma~\ref{lem:rate},
on a compact set containing every curve of cost $\le U$ (bounded by coercivity,
Proposition~\ref{prop:degree}); it is $C^k$, vanishes near the endpoints, and
$\nabla_{\m x}g_j\cdot\g\phi\ge2$ where $g_j<\kappa$. Let $\g\gamma$ be feasible in
$\mathcal S_N$ with cost $\le U$, and $\m v$ the Hermite interpolant of
$t\mapsto\g\phi(\g\gamma(t),t)$. By equicontinuity, for large $N$ the end pieces of
$\g\gamma$ lie where $\g\phi=0$, so $\m v$ has zero boundary data and is an
admissible change of control points. As $\|(\g\phi\circ\g\gamma)^{(k)}\|_{L^2}\le c_4(U)$,
\eqref{eq:l12-sup} gives
\begin{equation}\label{eq:l16-interp}
\|\m v-\g\phi\circ\g\gamma\|_\infty\le c_4N^{-(k-1/2)}<1/G\ \text{ for large }N,
\end{equation}
so $\nabla_{\m x}g_j\cdot\m v\ge1$ where $g_j<\kappa$. Hence, for small $s>0$, the
piecewise polynomial $g_j+s\nabla_{\m x}g_j\cdot\m v$, with coefficients
$\g\pi_j+D\g\pi_j(s\m v)$, is $\ge g_j+s>0$ where $g_j<\kappa$ and
$\ge\kappa-sG\sup\|\m v\|>0$ elsewhere: its coefficients lie in
$\mathrm{int}\,\mathcal K_j$, which is Assumption~\ref{assum:rob} with $s\m v$.

\paragraph{The degree-refinement procedure.}
At level $r$, take degrees $d_r\uparrow\infty$ satisfying
Assumption~\ref{assum:structural}, and set
\[
h_r=\frac{(2d_r)^{-(k+1)}}{r},\qquad
\mathcal B_r=\big[-r\sqrt{d_r+1},\,r\sqrt{d_r+1}\big]^{nd_f(r)},\qquad
d_f(r)=d_r+1-2(\ell+1):
\]
the grid spacing and a box of free control points (in the rational chart
after the proof of Proposition~\ref{prop:fixeddegree}) centred at zero.

Degree elevation keeps the previous output without changing its curve or
cost. Algorithm~\ref{alg:pgd} may stop after finitely many iterations, since a
trial replaces its seed only when it is feasible and no more expensive.

Feasibility of each grid candidate and of each trial is tested exactly over
$[0,1]$, piece by piece, for rational piecewise-polynomial data and
breakpoints. For a trial, its free control points in the rational chart are
first rounded to rationals, and this rounded curve is the one kept. On each
piece, Sturm sequences isolate the roots in the open interval of the piece within disjoint closed rational subintervals of it, with no root at their endpoints. The sign of the block is then checked at one rational point between
consecutive isolating intervals, and between each end of the piece and its
nearest interval (or at one interior rational point if there is no root). A
piece on which the block vanishes identically passes.

\begin{algorithm}[t]
\caption{Degree and seed refinement for \eqref{eq:Pinf}}
\label{alg:refinement}
\begin{algorithmic}[1]
\Require degrees $d_r$ with charts $(\Gm_{\mathrm{aff}}^{(r)},\m N_r)$, spacings $h_r$, boxes $\mathcal B_r$; levels $R$; finite iteration caps for Algorithm~\ref{alg:pgd}
\State $\Gm_{\mathrm{best}}\gets\mathrm{none}$
\For{$r=1,\ldots,R$}
  \State $\mathcal G_r\gets$ feasible degree-$d_r$ curves from $h_r\mathbb Z^{n d_f(r)}\cap\mathcal B_r$
  \State \textbf{if} $\Gm_{\mathrm{best}}\ne\mathrm{none}$ \textbf{then} add its degree-$d_r$ elevation to $\mathcal G_r$
  \State \textbf{if} $\mathcal G_r=\emptyset$ \textbf{then continue}
  \State $\Gm_{\mathrm{seed}}\gets\operatorname*{arg\,min}_{\Gm\in\mathcal G_r}\mathrm{cost}(\Gm)$
  \State $\Gm_{\mathrm{trial},\free}\gets$ Algorithm~\ref{alg:pgd} with chart $(\Gm_{\mathrm{aff}}^{(r)},\m N_r)$, initialized at $(\Gm_{\mathrm{seed}}-\Gm_{\mathrm{aff}}^{(r)})\m N_r$
  \State $\Gm_{\mathrm{trial}}\gets\Gm_{\mathrm{aff}}^{(r)}+\Gm_{\mathrm{trial},\free}\m N_r^\top$ \Comment{rounded to rationals in the rational chart}
  \State $\Gm_{\mathrm{best}}\gets$ cheaper feasible member of $\{\Gm_{\mathrm{seed}},\Gm_{\mathrm{trial}}\}$
\EndFor
\State \Return $\Gm_{\mathrm{best}}$ \Comment{or no feasible curve found within the $R$ levels}
\end{algorithmic}
\end{algorithm}

\begin{proposition}[degree refinement]
\label{prop:degree}
Assume Assumptions~\ref{assum:data}, \ref{assum:curv} and~\ref{assum:mfcq}.
Once Algorithm~\ref{alg:refinement} has found a first feasible curve, its
output costs never increase and converge to the optimum $E^\star_\infty$; the
optimum $E^\star(d)$ of \eqref{eq:P} at degree $d$ also decreases to
$E^\star_\infty$; and $E^\star_\infty$ is attained by some curve. Every limit
point of the output curves (weakly in $\trajspace^k$, and uniformly together
with the derivatives below $k$) is a global minimizer of \eqref{eq:Pinf}.
\end{proposition}

\paragraph{Proof of Proposition~\ref{prop:degree}.}
We use two properties, proved at the end: \emph{coercivity} (curves with the
boundary data and bounded cost are bounded in $\trajspace^k$) and the
\emph{approach from inside} (for every $\varepsilon>0$ some $C^\infty$ curve with the
boundary data costs at most $E^\star_\infty+\varepsilon$ and has every constraint
bounded below by a positive number). The constraints $g_j$ need only depend
continuously on $t$ and on derivatives of order below $k$.

Existence: a minimizing sequence is bounded, so a subsequence converges weakly
in $\trajspace^k$ and uniformly with the derivatives below $k$; the limit keeps
the data and constraints, and $E$ is weakly lower semicontinuous.

Approximation: fix $\varepsilon$ and $\g\gamma_\varepsilon$ from the approach from inside,
and $\delta>0$ such that every curve with the boundary data within $2\delta$ of
$\g\gamma_\varepsilon$ in $C^k$ is feasible with cost $\le E^\star_\infty+2\varepsilon$. The
Bernstein polynomials of $\g\gamma_\varepsilon$ converge in $C^{\max\{k,\ell\}}$
\citep{lorentz1986bernstein}; correcting their boundary data by a Hermite
polynomial of degree $2\ell+1$ gives $\g\gamma_d$ with the data and
$\|\g\gamma_d-\g\gamma_\varepsilon\|_{C^k}\le\delta$ for $d\ge d(\varepsilon)$. For a degree-$d$
curve with control points $\Gm$,
\begin{equation}\label{eq:p17-deriv}
\g\gamma^{(i)}=\frac{d!}{(d-i)!}\sum_{l}(\Delta^i\Gm)_l\,[\m b_{d-i}]_l ,\qquad\text{so}\qquad
\|\g\gamma^{(i)}\|_\infty\le(2d)^i\|\Gm\|_F ,
\end{equation}
since $\|(\Delta^i\Gm)_l\|\le2^i\|\Gm\|_F$ and the weights are nonnegative with
sum one. In the rational chart (an isometry for one segment), rounding the
free coordinates of $\g\gamma_{d_r}$ to $h_r\mathbb Z^{nd_f(r)}$ moves the curve in
$C^k$ by at most
\[
(2d_r)^kh_r\sqrt{n(d_r+1)}/2\le\sqrt n/(2r)<\delta
\]
for large $r$, and the grid point lies in $\mathcal B_r$, since the control points
of $\g\gamma_d$ are bounded independently of $d$. So for $r\ge r(\varepsilon)$ the seed,
hence the output, costs at most $E^\star_\infty+2\varepsilon$, and every output costs at
least $E^\star(d_r)\ge E^\star_\infty$. The previous output, elevated, is a candidate
at each level, so the costs never increase, and $E^\star(d)$ is nonincreasing and
between $E^\star_\infty$ and the output costs. Limit points of the outputs are
feasible with cost $\le E^\star_\infty$, as in the existence step.

Coercivity: if $\|\g\gamma_i\|_{\trajspace^k}\to\infty$ with bounded cost, the
normalized curves have $k$th derivatives and boundary data tending to zero; by
the compact embedding $\trajspace^k\subset\trajspace^{k-1}$, a subsequence converges in
$\trajspace^{k-1}$, hence (the $k$th derivatives tending to zero) in $\trajspace^k$, to a polynomial of degree below $k$ and norm one whose derivatives
of orders $0,\dots,\min\{\ell,k-1\}$ vanish at both ends. Under
Assumption~\ref{assum:curv} it is zero (proof of Lemma~\ref{lem:hpos}), a
contradiction.

Approach from inside: apply the space-time pushes of the proof of
Lemma~\ref{lem:rate} to a minimizer $\g\gamma^\star$. The pushed curve keeps the
data, has $g_j\ge\mu$, and costs at most $E^\star_\infty+c_2\mu\le E^\star_\infty+\varepsilon/3$.
Smoothing and a small Hermite correction add at most $\varepsilon/3$ and keep
$g_j\ge\mu/2$; if $\ell\ge k$, the end correction of Lemma~\ref{lem:rate} on a
layer of width $w$ adds the data of orders $k,\dots,\ell$ at cost $O(w)$ and moves
the curve by $O(w^k)$, so for small $w$ the cost is below $E^\star_\infty+\varepsilon$ and
$g_j\ge\mu/4$.

\paragraph{Proof of Proposition~\ref{prop:globalC}.}
Let $U=E^\star_\infty+1$ and let $N_0$ exceed the thresholds of
Lemmas~\ref{lem:rate} and~\ref{lem:mfcq-rob} with $CN_0^{-\min\{k-1/2,1\}}<1$, so that
$E^\star_N<U$ for $N\ge N_0$. By Lemma~\ref{lem:mfcq-rob}, Assumption~\ref{assum:rob}
holds on the feasible curves of $\mathcal S_N$ with cost $\le U$, which contain the
global minimizers over $\mathcal S_N$ and all limit points from starts of cost
$\le U$. Assumption~\ref{assum:structural} enters
Proposition~\ref{prop:fixeddegree} and Corollary~\ref{cor:random} only through
$\m H\succ0$ and an affine chart. On $\mathcal S_N$, $\m H\succ0$: a $C^{k-1}$
spline with zero $k$th derivative is one polynomial of degree below $k$, which
the zero boundary data annihilate (proof of Lemma~\ref{lem:hpos}). The boundary
and junction equations are linear, consistent and of full column rank for
every $N$ when $\ell<k$ (Hermite interpolation at the nodes) and for $N\ge3$
when $\ell\ge k$ (a middle piece of degree $\ge2k-1$ joins the end pieces), so
the chart of Lemma~\ref{lem:affchart} exists (for smaller $N$, a feasible start
makes the equations consistent, and an orthonormal basis of their null space
gives the chart used by Theorem~\ref{thm:globalinf}(i)). There
\[
\beta=2\lambda_{\max}(\m N^\top\m Q\m N),
\]
with $\m N$ the orthonormal basis of this chart (to which $\m z_h$ is mapped) and $\m Q$ block diagonal (each block $N^{2k-1}$ times the cost matrix of one
piece on $[0,1]$).

(i) Proposition~\ref{prop:fixeddegree} over $\mathcal S_N$, with $\mathcal B$ the whole
chart, gives limit points that are global minimizers over $\mathcal S_N$;
Corollary~\ref{cor:random} gives the same for independent uniform feasible
starts in a bounded box containing a global minimizer in its interior.

(ii) The bounds are Lemma~\ref{lem:rate}. The minimizers over $\mathcal S_N$ have
bounded cost, hence are bounded in $\trajspace^k$ (coercivity); a subsequence
converges weakly in $\trajspace^k$ and uniformly with the derivatives below $k$ to
a feasible curve with the boundary data of order below $k$ and cost at most
$\lim E^\star_N=E^\star_\infty$, a global minimizer of \eqref{eq:Pinf}.

\subsection{Scope of the assumptions and stability}
\label{app:scope}
This subsection records which parts of Assumptions~\ref{assum:data}
and~\ref{assum:curv} the proofs need and which can be weakened, when the
optimal value is stable, and when Assumption~\ref{assum:mfcq}, used only in
Propositions~\ref{prop:globalC} and~\ref{prop:degree} and in the linear stability bound below, holds and why Proposition~\ref{prop:globalC} cannot drop it.
Theorem~\ref{thm:globalinf} needs only Assumptions~\ref{assum:data} and~\ref{assum:curv}.

\paragraph{Piecewise-polynomial data.}
Assumption~\ref{assum:data} asks for continuous, piecewise-polynomial obstacle
centres and reference, which covers the piecewise-linear reference
\eqref{eq:policyref} of a learned policy. The exact test and the critical-value step of the proof of Proposition~\ref{prop:fixeddegree} use polynomial data.

The cone test stays exact for rational data and
breakpoints: on each interval of a common
partition of $[0,1]$ into the pieces of the curve and of the data, the block
is a polynomial in $t$, and $\mathcal K_j$ becomes the product of the cones of
these intervals, still closed and convex. The problem over $\mathcal S_N$
stays semialgebraic, which is all that step needs.

The
push, the rate and Lemma~\ref{lem:mfcq-rob} use only continuity of the
centres and the reference: the field $\g\phi$ of Lemma~\ref{lem:push} is a
smooth partition of unity times fixed vectors, and continuity of $g_j$ and
$\nabla_{\m x}g_j$ in $(\m x,t)$ is all its construction needs. Beyond the
exact test, the proof of Theorem~\ref{thm:globalinf} uses only continuity of the
data.

\paragraph{Bounded curves instead of Assumption~\ref{assum:curv}.}
Assumption~\ref{assum:curv} enters only through boundedness: curves of
bounded cost are bounded. The descent step uses the upper curvature $\beta$
and not $\alpha$. When the steering constraint is present ($\bar r<\infty$),
every feasible curve has $\|\g\gamma\|_\infty\le\|\g\rho\|_\infty+\bar r$,
and the coercivity argument of Proposition~\ref{prop:degree} applies
unchanged: a normalized sequence with vanishing $k$th derivatives and
vanishing sup norm cannot converge to a polynomial of norm one. So the
conclusions of Theorem~\ref{thm:globalinf} and Proposition~\ref{prop:globalC} hold with $\bar r<\infty$ in place
of Assumption~\ref{assum:curv}, with $f$ convex rather than strongly convex
over $\mathcal S_N$ and compactness supplied by the bounded feasible set.

\paragraph{Stability of the optimal value.}
$E^\star_\infty(\tilde r)$ is nondecreasing in $\tilde r$ and left-continuous on
$(-\tilde{r}_{\max},\tilde{r}_{\max}]$: if $\tilde r_i\uparrow\tilde r$, minimizers of
$(\mathrm P^{\tilde r_i}_\infty)$ cost at most $E^\star_\infty(\tilde r)$, and, as in the
proof of Theorem~\ref{thm:globalinf}(iii), a subsequence converges to a curve
feasible for $(\mathrm P^{\tilde r}_\infty)$ of cost at most
$\lim_iE^\star_\infty(\tilde r_i)$ (or $E^\star_\infty(\tilde r_i)\to\infty$ if
$(\mathrm P^{\tilde r}_\infty)$ is infeasible). Hence $E^\star_\infty$ has at most countably
many jumps, all upward on the right. Let
\begin{equation}\label{eq:Erob}
E_{\mathrm{rob}}:=\lim_{\tilde{r}\downarrow0}E^\star_\infty(\tilde{r}).
\end{equation}
By compactness, a curve along which every constraint holds strictly keeps a
positive margin, so every method whose curves keep a positive clearance costs
at least $E_{\mathrm{rob}}$, the analogue of the robust optimum of
\citet{solovey2020revisiting}. We call \eqref{eq:Pinf} \emph{stable} if
$E_{\mathrm{rob}}=E^\star_\infty$; with $\tilde r$ in place of $0$, this holds at all but
countably many $\tilde r$.

On a stable instance, $E^\star_\infty(\mu)<\infty$ for some $\mu\in(0,\tilde r_{\max}]$, so Theorem~\ref{thm:globalinf}(ii) at $\tilde{r}=0$ (with $U=E^\star_\infty(\mu)$) gives curves
feasible for \eqref{eq:Pinf} with costs in
$[E^\star_\infty,\,E^\star_\infty(\delta_N)+\delta_N]\to E^\star_\infty$, and if
\begin{equation}\label{eq:linstab}
E^\star_\infty(\tilde{r})\le E^\star_\infty+c\,\tilde{r}\qquad(\text{small }\tilde{r}>0),
\end{equation}
within $(c+1)\delta_N$ of $E^\star_\infty$. Bound \eqref{eq:linstab} holds if
Assumption~\ref{assum:mfcq} holds along one optimal curve (an open condition):
the construction of Lemma~\ref{lem:push} near that curve, used as in the proof
of Lemma~\ref{lem:rate}, gives a curve with every $g_j\ge\varrho$ and cost at
most $E^\star_\infty+c_2\varrho$, and $g_j\ge\varrho$ is a margin of at least
$\varrho/(4G')$ for a constant $G'$ of the data. Theorem~\ref{thm:globalinf} then
recovers the rate of Proposition~\ref{prop:globalC}(ii) with exactly feasible
curves.

\paragraph{When Assumption~\ref{assum:mfcq} holds.}
It is needed only at feasible positions within distance $1$ of curves of cost
at most $E^\star_\infty+1$. It always holds for robots without obstacles, walls
or steering: the direction $\m v=\m x$, which spreads every robot away from the
origin, gives
\[
\nabla g_j\cdot\m v=2\|\m x_a-\m x_b\|^2>0
\]
for every touching pair. Spreading about the centre of a single ball obstacle
does the same for that obstacle and any number of robots.

For static constraints it holds for almost every choice of margins. Give block
$j$ the radius $r_j+\delta_j$. For each set $\mathcal J$ of blocks, the map
$\m x\mapsto(g_j(\m x))_{j\in\mathcal J}$ is polynomial, so by Sard's theorem
almost every value is regular. At every configuration where exactly the blocks
of $\mathcal J$ are active, their gradients are then linearly independent, which
implies Assumption~\ref{assum:mfcq}. There are finitely many sets $\mathcal J$,
so the exceptional margins $\g\delta$ form a null set. For moving obstacles or a
moving reference the same count suggests that, generically, the condition
fails only at isolated points of space-time, and a cheapest curve may pass
through one.

\paragraph{Proposition~\ref{prop:globalC} cannot drop Assumption~\ref{assum:mfcq}.}
The condition is needed by the algorithm, not only by the proof. Take a point
robot in $\R^2$, two discs of radius $r$ centred at $(\pm r,0)$, which touch at
the origin, where Assumption~\ref{assum:mfcq} fails, and the cost $c\,x_2$ with
$c>0$ over $x_2\ge-1$ (adding $(\epsilon/2)\|\m x\|^2$, $\epsilon>0$, changes
nothing below). The vertical line $x_1=0$ is feasible throughout. From $(0,y)$
with $0<y\le\eta c$, the linearized set is
\[
\{\m w:\ w_2\ge y/2+(r/y)|w_1|\},
\]
and the repair returns its apex $(0,y/2)$. So the iterates halve their
distance to the origin at every step and converge to it. The origin is not a
KKT point, since the cost still decreases along the feasible line. A cheapest
curve through a zero-width gap between two touching obstacles can stall
GlobMP in the same way.

\section{Addendum to Section~\ref{sec:experiments}: experiments}

\label{app:extexp}

\subsection{Experimental setup}
\begingroup
\setcounter{topnumber}{4}\setcounter{bottomnumber}{3}\setcounter{totalnumber}{8}
\renewcommand{\topfraction}{0.95}\renewcommand{\bottomfraction}{0.9}
\renewcommand{\textfraction}{0.05}\renewcommand{\floatpagefraction}{0.8}
\label{app:benchmarks}

This subsection summarizes the hardware, robot platforms and software shared across all experiments. GlobMP's runs below (D3IL aside, Appendix~\ref{app:avoiding}) are instances of \eqref{eq:P}, each built from blocks of \eqref{eq:Pinf}, Appendix~\ref{app:variants} or the maze keep-outs \eqref{eq:maze-keepout} of Appendix~\ref{sec:exp-maze} and solved by
Algorithms~\ref{alg:pgd}--\ref{alg:proj} under the protocol of its source. All reported baseline results are obtained from our own
reproductions.

\label{app:expdetail}
\label{app:addexp}

\paragraph{Hardware.} The D3IL and maze baselines
ran on a workstation with a $24$-core Intel Core Ultra~9 $285$K at $3.7$\,GHz
and one NVIDIA GeForce RTX~$5090$ ($32$\,GB); the synthetic study of \S\ref{sec:exp-synthetic} and the multi-robot suite of \S\ref{app:exp-multirobot} ran on a shared node with
two $32$-core AMD EPYC~$7502$ CPUs and eight NVIDIA A$100$-PCIE GPUs
($40$\,GB each), as did GlobMP's maze runs, so the maze times compare different machines. Training, sampling and model evaluation were performed on a
single GPU. The RRT, GCS and SDP-MP baselines ran on CPU with each run pinned
to fixed cores, and Algorithms~\ref{alg:pgd}--\ref{alg:proj} ran on CPU or
GPU as stated with each experiment (for the synthetic study, in
Appendix~\ref{app:exp-synthetic-details}).

\paragraph{Robot platforms.}
\label{app:robot-system}
The bimanual demonstrations are collected on the setup of Figure~\ref{fig:teleop-rig}, using a Dexmate Vega-1 with two XHAND1 dexterous hands. Two ZED X One S cameras are mounted at the wrists and one ZED X Mini camera is mounted on the head. Teleoperation uses Manus gloves for hand target
gestures and HTC VIVE trackers for wrist target poses. Only the resulting
trajectories enter this paper; none of the teleoperation hardware is used at
inference. Figure~\ref{fig:modalities} shows the modalities, the
task string, the camera images and the fingertip tactile arrays, at the instant
the tissue box changes hands in the handover task.

The Robotarium experiments of \S\ref{sec:exp-mrmp} use the Georgia Tech
Robotarium~\citep{wilson2020robotarium}, with its $3.2\times2.0$\,m arena,
GRITSBot geometry, and differential-drive dynamics. The benchmark environments
of \S\ref{sec:exp-mrmp} are adapted to the Robotarium workspace, with obstacles scaled by $0.85$ (\texttt{Cross-Room}), $0.4$ (\texttt{Drop-Region}), $0.2$ (\texttt{Corridor-Bay}, eight posts) and $1.1$ (\texttt{Conveyor}, with wider channels and a model retrained for them), represented virtually in the arena. The GRITSBots have
physical radius $0.055$\,m, while planning uses an inflated radius
$r=0.13$\,m to account for tracking error, the testbed's
barrier-certificate safety layer and the non-holonomic conversion of the plan. The planned trajectories are tracked through
the Robotarium's single-integrator-to-unicycle controller, and an overhead
camera records the executed motion. The virtual obstacles are overlaid from
the corresponding map definitions, and clearance is evaluated from the robot
trajectories tracked in the video after correcting for lens distortion.
Across the four obstacle maps, no GlobMP robot comes closer than $0.06$\,m to
an obstacle; the closest pass occurs on \texttt{Conveyor} and is shown in
Figure~\ref{fig:robotarium-setup}; Figures~\ref{fig:rg-empty}--\ref{fig:rg-crossroom}, at the end of the appendix, show every method except PCD-LB (at zero in Figure~\ref{fig:robotarium-sr-one}) on each map.

\begin{figure}[t!]
\begin{center}

\includegraphics[width=\linewidth]{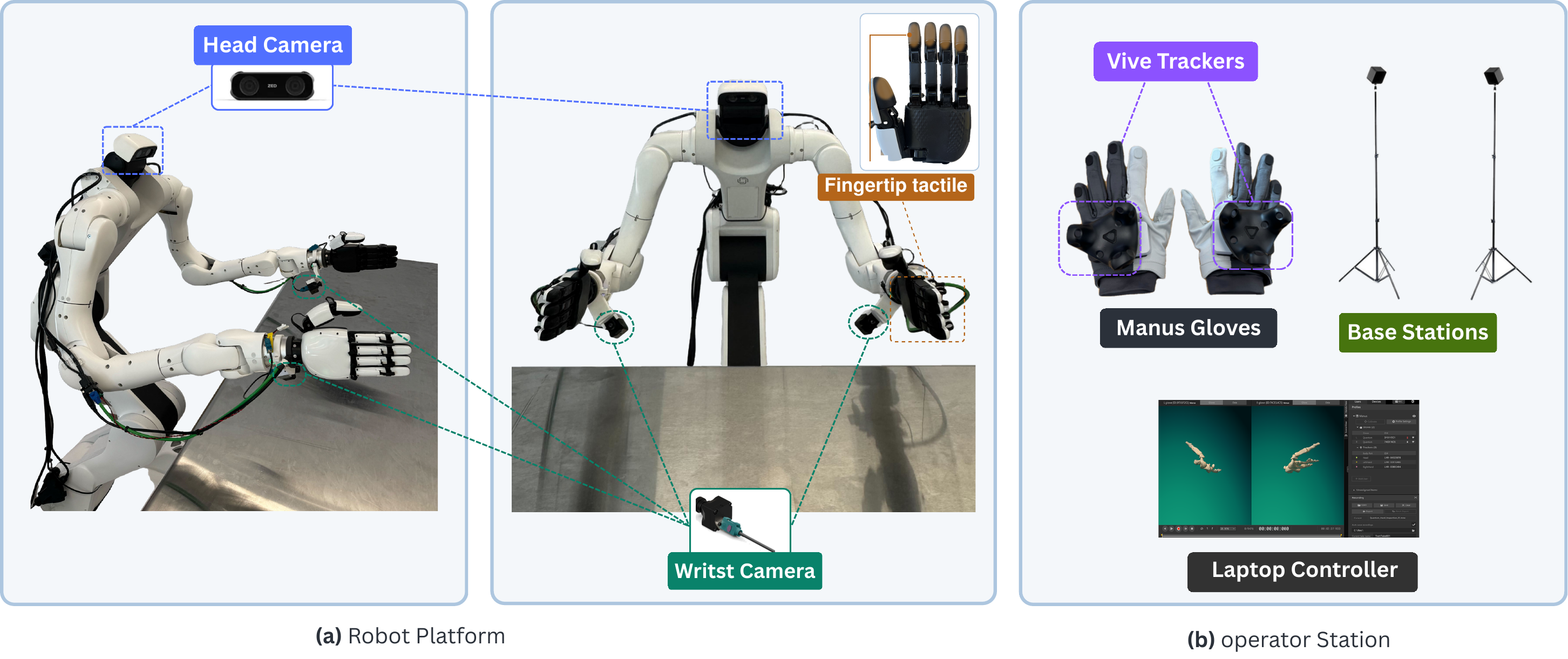}
\end{center}

\caption{Robot platform and operator station. The operator's
finger joints come from Manus gloves and the wrist poses from VIVE trackers
against two base stations; retargeting turns both into absolute joint targets for
the $38$ arm and hand joints of a Dexmate Vega-1 carrying two XHAND1
hands, sent at $30$\,Hz. Head and wrist
cameras and the fingertip force sensors record alongside. The close-up is the same hand
photographed on the bench and cut out of its background; the five fingertip pads, $120$ taxels each, are
tinted amber and the leader points at one of them; they are
the only tactile surfaces the hand has, the thumb's included. The palm and the finger links sense nothing. Only the resulting trajectories are
used in this paper; none of the operator hardware is present at inference.}
\label{fig:teleop-rig}
\end{figure}

\begin{figure}[htbp]
\centering
\includegraphics[width=\linewidth]{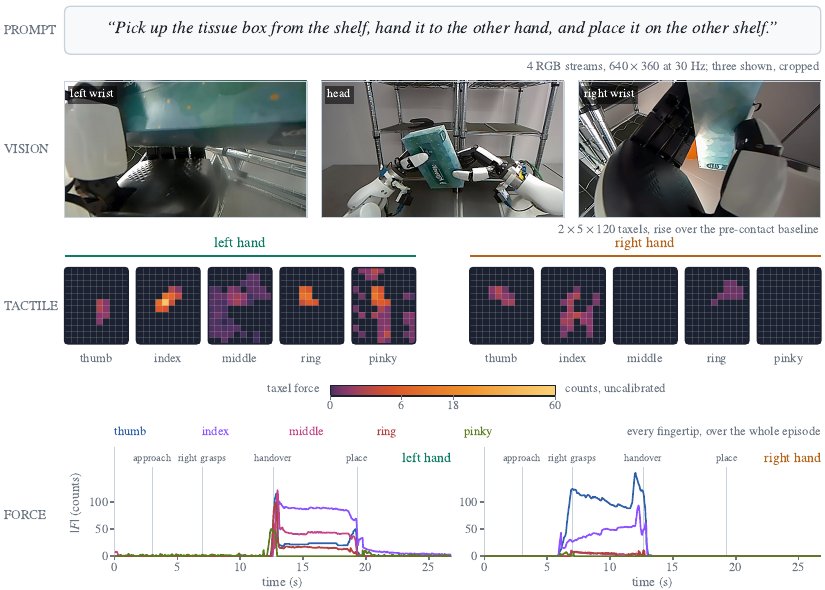}
\caption{Given modalities, at the instant the tissue box changes
hands (one episode of the handover task, $27$\,s at $30$\,Hz). \emph{Prompt}: the
task string the policy is conditioned on, one per episode. \emph{Vision}: three of
the four RGB streams. \emph{Tactile}: both hands at that same instant, five
fingertip pads of $120$ taxels each, every pad unrolled into its $10\times12$
measurement grid, one square per taxel, shown as its rise over the pre-contact
baseline on a stretched colour scale. \emph{Force}: the same fingertip
sensors over the whole episode, every finger on its own, one panel per hand.
Both scales are the sensor's own counts, as measured: the hand reports integers
and the vendor documents no conversion to newtons, so the counts carry no unit
and the figure claims none. The load passes from the right hand to
the left at the handover and falls when the box is placed.}
\label{fig:modalities}
\end{figure}

\begin{figure}[htbp]
\centering
\includegraphics[width=0.95\linewidth]{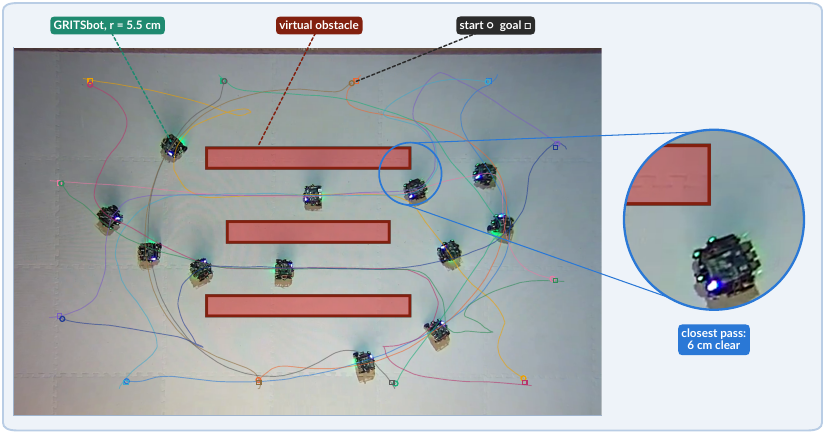}
\caption{GlobMP on the Georgia Tech Robotarium: twelve GRITSBots on \texttt{Conveyor} in the
$3.2\times2.0$\,m arena, at $t=25$\,s. The paths the
robots drove, starts (circles), goals (squares), and the virtual obstacles drawn from
the map. The circle magnifies the run's closest pass: $6$\,cm between the robot's body and a belt.}
\label{fig:robotarium-setup}
\end{figure}

\paragraph{Software.} All experiments are based on PyTorch ($2.13$; $2.6$ for Maze2D) for automatic
differentiation, JAX~$0.6.2$ for DiRecT's SCP--ADMM projector, and IPOPT
through Pyomo for the augmented-Lagrangian projection, under an external
wall-clock cap. The RRT variants run through OMPL~\citep{sucan2012ompl}; IRIS and GCS through Drake~\citep{drake}; and the curve fits, SDP-MP and GlobMP's SQP polish through CVXPY~\citep{diamond2016cvxpy} with Clarabel~\citep{goulart2024clarabel}. The
multi-robot baselines build on MMD \citep{shaoul2025mmd} (MIT), MDOC
\citep{he2026mdoc} (MIT) and DiRecT \citep{direct2026} (MIT), itself an
extension of \textsc{Projected Coupled Diffusion} \citep{luan2026pcd} (MIT). Every learned multi-robot, maze and D3IL baseline is trained from scratch under
its source's protocol rather than scored from a released checkpoint.

\subsection{Synthetic experiments}
\label{app:synthetic-tables}
\label{app:exp-synthetic-details}

The synthetic study of \S\ref{sec:exp-synthetic} tests continuous-time
feasibility, the cost of the returned curve, and scalability, on Problem~\ref{prob:B} and a variant of it.
Every instance is a minimum-derivative problem with snap cost $k=4$,
endpoint data $\ell=4$, piecewise degree-$9$ polynomials, $C^3$ junctions and
margin $0.005$\,m (added once to pair separations), so the piecewise conditions of
Appendix~\ref{app:piecewise} apply and the planning curve lies in the $\trajspace^4$
class of \eqref{eq:Pinf}.

\begin{problem}[robots exchanging places among spheres]\label{prob:B}
$N_a$ robots, $N_a$ even, are stacked into one curve $\g\gamma$ with values in $\R^{2N_a}$ and form
$N_a/2$ exchanging pairs: in pair $i$ one robot moves from $(x_0,y_i)$ to
$(x_0+D_i,y_i)$ and the other the reverse way along $y_i+o_i$, with
$D_i$ uniform on $[6,8]$\,m and $o_i$ uniform on $[0.3,0.9]\,R_S$ with $R_S=1.2$\,m, consecutive pairs separated by $2.8$ to $3.6\,R_S$ (from the upper row of one to the lower row of the next). Every pair of robots
carries a separation block
($\g\Sigma_j=\m E_{ab}^{\!\top}\m E_{ab}$, $\m c_j=\m 0$, $r_j=R_S$; Appendix~\ref{app:variants}), and every robot a
ball block for each static sphere, with $\m E_a\m c_j$ its centre, $\g\Sigma_j=\m E_a^{\!\top}\m E_a$
and $r_j$ its radius. Two spheres of radius uniform on $[0.45,0.65]$\,m sit $0.3$ to $0.9$ radii
below the first row, so the first robot's straight path enters both, and
every start and goal clears them by $0.2$\,m. There is no workspace bound, so
every inequality block has a nonconvex feasible set. Each exchanging pair, and each robot whose straight path enters or passes within $5$\,cm of a sphere, for that sphere, is passed on one of two sides, which gives $8$ to $128$
route classes per instance.
\end{problem}

The main comparison uses $N=7$ pieces and $20$ instances at each of
$N_a=2,4,6,8$. The scaling study uses $N=4$ pieces, lower rows of consecutive pairs spaced $3R_S$
apart with $o_i=0.6R_S$ and one $D_i$ shared by all pairs, three spheres, and, from $2$ to $200$ robots, $10$
seeds per size for GlobMP and $5$ for the baselines. GlobMP's scaling runs start from the straight-line curve (Figure~\ref{fig:prob-b} reports only this stage; no route classes, polish or finishing step enter it) and are timed from process start to the first feasible curve, setup included, on CPU (Figure~\ref{fig:prob-b}a) and separately on one A100 ($28$\,s median at $200$ robots). Blocks are counted per piece: at $200$ robots, $19{,}900$ pair and $600$ robot-sphere blocks over $4$ pieces give the $82{,}000$ blocks of \S\ref{sec:exp-synthetic} and Figure~\ref{fig:prob-b}c.

\paragraph{Methods.} In the main comparison, GlobMP opens each route class with a waypoint path that
passes every pair and every crossed sphere on the chosen side, runs
Algorithms~\ref{alg:pgd}--\ref{alg:proj} (then an SQP polish and the
finishing step of Appendix~\ref{app:impl}), and returns the cheapest feasible result over the
classes. RRT-Connect and RRT$^*$ are the OMPL implementations
\citep{kuffner2000rrt,karaman2011sampling} in $\R^{2N_a}$ with the exact
validity test (every pair at least $R_S$ apart, every robot outside every
sphere); they sample in the rectangle spanned by the instance plus $3$\,m,
since a sampling planner needs a bounded region, and get a $60$\,s budget ($600$\,s
in the scaling study). Their path is collision-free as a polyline but is
not a curve of the class, so the returned curve is the cheapest
collision-free fit among several (else the least colliding) on the true geometry (dense check, $0.1$\,mm tolerance): the
minimum-snap interpolant through the path at the segment breakpoints, which
cuts corners; fits that trade snap against the squared distance to the path;
and minimum-snap fits kept within a corridor around the path wherever the
current fit collides, with the corridor width set by the path's own
clearance \citep{mellinger2011minimum,richter2016polynomial}, each on the path and on a
shortcut version of it. GCS
\citep{marcucci2023gcs} works in the same joint space, where both collision
sets are convex: robot $a$ inside a sphere is $\{\m x:\|\m x_a-\m E_a\m c_j\|\le r_j\}$ and
a colliding pair is $\{\m x:\|\m x_a-\m x_b\|\le R_S\}$. IRIS \citep{deits2015computing}
grows regions from seeds every $0.5$\,m along a $60$\,s RRT-Connect path (GCS budget $2400$\,s, $600$\,s in the scaling study), Drake's
GCS trajectory optimization finds the minimum path-energy path through them (plus a $10^{-3}$ time cost) with Bézier
order $5$, falling back to $3$ and $2$ where its convex solve fails
numerically, and its curve is fitted in the same way. SDP-MP, the
semidefinite relaxation of \citet{graesdal2026sdp} extended to pair blocks, is written on the same program in its own affine
chart, so that the endpoint and junction conditions hold by construction;
the per-segment form with lifted junctions stops with a numerical error in
Clarabel from two segments on. The chart form is joint over all robots and segments, and each block must lie on each segment in the cone of polynomials
nonnegative on $[0,1]$ of \S\ref{sec:setup}, written as sums of squares. On an obstacle-free instance it
reproduces the exact minimum-snap optimum. Its rounded curve is the
first-order part of the solution. The baselines run on CPU, one run per
four pinned cores, and a run that returns nothing within one hour counts as
a failure. GlobMP's route classes run on CPU in the same way at $N_a=2,4$
and on the first $10$ instances at $N_a=6$, and on A100 GPUs, four to ten runs per
GPU, for the rest, so its times are totals of work rather than a
like-for-like speed comparison.

\paragraph{Rate in the number of pieces.} Figure~\ref{fig:rate} uses
instance $0$ of the main comparison with $N_a=2$, degree $9$, $C^{k-1}$ joins
and the $0.5$\,cm margin, for $k\in\{1,2,3,4\}$ and $\ell\in\{k-1,k\}$, and
$N$ up to $12$ ($k=1$), $16$ ($k=4$), $24$ ($k=3$) and $32$ ($k=2$). All route-class runs
converge to the same class, so $E^\star_N$ is estimated by the cost GlobMP reaches from
it: Algorithms~\ref{alg:pgd}--\ref{alg:proj}, the SQP polish and the finishing
step, from the waypoint path for $N\le8$ and from the $N=8$ curve,
refitted by least squares, for larger $N$; a run counts only at a KKT point
that keeps the margin. The cost is evaluated in global time from the
$k$th-derivative coefficients, with the cost quadratic expanded around the
obstacle-free minimizer (a change of chart): the expansion around the
default chart point loses all digits beyond about $16$ pieces at $k=4$, and
two starts agree to $5\cdot10^{-10}$ relative; Figure~\ref{fig:rate}a floors at $2\cdot10^{-9}$. Since \eqref{eq:Pinf} keeps only the orders below
$k$, both values of $\ell$ share $E^\star_\infty$, estimated by the cheapest
$\ell=k-1$ value at $N\ge8$. Panel (c) takes $N=8$, $\ell=k-1$ and margins (added once to the pair separation) from $0$ to
$5$\,cm, up to $10\%$ of the smaller sphere radius ($0.51$\,m; the other is $0.58$\,m); for $k\ge2$ the pair block stays inactive, at least $19$\,cm clear, so these are the margins $\tilde r$ of Theorem~\ref{thm:globalinf} (for $k=1$ it is active, so the $k=1$ curve in (c) raises the pair separation by $\tilde r$, not $2\tilde r$, and is not $E^\star_\infty(\tilde r)$). Velocity ($k=1$) wraps the spheres along arcs, and its finishing step reaches a
KKT point only on some runs, each started from the waypoint path or the curve at $N/2$: the panels show the runs that do, against $E^\star_\infty$ and, in
(c), the zero-margin cost, each estimated by the cheapest $\ell=0$ curve
feasible on the true geometry ($N=10$ and $N=8$).

\FloatBarrier

\subsection{Maze navigation}
\label{sec:exp-maze}
We examine whether the proposed formulation remains effective in a nonconvex maze. We consider the safe maze navigation benchmark of \citet{xiao2025safediffuser}, following the setting reported by \citet{direct2026}. A point robot traverses the D4RL \texttt{maze2d-large} environment \citep{fu2020d4rl} between fixed endpoints while avoiding keep-out regions placed within the maze corridors. The state contains planar position and velocity, and the action is a planar force. A rollout is considered safe if the executed trajectory remains outside every keep-out region. We evaluate two settings, \texttt{Broad}, containing two keep-outs, and \texttt{Narrow}, containing six.

This benchmark combines both constraint constructions used in this paper. The maze walls form a union of cells and therefore define a disjunctive free space that cannot be represented by a single membership in $\mathcal K$. A breadth-first search on the cell graph selects a corridor in ${\approx}1$\,ms, after which the continuous trajectory is optimized and verified within the selected corridor. Each of the $12$ corridor boxes along the route contributes four affine blocks \eqref{eq:wallblock} and one B\'ezier segment, with adjacent segments joined $C^{\ell}$, resulting in $48$ affine blocks. Each keep-out contributes one block \eqref{eq:maze-keepout}, quadratic for $e=2$ and quartic for $e=4$, beyond the quadratic blocks of \eqref{eq:Pinf}. We use $k=\ell=2$ with pieces of degree $9$, so Corollary~\ref{cor:eps}(i) ($\ell<k$) does not apply here. The maze geometry is not learned. The baselines rely on the diffusion prior to remain within the maze and explicitly constrain only the keep-outs among the geometric constraints, whereas our formulation enforces the corridor boundaries through affine blocks and does not use a learned prior.

\begin{table}[t]
\caption{\texttt{Maze2D Broad}, in the format of \citet{direct2026}'s Table~5.
Mean $\pm$ standard deviation. The four \emph{Gen.}\ columns score the generated trajectory.
Best value(s) in bold, Time excepted. All baselines are run here over $10$ evaluation episodes, sampling from the same diffusion planner we train, so the values are measured rather than quoted; our rows are deterministic and report a single trajectory.}
\label{tab:maze2d}
\begin{center}
\footnotesize
\resizebox{\ifdim\width>\linewidth\linewidth\else\width\fi}{!}{%
\begin{tabular}{lcccc|cccc}
\toprule
Method & SR $\uparrow$ & Viol.\ $\downarrow$ & Score $\uparrow$ & Time\,(s) $\downarrow$
& Gen.\ SR $\uparrow$ & Gen.\ Viol.\ $\downarrow$ & Gen.\ AS $\downarrow$ & Gen.\ CS $\downarrow$ \\
\midrule
Diffuser & $0.000$ & $44.600\pm11.298$ & $1.612\pm0.004$ & $0.256\pm0.099$ & $0.000$ & $39.700\pm18.815$ & $0.0396\pm0.0022$ & $0.4281\pm0.0427$ \\
Classifier Guidance & $0.100$ & $27.800\pm13.934$ & $1.612\pm0.004$ & $0.548\pm0.091$ & $0.300$ & $9.000\pm11.340$ & $0.0400\pm0.0022$ & $0.4352\pm0.0552$ \\
Gradient Guidance & $0.200$ & $13.000\pm6.885$ & $1.613\pm0.005$ & $2.663\pm0.107$ & $\mathbf{1.000}$ & $\mathbf{0.000}\pm0.000$ & $0.0404\pm0.0023$ & $0.4471\pm0.0460$ \\
Projected Diffusion & $0.400$ & $38.500\pm45.247$ & $1.314\pm0.463$ & $103.729\pm26.086$ & $\mathbf{1.000}$ & $\mathbf{0.000}\pm0.000$ & $0.0026\pm0.0002$ & $0.0092\pm0.0024$ \\
Primal--Dual & $0.000$ & $30.600\pm12.492$ & $0.111\pm0.320$ & $4.426\pm0.118$ & $0.800$ & $0.300\pm0.640$ & $20.1566\pm0.3115$ & $1.8145\pm0.0146$ \\
Augmented Lagrangian & $0.500$ & $9.600\pm10.394$ & $-0.025\pm0.000$ & $4.229\pm0.212$ & $0.000$ & $9.900\pm7.382$ & $1.8568\pm0.4291$ & $1.1980\pm0.0542$ \\
SafeDiffuser-RoS & $0.000$ & $59.900\pm41.293$ & $1.070\pm0.562$ & $145.924\pm21.730$ & $\mathbf{1.000}$ & $\mathbf{0.000}\pm0.000$ & $0.0082\pm0.0066$ & $0.0115\pm0.0043$ \\
SafeDiffuser-ReS & $0.400$ & $9.800\pm8.600$ & $0.655\pm0.833$ & $30.553\pm2.242$ & $\mathbf{1.000}$ & $\mathbf{0.000}\pm0.000$ & $0.0039\pm0.0048$ & $0.0164\pm0.0059$ \\
SafeDiffuser-TVS & $0.400$ & $42.400\pm38.733$ & $1.180\pm0.625$ & $100.873\pm18.575$ & $\mathbf{1.000}$ & $\mathbf{0.000}\pm0.000$ & $0.0048\pm0.0029$ & $0.0117\pm0.0056$ \\
DiRecT & $\mathbf{1.000}$ & $\mathbf{0.000}\pm0.000$ & $1.610\pm0.004$ & $7.746\pm0.485$ & $\mathbf{1.000}$ & $\mathbf{0.000}\pm0.000$ & $0.0024\pm0.0000$ & $0.0080\pm0.0003$ \\
GlobMP (Ours), disc cover & $\mathbf{1.000}$ & $\mathbf{0.000}$ & $1.591$ & $270.3$ & $\mathbf{1.000}$ & $\mathbf{0.000}$ & $0.0009$ & $0.0010$ \\
GlobMP (Ours), exact $e{=}4$ & $\mathbf{1.000}$ & $\mathbf{0.000}$ & $\mathbf{1.629}$ & $32.8$ & $\mathbf{1.000}$ & $\mathbf{0.000}$ & $\mathbf{0.0005}$ & $\mathbf{0.0001}$ \\
\bottomrule
\end{tabular}}
\end{center}
\end{table}

\textbf{Baselines.} We compare against ten diffusion planners, nine of them constrained: \emph{Diffuser} \citep{janner2022diffuser}, \emph{Classifier Guidance} \citep{dhariwal2021diffusion}, \emph{Gradient Guidance} \citep{chung2023diffusion}, \emph{Projected Diffusion} \citep{christopher2024projected}, \emph{Primal--Dual} and \emph{Augmented Lagrangian} \citep{zhang2025constrained}, the \emph{SafeDiffuser} variants RoS, ReS, and TVS \citep{xiao2025safediffuser}, and \emph{DiRecT} \citep{direct2026}, on the task setup of \citet{direct2026}. Methods requiring a nonlinear solver use IPOPT \citep{wachter2006implementation}. The nine constrained methods impose the keep-out regions but leave satisfaction of the maze walls to the learned prior, and none guarantees feasibility in continuous time. Their constraints are enforced pointwise during sampling or projection, so feasibility holds at the sampled states rather than along the whole curve.

\textbf{Training and environment details.} Each keep-out region is represented by an ellipse or super-ellipse centred at $(c_x,c_y)$ with semi-axes $(r_x,r_y)$ and order $e$, \begin{equation} \label{eq:maze-keepout} \Big|\tfrac{x-c_x}{r_x}\Big|^{e} +\Big|\tfrac{y-c_y}{r_y}\Big|^{e}\;\ge\;1 . \end{equation} The \texttt{Broad} setting uses one region with $e=2$ and one with $e=4$, while all six keep-outs in \texttt{Narrow} use $e=4$. Each region straddles a wall stub and extends into the corridor, and GlobMP's route touches the $0.4$ margin of both keep-outs in \texttt{Broad} and of four of the six in \texttt{Narrow}, as illustrated in Figures~\ref{fig:maze2d-rollouts-broad} and~\ref{fig:maze2d-rollouts-narrow}. A single continuous-time diffusion model with a temporal UNet is used by all baselines. It is trained for $500{,}000$ steps with Adam, batch size $256$, on the D4RL \texttt{maze2d-large-v1} dataset using min--max normalization. The learning rate is $10^{-3}$ with PyTorch's cosine annealing ($T_{\max}=10{,}000$ steps, so periodic over training), the exponential moving average decay is $0.995$, and the planning horizon is $384$ steps. The model outputs a state--action trajectory. Dynamics are identified from the same dataset by ordinary least squares over $3{,}993{,}503$ transitions and are imposed on the baselines as equality constraints. Because open-loop execution of the fitted linear model diverges over this horizon, all methods execute their planned trajectories using a proportional--derivative controller with proportional gain $5$ and derivative gain $1$.

\textbf{Evaluation.}
We follow the protocol of \citet{direct2026}. Each baseline generates $10$
i.i.d.\ trajectories conditioned on the same fixed endpoints, with each
trajectory executed for at most $800$ environment steps. We report
\emph{SR}, the fraction of collision-free rollouts; \emph{Viol.}, the mean
number of steps spent inside a keep-out; \emph{Score}, the D4RL normalized
return; and \emph{Time}, the planning time. The \emph{Gen.}\ columns evaluate the
generated trajectory before execution and additionally report acceleration
smoothness, \emph{AS}, and curvature smoothness, \emph{CS} (not the constraint satisfaction of Appendix~\ref{app:exp-multirobot}). Baseline values
are means over $10$ samples, whereas our method is deterministic and reports
a single trajectory.

Tables~\ref{tab:maze2d} and \ref{tab:maze2d-narrow} reveal an important
difference between planning and execution. In \texttt{Broad}, several
baselines generate trajectories with \emph{Gen.\ SR}${}=1.00$ but lose safety
during execution: Gradient Guidance falls to $0.20$, SafeDiffuser-RoS to
$0.00$, and Projected Diffusion, SafeDiffuser-ReS and SafeDiffuser-TVS to
$0.40$. DiRecT is the only baseline that carries \emph{Gen.\ SR}${}=1.00$
through execution. Both of our formulations
maintain \emph{SR}${}=1.00$ with zero violations before and after execution.
The disc-cover formulation attains a score of $1.591$, while the
exact $e=4$ formulation improves the score to $1.629$ and achieves the lowest
\emph{AS} and \emph{CS}. Both satisfy the keep-out constraints for every $t$.

The distinction becomes more pronounced in \texttt{Narrow}, where feasibility
is difficult already at the planning stage. Gradient Guidance, SafeDiffuser-ReS
and DiRecT are the only baselines reaching \emph{Gen.\ SR}${}=1.00$, and of these
only DiRecT keeps it through execution; the other two fall to $0.10$. Our
trajectory maintains \emph{SR}${}=1.00$ with zero violations, achieves a score of
$1.625$ compared with $1.615$ for DiRecT, and obtains the
lowest \emph{AS} and \emph{CS}, at a planning time of $90.0$\,s compared with
$16.0$\,s for DiRecT (different machines, Appendix~\ref{app:addexp}). That time is not a single descent: Algorithm~\ref{alg:pgd}
is within $1\%$ of its own final cost after $19$\,s and stops at $57$\,s ($60$\,s cap),
and an SQP polish then supplies the final $6\%$ of the cost ($30$\,s, plus $3$\,s for the finishing step); letting
Algorithm~\ref{alg:pgd} run to its own stopping rule instead ($532$\,s) changes
neither the cost nor the score. A $20$\,s cap gives $32.8$\,s on \texttt{Broad}
against $7.7$\,s for DiRecT. On both maps the returned trajectory satisfies
the keep-out constraints at every $t\in[0,1]$, not only at the sampled
states that the benchmark checks.

In Tables~\ref{tab:maze2d} and~\ref{tab:maze2d-narrow}, several baselines lose
safety between the generated plan and the executed rollout, whereas every
GlobMP row is at $1.00$ on both. Figures~\ref{fig:maze2d-rollouts-broad}
and~\ref{fig:maze2d-rollouts-narrow} show the executed rollouts of every method
on the two maps.
GlobMP keeps the left side of \eqref{eq:maze-keepout} at least $1.4$, a margin of $0.4$; the plan touches this margin exactly at its active contacts (minimum
clearance $3.1\times10^{-7}$ on \texttt{Broad} and $3.7\times10^{-10}$ on
\texttt{Narrow}). Tracking error takes the executed
rollout up to $0.24$ into the margin band, which is what the margin is for:
the rollout stays outside every keep-out, and the benchmark counts zero
violations.

\begin{table}[tb]
\caption{\texttt{Maze2D Narrow}, the six-keep-out setting. Mean $\pm$ standard
deviation. The four \emph{Gen.}\ columns score the generated trajectory. Best value(s) in bold, Time excepted. All baselines are run here over $10$ evaluation episodes, sampling from the same diffusion planner we train, so the values are measured rather than quoted; our row is deterministic and reports a single trajectory.}
\label{tab:maze2d-narrow}
\begin{center}
\footnotesize
\resizebox{\ifdim\width>\linewidth\linewidth\else\width\fi}{!}{%
\begin{tabular}{lcccc|cccc}
\toprule
Method & SR $\uparrow$ & Viol.\ $\downarrow$ & Score $\uparrow$ & Time\,(s) $\downarrow$
& Gen.\ SR $\uparrow$ & Gen.\ Viol.\ $\downarrow$ & Gen.\ AS $\downarrow$ & Gen.\ CS $\downarrow$ \\
\midrule
Diffuser & $0.000$ & $94.100\pm7.327$ & $1.612\pm0.004$ & $0.260\pm0.106$ & $0.000$ & $82.800\pm16.726$ & $0.0396\pm0.0022$ & $0.4281\pm0.0427$ \\
Classifier Guidance & $0.000$ & $64.400\pm14.908$ & $1.619\pm0.005$ & $0.549\pm0.104$ & $0.000$ & $27.700\pm23.367$ & $0.0402\pm0.0023$ & $0.4385\pm0.0559$ \\
Gradient Guidance & $0.100$ & $21.700\pm17.211$ & $1.622\pm0.006$ & $2.607\pm0.086$ & $\mathbf{1.000}$ & $\mathbf{0.000}\pm0.000$ & $0.0404\pm0.0028$ & $0.4475\pm0.0535$ \\
Projected Diffusion & $0.100$ & $83.200\pm158.026$ & $1.057\pm0.721$ & $756.264\pm166.998$ & $0.800$ & $4.800\pm9.765$ & $0.0145\pm0.0225$ & $0.0221\pm0.0258$ \\
Primal--Dual & $0.000$ & $125.900\pm114.102$ & $0.103\pm0.294$ & $4.379\pm0.138$ & $0.000$ & $11.400\pm4.409$ & $20.1512\pm0.3103$ & $1.8129\pm0.0146$ \\
Augmented Lagrangian & $0.000$ & $222.200\pm181.693$ & $-0.025\pm0.000$ & $4.394\pm0.204$ & $0.000$ & $159.300\pm40.020$ & $2.5444\pm0.6313$ & $1.2043\pm0.0519$ \\
SafeDiffuser-RoS & $0.000$ & $258.200\pm206.817$ & $0.927\pm0.628$ & $1215.129\pm163.413$ & $0.000$ & $20.800\pm7.859$ & $0.5856\pm0.0694$ & $0.2983\pm0.0335$ \\
SafeDiffuser-ReS & $0.100$ & $223.800\pm230.466$ & $0.808\pm0.834$ & $179.773\pm90.523$ & $\mathbf{1.000}$ & $\mathbf{0.000}\pm0.000$ & $0.0350\pm0.0569$ & $0.0283\pm0.0121$ \\
SafeDiffuser-TVS & $0.000$ & $266.300\pm200.686$ & $0.908\pm0.622$ & $559.372\pm114.038$ & $0.000$ & $31.700\pm19.990$ & $0.0458\pm0.0083$ & $0.0306\pm0.0076$ \\
DiRecT & $\mathbf{1.000}$ & $\mathbf{0.000}\pm0.000$ & $1.615\pm0.003$ & $15.957\pm1.406$ & $\mathbf{1.000}$ & $\mathbf{0.000}\pm0.000$ & $0.0024\pm0.0000$ & $0.0084\pm0.0008$ \\
GlobMP (Ours), exact $e{=}4$ & $\mathbf{1.000}$ & $\mathbf{0.000}$ & $\mathbf{1.625}$ & $90.0$ & $\mathbf{1.000}$ & $\mathbf{0.000}$ & $\mathbf{0.0006}$ & $\mathbf{0.0002}$ \\
\bottomrule
\end{tabular}}
\end{center}
\end{table}

\begin{figure}[p]
\centering
\includegraphics[width=\linewidth]{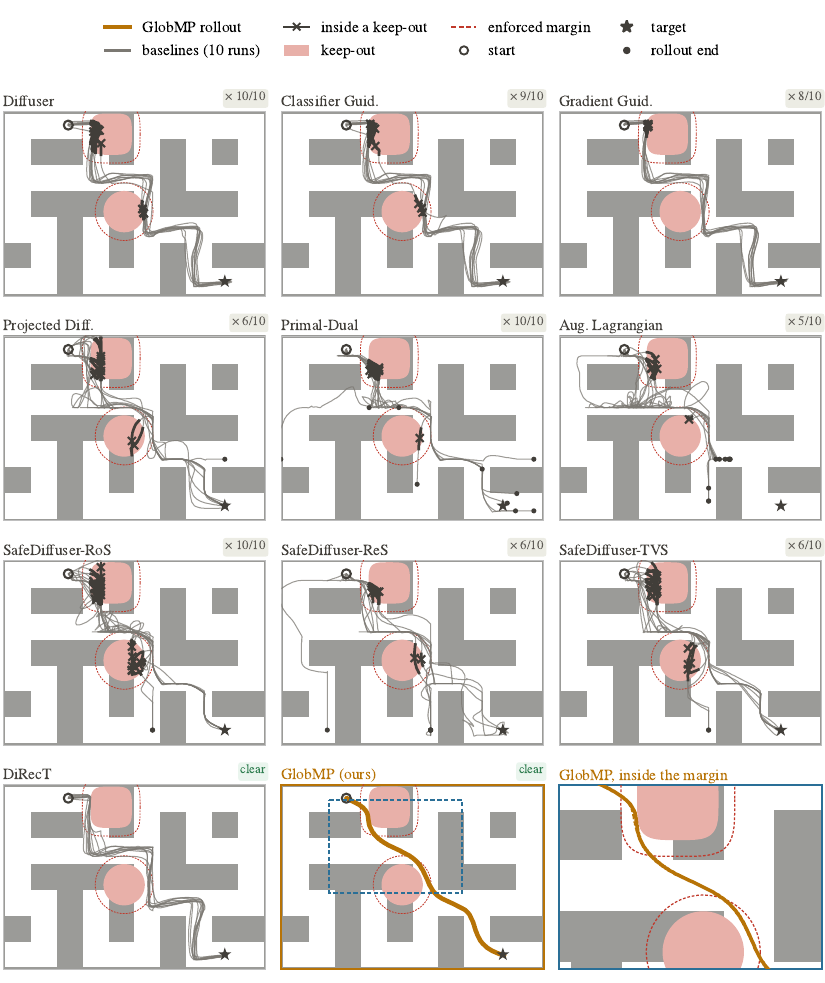}
\vspace{-6pt}
\caption{Executed rollouts on \texttt{Maze2D Broad}, one panel per method, under
the benchmark's PD controller. Baseline panels show all ten evaluation runs
behind Table~\ref{tab:maze2d}, unfiltered; GlobMP is deterministic and returns
one. Stretches drawn dark and marked $\times$ are where the robot is inside a
keep-out, and the badge counts the runs that entered one, which is $10(1-$\emph{SR}$)$
of the table. Every baseline of the table appears, in the table's order, with GlobMP's exact $e{=}4$ row. The last
cell magnifies every stretch where GlobMP runs inside the enforced margin, the
region marked by the dashed box on its own panel: the continuous plan (dashed)
is held off the keep-out by the margin, and the executed rollout (solid) dips
into the margin band but stays outside the keep-out, which is what the margin is
for.}
\label{fig:maze2d-rollouts-broad}
\end{figure}

\begin{figure}[p]
\centering
\includegraphics[width=\linewidth]{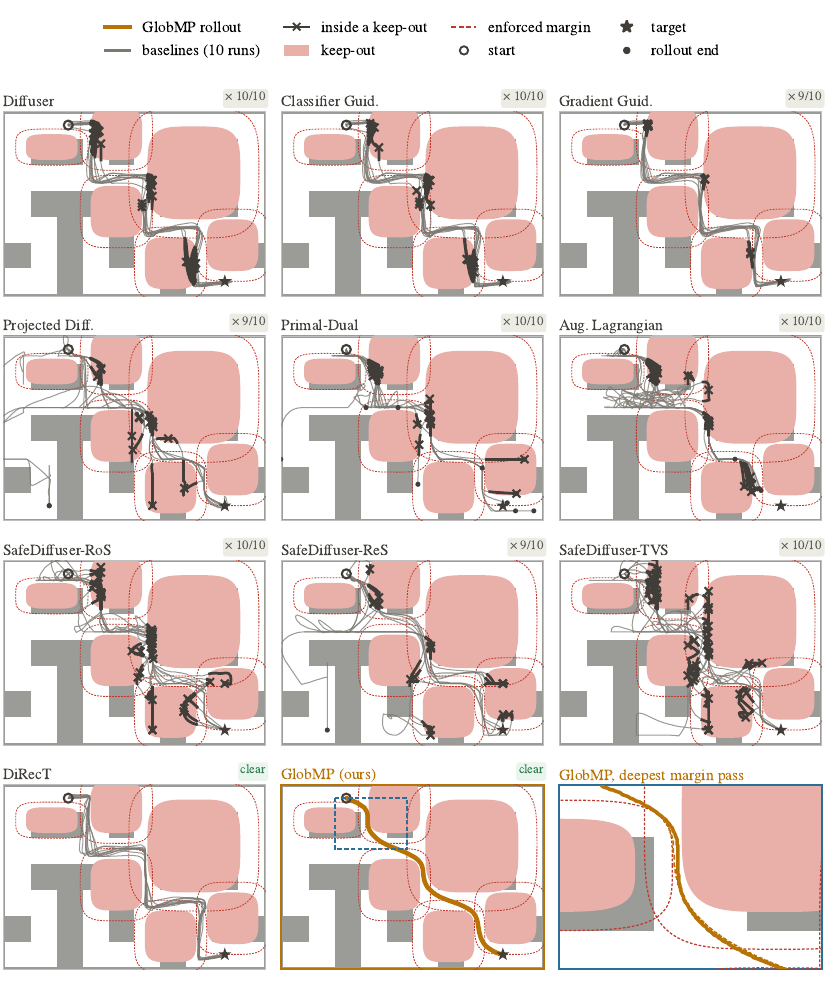}
\vspace{-6pt}
\caption{Executed rollouts on \texttt{Maze2D Narrow}, drawn as in
Figure~\ref{fig:maze2d-rollouts-broad}, over the ten evaluation runs behind
Table~\ref{tab:maze2d-narrow}. Only DiRecT and GlobMP keep every run outside the
keep-outs; the keep-outs leave little of the corridor, and most baselines enter
one on every run. So little is left that GlobMP's trajectory is inside the
enforced margin along most of the route, so the last cell magnifies its deepest
pass rather than every stretch.}
\label{fig:maze2d-rollouts-narrow}
\end{figure}

\subsection{Collision-free obstacle avoidance in end-effector space}
\label{app:avoiding}
We evaluate GlobMP against the baselines on the constrained manipulation benchmark of \citet{jia2024d3il}, with GlobMP as the projection step of DiRecT's own diffusion sampler. A planar end-effector must reach a goal region while avoiding a set of pillars (Figure~\ref{fig:manipulation}). The state lies in $\R^4$ and contains the current and desired positions, while the action lies in $\R^2$ and specifies velocity. The initial position is fixed, the goal lies on $y=0.35$, and safety requires the entire trajectory to remain outside every keep-out region. Circular keep-outs are represented by the quadratic block \eqref{eq:MjY} at degree $2d$, while planar constraints use the affine block of \S\ref{app:affine}. GlobMP's plan is a chain of $15$ straight segments ($C^0$ joins, start fixed, end free) under DiRecT's own projection cost, which lies outside the curve class of Theorem~\ref{thm:globalinf}. 

\textbf{Baselines.} We run the ten diffusion planners of the maze benchmark (\S\ref{sec:exp-maze}), which cover soft guidance, projection, dual optimization, barrier-function guidance, and receding-horizon projection. Methods requiring a nonlinear solver use IPOPT \citep{wachter2006implementation}. 

\textbf{Training and environment details.} A continuous-time diffusion model with a temporal UNet is trained for $500{,}000$ steps with Adam on the D3IL \texttt{avoiding} dataset using batch size $32$, learning rate $2\times10^{-4}$, PyTorch's cosine annealing ($T_{\max}=100{,}000$ steps, so periodic over training), and an exponential moving average with decay $0.995$, at horizon $H{=}16$. The same model is shared by all baselines, and constraints are introduced only during sampling, so no retraining is required when the constraint set changes. The original task, learned from $96$ human demonstrations, contains six circular pillars, $0.05$--$0.06$\,m across. Following \citet{direct2026}, we introduce unseen test-time constraints by enlarging the first pillar from radius $0.03$\,m to $0.08$\,m and restricting lateral motion with two planar regions. We further activate three keep-outs of radius $0.02$\,m centred at $(0.5,0.16)$, $(0.4,0.175)$, and $(0.6,0.175)$, which are included in their release but disabled by default. The resulting problem contains ten circular keep-outs and two planar constraints. These constraints deliberately differ from the training distribution. Of the $96$ demonstrations, $77$ violate the enlarged pillar, $62$ violate the planar constraints, and $53$ violate the three additional keep-outs. The resulting free space contains four route classes, with a minimum passage width of $0.061$\,m.

\begin{figure}[tb]
\begin{center}
\includegraphics[width=\linewidth]{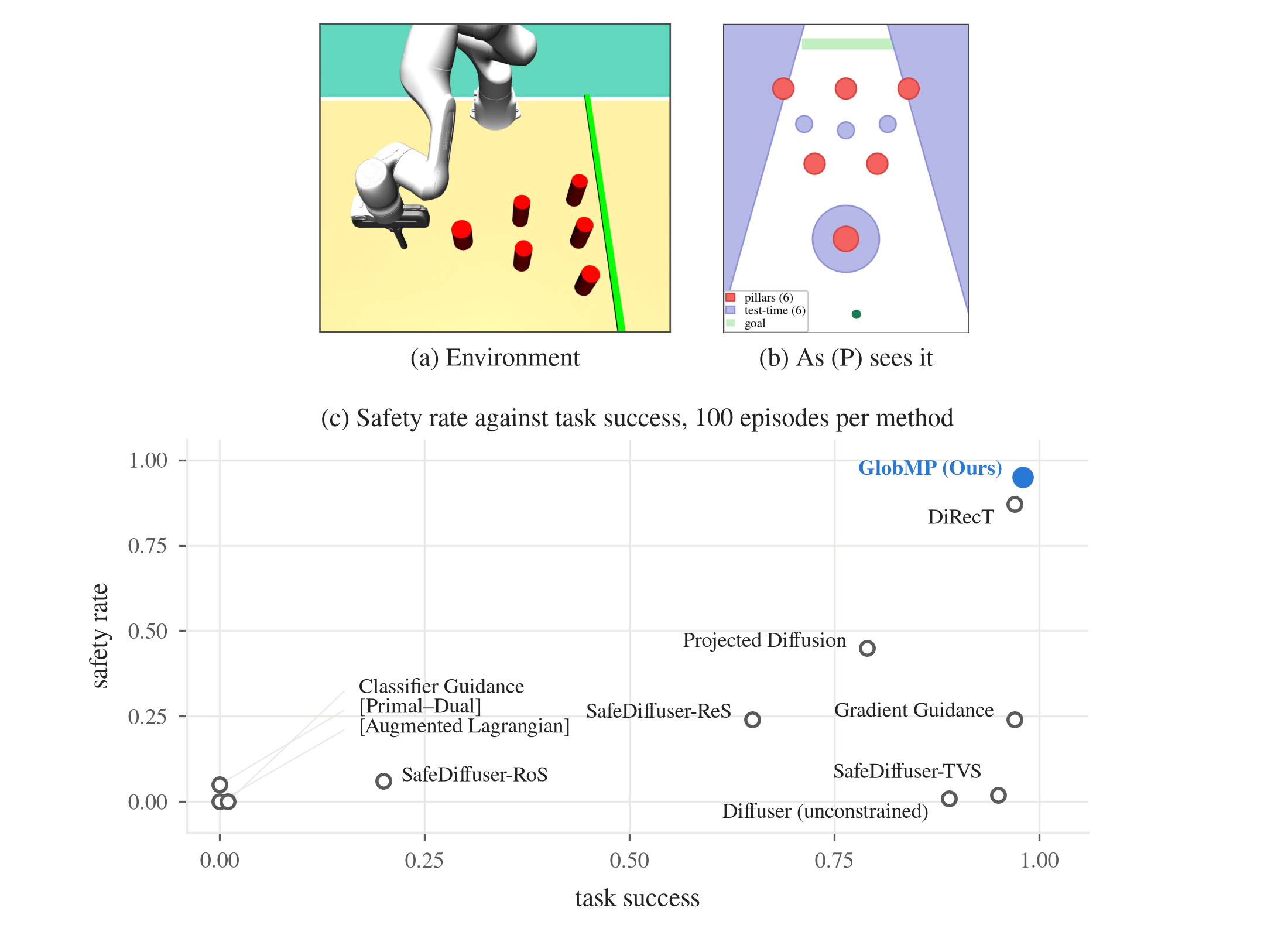}
\end{center}
\caption{D3IL \texttt{avoiding-v0}.
\emph{(a)} The environment's own view; the test-time constraints have no physical counterpart, so they appear only in \emph{(b)}. \emph{(b)} The scene
as GlobMP sees it, with all ten discs and both half-planes. \emph{(c)} Safety rate against task success over $100$ episodes, one point per method: a
point at the lower right reaches the goal without avoiding, and only the upper right does
both. Brackets as in Table~\ref{tab:d3il-table6}.}
\label{fig:manipulation}
\end{figure}

\textbf{Evaluation.}
We evaluate each method over $100$ i.i.d.\ rollouts from the same initial position at seed $0$, with at most $100$ environment steps, planning horizon $H{=}16$, and execution horizon of $4$ steps. We report \emph{SR}, the fraction of rollouts satisfying all constraints; \emph{TS}, the fraction reaching the goal within the episode budget; \emph{Steps(Safe)}, the mean steps among safe rollouts; and \emph{Time}, the mean computation time per replanning step. All values are obtained from our executions of the released implementations.

\textbf{GlobMP.} We keep DiRecT's checkpoint, sampler, dynamics fit and evaluation, and replace only its projection: at every denoising step GlobMP projects the sampler's plan, minimizing DiRecT's own cost, onto plans whose straight-line motion between the $16$ states clears every keep-out by the margin at every instant, not only at the states. The actions follow from the planned positions through the dynamics fit, so their limits (the environment clips each coordinate of a step to at most $1.25$\,cm, and never backwards) and their rate of change enter as linear blocks, and every planned action is within them; the benchmark's PD tracker (gains $5$ and $1$) corrects each executed action, which the environment may still clip. A plan that starts closer to a keep-out than the margin must win the margin back, by at most $1$\,cm per step. Within one plan the start state is fixed, so each projection starts from the previous one's solution, which is feasible when that projection succeeded.

Table~\ref{tab:d3il-table6} shows that safety and task completion do not necessarily coincide,
and Figure~\ref{fig:d3il-grid} shows the rollouts behind it. Over $100$ episodes at the same seed, GlobMP is safe in $0.95$ of them and reaches the goal in $0.98$, ahead of DiRecT ($0.87$ and $0.97$), and every plan it returns is feasible at every instant. Over four seeds ($400$ episodes) GlobMP averages $0.865$ and $0.97$, on par with DiRecT's single seed, and it is slower: safe rollouts take $78$ steps against $61$, and each plan takes longer to compute. Projected Diffusion attains the next-highest safety rate at $0.45$. In contrast, Gradient Guidance and SafeDiffuser-TVS reach the goal in $0.97$ and $0.95$ of episodes but achieve safety rates of only $0.24$ and $0.02$, respectively. The two metrics should therefore be interpreted jointly, since high task success may still correspond to frequent constraint violations.

For DiRecT, the remaining safety failures occur primarily during execution rather than planning. IPOPT converges on every solve and $815$ of $819$ generated plans are feasible; moreover, in $10$ of the $13$ unsafe episodes, every generated plan is feasible before execution. Thus, even when the
planner satisfies the test-time constraints, tracking can still drive the executed trajectory outside the feasible set. As in \citet{direct2026}, these experiments evaluate inference-time constraint handling: the additional constraints are introduced only at test time and no method is retrained.

\begin{figure}[tbp]
\begin{center}
\includegraphics[width=\linewidth]{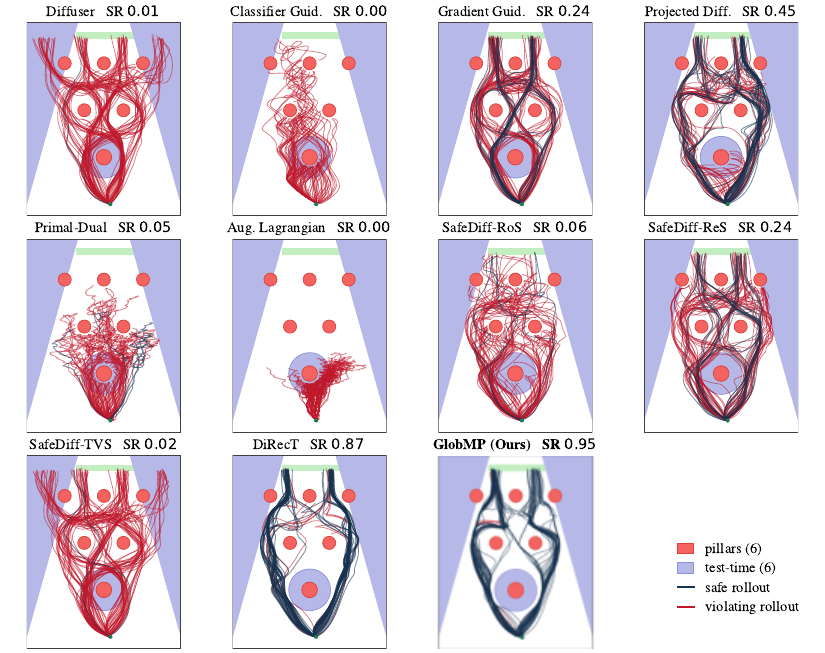}
\end{center}
\caption{All $100$ rollouts of every method on D3IL \texttt{avoiding-v0}, from
the common start (green), GlobMP last. A rollout is drawn red if it violates any of the
twelve constraints under the benchmark's own predicate and navy otherwise, so the fraction of navy is the safety rate in the panel title. The six
pillars belong to the environment; the six lavender regions are imposed only at
test time.}
\label{fig:d3il-grid}
\end{figure}

\begin{table}[tb]
\caption{D3IL \texttt{avoiding-v0}, in the format of \citet{direct2026}'s Table~6, $100$ episodes per method; mean $\pm$ standard deviation over episodes. Every value is our own execution of the released code from a single shared checkpoint. Bracketed rows required a non-default \texttt{-{}-guidance-last-steps 50} to run at all: at the default both crash the simulator. Boldface marks the best safety rate and task success among the methods that enforce the constraints; the cost columns
carry none. GlobMP is the projection step of DiRecT's sampler (replanning every $4$ steps, $28$ CPU workers for $50$ episodes at a time), on the same seed; over four seeds it reaches $0.865$ and $0.97$.}
\label{tab:d3il-table6}
\begin{center}
\footnotesize
\begin{tabular}{lcccc}
\toprule
Method                      & SR $\uparrow$ & TS $\uparrow$ & Steps(Safe) $\downarrow$ & Time\,(s) $\downarrow$ \\
\midrule
Diffuser (unconstrained)    & 0.01 & 0.89 & $61.00\pm0.00$ & $0.09\pm0.00$ \\
Classifier Guidance         & 0.00 & 0.01 & --- & $0.17\pm0.02$ \\
Gradient Guidance           & 0.24 & 0.97 & $62.96\pm7.27$ & $1.16\pm0.07$ \\
Projected Diffusion         & 0.45 & 0.79 & $65.18\pm8.09$ & $7.16\pm3.11$ \\
$[$Primal--Dual$]$          & 0.05 & 0.00 & $100.00\pm0.00$ & $0.74\pm0.02$ \\
$[$Augmented Lagrangian$]$  & 0.00 & 0.00 & --- & $0.79\pm0.01$ \\
SafeDiffuser-RoS            & 0.06 & 0.20 & $67.50\pm6.55$ & $7.34\pm0.18$ \\
SafeDiffuser-ReS            & 0.24 & 0.65 & $62.17\pm5.10$ & $5.44\pm0.27$ \\
SafeDiffuser-TVS            & 0.02 & 0.95 & $58.50\pm1.50$ & $5.26\pm0.09$ \\
DiRecT                      & 0.87 & 0.97 & $61.38\pm4.84$ & $3.72\pm0.33$ \\
GlobMP (Ours)               & \textbf{0.95} & \textbf{0.98} & $78.20\pm9.22$ & $15.12\pm9.82$ \\
\bottomrule
\end{tabular}
\end{center}
\end{table}

\FloatBarrier   %
\setcounter{secnumdepth}{3}\setcounter{tocdepth}{3}
\subsection{Multi-robot motion planning}
\label{sec:exp-multirobot}
\label{app:exp-multirobot}
\setcounter{totalnumber}{12}\setcounter{topnumber}{6}\setcounter{bottomnumber}{4}
\begin{figure}[tbp]
\begin{center}
\includegraphics[width=0.97\linewidth]{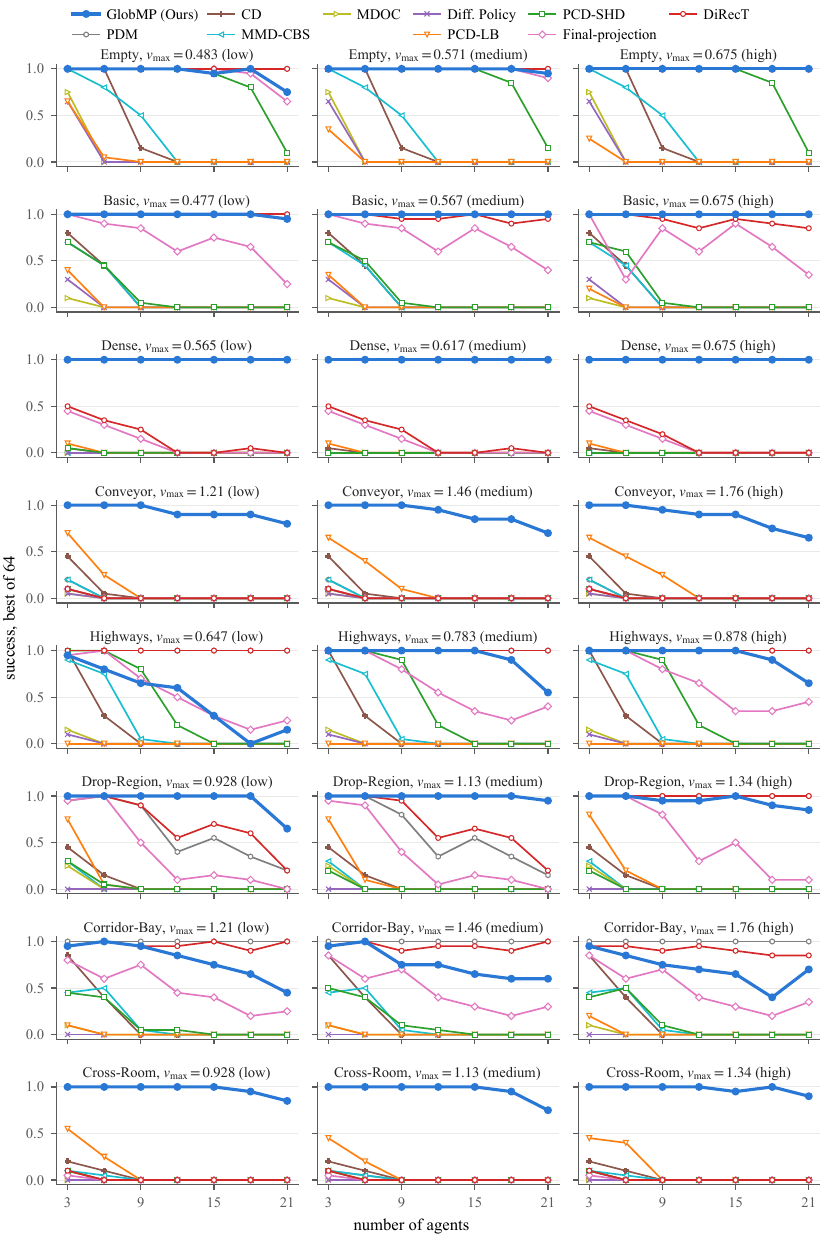}
\end{center}
\caption{Best-of-$64$ success rate for $N_a\in\{3,6,\dots,21\}$ agents on every cell of the
suite, one row per map and one column per velocity cap: an instance counts as solved when any of
the $64$ candidates a baseline samples is collision-free. This is the SR reported in Tables~\ref{tab:hard-empty}--\ref{tab:hard-crossroom};
our method still returns one trajectory.}
\label{fig:hardsuite-best64}
\end{figure}
\begin{figure}[htbp]
\centering
\includegraphics[width=\linewidth]{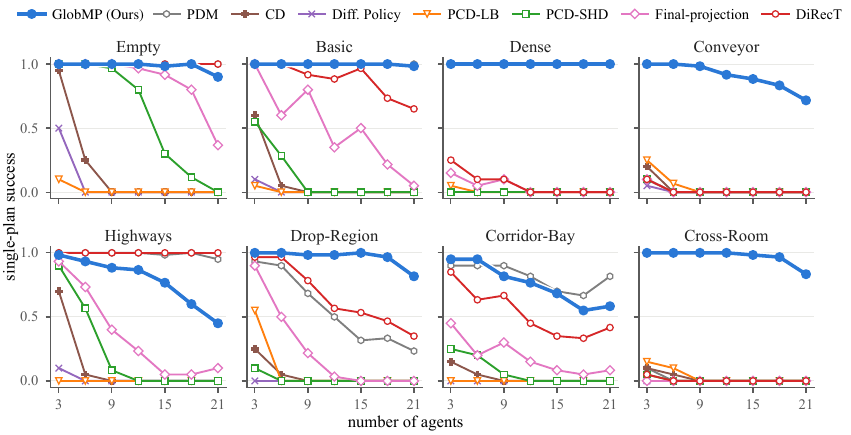}
\caption{Success of a \emph{single} returned plan for $N_a\in\{3,6,\dots,21\}$ agents, one
panel per map, averaged over the three velocity caps: the first candidate each baseline samples,
and the one trajectory our method returns. MDOC and MMD-CBS store no per-candidate data and
are omitted. Figure~\ref{fig:hardsuite-best64} gives the best-of-$64$ counterpart.}
\label{fig:hardsuite-scaling}
\end{figure}

We extend the multi-robot evaluation of \S\ref{sec:exp-mrmp} to a broader benchmark with eight environments, team sizes $N_a\in\{3,6,\dots,21\}$, and three test-time velocity caps (results in Figures~\ref{fig:hardsuite-best64}--\ref{fig:hardsuite-scaling}). The benchmark follows \citet{shaoul2025mmd}, with the velocity constraint of \citet{luan2026pcd} and the evaluation protocol of \citet{direct2026}. Six environments are from \citet{shaoul2025mmd}; \texttt{Corridor-Bay} and \texttt{Cross-Room} are reconstructed from \citet{liang2025smd}. Start--goal pairs follow the sampling procedure of \citet{shaoul2025mmd}, which places every endpoint at
least $2.1r+0.01$ from an obstacle of the source map (before the enlargement below) and $\max(0.15,2.2r)+0.01$ from every other endpoint.

\textbf{Problem setup.}
Each agent follows a planar B\'ezier trajectory, with one block per agent
and obstacle and one per pair of agents at separation $2r$
(\S\ref{app:multiagent}). Box obstacles use the affine blocks of
\S\ref{app:affine}, one half-space per agent, segment and box, with the
side selected by an independent grid seed.

The eight environments cover different forms of geometric and coordinated motion:
\begin{itemize}\itemsep2pt \parskip0pt
\item \texttt{Empty} contains no obstacles and isolates inter-agent coordination.
\item \texttt{Basic} and \texttt{Dense} contain ten and twenty scattered circular obstacles, respectively.
\item \texttt{Conveyor} contains two narrow corridors with opposing preferred traffic directions, requiring coordinated single-file motion.
\item \texttt{Highways} contains a central block and eight boundary blocks that leave a ring road with a preferred counterclockwise circulation.
\item \texttt{Drop-Region} contains four square obstacles and rewards pausing at drop-off points beside them.
\item \texttt{Corridor-Bay} contains three parallel channels joined by a narrow crossing, requiring teams to coordinate their passage.
\item \texttt{Cross-Room} contains a plus-shaped divider that forms four quadrants connected only through gaps near the workspace boundary.
\end{itemize}
The demonstrations associated with each environment encode its characteristic motion pattern, which is measured by the data-adherence metric \emph{DA}; on \texttt{Corridor-Bay} and \texttt{Cross-Room}, which have no preferred pattern, \emph{DA} is the start-to-goal distance divided by the path length.

In the normalized workspace $[-1,1]^2$, robots have radius $r=0.05$ on every map except \texttt{Empty}, where $r=0.10$, with minimum pairwise separation $2r$. Obstacles are those of the source maps, enlarged by $2.20$ (\texttt{Basic}), $1.50$ (\texttt{Conveyor}), $1.45$ (\texttt{Highways}), $1.60$ (\texttt{Drop-Region}), $1.15$ (\texttt{Corridor-Bay}) and $1.10$ (\texttt{Cross-Room}); \texttt{Dense} keeps its published sizes. Each environment is evaluated under Low, Medium, and High test-time velocity caps. In the order Low, Medium, High, these are
$(0.483,0.571,0.675)$ for \texttt{Empty},
$(0.477,0.567,0.675)$ for \texttt{Basic},
$(0.565,0.617,0.675)$ for \texttt{Dense},
$(1.21,1.46,1.76)$ for \texttt{Conveyor},
$(0.647,0.783,0.878)$ for \texttt{Highways},
$(0.928,1.13,1.34)$ for \texttt{Drop-Region},
$(1.21,1.46,1.76)$ for \texttt{Corridor-Bay}, and
$(0.928,1.13,1.34)$ for \texttt{Cross-Room}.

\textbf{Baselines.}
We compare GlobMP against Diffusion Policy \citep{chi2023diffusionpolicy}, CD, PCD-LB, and PCD-SHD \citep{luan2026pcd}, Final~Proj.\ and DiRecT \citep{direct2026}, MMD-CBS \citep{shaoul2025mmd}, and MDOC \citep{he2026mdoc}. Diffusion Policy provides the unguided reference; CD uses constraint guidance; PCD-LB and PCD-SHD augment guidance with projection; Final~Proj.\ applies projection after sampling; MMD-CBS combines sampled single-agent motions through conflict-based search; DiRecT applies projection throughout denoising; and MDOC uses model-based diffusion optimal control.
Each method is evaluated on the same start--goal instances.

\FloatBarrier

\textbf{Evaluation.}
We evaluate each method for $N_a\in\{3,6,\dots,21\}$ on all eight maps and
under the three velocity caps. For each (map, team size) pair, the same $20$
initial configurations are shared across methods, and each sampling baseline generates $64$ candidate trajectories per configuration (Final~Proj.\ samples $128$ on \texttt{Empty} and \texttt{Cross-Room}, scored on the first $64$ and timed over all $128$); MMD-CBS and MDOC return one plan, and GlobMP one trajectory. Following
\citet{direct2026}, we report \emph{SR}, the fraction of configurations with at least one collision-free trajectory; \emph{CS}, the fraction of all trajectories satisfying the collision and kinematic constraints; \emph{Coll}, the mean number of colliding agent pairs; \emph{DA}, environment-specific motion-pattern adherence; and \emph{Time}, wall-clock time per configuration in seconds. Tables~\ref{tab:hard-empty}--\ref{tab:hard-crossroom} give SR and CS in percent; the text quotes them as fractions. A dash marks a quantity that was not measured or is not comparable, rather than a zero. GlobMP's CS checks collisions only, since none of its counted plans exceeds the cap. The other conventions are restated in each table's caption.

\begin{table}[p]
\caption{Multi-robot planning performance on \texttt{Empty} for $N_a\in\{3,6,\dots,21\}$ under the Low, Medium and High velocity caps. \emph{SR} and \emph{CS} measure trajectory feasibility, \emph{Coll} reports inter-robot collisions, \emph{DA} adherence to the environment-specific motion pattern, and \emph{Time} the planning time per configuration: for \emph{GlobMP (Ours)} the median over the configurations with the standard deviation as subscript, for the baselines the mean, so the two are not directly comparable. $^{\dagger}$ marks a method that applies no velocity limit, so its three blocks report the same run; MMD-CBS also returns one plan per configuration rather than $64$ candidates. A dash under \emph{Coll} marks a count that is not commensurable with the others, as MDOC and MMD-CBS count colliding time steps where every other row counts agent pairs with any collision; MDOC's \emph{CS} is omitted because its plans that do not cover every agent were counted as feasible. MMD-CBS plans under $60$\,s per configuration and \emph{GlobMP (Ours)} under $300$\,s per team.}
\label{tab:hard-empty}
\begin{center}\tiny\setlength{\tabcolsep}{2.2pt}
\resizebox{\linewidth}{!}{%
\begin{tabular}{ll ccccc ccccc ccccc}
\toprule
& & \multicolumn{5}{c}{Low ($v_{\max}=0.483$)} & \multicolumn{5}{c}{Medium ($v_{\max}=0.571$)} & \multicolumn{5}{c}{High ($v_{\max}=0.675$)}\\
\cmidrule(lr){3-7}\cmidrule(lr){8-12}\cmidrule(lr){13-17}
$N_a$ & Method & SR$\uparrow$ & CS$\uparrow$ & Coll$\downarrow$ & DA$\uparrow$ & Time$\downarrow$ & SR$\uparrow$ & CS$\uparrow$ & Coll$\downarrow$ & DA$\uparrow$ & Time$\downarrow$ & SR$\uparrow$ & CS$\uparrow$ & Coll$\downarrow$ & DA$\uparrow$ & Time$\downarrow$\\
\midrule
3 & Diff.\ Policy$^{\dagger}$ & 65.0 & 48.2$_{\pm46.5}$ & 0.67$_{\pm0.77}$ & 0.999$_{\pm0.001}$ & 0.5 & 65.0 & 48.2$_{\pm46.5}$ & 0.67$_{\pm0.77}$ & 0.999$_{\pm0.001}$ & 0.3 & 65.0 & 48.2$_{\pm46.5}$ & 0.67$_{\pm0.77}$ & 0.999$_{\pm0.001}$ & 0.5\\
  & CD$^{\dagger}$ & 100.0 & 1.1$_{\pm4.4}$ & 0.14$_{\pm0.28}$ & 0.517$_{\pm0.149}$ & 1.0$_{\pm0.1}$ & 100.0 & 7.1$_{\pm20.2}$ & 0.14$_{\pm0.28}$ & 0.517$_{\pm0.149}$ & 1.1$_{\pm0.1}$ & 100.0 & 20.5$_{\pm26.1}$ & 0.14$_{\pm0.28}$ & 0.517$_{\pm0.149}$ & 0.5$_{\pm0.0}$\\
  & MMD-CBS$^{\dagger}$ & 100.0 & 70.0$_{\pm45.8}$ & --- & 0.946$_{\pm0.071}$ & 3.1$_{\pm2.5}$ & 100.0 & 70.0$_{\pm45.8}$ & --- & 0.946$_{\pm0.071}$ & 3.1$_{\pm2.5}$ & 100.0 & 70.0$_{\pm45.8}$ & --- & 0.946$_{\pm0.071}$ & 3.1$_{\pm2.5}$\\
  & MDOC$^{\dagger}$ & 75.0 & --- & --- & 0.572$_{\pm0.400}$ & 39.0$_{\pm6.1}$ & 75.0 & --- & --- & 0.572$_{\pm0.400}$ & 34.8$_{\pm4.3}$ & 75.0 & --- & --- & 0.572$_{\pm0.400}$ & 39.6$_{\pm4.8}$\\
  & PCD-LB & 65.0 & 12.3$_{\pm20.2}$ & 0.02$_{\pm0.05}$ & 0.325$_{\pm0.135}$ & 2.8$_{\pm0.8}$ & 35.0 & 4.6$_{\pm11.9}$ & 0.01$_{\pm0.02}$ & 0.316$_{\pm0.133}$ & 4.2$_{\pm0.9}$ & 25.0 & 3.3$_{\pm9.1}$ & 0.01$_{\pm0.01}$ & 0.307$_{\pm0.130}$ & 3.2$_{\pm0.6}$\\
  & PCD-SHD & \textbf{100.0} & \textbf{100.0}$_{\pm0.0}$ & \textbf{0.00}$_{\pm0.00}$ & 0.828$_{\pm0.113}$ & 3.0$_{\pm0.4}$ & \textbf{100.0} & \textbf{100.0}$_{\pm0.0}$ & \textbf{0.00}$_{\pm0.00}$ & 0.831$_{\pm0.109}$ & 1.4$_{\pm0.1}$ & \textbf{100.0} & \textbf{100.0}$_{\pm0.0}$ & \textbf{0.00}$_{\pm0.00}$ & 0.833$_{\pm0.107}$ & 1.4$_{\pm0.2}$\\
  & Final Proj. & \textbf{100.0} & \textbf{100.0}$_{\pm0.0}$ & \textbf{0.00}$_{\pm0.00}$ & \textbf{0.998}$_{\pm0.003}$ & 3.7$_{\pm0.3}$ & \textbf{100.0} & \textbf{100.0}$_{\pm0.0}$ & \textbf{0.00}$_{\pm0.00}$ & \textbf{0.998}$_{\pm0.003}$ & 2.0$_{\pm0.5}$ & \textbf{100.0} & \textbf{100.0}$_{\pm0.0}$ & \textbf{0.00}$_{\pm0.00}$ & \textbf{0.998}$_{\pm0.003}$ & 3.7$_{\pm0.5}$\\
  & DiRecT & \textbf{100.0} & \textbf{100.0}$_{\pm0.0}$ & \textbf{0.00}$_{\pm0.00}$ & 0.993$_{\pm0.018}$ & 31.5$_{\pm0.6}$ & \textbf{100.0} & \textbf{100.0}$_{\pm0.0}$ & \textbf{0.00}$_{\pm0.00}$ & 0.993$_{\pm0.018}$ & 31.4$_{\pm0.5}$ & \textbf{100.0} & \textbf{100.0}$_{\pm0.0}$ & \textbf{0.00}$_{\pm0.00}$ & 0.993$_{\pm0.017}$ & 31.0$_{\pm1.3}$\\
  & GlobMP (Ours) & \textbf{100.0} & \textbf{100.0}$_{\pm0.0}$ & \textbf{0.00}$_{\pm0.00}$ & 0.933$_{\pm0.149}$ & 4.0$_{\pm17.1}$ & \textbf{100.0} & \textbf{100.0}$_{\pm0.0}$ & \textbf{0.00}$_{\pm0.00}$ & 0.934$_{\pm0.149}$ & 4.0$_{\pm6.3}$ & \textbf{100.0} & \textbf{100.0}$_{\pm0.0}$ & \textbf{0.00}$_{\pm0.00}$ & 0.984$_{\pm0.043}$ & 0.2$_{\pm0.1}$\\
\midrule
6 & Diff.\ Policy$^{\dagger}$ & 0.0 & 0.0$_{\pm0.0}$ & 4.77$_{\pm2.73}$ & 0.999$_{\pm0.001}$ & 0.5 & 0.0 & 0.0$_{\pm0.0}$ & 4.77$_{\pm2.73}$ & 0.999$_{\pm0.001}$ & 0.6 & 0.0 & 0.0$_{\pm0.0}$ & 4.77$_{\pm2.73}$ & 0.999$_{\pm0.001}$ & 0.3\\
  & CD$^{\dagger}$ & 100.0 & 0.0$_{\pm0.0}$ & 0.96$_{\pm0.85}$ & 0.582$_{\pm0.079}$ & 1.1$_{\pm0.1}$ & 100.0 & 0.0$_{\pm0.0}$ & 0.96$_{\pm0.85}$ & 0.582$_{\pm0.079}$ & 0.9$_{\pm0.3}$ & 100.0 & 0.2$_{\pm0.7}$ & 0.96$_{\pm0.85}$ & 0.582$_{\pm0.079}$ & 1.1$_{\pm0.1}$\\
  & MMD-CBS$^{\dagger}$ & 80.0 & 15.0$_{\pm35.7}$ & --- & 0.879$_{\pm0.065}$ & 30.3$_{\pm21.2}$ & 80.0 & 15.0$_{\pm35.7}$ & --- & 0.879$_{\pm0.065}$ & 30.3$_{\pm21.2}$ & 80.0 & 15.0$_{\pm35.7}$ & --- & 0.879$_{\pm0.065}$ & 30.3$_{\pm21.2}$\\
  & MDOC$^{\dagger}$ & 0.0 & --- & --- & 0.000$_{\pm0.000}$ & 57.6$_{\pm4.8}$ & 0.0 & --- & --- & 0.000$_{\pm0.000}$ & 58.2$_{\pm4.8}$ & 0.0 & --- & --- & 0.000$_{\pm0.000}$ & 67.8$_{\pm5.2}$\\
  & PCD-LB & 5.0 & 0.1$_{\pm0.3}$ & 1.01$_{\pm0.73}$ & 0.299$_{\pm0.050}$ & 3.8$_{\pm0.7}$ & 0.0 & 0.0$_{\pm0.0}$ & 1.04$_{\pm0.72}$ & 0.300$_{\pm0.056}$ & 5.8$_{\pm1.1}$ & 0.0 & 0.0$_{\pm0.0}$ & 1.08$_{\pm0.69}$ & 0.302$_{\pm0.060}$ & 2.3$_{\pm0.3}$\\
  & PCD-SHD & \textbf{100.0} & 99.8$_{\pm0.5}$ & 0.00$_{\pm0.00}$ & 0.608$_{\pm0.073}$ & 4.3$_{\pm1.4}$ & \textbf{100.0} & \textbf{100.0}$_{\pm0.0}$ & \textbf{0.00}$_{\pm0.00}$ & 0.612$_{\pm0.071}$ & 2.0$_{\pm0.3}$ & \textbf{100.0} & \textbf{100.0}$_{\pm0.0}$ & \textbf{0.00}$_{\pm0.00}$ & 0.613$_{\pm0.068}$ & 1.6$_{\pm0.1}$\\
  & Final Proj. & \textbf{100.0} & 98.3$_{\pm7.5}$ & 0.04$_{\pm0.15}$ & \textbf{0.995}$_{\pm0.006}$ & 3.2$_{\pm0.5}$ & \textbf{100.0} & \textbf{100.0}$_{\pm0.0}$ & \textbf{0.00}$_{\pm0.00}$ & \textbf{0.996}$_{\pm0.005}$ & 1.8$_{\pm0.5}$ & \textbf{100.0} & \textbf{100.0}$_{\pm0.0}$ & \textbf{0.00}$_{\pm0.00}$ & \textbf{0.996}$_{\pm0.004}$ & 1.8$_{\pm0.8}$\\
  & DiRecT & \textbf{100.0} & \textbf{100.0}$_{\pm0.0}$ & \textbf{0.00}$_{\pm0.00}$ & 0.985$_{\pm0.017}$ & 27.6$_{\pm3.6}$ & \textbf{100.0} & \textbf{100.0}$_{\pm0.0}$ & \textbf{0.00}$_{\pm0.00}$ & 0.985$_{\pm0.016}$ & 24.6$_{\pm5.9}$ & \textbf{100.0} & \textbf{100.0}$_{\pm0.0}$ & \textbf{0.00}$_{\pm0.00}$ & 0.985$_{\pm0.016}$ & 24.0$_{\pm6.4}$\\
  & GlobMP (Ours) & \textbf{100.0} & \textbf{100.0}$_{\pm0.0}$ & \textbf{0.00}$_{\pm0.00}$ & 0.975$_{\pm0.090}$ & 4.4$_{\pm22.0}$ & \textbf{100.0} & \textbf{100.0}$_{\pm0.0}$ & \textbf{0.00}$_{\pm0.00}$ & 0.976$_{\pm0.090}$ & 4.3$_{\pm4.2}$ & \textbf{100.0} & \textbf{100.0}$_{\pm0.0}$ & \textbf{0.00}$_{\pm0.00}$ & 0.960$_{\pm0.089}$ & 0.2$_{\pm0.1}$\\
\midrule
9 & Diff.\ Policy$^{\dagger}$ & 0.0 & 0.0$_{\pm0.0}$ & 9.22$_{\pm3.01}$ & 0.999$_{\pm0.001}$ & 0.6 & 0.0 & 0.0$_{\pm0.0}$ & 9.22$_{\pm3.01}$ & 0.999$_{\pm0.001}$ & 0.6 & 0.0 & 0.0$_{\pm0.0}$ & 9.22$_{\pm3.01}$ & 0.999$_{\pm0.001}$ & 0.6\\
  & CD$^{\dagger}$ & 15.0 & 0.0$_{\pm0.0}$ & 3.80$_{\pm1.33}$ & 0.569$_{\pm0.075}$ & 1.1$_{\pm0.1}$ & 15.0 & 0.0$_{\pm0.0}$ & 3.80$_{\pm1.33}$ & 0.569$_{\pm0.075}$ & 0.5$_{\pm0.1}$ & 15.0 & 0.0$_{\pm0.0}$ & 3.80$_{\pm1.33}$ & 0.569$_{\pm0.075}$ & 1.1$_{\pm0.1}$\\
  & MMD-CBS$^{\dagger}$ & 50.0 & 0.0$_{\pm0.0}$ & --- & 0.877$_{\pm0.062}$ & 57.5$_{\pm8.9}$ & 50.0 & 0.0$_{\pm0.0}$ & --- & 0.877$_{\pm0.062}$ & 57.5$_{\pm8.9}$ & 50.0 & 0.0$_{\pm0.0}$ & --- & 0.877$_{\pm0.062}$ & 57.5$_{\pm8.9}$\\
  & MDOC$^{\dagger}$ & 0.0 & --- & --- & 0.000$_{\pm0.000}$ & 83.3$_{\pm5.9}$ & 0.0 & --- & --- & 0.000$_{\pm0.000}$ & 84.1$_{\pm5.7}$ & 0.0 & --- & --- & 0.000$_{\pm0.000}$ & 97.4$_{\pm7.2}$\\
  & PCD-LB & 0.0 & 0.0$_{\pm0.0}$ & 4.78$_{\pm1.47}$ & 0.253$_{\pm0.038}$ & 4.6$_{\pm0.7}$ & 0.0 & 0.0$_{\pm0.0}$ & 5.26$_{\pm1.65}$ & 0.260$_{\pm0.037}$ & 6.4$_{\pm0.8}$ & 0.0 & 0.0$_{\pm0.0}$ & 5.78$_{\pm1.64}$ & 0.271$_{\pm0.037}$ & 4.6$_{\pm0.5}$\\
  & PCD-SHD & \textbf{100.0} & 94.8$_{\pm8.7}$ & 0.05$_{\pm0.09}$ & 0.547$_{\pm0.052}$ & 5.4$_{\pm0.9}$ & \textbf{100.0} & 96.2$_{\pm5.0}$ & 0.04$_{\pm0.05}$ & 0.557$_{\pm0.052}$ & 2.2$_{\pm0.2}$ & \textbf{100.0} & 97.5$_{\pm3.8}$ & 0.02$_{\pm0.04}$ & 0.562$_{\pm0.052}$ & 1.9$_{\pm0.1}$\\
  & Final Proj. & \textbf{100.0} & 99.4$_{\pm2.1}$ & 0.01$_{\pm0.04}$ & \textbf{0.993}$_{\pm0.006}$ & 4.1$_{\pm0.5}$ & \textbf{100.0} & \textbf{100.0}$_{\pm0.0}$ & \textbf{0.00}$_{\pm0.00}$ & \textbf{0.995}$_{\pm0.005}$ & 2.2$_{\pm0.5}$ & \textbf{100.0} & \textbf{100.0}$_{\pm0.0}$ & \textbf{0.00}$_{\pm0.00}$ & \textbf{0.995}$_{\pm0.004}$ & 2.1$_{\pm0.8}$\\
  & DiRecT & \textbf{100.0} & \textbf{100.0}$_{\pm0.0}$ & \textbf{0.00}$_{\pm0.00}$ & 0.972$_{\pm0.020}$ & 33.5$_{\pm7.9}$ & \textbf{100.0} & \textbf{100.0}$_{\pm0.0}$ & \textbf{0.00}$_{\pm0.00}$ & 0.973$_{\pm0.020}$ & 33.3$_{\pm7.9}$ & \textbf{100.0} & \textbf{100.0}$_{\pm0.0}$ & \textbf{0.00}$_{\pm0.00}$ & 0.974$_{\pm0.019}$ & 33.2$_{\pm8.0}$\\
  & GlobMP (Ours) & \textbf{100.0} & \textbf{100.0}$_{\pm0.0}$ & \textbf{0.00}$_{\pm0.00}$ & 0.989$_{\pm0.062}$ & 6.1$_{\pm19.5}$ & \textbf{100.0} & \textbf{100.0}$_{\pm0.0}$ & \textbf{0.00}$_{\pm0.00}$ & 0.990$_{\pm0.062}$ & 4.7$_{\pm1.9}$ & \textbf{100.0} & \textbf{100.0}$_{\pm0.0}$ & \textbf{0.00}$_{\pm0.00}$ & 0.932$_{\pm0.116}$ & 0.4$_{\pm0.2}$\\
\midrule
12 & Diff.\ Policy$^{\dagger}$ & 0.0 & 0.0$_{\pm0.0}$ & 17.68$_{\pm4.12}$ & 0.999$_{\pm0.001}$ & 0.6 & 0.0 & 0.0$_{\pm0.0}$ & 17.68$_{\pm4.12}$ & 0.999$_{\pm0.001}$ & 0.6 & 0.0 & 0.0$_{\pm0.0}$ & 17.68$_{\pm4.12}$ & 0.999$_{\pm0.001}$ & 0.6\\
  & CD$^{\dagger}$ & 0.0 & 0.0$_{\pm0.0}$ & 8.90$_{\pm2.05}$ & 0.555$_{\pm0.045}$ & 0.5$_{\pm0.0}$ & 0.0 & 0.0$_{\pm0.0}$ & 8.90$_{\pm2.05}$ & 0.555$_{\pm0.045}$ & 1.1$_{\pm0.1}$ & 0.0 & 0.0$_{\pm0.0}$ & 8.90$_{\pm2.05}$ & 0.555$_{\pm0.045}$ & 0.5$_{\pm0.0}$\\
  & MMD-CBS$^{\dagger}$ & 0.0 & 0.0$_{\pm0.0}$ & --- & 0.961$_{\pm0.022}$ & 67.5$_{\pm4.9}$ & 0.0 & 0.0$_{\pm0.0}$ & --- & 0.961$_{\pm0.022}$ & 67.5$_{\pm4.9}$ & 0.0 & 0.0$_{\pm0.0}$ & --- & 0.961$_{\pm0.022}$ & 67.5$_{\pm4.9}$\\
  & MDOC$^{\dagger}$ & 0.0 & --- & --- & 0.000$_{\pm0.000}$ & 110.0$_{\pm8.6}$ & 0.0 & --- & --- & 0.000$_{\pm0.000}$ & 111.2$_{\pm8.5}$ & 0.0 & --- & --- & 0.000$_{\pm0.000}$ & 128.9$_{\pm10.2}$\\
  & PCD-LB & 0.0 & 0.0$_{\pm0.0}$ & 11.68$_{\pm1.53}$ & 0.259$_{\pm0.038}$ & 5.8$_{\pm0.8}$ & 0.0 & 0.0$_{\pm0.0}$ & 13.11$_{\pm1.52}$ & 0.265$_{\pm0.038}$ & 8.0$_{\pm1.0}$ & 0.0 & 0.0$_{\pm0.0}$ & 14.36$_{\pm1.63}$ & 0.279$_{\pm0.040}$ & 3.1$_{\pm0.4}$\\
  & PCD-SHD & \textbf{100.0} & 74.0$_{\pm15.1}$ & 0.27$_{\pm0.19}$ & 0.515$_{\pm0.047}$ & 7.1$_{\pm1.4}$ & \textbf{100.0} & 80.6$_{\pm13.5}$ & 0.18$_{\pm0.13}$ & 0.521$_{\pm0.043}$ & 2.9$_{\pm0.5}$ & \textbf{100.0} & 81.6$_{\pm12.9}$ & 0.17$_{\pm0.13}$ & 0.524$_{\pm0.040}$ & 2.3$_{\pm0.1}$\\
  & Final Proj. & \textbf{100.0} & 97.1$_{\pm6.5}$ & 0.05$_{\pm0.11}$ & 0.988$_{\pm0.009}$ & 6.4$_{\pm0.6}$ & \textbf{100.0} & 99.9$_{\pm0.3}$ & 0.00$_{\pm0.01}$ & 0.990$_{\pm0.007}$ & 6.5$_{\pm0.6}$ & \textbf{100.0} & \textbf{100.0}$_{\pm0.0}$ & \textbf{0.00}$_{\pm0.00}$ & \textbf{0.992}$_{\pm0.006}$ & 3.0$_{\pm0.8}$\\
  & DiRecT & \textbf{100.0} & \textbf{100.0}$_{\pm0.0}$ & \textbf{0.00}$_{\pm0.00}$ & 0.955$_{\pm0.023}$ & 45.2$_{\pm9.5}$ & \textbf{100.0} & \textbf{100.0}$_{\pm0.0}$ & \textbf{0.00}$_{\pm0.00}$ & 0.957$_{\pm0.022}$ & 45.8$_{\pm9.8}$ & \textbf{100.0} & \textbf{100.0}$_{\pm0.0}$ & \textbf{0.00}$_{\pm0.00}$ & 0.958$_{\pm0.021}$ & 50.0$_{\pm5.4}$\\
  & GlobMP (Ours) & \textbf{100.0} & \textbf{100.0}$_{\pm0.0}$ & \textbf{0.00}$_{\pm0.00}$ & \textbf{0.994}$_{\pm0.024}$ & 9.7$_{\pm31.5}$ & \textbf{100.0} & \textbf{100.0}$_{\pm0.0}$ & \textbf{0.00}$_{\pm0.00}$ & \textbf{0.996}$_{\pm0.016}$ & 7.0$_{\pm5.7}$ & \textbf{100.0} & \textbf{100.0}$_{\pm0.0}$ & \textbf{0.00}$_{\pm0.00}$ & 0.904$_{\pm0.133}$ & 0.6$_{\pm0.3}$\\
\midrule
15 & Diff.\ Policy$^{\dagger}$ & 0.0 & 0.0$_{\pm0.0}$ & 26.96$_{\pm5.94}$ & 0.999$_{\pm0.001}$ & 0.6 & 0.0 & 0.0$_{\pm0.0}$ & 26.96$_{\pm5.94}$ & 0.999$_{\pm0.001}$ & 0.6 & 0.0 & 0.0$_{\pm0.0}$ & 26.96$_{\pm5.94}$ & 0.999$_{\pm0.001}$ & 0.3\\
  & CD$^{\dagger}$ & 0.0 & 0.0$_{\pm0.0}$ & 15.53$_{\pm2.54}$ & 0.587$_{\pm0.059}$ & 1.2$_{\pm0.1}$ & 0.0 & 0.0$_{\pm0.0}$ & 15.53$_{\pm2.54}$ & 0.587$_{\pm0.059}$ & 0.9$_{\pm0.3}$ & 0.0 & 0.0$_{\pm0.0}$ & 15.53$_{\pm2.54}$ & 0.587$_{\pm0.059}$ & 1.2$_{\pm0.1}$\\
  & MMD-CBS$^{\dagger}$ & 0.0 & 0.0$_{\pm0.0}$ & --- & 0.989$_{\pm0.013}$ & 74.4$_{\pm9.3}$ & 0.0 & 0.0$_{\pm0.0}$ & --- & 0.989$_{\pm0.013}$ & 74.4$_{\pm9.3}$ & 0.0 & 0.0$_{\pm0.0}$ & --- & 0.989$_{\pm0.013}$ & 74.4$_{\pm9.3}$\\
  & MDOC$^{\dagger}$ & 0.0 & --- & --- & 0.000$_{\pm0.000}$ & 139.2$_{\pm10.6}$ & 0.0 & --- & --- & 0.000$_{\pm0.000}$ & 139.0$_{\pm10.6}$ & 0.0 & --- & --- & 0.000$_{\pm0.000}$ & 153.1$_{\pm12.1}$\\
  & PCD-LB & 0.0 & 0.0$_{\pm0.0}$ & 20.89$_{\pm2.48}$ & 0.280$_{\pm0.038}$ & 6.8$_{\pm0.9}$ & 0.0 & 0.0$_{\pm0.0}$ & 23.55$_{\pm2.73}$ & 0.284$_{\pm0.037}$ & 9.2$_{\pm0.7}$ & 0.0 & 0.0$_{\pm0.0}$ & 26.34$_{\pm2.90}$ & 0.300$_{\pm0.037}$ & 6.8$_{\pm0.4}$\\
  & PCD-SHD & 95.0 & 33.7$_{\pm20.2}$ & 0.99$_{\pm0.52}$ & 0.506$_{\pm0.028}$ & 9.5$_{\pm2.0}$ & \textbf{100.0} & 37.2$_{\pm16.9}$ & 0.79$_{\pm0.44}$ & 0.517$_{\pm0.026}$ & 2.7$_{\pm0.3}$ & \textbf{100.0} & 37.5$_{\pm18.1}$ & 0.73$_{\pm0.43}$ & 0.523$_{\pm0.027}$ & 2.7$_{\pm0.1}$\\
  & Final Proj. & \textbf{100.0} & 82.8$_{\pm26.1}$ & 0.60$_{\pm1.61}$ & 0.975$_{\pm0.013}$ & 5.2$_{\pm0.4}$ & \textbf{100.0} & 94.7$_{\pm13.7}$ & 0.13$_{\pm0.38}$ & 0.979$_{\pm0.011}$ & 8.5$_{\pm2.0}$ & \textbf{100.0} & 98.9$_{\pm3.4}$ & 0.02$_{\pm0.07}$ & \textbf{0.982}$_{\pm0.010}$ & 4.2$_{\pm0.8}$\\
  & DiRecT & \textbf{100.0} & \textbf{100.0}$_{\pm0.0}$ & \textbf{0.00}$_{\pm0.00}$ & 0.926$_{\pm0.028}$ & 61.9$_{\pm11.6}$ & \textbf{100.0} & \textbf{100.0}$_{\pm0.0}$ & \textbf{0.00}$_{\pm0.00}$ & 0.930$_{\pm0.025}$ & 68.1$_{\pm0.6}$ & \textbf{100.0} & \textbf{100.0}$_{\pm0.0}$ & \textbf{0.00}$_{\pm0.00}$ & 0.932$_{\pm0.024}$ & 46.4$_{\pm18.9}$\\
  & GlobMP (Ours) & 95.0 & 95.0$_{\pm21.8}$ & 1.65$_{\pm7.19}$ & \textbf{0.983}$_{\pm0.043}$ & 24.5$_{\pm61.2}$ & \textbf{100.0} & \textbf{100.0}$_{\pm0.0}$ & \textbf{0.00}$_{\pm0.00}$ & \textbf{0.987}$_{\pm0.033}$ & 14.6$_{\pm40.9}$ & \textbf{100.0} & \textbf{100.0}$_{\pm0.0}$ & \textbf{0.00}$_{\pm0.00}$ & 0.905$_{\pm0.134}$ & 1.1$_{\pm1.6}$\\
\midrule
18 & Diff.\ Policy$^{\dagger}$ & 0.0 & 0.0$_{\pm0.0}$ & 39.06$_{\pm6.29}$ & 0.999$_{\pm0.001}$ & 0.7 & 0.0 & 0.0$_{\pm0.0}$ & 39.06$_{\pm6.29}$ & 0.999$_{\pm0.001}$ & 0.7 & 0.0 & 0.0$_{\pm0.0}$ & 39.06$_{\pm6.29}$ & 0.999$_{\pm0.001}$ & 0.7\\
  & CD$^{\dagger}$ & 0.0 & 0.0$_{\pm0.0}$ & 22.94$_{\pm1.70}$ & 0.565$_{\pm0.049}$ & 1.2$_{\pm0.1}$ & 0.0 & 0.0$_{\pm0.0}$ & 22.94$_{\pm1.70}$ & 0.565$_{\pm0.049}$ & 0.6$_{\pm0.0}$ & 0.0 & 0.0$_{\pm0.0}$ & 22.94$_{\pm1.70}$ & 0.565$_{\pm0.049}$ & 1.2$_{\pm0.1}$\\
  & MMD-CBS$^{\dagger}$ & 0.0 & 0.0$_{\pm0.0}$ & --- & 0.996$_{\pm0.009}$ & 81.6$_{\pm17.1}$ & 0.0 & 0.0$_{\pm0.0}$ & --- & 0.996$_{\pm0.009}$ & 81.6$_{\pm17.1}$ & 0.0 & 0.0$_{\pm0.0}$ & --- & 0.996$_{\pm0.009}$ & 81.6$_{\pm17.1}$\\
  & MDOC$^{\dagger}$ & 0.0 & --- & --- & 0.000$_{\pm0.000}$ & 171.7$_{\pm13.7}$ & 0.0 & --- & --- & 0.000$_{\pm0.000}$ & 172.6$_{\pm14.8}$ & 0.0 & --- & --- & 0.000$_{\pm0.000}$ & 185.6$_{\pm17.7}$\\
  & PCD-LB & 0.0 & 0.0$_{\pm0.0}$ & 29.66$_{\pm2.84}$ & 0.271$_{\pm0.038}$ & 8.1$_{\pm0.9}$ & 0.0 & 0.0$_{\pm0.0}$ & 33.68$_{\pm2.86}$ & 0.276$_{\pm0.037}$ & 11.0$_{\pm1.2}$ & 0.0 & 0.0$_{\pm0.0}$ & 38.10$_{\pm2.86}$ & 0.289$_{\pm0.037}$ & 4.1$_{\pm0.5}$\\
  & PCD-SHD & 80.0 & 7.1$_{\pm7.0}$ & 2.21$_{\pm0.82}$ & 0.485$_{\pm0.034}$ & 5.2$_{\pm1.0}$ & 85.0 & 11.2$_{\pm9.5}$ & 1.83$_{\pm0.78}$ & 0.494$_{\pm0.035}$ & 3.3$_{\pm0.5}$ & 85.0 & 9.9$_{\pm8.8}$ & 1.69$_{\pm0.73}$ & 0.502$_{\pm0.036}$ & 3.0$_{\pm0.2}$\\
  & Final Proj. & 95.0 & 67.3$_{\pm36.5}$ & 1.60$_{\pm2.70}$ & 0.966$_{\pm0.015}$ & 13.6$_{\pm0.6}$ & \textbf{100.0} & 85.0$_{\pm28.5}$ & 0.45$_{\pm0.95}$ & 0.971$_{\pm0.012}$ & 10.2$_{\pm3.0}$ & \textbf{100.0} & 95.4$_{\pm11.3}$ & 0.08$_{\pm0.24}$ & \textbf{0.975}$_{\pm0.011}$ & 5.8$_{\pm0.8}$\\
  & DiRecT & \textbf{100.0} & 99.1$_{\pm4.1}$ & 0.00$_{\pm0.00}$ & 0.901$_{\pm0.025}$ & 38.0$_{\pm0.6}$ & \textbf{100.0} & 99.1$_{\pm3.7}$ & 0.00$_{\pm0.00}$ & 0.907$_{\pm0.024}$ & 38.2$_{\pm0.6}$ & \textbf{100.0} & 99.2$_{\pm3.4}$ & \textbf{0.00}$_{\pm0.00}$ & 0.911$_{\pm0.023}$ & 38.2$_{\pm0.6}$\\
  & GlobMP (Ours) & \textbf{100.0} & \textbf{100.0}$_{\pm0.0}$ & \textbf{0.00}$_{\pm0.00}$ & \textbf{0.977}$_{\pm0.047}$ & 27.5$_{\pm51.1}$ & \textbf{100.0} & \textbf{100.0}$_{\pm0.0}$ & \textbf{0.00}$_{\pm0.00}$ & \textbf{0.984}$_{\pm0.034}$ & 15.0$_{\pm25.0}$ & \textbf{100.0} & \textbf{100.0}$_{\pm0.0}$ & \textbf{0.00}$_{\pm0.00}$ & 0.876$_{\pm0.161}$ & 1.6$_{\pm1.9}$\\
\midrule
21 & Diff.\ Policy$^{\dagger}$ & 0.0 & 0.0$_{\pm0.0}$ & 60.02$_{\pm7.50}$ & 0.999$_{\pm0.001}$ & 0.7 & 0.0 & 0.0$_{\pm0.0}$ & 60.02$_{\pm7.50}$ & 0.999$_{\pm0.001}$ & 0.3 & 0.0 & 0.0$_{\pm0.0}$ & 60.02$_{\pm7.50}$ & 0.999$_{\pm0.001}$ & 0.7\\
  & CD$^{\dagger}$ & 0.0 & 0.0$_{\pm0.0}$ & 35.83$_{\pm3.71}$ & 0.587$_{\pm0.041}$ & 1.3$_{\pm0.1}$ & 0.0 & 0.0$_{\pm0.0}$ & 35.83$_{\pm3.71}$ & 0.587$_{\pm0.041}$ & 1.2$_{\pm0.0}$ & 0.0 & 0.0$_{\pm0.0}$ & 35.83$_{\pm3.71}$ & 0.587$_{\pm0.041}$ & 1.3$_{\pm0.1}$\\
  & MMD-CBS$^{\dagger}$ & 0.0 & 0.0$_{\pm0.0}$ & --- & 1.000$_{\pm0.000}$ & 82.4$_{\pm9.9}$ & 0.0 & 0.0$_{\pm0.0}$ & --- & 1.000$_{\pm0.000}$ & 82.4$_{\pm9.9}$ & 0.0 & 0.0$_{\pm0.0}$ & --- & 1.000$_{\pm0.000}$ & 82.4$_{\pm9.9}$\\
  & MDOC$^{\dagger}$ & 0.0 & --- & --- & 0.000$_{\pm0.000}$ & 193.5$_{\pm25.2}$ & 0.0 & --- & --- & 0.000$_{\pm0.000}$ & 192.8$_{\pm25.2}$ & 0.0 & --- & --- & 0.000$_{\pm0.000}$ & 217.6$_{\pm36.3}$\\
  & PCD-LB & 0.0 & 0.0$_{\pm0.0}$ & 45.06$_{\pm2.97}$ & 0.255$_{\pm0.026}$ & 9.9$_{\pm1.5}$ & 0.0 & 0.0$_{\pm0.0}$ & 51.43$_{\pm2.86}$ & 0.260$_{\pm0.022}$ & 13.2$_{\pm1.8}$ & 0.0 & 0.0$_{\pm0.0}$ & 57.66$_{\pm3.16}$ & 0.273$_{\pm0.020}$ & 9.3$_{\pm0.8}$\\
  & PCD-SHD & 10.0 & 0.2$_{\pm0.7}$ & 5.13$_{\pm1.31}$ & 0.467$_{\pm0.028}$ & 6.8$_{\pm1.5}$ & 15.0 & 0.2$_{\pm0.6}$ & 4.10$_{\pm1.05}$ & 0.480$_{\pm0.024}$ & 4.3$_{\pm0.6}$ & 10.0 & 0.4$_{\pm1.2}$ & 3.64$_{\pm0.92}$ & 0.491$_{\pm0.023}$ & 3.6$_{\pm0.2}$\\
  & Final Proj. & 65.0 & 19.2$_{\pm27.0}$ & 7.83$_{\pm6.74}$ & 0.940$_{\pm0.020}$ & 10.8$_{\pm2.2}$ & 90.0 & 45.0$_{\pm35.2}$ & 3.88$_{\pm4.47}$ & 0.948$_{\pm0.018}$ & 14.5$_{\pm4.3}$ & \textbf{100.0} & 62.4$_{\pm34.9}$ & 1.70$_{\pm2.11}$ & \textbf{0.954}$_{\pm0.016}$ & 7.9$_{\pm0.8}$\\
  & DiRecT & \textbf{100.0} & \textbf{99.8}$_{\pm0.7}$ & \textbf{0.00}$_{\pm0.00}$ & 0.857$_{\pm0.033}$ & 99.1$_{\pm31.9}$ & \textbf{100.0} & \textbf{99.9}$_{\pm0.3}$ & \textbf{0.00}$_{\pm0.00}$ & 0.867$_{\pm0.031}$ & 91.9$_{\pm33.8}$ & \textbf{100.0} & 99.9$_{\pm0.3}$ & \textbf{0.00}$_{\pm0.00}$ & 0.873$_{\pm0.028}$ & 85.6$_{\pm34.6}$\\
  & GlobMP (Ours) & 75.0 & 75.0$_{\pm43.3}$ & 17.15$_{\pm30.27}$ & \textbf{0.972}$_{\pm0.055}$ & 72.7$_{\pm81.0}$ & 95.0 & 95.0$_{\pm21.8}$ & 4.50$_{\pm19.62}$ & \textbf{0.972}$_{\pm0.050}$ & 31.2$_{\pm46.6}$ & \textbf{100.0} & \textbf{100.0}$_{\pm0.0}$ & \textbf{0.00}$_{\pm0.00}$ & 0.839$_{\pm0.185}$ & 3.1$_{\pm3.8}$\\
\bottomrule
\end{tabular}}
\end{center}
\end{table}

\begin{table}[p]
\caption{Multi-robot planning performance on \texttt{Basic} for $N_a\in\{3,6,\dots,21\}$ under the Low, Medium and High velocity caps. \emph{SR} and \emph{CS} measure trajectory feasibility, \emph{Coll} reports inter-robot collisions, \emph{DA} adherence to the environment-specific motion pattern, and \emph{Time} the planning time per configuration: for \emph{GlobMP (Ours)} the median over the configurations with the standard deviation as subscript, for the baselines the mean, so the two are not directly comparable. $^{\dagger}$ marks a method that applies no velocity limit, so its three blocks report the same run; MMD-CBS also returns one plan per configuration rather than $64$ candidates. A dash under \emph{Coll} marks a count that is not commensurable with the others, as MDOC and MMD-CBS count colliding time steps where every other row counts agent pairs with any collision; MDOC's \emph{CS} is omitted because its plans that do not cover every agent were counted as feasible. MMD-CBS plans under $60$\,s per configuration and \emph{GlobMP (Ours)} under $300$\,s per team.}
\label{tab:hard-basic}
\begin{center}\tiny\setlength{\tabcolsep}{2.2pt}
\resizebox{\linewidth}{!}{%
\begin{tabular}{ll ccccc ccccc ccccc}
\toprule
& & \multicolumn{5}{c}{Low ($v_{\max}=0.477$)} & \multicolumn{5}{c}{Medium ($v_{\max}=0.567$)} & \multicolumn{5}{c}{High ($v_{\max}=0.675$)}\\
\cmidrule(lr){3-7}\cmidrule(lr){8-12}\cmidrule(lr){13-17}
$N_a$ & Method & SR$\uparrow$ & CS$\uparrow$ & Coll$\downarrow$ & DA$\uparrow$ & Time$\downarrow$ & SR$\uparrow$ & CS$\uparrow$ & Coll$\downarrow$ & DA$\uparrow$ & Time$\downarrow$ & SR$\uparrow$ & CS$\uparrow$ & Coll$\downarrow$ & DA$\uparrow$ & Time$\downarrow$\\
\midrule
3 & Diff.\ Policy$^{\dagger}$ & 30.0 & 13.0$_{\pm29.6}$ & 0.35$_{\pm0.64}$ & 0.998$_{\pm0.001}$ & 0.3 & 30.0 & 13.0$_{\pm29.6}$ & 0.35$_{\pm0.64}$ & 0.998$_{\pm0.001}$ & 0.3 & 30.0 & 13.0$_{\pm29.6}$ & 0.35$_{\pm0.64}$ & 0.998$_{\pm0.001}$ & 0.3\\
  & CD$^{\dagger}$ & 80.0 & 0.0$_{\pm0.0}$ & 0.05$_{\pm0.08}$ & 0.611$_{\pm0.138}$ & 1.1$_{\pm0.0}$ & 80.0 & 0.0$_{\pm0.0}$ & 0.05$_{\pm0.08}$ & 0.611$_{\pm0.138}$ & 0.4$_{\pm0.0}$ & 80.0 & 0.4$_{\pm1.0}$ & 0.05$_{\pm0.08}$ & 0.611$_{\pm0.138}$ & 1.1$_{\pm0.1}$\\
  & MMD-CBS$^{\dagger}$ & 70.0 & 35.0$_{\pm47.7}$ & --- & 0.895$_{\pm0.113}$ & 2.5$_{\pm1.7}$ & 70.0 & 50.0$_{\pm50.0}$ & --- & 0.895$_{\pm0.113}$ & 2.5$_{\pm1.7}$ & 70.0 & 55.0$_{\pm49.7}$ & --- & 0.895$_{\pm0.113}$ & 2.5$_{\pm1.7}$\\
  & MDOC$^{\dagger}$ & 10.0 & --- & --- & 0.090$_{\pm0.269}$ & 33.4$_{\pm4.5}$ & 10.0 & --- & --- & 0.090$_{\pm0.269}$ & 33.6$_{\pm4.4}$ & 10.0 & --- & --- & 0.090$_{\pm0.269}$ & 33.5$_{\pm4.6}$\\
  & PCD-LB & 40.0 & 5.2$_{\pm10.1}$ & 0.01$_{\pm0.03}$ & 0.361$_{\pm0.114}$ & 5.2$_{\pm1.6}$ & 35.0 & 2.2$_{\pm3.6}$ & 0.01$_{\pm0.03}$ & 0.349$_{\pm0.112}$ & 4.4$_{\pm0.6}$ & 20.0 & 2.1$_{\pm6.1}$ & 0.01$_{\pm0.03}$ & 0.336$_{\pm0.106}$ & 3.6$_{\pm0.4}$\\
  & PCD-SHD & 70.0 & 61.2$_{\pm43.9}$ & 0.06$_{\pm0.14}$ & 0.847$_{\pm0.120}$ & 3.1$_{\pm0.2}$ & 70.0 & 61.1$_{\pm43.9}$ & 0.06$_{\pm0.14}$ & 0.848$_{\pm0.118}$ & 1.5$_{\pm0.1}$ & 70.0 & 60.7$_{\pm43.8}$ & 0.05$_{\pm0.12}$ & 0.849$_{\pm0.118}$ & 3.0$_{\pm0.1}$\\
  & Final Proj. & \textbf{100.0} & 95.1$_{\pm21.5}$ & \textbf{0.00}$_{\pm0.00}$ & \textbf{0.964}$_{\pm0.050}$ & 4.0$_{\pm0.7}$ & \textbf{100.0} & 95.1$_{\pm21.5}$ & \textbf{0.00}$_{\pm0.00}$ & \textbf{0.969}$_{\pm0.043}$ & 4.0$_{\pm0.8}$ & \textbf{100.0} & 95.1$_{\pm21.5}$ & \textbf{0.00}$_{\pm0.00}$ & 0.973$_{\pm0.037}$ & 2.2$_{\pm0.8}$\\
  & DiRecT & \textbf{100.0} & 99.8$_{\pm1.0}$ & \textbf{0.00}$_{\pm0.00}$ & 0.925$_{\pm0.098}$ & 26.3$_{\pm6.4}$ & \textbf{100.0} & 95.6$_{\pm19.1}$ & \textbf{0.00}$_{\pm0.00}$ & 0.930$_{\pm0.093}$ & 26.8$_{\pm6.5}$ & \textbf{100.0} & 95.5$_{\pm19.4}$ & \textbf{0.00}$_{\pm0.00}$ & 0.931$_{\pm0.092}$ & 27.7$_{\pm7.3}$\\
  & GlobMP (Ours) & \textbf{100.0} & \textbf{100.0}$_{\pm0.0}$ & \textbf{0.00}$_{\pm0.00}$ & 0.878$_{\pm0.169}$ & 2.1$_{\pm2.0}$ & \textbf{100.0} & \textbf{100.0}$_{\pm0.0}$ & \textbf{0.00}$_{\pm0.00}$ & 0.851$_{\pm0.184}$ & 0.2$_{\pm2.0}$ & \textbf{100.0} & \textbf{100.0}$_{\pm0.0}$ & \textbf{0.00}$_{\pm0.00}$ & \textbf{0.987}$_{\pm0.026}$ & 0.0$_{\pm0.1}$\\
\midrule
6 & Diff.\ Policy$^{\dagger}$ & 0.0 & 0.0$_{\pm0.0}$ & 3.63$_{\pm1.90}$ & 0.998$_{\pm0.001}$ & 0.3 & 0.0 & 0.0$_{\pm0.0}$ & 3.63$_{\pm1.90}$ & 0.998$_{\pm0.001}$ & 0.4 & 0.0 & 0.0$_{\pm0.0}$ & 3.63$_{\pm1.90}$ & 0.998$_{\pm0.001}$ & 0.5\\
  & CD$^{\dagger}$ & 45.0 & 0.0$_{\pm0.0}$ & 1.14$_{\pm0.58}$ & 0.633$_{\pm0.079}$ & 1.1$_{\pm0.1}$ & 45.0 & 0.0$_{\pm0.0}$ & 1.14$_{\pm0.58}$ & 0.633$_{\pm0.079}$ & 1.2$_{\pm0.0}$ & 45.0 & 0.0$_{\pm0.0}$ & 1.14$_{\pm0.58}$ & 0.633$_{\pm0.079}$ & 1.1$_{\pm0.1}$\\
  & MMD-CBS$^{\dagger}$ & 45.0 & 0.0$_{\pm0.0}$ & --- & 0.797$_{\pm0.102}$ & 9.3$_{\pm9.6}$ & 45.0 & 5.0$_{\pm21.8}$ & --- & 0.797$_{\pm0.102}$ & 9.3$_{\pm9.6}$ & 45.0 & 10.0$_{\pm30.0}$ & --- & 0.797$_{\pm0.102}$ & 9.3$_{\pm9.6}$\\
  & MDOC$^{\dagger}$ & 0.0 & --- & --- & 0.000$_{\pm0.000}$ & 90.6$_{\pm21.6}$ & 0.0 & --- & --- & 0.000$_{\pm0.000}$ & 69.7$_{\pm5.3}$ & 0.0 & --- & --- & 0.000$_{\pm0.000}$ & 77.3$_{\pm14.4}$\\
  & PCD-LB & 0.0 & 0.0$_{\pm0.0}$ & 0.69$_{\pm0.50}$ & 0.360$_{\pm0.064}$ & 6.4$_{\pm0.9}$ & 0.0 & 0.0$_{\pm0.0}$ & 0.72$_{\pm0.44}$ & 0.347$_{\pm0.056}$ & 4.9$_{\pm0.5}$ & 0.0 & 0.0$_{\pm0.0}$ & 0.78$_{\pm0.43}$ & 0.336$_{\pm0.051}$ & 1.8$_{\pm0.1}$\\
  & PCD-SHD & 45.0 & 22.2$_{\pm33.5}$ & 0.46$_{\pm0.35}$ & 0.773$_{\pm0.084}$ & 3.3$_{\pm0.2}$ & 50.0 & 21.4$_{\pm33.1}$ & 0.45$_{\pm0.33}$ & 0.777$_{\pm0.082}$ & 3.2$_{\pm0.1}$ & 60.0 & 20.8$_{\pm31.9}$ & 0.44$_{\pm0.31}$ & 0.780$_{\pm0.079}$ & 3.1$_{\pm0.1}$\\
  & Final Proj. & 90.0 & 85.6$_{\pm32.3}$ & 0.09$_{\pm0.33}$ & \textbf{0.959}$_{\pm0.031}$ & 4.3$_{\pm0.8}$ & 90.0 & 88.6$_{\pm29.9}$ & 0.02$_{\pm0.06}$ & \textbf{0.965}$_{\pm0.026}$ & 4.4$_{\pm0.7}$ & 30.0 & 9.1$_{\pm22.6}$ & 0.08$_{\pm0.11}$ & \textbf{0.983}$_{\pm0.011}$ & 4.2$_{\pm0.6}$\\
  & DiRecT & \textbf{100.0} & \textbf{100.0}$_{\pm0.0}$ & \textbf{0.00}$_{\pm0.00}$ & 0.919$_{\pm0.059}$ & 30.6$_{\pm7.9}$ & \textbf{100.0} & \textbf{100.0}$_{\pm0.0}$ & \textbf{0.00}$_{\pm0.00}$ & 0.921$_{\pm0.058}$ & 29.4$_{\pm7.8}$ & \textbf{100.0} & \textbf{100.0}$_{\pm0.0}$ & \textbf{0.00}$_{\pm0.00}$ & 0.922$_{\pm0.057}$ & 30.9$_{\pm8.8}$\\
  & GlobMP (Ours) & \textbf{100.0} & \textbf{100.0}$_{\pm0.0}$ & \textbf{0.00}$_{\pm0.00}$ & 0.932$_{\pm0.141}$ & 4.2$_{\pm18.1}$ & \textbf{100.0} & \textbf{100.0}$_{\pm0.0}$ & \textbf{0.00}$_{\pm0.00}$ & 0.934$_{\pm0.137}$ & 4.2$_{\pm5.0}$ & \textbf{100.0} & \textbf{100.0}$_{\pm0.0}$ & \textbf{0.00}$_{\pm0.00}$ & 0.980$_{\pm0.041}$ & 0.2$_{\pm0.1}$\\
\midrule
9 & Diff.\ Policy$^{\dagger}$ & 0.0 & 0.0$_{\pm0.0}$ & 6.58$_{\pm2.27}$ & 0.999$_{\pm0.001}$ & 0.3 & 0.0 & 0.0$_{\pm0.0}$ & 6.58$_{\pm2.27}$ & 0.999$_{\pm0.001}$ & 0.3 & 0.0 & 0.0$_{\pm0.0}$ & 6.58$_{\pm2.27}$ & 0.999$_{\pm0.001}$ & 0.3\\
  & CD$^{\dagger}$ & 0.0 & 0.0$_{\pm0.0}$ & 2.86$_{\pm1.16}$ & 0.630$_{\pm0.065}$ & 0.5$_{\pm0.1}$ & 0.0 & 0.0$_{\pm0.0}$ & 2.86$_{\pm1.16}$ & 0.630$_{\pm0.065}$ & 1.2$_{\pm0.0}$ & 0.0 & 0.0$_{\pm0.0}$ & 2.86$_{\pm1.16}$ & 0.630$_{\pm0.065}$ & 1.1$_{\pm0.1}$\\
  & MMD-CBS$^{\dagger}$ & 0.0 & 0.0$_{\pm0.0}$ & --- & 0.756$_{\pm0.082}$ & 15.1$_{\pm23.5}$ & 0.0 & 0.0$_{\pm0.0}$ & --- & 0.756$_{\pm0.082}$ & 15.1$_{\pm23.5}$ & 0.0 & 0.0$_{\pm0.0}$ & --- & 0.756$_{\pm0.082}$ & 15.1$_{\pm23.5}$\\
  & MDOC$^{\dagger}$ & 0.0 & --- & --- & 0.000$_{\pm0.000}$ & 92.2$_{\pm17.8}$ & 0.0 & --- & --- & 0.000$_{\pm0.000}$ & 81.2$_{\pm17.7}$ & 0.0 & --- & --- & 0.000$_{\pm0.000}$ & 74.6$_{\pm5.1}$\\
  & PCD-LB & 0.0 & 0.0$_{\pm0.0}$ & 2.85$_{\pm1.07}$ & 0.365$_{\pm0.072}$ & 7.4$_{\pm0.5}$ & 0.0 & 0.0$_{\pm0.0}$ & 3.06$_{\pm1.03}$ & 0.365$_{\pm0.065}$ & 5.6$_{\pm0.4}$ & 0.0 & 0.0$_{\pm0.0}$ & 3.27$_{\pm0.98}$ & 0.361$_{\pm0.061}$ & 4.4$_{\pm0.3}$\\
  & PCD-SHD & 5.0 & 0.3$_{\pm1.4}$ & 0.81$_{\pm0.55}$ & 0.735$_{\pm0.061}$ & 3.6$_{\pm0.1}$ & 5.0 & 0.5$_{\pm2.0}$ & 0.81$_{\pm0.53}$ & 0.738$_{\pm0.060}$ & 3.2$_{\pm0.6}$ & 5.0 & 0.4$_{\pm1.7}$ & 0.84$_{\pm0.50}$ & 0.740$_{\pm0.058}$ & 2.8$_{\pm0.8}$\\
  & Final Proj. & 85.0 & 76.3$_{\pm38.4}$ & 0.11$_{\pm0.24}$ & 0.949$_{\pm0.027}$ & 5.7$_{\pm0.7}$ & 85.0 & 78.4$_{\pm38.7}$ & 0.06$_{\pm0.21}$ & 0.957$_{\pm0.023}$ & 3.1$_{\pm0.7}$ & 85.0 & 79.0$_{\pm38.6}$ & 0.05$_{\pm0.20}$ & 0.962$_{\pm0.020}$ & 5.8$_{\pm0.7}$\\
  & DiRecT & \textbf{100.0} & 99.0$_{\pm4.4}$ & 0.01$_{\pm0.04}$ & 0.903$_{\pm0.050}$ & 46.6$_{\pm13.9}$ & 95.0 & 93.9$_{\pm22.0}$ & 0.01$_{\pm0.04}$ & 0.902$_{\pm0.051}$ & 47.1$_{\pm14.1}$ & 95.0 & 94.0$_{\pm21.9}$ & 0.01$_{\pm0.04}$ & 0.907$_{\pm0.046}$ & 42.6$_{\pm12.5}$\\
  & GlobMP (Ours) & \textbf{100.0} & \textbf{100.0}$_{\pm0.0}$ & \textbf{0.00}$_{\pm0.00}$ & \textbf{0.969}$_{\pm0.100}$ & 4.4$_{\pm8.8}$ & \textbf{100.0} & \textbf{100.0}$_{\pm0.0}$ & \textbf{0.00}$_{\pm0.00}$ & \textbf{0.969}$_{\pm0.099}$ & 4.3$_{\pm1.2}$ & \textbf{100.0} & \textbf{100.0}$_{\pm0.0}$ & \textbf{0.00}$_{\pm0.00}$ & \textbf{0.972}$_{\pm0.055}$ & 0.3$_{\pm0.1}$\\
\midrule
12 & Diff.\ Policy$^{\dagger}$ & 0.0 & 0.0$_{\pm0.0}$ & 13.12$_{\pm3.83}$ & 0.999$_{\pm0.001}$ & 0.3 & 0.0 & 0.0$_{\pm0.0}$ & 13.12$_{\pm3.83}$ & 0.999$_{\pm0.001}$ & 0.3 & 0.0 & 0.0$_{\pm0.0}$ & 13.12$_{\pm3.83}$ & 0.999$_{\pm0.001}$ & 0.3\\
  & CD$^{\dagger}$ & 0.0 & 0.0$_{\pm0.0}$ & 7.58$_{\pm1.54}$ & 0.638$_{\pm0.033}$ & 1.2$_{\pm0.1}$ & 0.0 & 0.0$_{\pm0.0}$ & 7.58$_{\pm1.54}$ & 0.638$_{\pm0.033}$ & 0.5$_{\pm0.1}$ & 0.0 & 0.0$_{\pm0.0}$ & 7.58$_{\pm1.54}$ & 0.638$_{\pm0.033}$ & 1.2$_{\pm0.0}$\\
  & MMD-CBS$^{\dagger}$ & 0.0 & 0.0$_{\pm0.0}$ & --- & 0.797$_{\pm0.027}$ & 9.5$_{\pm18.8}$ & 0.0 & 0.0$_{\pm0.0}$ & --- & 0.797$_{\pm0.027}$ & 9.5$_{\pm18.8}$ & 0.0 & 0.0$_{\pm0.0}$ & --- & 0.797$_{\pm0.027}$ & 9.5$_{\pm18.8}$\\
  & MDOC$^{\dagger}$ & 0.0 & --- & --- & 0.000$_{\pm0.000}$ & 104.7$_{\pm8.5}$ & 0.0 & --- & --- & 0.000$_{\pm0.000}$ & 113.0$_{\pm6.5}$ & 0.0 & --- & --- & 0.000$_{\pm0.000}$ & 109.8$_{\pm6.6}$\\
  & PCD-LB & 0.0 & 0.0$_{\pm0.0}$ & 7.89$_{\pm1.38}$ & 0.340$_{\pm0.044}$ & 8.6$_{\pm0.8}$ & 0.0 & 0.0$_{\pm0.0}$ & 8.44$_{\pm1.23}$ & 0.336$_{\pm0.044}$ & 3.1$_{\pm0.3}$ & 0.0 & 0.0$_{\pm0.0}$ & 8.87$_{\pm1.18}$ & 0.334$_{\pm0.043}$ & 4.9$_{\pm0.4}$\\
  & PCD-SHD & 0.0 & 0.0$_{\pm0.0}$ & 1.98$_{\pm0.70}$ & 0.713$_{\pm0.052}$ & 3.9$_{\pm0.2}$ & 0.0 & 0.0$_{\pm0.0}$ & 1.93$_{\pm0.69}$ & 0.718$_{\pm0.051}$ & 3.7$_{\pm0.1}$ & 0.0 & 0.0$_{\pm0.0}$ & 1.97$_{\pm0.65}$ & 0.721$_{\pm0.049}$ & 3.8$_{\pm0.1}$\\
  & Final Proj. & 60.0 & 38.0$_{\pm42.5}$ & 0.72$_{\pm1.00}$ & 0.960$_{\pm0.021}$ & 6.6$_{\pm0.7}$ & 60.0 & 40.6$_{\pm42.5}$ & 0.36$_{\pm0.53}$ & 0.966$_{\pm0.018}$ & 4.6$_{\pm1.5}$ & 60.0 & 43.6$_{\pm43.7}$ & 0.28$_{\pm0.38}$ & \textbf{0.971}$_{\pm0.015}$ & 6.5$_{\pm0.7}$\\
  & DiRecT & \textbf{100.0} & 99.2$_{\pm2.1}$ & 0.00$_{\pm0.01}$ & 0.905$_{\pm0.047}$ & 41.5$_{\pm0.9}$ & 95.0 & 82.7$_{\pm35.4}$ & 0.00$_{\pm0.01}$ & 0.906$_{\pm0.046}$ & 41.4$_{\pm0.9}$ & 85.0 & 79.8$_{\pm39.7}$ & 0.04$_{\pm0.15}$ & 0.915$_{\pm0.042}$ & 42.4$_{\pm1.4}$\\
  & GlobMP (Ours) & \textbf{100.0} & \textbf{100.0}$_{\pm0.0}$ & \textbf{0.00}$_{\pm0.00}$ & \textbf{0.983}$_{\pm0.084}$ & 5.2$_{\pm13.2}$ & \textbf{100.0} & \textbf{100.0}$_{\pm0.0}$ & \textbf{0.00}$_{\pm0.00}$ & \textbf{0.984}$_{\pm0.078}$ & 4.7$_{\pm1.0}$ & \textbf{100.0} & \textbf{100.0}$_{\pm0.0}$ & \textbf{0.00}$_{\pm0.00}$ & 0.968$_{\pm0.065}$ & 0.4$_{\pm0.2}$\\
\midrule
15 & Diff.\ Policy$^{\dagger}$ & 0.0 & 0.0$_{\pm0.0}$ & 22.25$_{\pm4.65}$ & 0.999$_{\pm0.001}$ & 0.3 & 0.0 & 0.0$_{\pm0.0}$ & 22.25$_{\pm4.65}$ & 0.999$_{\pm0.001}$ & 0.3 & 0.0 & 0.0$_{\pm0.0}$ & 22.25$_{\pm4.65}$ & 0.999$_{\pm0.001}$ & 0.4\\
  & CD$^{\dagger}$ & 0.0 & 0.0$_{\pm0.0}$ & 13.94$_{\pm2.28}$ & 0.625$_{\pm0.038}$ & 1.2$_{\pm0.1}$ & 0.0 & 0.0$_{\pm0.0}$ & 13.94$_{\pm2.28}$ & 0.625$_{\pm0.038}$ & 1.2$_{\pm0.1}$ & 0.0 & 0.0$_{\pm0.0}$ & 13.94$_{\pm2.28}$ & 0.625$_{\pm0.038}$ & 1.2$_{\pm0.0}$\\
  & MMD-CBS$^{\dagger}$ & 0.0 & 0.0$_{\pm0.0}$ & --- & --- & 6.6$_{\pm3.5}$ & 0.0 & 0.0$_{\pm0.0}$ & --- & --- & 6.6$_{\pm3.5}$ & 0.0 & 0.0$_{\pm0.0}$ & --- & --- & 6.6$_{\pm3.5}$\\
  & MDOC$^{\dagger}$ & 0.0 & --- & --- & 0.000$_{\pm0.000}$ & 150.5$_{\pm11.2}$ & 0.0 & --- & --- & 0.000$_{\pm0.000}$ & 165.2$_{\pm13.5}$ & 0.0 & --- & --- & 0.000$_{\pm0.000}$ & 160.4$_{\pm31.5}$\\
  & PCD-LB & 0.0 & 0.0$_{\pm0.0}$ & 14.53$_{\pm2.15}$ & 0.324$_{\pm0.037}$ & 9.2$_{\pm1.2}$ & 0.0 & 0.0$_{\pm0.0}$ & 15.91$_{\pm2.00}$ & 0.320$_{\pm0.032}$ & 6.8$_{\pm0.6}$ & 0.0 & 0.0$_{\pm0.0}$ & 17.16$_{\pm1.99}$ & 0.321$_{\pm0.026}$ & 5.3$_{\pm0.4}$\\
  & PCD-SHD & 0.0 & 0.0$_{\pm0.0}$ & 3.93$_{\pm1.17}$ & 0.674$_{\pm0.039}$ & 4.3$_{\pm0.2}$ & 0.0 & 0.0$_{\pm0.0}$ & 3.89$_{\pm1.14}$ & 0.679$_{\pm0.038}$ & 4.1$_{\pm0.1}$ & 0.0 & 0.0$_{\pm0.0}$ & 3.84$_{\pm1.12}$ & 0.682$_{\pm0.037}$ & 4.2$_{\pm0.2}$\\
  & Final Proj. & 75.0 & 44.5$_{\pm40.5}$ & 1.35$_{\pm1.67}$ & 0.946$_{\pm0.027}$ & 5.2$_{\pm0.5}$ & 85.0 & 54.4$_{\pm41.8}$ & 0.79$_{\pm1.19}$ & 0.954$_{\pm0.023}$ & 9.6$_{\pm0.6}$ & 90.0 & 60.1$_{\pm41.6}$ & 0.50$_{\pm0.81}$ & \textbf{0.960}$_{\pm0.020}$ & 6.0$_{\pm2.1}$\\
  & DiRecT & \textbf{100.0} & 97.7$_{\pm4.9}$ & 0.03$_{\pm0.05}$ & 0.887$_{\pm0.041}$ & 64.3$_{\pm2.3}$ & \textbf{100.0} & 92.0$_{\pm20.0}$ & 0.02$_{\pm0.05}$ & 0.890$_{\pm0.039}$ & 71.4$_{\pm11.9}$ & 95.0 & 91.3$_{\pm22.3}$ & 0.05$_{\pm0.12}$ & 0.895$_{\pm0.038}$ & 72.6$_{\pm13.9}$\\
  & GlobMP (Ours) & \textbf{100.0} & \textbf{100.0}$_{\pm0.0}$ & \textbf{0.00}$_{\pm0.00}$ & \textbf{1.000}$_{\pm0.001}$ & 5.7$_{\pm29.2}$ & \textbf{100.0} & \textbf{100.0}$_{\pm0.0}$ & \textbf{0.00}$_{\pm0.00}$ & \textbf{1.000}$_{\pm0.000}$ & 5.3$_{\pm11.8}$ & \textbf{100.0} & \textbf{100.0}$_{\pm0.0}$ & \textbf{0.00}$_{\pm0.00}$ & 0.959$_{\pm0.082}$ & 0.9$_{\pm0.3}$\\
\midrule
18 & Diff.\ Policy$^{\dagger}$ & 0.0 & 0.0$_{\pm0.0}$ & 28.87$_{\pm4.75}$ & 0.999$_{\pm0.001}$ & 0.4 & 0.0 & 0.0$_{\pm0.0}$ & 28.87$_{\pm4.75}$ & 0.999$_{\pm0.001}$ & 0.5 & 0.0 & 0.0$_{\pm0.0}$ & 28.87$_{\pm4.75}$ & 0.999$_{\pm0.001}$ & 0.4\\
  & CD$^{\dagger}$ & 0.0 & 0.0$_{\pm0.0}$ & 21.88$_{\pm2.51}$ & 0.642$_{\pm0.035}$ & 0.6$_{\pm0.2}$ & 0.0 & 0.0$_{\pm0.0}$ & 21.88$_{\pm2.51}$ & 0.642$_{\pm0.035}$ & 1.3$_{\pm0.1}$ & 0.0 & 0.0$_{\pm0.0}$ & 21.88$_{\pm2.51}$ & 0.642$_{\pm0.035}$ & 1.3$_{\pm0.0}$\\
  & MMD-CBS$^{\dagger}$ & 0.0 & 0.0$_{\pm0.0}$ & --- & 0.829$_{\pm0.000}$ & 9.2$_{\pm15.5}$ & 0.0 & 0.0$_{\pm0.0}$ & --- & 0.829$_{\pm0.000}$ & 9.2$_{\pm15.5}$ & 0.0 & 0.0$_{\pm0.0}$ & --- & 0.829$_{\pm0.000}$ & 9.2$_{\pm15.5}$\\
  & MDOC$^{\dagger}$ & 0.0 & --- & --- & 0.000$_{\pm0.000}$ & 149.1$_{\pm10.6}$ & 0.0 & --- & --- & 0.000$_{\pm0.000}$ & 143.3$_{\pm7.8}$ & 0.0 & --- & --- & 0.000$_{\pm0.000}$ & 142.9$_{\pm7.8}$\\
  & PCD-LB & 0.0 & 0.0$_{\pm0.0}$ & 22.51$_{\pm2.98}$ & 0.333$_{\pm0.034}$ & 9.3$_{\pm0.9}$ & 0.0 & 0.0$_{\pm0.0}$ & 24.71$_{\pm2.90}$ & 0.337$_{\pm0.031}$ & 3.6$_{\pm0.3}$ & 0.0 & 0.0$_{\pm0.0}$ & 27.06$_{\pm2.84}$ & 0.342$_{\pm0.032}$ & 5.5$_{\pm0.4}$\\
  & PCD-SHD & 0.0 & 0.0$_{\pm0.0}$ & 6.51$_{\pm1.58}$ & 0.683$_{\pm0.043}$ & 3.9$_{\pm0.9}$ & 0.0 & 0.0$_{\pm0.0}$ & 6.28$_{\pm1.53}$ & 0.688$_{\pm0.043}$ & 4.2$_{\pm0.1}$ & 0.0 & 0.0$_{\pm0.0}$ & 6.19$_{\pm1.47}$ & 0.691$_{\pm0.043}$ & 4.2$_{\pm0.1}$\\
  & Final Proj. & 65.0 & 19.4$_{\pm30.9}$ & 2.66$_{\pm2.15}$ & 0.941$_{\pm0.029}$ & 12.1$_{\pm0.7}$ & 65.0 & 22.1$_{\pm32.5}$ & 1.84$_{\pm1.67}$ & 0.949$_{\pm0.025}$ & 12.0$_{\pm0.6}$ & 65.0 & 24.5$_{\pm32.3}$ & 1.48$_{\pm1.35}$ & \textbf{0.956}$_{\pm0.021}$ & 7.3$_{\pm2.0}$\\
  & DiRecT & \textbf{100.0} & 92.6$_{\pm11.6}$ & 0.06$_{\pm0.11}$ & 0.881$_{\pm0.045}$ & 93.6$_{\pm18.4}$ & 90.0 & 78.4$_{\pm36.7}$ & 0.21$_{\pm0.39}$ & 0.888$_{\pm0.044}$ & 92.6$_{\pm16.9}$ & 90.0 & 77.4$_{\pm37.0}$ & 0.31$_{\pm0.60}$ & 0.893$_{\pm0.042}$ & 84.1$_{\pm7.1}$\\
  & GlobMP (Ours) & \textbf{100.0} & \textbf{100.0}$_{\pm0.0}$ & \textbf{0.00}$_{\pm0.00}$ & \textbf{1.000}$_{\pm0.000}$ & 6.4$_{\pm22.9}$ & \textbf{100.0} & \textbf{100.0}$_{\pm0.0}$ & \textbf{0.00}$_{\pm0.00}$ & \textbf{1.000}$_{\pm0.000}$ & 5.3$_{\pm4.2}$ & \textbf{100.0} & \textbf{100.0}$_{\pm0.0}$ & \textbf{0.00}$_{\pm0.00}$ & 0.953$_{\pm0.079}$ & 1.0$_{\pm0.5}$\\
\midrule
21 & Diff.\ Policy$^{\dagger}$ & 0.0 & 0.0$_{\pm0.0}$ & 42.00$_{\pm5.66}$ & 0.999$_{\pm0.001}$ & 0.4 & 0.0 & 0.0$_{\pm0.0}$ & 42.00$_{\pm5.66}$ & 0.999$_{\pm0.001}$ & 0.4 & 0.0 & 0.0$_{\pm0.0}$ & 42.00$_{\pm5.66}$ & 0.999$_{\pm0.001}$ & 0.8\\
  & CD$^{\dagger}$ & 0.0 & 0.0$_{\pm0.0}$ & 31.60$_{\pm2.41}$ & 0.636$_{\pm0.033}$ & 1.3$_{\pm0.0}$ & 0.0 & 0.0$_{\pm0.0}$ & 31.60$_{\pm2.41}$ & 0.636$_{\pm0.033}$ & 1.3$_{\pm0.0}$ & 0.0 & 0.0$_{\pm0.0}$ & 31.60$_{\pm2.41}$ & 0.636$_{\pm0.033}$ & 1.3$_{\pm0.1}$\\
  & MMD-CBS$^{\dagger}$ & 0.0 & 0.0$_{\pm0.0}$ & --- & --- & 4.4$_{\pm2.8}$ & 0.0 & 0.0$_{\pm0.0}$ & --- & --- & 4.4$_{\pm2.8}$ & 0.0 & 0.0$_{\pm0.0}$ & --- & --- & 4.4$_{\pm2.8}$\\
  & MDOC$^{\dagger}$ & 0.0 & --- & --- & 0.000$_{\pm0.000}$ & 165.5$_{\pm10.3}$ & 0.0 & --- & --- & 0.000$_{\pm0.000}$ & 162.1$_{\pm8.8}$ & 0.0 & --- & --- & 0.000$_{\pm0.000}$ & 162.7$_{\pm8.7}$\\
  & PCD-LB & 0.0 & 0.0$_{\pm0.0}$ & 32.93$_{\pm3.29}$ & 0.337$_{\pm0.031}$ & 11.2$_{\pm1.0}$ & 0.0 & 0.0$_{\pm0.0}$ & 36.75$_{\pm3.00}$ & 0.335$_{\pm0.026}$ & 8.2$_{\pm0.6}$ & 0.0 & 0.0$_{\pm0.0}$ & 40.37$_{\pm3.01}$ & 0.338$_{\pm0.025}$ & 6.4$_{\pm0.3}$\\
  & PCD-SHD & 0.0 & 0.0$_{\pm0.0}$ & 8.70$_{\pm1.44}$ & 0.662$_{\pm0.030}$ & 5.0$_{\pm0.3}$ & 0.0 & 0.0$_{\pm0.0}$ & 8.63$_{\pm1.36}$ & 0.668$_{\pm0.028}$ & 4.6$_{\pm0.1}$ & 0.0 & 0.0$_{\pm0.0}$ & 8.37$_{\pm1.38}$ & 0.672$_{\pm0.026}$ & 4.6$_{\pm0.1}$\\
  & Final Proj. & 25.0 & 7.7$_{\pm20.9}$ & 4.65$_{\pm4.26}$ & 0.946$_{\pm0.026}$ & 15.1$_{\pm0.6}$ & 40.0 & 8.7$_{\pm21.3}$ & 3.21$_{\pm2.92}$ & 0.954$_{\pm0.023}$ & 11.8$_{\pm3.7}$ & 35.0 & 9.5$_{\pm21.0}$ & 2.42$_{\pm2.03}$ & \textbf{0.960}$_{\pm0.020}$ & 15.1$_{\pm0.7}$\\
  & DiRecT & \textbf{100.0} & 93.9$_{\pm8.7}$ & \textbf{0.06}$_{\pm0.10}$ & 0.875$_{\pm0.039}$ & 109.4$_{\pm9.7}$ & 95.0 & 60.2$_{\pm38.2}$ & 0.11$_{\pm0.13}$ & 0.880$_{\pm0.038}$ & 101.1$_{\pm21.0}$ & 85.0 & 54.4$_{\pm43.5}$ & 0.64$_{\pm0.70}$ & 0.894$_{\pm0.040}$ & 78.0$_{\pm0.8}$\\
  & GlobMP (Ours) & 95.0 & \textbf{95.0}$_{\pm21.8}$ & 2.35$_{\pm10.24}$ & \textbf{1.000}$_{\pm0.003}$ & 9.7$_{\pm47.9}$ & \textbf{100.0} & \textbf{100.0}$_{\pm0.0}$ & \textbf{0.00}$_{\pm0.00}$ & \textbf{1.000}$_{\pm0.002}$ & 6.3$_{\pm3.8}$ & \textbf{100.0} & \textbf{100.0}$_{\pm0.0}$ & \textbf{0.00}$_{\pm0.00}$ & 0.945$_{\pm0.098}$ & 1.5$_{\pm2.5}$\\
\bottomrule
\end{tabular}}
\end{center}
\end{table}

\begin{table}[p]
\caption{Multi-robot planning performance on \texttt{Dense} for $N_a\in\{3,6,\dots,21\}$ under the Low, Medium and High velocity caps. \emph{SR} and \emph{CS} measure trajectory feasibility, \emph{Coll} reports inter-robot collisions, \emph{DA} adherence to the environment-specific motion pattern, and \emph{Time} the planning time per configuration: for \emph{GlobMP (Ours)} the median over the configurations with the standard deviation as subscript, for the baselines the mean, so the two are not directly comparable. $^{\dagger}$ marks a method that applies no velocity limit, so its three blocks report the same run; MMD-CBS also returns one plan per configuration rather than $64$ candidates. A dash under \emph{Coll} marks a count that is not commensurable with the others, as MDOC and MMD-CBS count colliding time steps where every other row counts agent pairs with any collision; MDOC's \emph{CS} is omitted because its plans that do not cover every agent were counted as feasible. MMD-CBS plans under $60$\,s per configuration and \emph{GlobMP (Ours)} under $300$\,s per team.}
\label{tab:hard-dense}
\begin{center}\tiny\setlength{\tabcolsep}{2.2pt}
\resizebox{\linewidth}{!}{%
\begin{tabular}{ll ccccc ccccc ccccc}
\toprule
& & \multicolumn{5}{c}{Low ($v_{\max}=0.565$)} & \multicolumn{5}{c}{Medium ($v_{\max}=0.617$)} & \multicolumn{5}{c}{High ($v_{\max}=0.675$)}\\
\cmidrule(lr){3-7}\cmidrule(lr){8-12}\cmidrule(lr){13-17}
$N_a$ & Method & SR$\uparrow$ & CS$\uparrow$ & Coll$\downarrow$ & DA$\uparrow$ & Time$\downarrow$ & SR$\uparrow$ & CS$\uparrow$ & Coll$\downarrow$ & DA$\uparrow$ & Time$\downarrow$ & SR$\uparrow$ & CS$\uparrow$ & Coll$\downarrow$ & DA$\uparrow$ & Time$\downarrow$\\
\midrule
3 & Diff.\ Policy$^{\dagger}$ & 0.0 & 0.0$_{\pm0.0}$ & 0.56$_{\pm0.63}$ & 0.999$_{\pm0.001}$ & 0.4 & 0.0 & 0.0$_{\pm0.0}$ & 0.56$_{\pm0.63}$ & 0.999$_{\pm0.001}$ & 0.4 & 0.0 & 0.0$_{\pm0.0}$ & 0.56$_{\pm0.63}$ & 0.999$_{\pm0.001}$ & 0.5\\
  & CD$^{\dagger}$ & 5.0 & 0.0$_{\pm0.0}$ & 0.14$_{\pm0.20}$ & 0.674$_{\pm0.096}$ & 1.1$_{\pm0.1}$ & 5.0 & 0.0$_{\pm0.0}$ & 0.14$_{\pm0.20}$ & 0.674$_{\pm0.096}$ & 1.1$_{\pm0.0}$ & 5.0 & 0.0$_{\pm0.0}$ & 0.14$_{\pm0.20}$ & 0.674$_{\pm0.096}$ & 1.1$_{\pm0.0}$\\
  & MMD-CBS$^{\dagger}$ & 0.0 & 0.0$_{\pm0.0}$ & --- & --- & 0.8$_{\pm0.5}$ & 0.0 & 0.0$_{\pm0.0}$ & --- & --- & 0.8$_{\pm0.5}$ & 0.0 & 0.0$_{\pm0.0}$ & --- & --- & 0.8$_{\pm0.5}$\\
  & MDOC$^{\dagger}$ & 0.0 & --- & --- & 0.000$_{\pm0.000}$ & 34.0$_{\pm5.3}$ & 0.0 & --- & --- & 0.000$_{\pm0.000}$ & 45.5$_{\pm8.4}$ & 0.0 & --- & --- & 0.000$_{\pm0.000}$ & 50.9$_{\pm17.6}$\\
  & PCD-LB & 10.0 & 5.2$_{\pm19.5}$ & 0.01$_{\pm0.02}$ & 0.386$_{\pm0.134}$ & 5.1$_{\pm1.5}$ & 10.0 & 4.8$_{\pm20.1}$ & 0.00$_{\pm0.01}$ & 0.381$_{\pm0.133}$ & 2.1$_{\pm0.6}$ & 10.0 & 4.7$_{\pm19.7}$ & 0.00$_{\pm0.01}$ & 0.376$_{\pm0.132}$ & 4.0$_{\pm0.8}$\\
  & PCD-SHD & 5.0 & 0.1$_{\pm0.3}$ & 0.11$_{\pm0.20}$ & 0.798$_{\pm0.100}$ & 3.1$_{\pm0.1}$ & 0.0 & 0.0$_{\pm0.0}$ & 0.11$_{\pm0.20}$ & 0.796$_{\pm0.099}$ & 3.1$_{\pm0.1}$ & 0.0 & 0.0$_{\pm0.0}$ & 0.12$_{\pm0.20}$ & 0.794$_{\pm0.099}$ & 3.0$_{\pm0.1}$\\
  & Final Proj. & 45.0 & 19.2$_{\pm35.9}$ & 0.11$_{\pm0.27}$ & \textbf{0.983}$_{\pm0.021}$ & 5.1$_{\pm0.7}$ & 45.0 & 19.2$_{\pm35.9}$ & 0.10$_{\pm0.27}$ & \textbf{0.984}$_{\pm0.019}$ & 5.1$_{\pm0.8}$ & 45.0 & 19.2$_{\pm35.9}$ & 0.09$_{\pm0.24}$ & 0.986$_{\pm0.017}$ & 5.1$_{\pm0.8}$\\
  & DiRecT & 50.0 & 22.2$_{\pm36.6}$ & 0.04$_{\pm0.11}$ & 0.928$_{\pm0.073}$ & 36.2$_{\pm9.4}$ & 50.0 & 22.0$_{\pm36.8}$ & 0.04$_{\pm0.10}$ & 0.928$_{\pm0.073}$ & 37.1$_{\pm11.0}$ & 50.0 & 21.9$_{\pm36.7}$ & 0.02$_{\pm0.09}$ & 0.928$_{\pm0.073}$ & 37.9$_{\pm10.4}$\\
  & GlobMP (Ours) & \textbf{100.0} & \textbf{100.0}$_{\pm0.0}$ & \textbf{0.00}$_{\pm0.00}$ & 0.980$_{\pm0.087}$ & 4.0$_{\pm8.0}$ & \textbf{100.0} & \textbf{100.0}$_{\pm0.0}$ & \textbf{0.00}$_{\pm0.00}$ & 0.864$_{\pm0.153}$ & 0.2$_{\pm5.7}$ & \textbf{100.0} & \textbf{100.0}$_{\pm0.0}$ & \textbf{0.00}$_{\pm0.00}$ & \textbf{0.989}$_{\pm0.024}$ & 0.2$_{\pm0.1}$\\
\midrule
6 & Diff.\ Policy$^{\dagger}$ & 0.0 & 0.0$_{\pm0.0}$ & 1.42$_{\pm0.99}$ & 0.999$_{\pm0.001}$ & 0.4 & 0.0 & 0.0$_{\pm0.0}$ & 1.42$_{\pm0.99}$ & 0.999$_{\pm0.001}$ & 0.4 & 0.0 & 0.0$_{\pm0.0}$ & 1.42$_{\pm0.99}$ & 0.999$_{\pm0.001}$ & 0.4\\
  & CD$^{\dagger}$ & 0.0 & 0.0$_{\pm0.0}$ & 0.80$_{\pm0.48}$ & 0.700$_{\pm0.074}$ & 1.1$_{\pm0.1}$ & 0.0 & 0.0$_{\pm0.0}$ & 0.80$_{\pm0.48}$ & 0.700$_{\pm0.074}$ & 0.5$_{\pm0.1}$ & 0.0 & 0.0$_{\pm0.0}$ & 0.80$_{\pm0.48}$ & 0.700$_{\pm0.074}$ & 1.1$_{\pm0.1}$\\
  & MMD-CBS$^{\dagger}$ & 0.0 & 0.0$_{\pm0.0}$ & --- & --- & 1.1$_{\pm0.8}$ & 0.0 & 0.0$_{\pm0.0}$ & --- & --- & 1.1$_{\pm0.8}$ & 0.0 & 0.0$_{\pm0.0}$ & --- & --- & 1.1$_{\pm0.8}$\\
  & MDOC$^{\dagger}$ & 0.0 & --- & --- & 0.000$_{\pm0.000}$ & 66.0$_{\pm9.4}$ & 0.0 & --- & --- & 0.000$_{\pm0.000}$ & 76.7$_{\pm6.1}$ & 0.0 & --- & --- & 0.000$_{\pm0.000}$ & 63.7$_{\pm15.4}$\\
  & PCD-LB & 0.0 & 0.0$_{\pm0.0}$ & 0.53$_{\pm0.39}$ & 0.378$_{\pm0.084}$ & 5.3$_{\pm1.1}$ & 0.0 & 0.0$_{\pm0.0}$ & 0.57$_{\pm0.41}$ & 0.373$_{\pm0.082}$ & 4.6$_{\pm0.8}$ & 0.0 & 0.0$_{\pm0.0}$ & 0.60$_{\pm0.41}$ & 0.372$_{\pm0.086}$ & 4.0$_{\pm0.6}$\\
  & PCD-SHD & 0.0 & 0.0$_{\pm0.0}$ & 1.04$_{\pm0.65}$ & 0.788$_{\pm0.062}$ & 3.3$_{\pm0.1}$ & 0.0 & 0.0$_{\pm0.0}$ & 1.07$_{\pm0.64}$ & 0.786$_{\pm0.061}$ & 3.2$_{\pm0.1}$ & 0.0 & 0.0$_{\pm0.0}$ & 1.06$_{\pm0.63}$ & 0.785$_{\pm0.061}$ & 3.2$_{\pm0.1}$\\
  & Final Proj. & 30.0 & 8.3$_{\pm20.9}$ & 0.16$_{\pm0.24}$ & 0.980$_{\pm0.013}$ & 4.2$_{\pm1.4}$ & 30.0 & 9.0$_{\pm22.4}$ & 0.10$_{\pm0.14}$ & \textbf{0.981}$_{\pm0.012}$ & 5.6$_{\pm0.8}$ & 30.0 & 9.1$_{\pm22.6}$ & 0.08$_{\pm0.11}$ & 0.983$_{\pm0.011}$ & 5.6$_{\pm0.5}$\\
  & DiRecT & 35.0 & 13.0$_{\pm26.8}$ & 0.04$_{\pm0.08}$ & 0.922$_{\pm0.046}$ & 44.0$_{\pm12.8}$ & 35.0 & 13.2$_{\pm26.9}$ & 0.03$_{\pm0.07}$ & 0.923$_{\pm0.046}$ & 47.7$_{\pm14.0}$ & 35.0 & 13.4$_{\pm27.2}$ & 0.02$_{\pm0.07}$ & 0.923$_{\pm0.045}$ & 47.4$_{\pm14.6}$\\
  & GlobMP (Ours) & \textbf{100.0} & \textbf{100.0}$_{\pm0.0}$ & \textbf{0.00}$_{\pm0.00}$ & \textbf{1.000}$_{\pm0.000}$ & 4.2$_{\pm5.9}$ & \textbf{100.0} & \textbf{100.0}$_{\pm0.0}$ & \textbf{0.00}$_{\pm0.00}$ & 0.915$_{\pm0.143}$ & 4.1$_{\pm2.0}$ & \textbf{100.0} & \textbf{100.0}$_{\pm0.0}$ & \textbf{0.00}$_{\pm0.00}$ & \textbf{0.985}$_{\pm0.051}$ & 0.2$_{\pm0.1}$\\
\midrule
9 & Diff.\ Policy$^{\dagger}$ & 0.0 & 0.0$_{\pm0.0}$ & 4.18$_{\pm1.43}$ & 0.999$_{\pm0.001}$ & 0.4 & 0.0 & 0.0$_{\pm0.0}$ & 4.18$_{\pm1.43}$ & 0.999$_{\pm0.001}$ & 0.4 & 0.0 & 0.0$_{\pm0.0}$ & 4.18$_{\pm1.43}$ & 0.999$_{\pm0.001}$ & 0.4\\
  & CD$^{\dagger}$ & 0.0 & 0.0$_{\pm0.0}$ & 3.39$_{\pm0.77}$ & 0.712$_{\pm0.056}$ & 1.2$_{\pm0.1}$ & 0.0 & 0.0$_{\pm0.0}$ & 3.39$_{\pm0.77}$ & 0.712$_{\pm0.056}$ & 1.2$_{\pm0.0}$ & 0.0 & 0.0$_{\pm0.0}$ & 3.39$_{\pm0.77}$ & 0.712$_{\pm0.056}$ & 1.2$_{\pm0.1}$\\
  & MMD-CBS$^{\dagger}$ & 0.0 & 0.0$_{\pm0.0}$ & --- & --- & 1.0$_{\pm0.5}$ & 0.0 & 0.0$_{\pm0.0}$ & --- & --- & 1.0$_{\pm0.5}$ & 0.0 & 0.0$_{\pm0.0}$ & --- & --- & 1.0$_{\pm0.5}$\\
  & MDOC$^{\dagger}$ & 0.0 & --- & --- & 0.000$_{\pm0.000}$ & 82.2$_{\pm12.5}$ & 0.0 & --- & --- & 0.000$_{\pm0.000}$ & 73.9$_{\pm5.1}$ & 0.0 & --- & --- & 0.000$_{\pm0.000}$ & 77.0$_{\pm6.6}$\\
  & PCD-LB & 0.0 & 0.0$_{\pm0.0}$ & 2.95$_{\pm1.17}$ & 0.403$_{\pm0.093}$ & 5.2$_{\pm1.7}$ & 0.0 & 0.0$_{\pm0.0}$ & 3.12$_{\pm1.15}$ & 0.400$_{\pm0.089}$ & 5.5$_{\pm1.0}$ & 0.0 & 0.0$_{\pm0.0}$ & 3.30$_{\pm1.13}$ & 0.399$_{\pm0.084}$ & 4.8$_{\pm0.7}$\\
  & PCD-SHD & 0.0 & 0.0$_{\pm0.0}$ & 2.08$_{\pm0.66}$ & 0.799$_{\pm0.047}$ & 3.6$_{\pm0.1}$ & 0.0 & 0.0$_{\pm0.0}$ & 2.12$_{\pm0.62}$ & 0.798$_{\pm0.046}$ & 1.7$_{\pm0.1}$ & 0.0 & 0.0$_{\pm0.0}$ & 2.17$_{\pm0.62}$ & 0.796$_{\pm0.045}$ & 3.5$_{\pm0.1}$\\
  & Final Proj. & 15.0 & 8.8$_{\pm25.7}$ & 0.46$_{\pm0.63}$ & 0.984$_{\pm0.010}$ & 8.2$_{\pm0.7}$ & 15.0 & 8.8$_{\pm25.7}$ & 0.36$_{\pm0.52}$ & \textbf{0.985}$_{\pm0.009}$ & 5.6$_{\pm2.1}$ & 15.0 & 8.8$_{\pm25.7}$ & 0.34$_{\pm0.50}$ & \textbf{0.987}$_{\pm0.009}$ & 8.4$_{\pm0.7}$\\
  & DiRecT & 25.0 & 11.1$_{\pm26.4}$ & 0.15$_{\pm0.18}$ & 0.939$_{\pm0.034}$ & 74.7$_{\pm22.0}$ & 25.0 & 11.1$_{\pm26.7}$ & 0.12$_{\pm0.14}$ & 0.940$_{\pm0.033}$ & 58.3$_{\pm13.7}$ & 20.0 & 11.2$_{\pm26.9}$ & 0.09$_{\pm0.13}$ & 0.940$_{\pm0.033}$ & 53.8$_{\pm0.9}$\\
  & GlobMP (Ours) & \textbf{100.0} & \textbf{100.0}$_{\pm0.0}$ & \textbf{0.00}$_{\pm0.00}$ & \textbf{1.000}$_{\pm0.000}$ & 4.3$_{\pm15.9}$ & \textbf{100.0} & \textbf{100.0}$_{\pm0.0}$ & \textbf{0.00}$_{\pm0.00}$ & 0.938$_{\pm0.129}$ & 4.3$_{\pm12.5}$ & \textbf{100.0} & \textbf{100.0}$_{\pm0.0}$ & \textbf{0.00}$_{\pm0.00}$ & 0.984$_{\pm0.037}$ & 0.3$_{\pm0.1}$\\
\midrule
12 & Diff.\ Policy$^{\dagger}$ & 0.0 & 0.0$_{\pm0.0}$ & 7.86$_{\pm2.57}$ & 0.999$_{\pm0.001}$ & 0.4 & 0.0 & 0.0$_{\pm0.0}$ & 7.86$_{\pm2.57}$ & 0.999$_{\pm0.001}$ & 0.4 & 0.0 & 0.0$_{\pm0.0}$ & 7.86$_{\pm2.57}$ & 0.999$_{\pm0.001}$ & 0.4\\
  & CD$^{\dagger}$ & 0.0 & 0.0$_{\pm0.0}$ & 6.72$_{\pm1.77}$ & 0.745$_{\pm0.049}$ & 0.6$_{\pm0.0}$ & 0.0 & 0.0$_{\pm0.0}$ & 6.72$_{\pm1.77}$ & 0.745$_{\pm0.049}$ & 1.2$_{\pm0.0}$ & 0.0 & 0.0$_{\pm0.0}$ & 6.72$_{\pm1.77}$ & 0.745$_{\pm0.049}$ & 1.2$_{\pm0.1}$\\
  & MMD-CBS$^{\dagger}$ & 0.0 & 0.0$_{\pm0.0}$ & --- & --- & 0.9$_{\pm0.4}$ & 0.0 & 0.0$_{\pm0.0}$ & --- & --- & 0.9$_{\pm0.4}$ & 0.0 & 0.0$_{\pm0.0}$ & --- & --- & 0.9$_{\pm0.4}$\\
  & MDOC$^{\dagger}$ & 0.0 & --- & --- & 0.000$_{\pm0.000}$ & 111.9$_{\pm6.7}$ & 0.0 & --- & --- & 0.000$_{\pm0.000}$ & 108.8$_{\pm6.3}$ & 0.0 & --- & --- & 0.000$_{\pm0.000}$ & 114.3$_{\pm6.4}$\\
  & PCD-LB & 0.0 & 0.0$_{\pm0.0}$ & 6.12$_{\pm2.04}$ & 0.405$_{\pm0.071}$ & 7.2$_{\pm1.0}$ & 0.0 & 0.0$_{\pm0.0}$ & 6.57$_{\pm2.01}$ & 0.404$_{\pm0.073}$ & 5.9$_{\pm0.7}$ & 0.0 & 0.0$_{\pm0.0}$ & 7.05$_{\pm2.03}$ & 0.405$_{\pm0.074}$ & 5.1$_{\pm0.6}$\\
  & PCD-SHD & 0.0 & 0.0$_{\pm0.0}$ & 4.12$_{\pm1.60}$ & 0.815$_{\pm0.039}$ & 3.8$_{\pm0.1}$ & 0.0 & 0.0$_{\pm0.0}$ & 4.12$_{\pm1.62}$ & 0.814$_{\pm0.037}$ & 3.7$_{\pm0.1}$ & 0.0 & 0.0$_{\pm0.0}$ & 4.26$_{\pm1.62}$ & 0.814$_{\pm0.037}$ & 3.7$_{\pm0.1}$\\
  & Final Proj. & 0.0 & 0.0$_{\pm0.0}$ & 1.10$_{\pm1.28}$ & 0.985$_{\pm0.009}$ & 10.1$_{\pm0.7}$ & 0.0 & 0.0$_{\pm0.0}$ & 1.00$_{\pm1.19}$ & 0.986$_{\pm0.008}$ & 6.3$_{\pm1.8}$ & 0.0 & 0.0$_{\pm0.0}$ & 0.92$_{\pm1.18}$ & \textbf{0.987}$_{\pm0.007}$ & 10.1$_{\pm0.7}$\\
  & DiRecT & 0.0 & 0.0$_{\pm0.0}$ & 0.49$_{\pm0.41}$ & 0.940$_{\pm0.029}$ & 68.9$_{\pm1.6}$ & 0.0 & 0.0$_{\pm0.0}$ & 0.40$_{\pm0.39}$ & 0.941$_{\pm0.029}$ & 69.3$_{\pm5.5}$ & 0.0 & 0.0$_{\pm0.0}$ & 0.33$_{\pm0.39}$ & 0.942$_{\pm0.029}$ & 74.8$_{\pm12.7}$\\
  & GlobMP (Ours) & \textbf{100.0} & \textbf{100.0}$_{\pm0.0}$ & \textbf{0.00}$_{\pm0.00}$ & \textbf{1.000}$_{\pm0.000}$ & 4.3$_{\pm18.6}$ & \textbf{100.0} & \textbf{100.0}$_{\pm0.0}$ & \textbf{0.00}$_{\pm0.00}$ & \textbf{0.992}$_{\pm0.054}$ & 4.3$_{\pm11.5}$ & \textbf{100.0} & \textbf{100.0}$_{\pm0.0}$ & \textbf{0.00}$_{\pm0.00}$ & 0.972$_{\pm0.073}$ & 0.3$_{\pm0.1}$\\
\midrule
15 & Diff.\ Policy$^{\dagger}$ & 0.0 & 0.0$_{\pm0.0}$ & 12.70$_{\pm3.56}$ & 0.999$_{\pm0.001}$ & 0.4 & 0.0 & 0.0$_{\pm0.0}$ & 12.70$_{\pm3.56}$ & 0.999$_{\pm0.001}$ & 0.4 & 0.0 & 0.0$_{\pm0.0}$ & 12.70$_{\pm3.56}$ & 0.999$_{\pm0.001}$ & 0.4\\
  & CD$^{\dagger}$ & 0.0 & 0.0$_{\pm0.0}$ & 13.31$_{\pm3.25}$ & 0.712$_{\pm0.044}$ & 1.2$_{\pm0.1}$ & 0.0 & 0.0$_{\pm0.0}$ & 13.31$_{\pm3.25}$ & 0.712$_{\pm0.044}$ & 1.2$_{\pm0.0}$ & 0.0 & 0.0$_{\pm0.0}$ & 13.31$_{\pm3.25}$ & 0.712$_{\pm0.044}$ & 1.2$_{\pm0.1}$\\
  & MMD-CBS$^{\dagger}$ & 0.0 & 0.0$_{\pm0.0}$ & --- & --- & 0.9$_{\pm0.4}$ & 0.0 & 0.0$_{\pm0.0}$ & --- & --- & 0.9$_{\pm0.4}$ & 0.0 & 0.0$_{\pm0.0}$ & --- & --- & 0.9$_{\pm0.4}$\\
  & MDOC$^{\dagger}$ & 0.0 & --- & --- & 0.000$_{\pm0.000}$ & 155.2$_{\pm7.7}$ & 0.0 & --- & --- & 0.000$_{\pm0.000}$ & 171.4$_{\pm25.9}$ & 0.0 & --- & --- & 0.000$_{\pm0.000}$ & 133.6$_{\pm10.9}$\\
  & PCD-LB & 0.0 & 0.0$_{\pm0.0}$ & 13.05$_{\pm2.30}$ & 0.368$_{\pm0.041}$ & 8.1$_{\pm1.1}$ & 0.0 & 0.0$_{\pm0.0}$ & 14.03$_{\pm2.31}$ & 0.367$_{\pm0.039}$ & 6.4$_{\pm1.2}$ & 0.0 & 0.0$_{\pm0.0}$ & 15.16$_{\pm2.36}$ & 0.366$_{\pm0.039}$ & 5.6$_{\pm0.7}$\\
  & PCD-SHD & 0.0 & 0.0$_{\pm0.0}$ & 7.59$_{\pm2.13}$ & 0.799$_{\pm0.032}$ & 4.1$_{\pm0.1}$ & 0.0 & 0.0$_{\pm0.0}$ & 7.65$_{\pm2.17}$ & 0.798$_{\pm0.032}$ & 4.0$_{\pm0.1}$ & 0.0 & 0.0$_{\pm0.0}$ & 7.81$_{\pm2.24}$ & 0.797$_{\pm0.033}$ & 4.0$_{\pm0.1}$\\
  & Final Proj. & 0.0 & 0.0$_{\pm0.0}$ & 1.55$_{\pm1.38}$ & 0.979$_{\pm0.007}$ & 12.8$_{\pm3.3}$ & 0.0 & 0.0$_{\pm0.0}$ & 1.23$_{\pm1.21}$ & 0.980$_{\pm0.007}$ & 14.7$_{\pm0.7}$ & 0.0 & 0.0$_{\pm0.0}$ & 1.06$_{\pm1.02}$ & \textbf{0.982}$_{\pm0.006}$ & 14.6$_{\pm0.7}$\\
  & DiRecT & 0.0 & 0.0$_{\pm0.0}$ & 0.75$_{\pm0.68}$ & 0.922$_{\pm0.023}$ & 113.1$_{\pm16.3}$ & 0.0 & 0.0$_{\pm0.0}$ & 0.53$_{\pm0.48}$ & 0.923$_{\pm0.023}$ & 114.6$_{\pm19.5}$ & 0.0 & 0.0$_{\pm0.0}$ & 0.36$_{\pm0.36}$ & 0.923$_{\pm0.023}$ & 102.4$_{\pm7.6}$\\
  & GlobMP (Ours) & \textbf{100.0} & \textbf{100.0}$_{\pm0.0}$ & \textbf{0.00}$_{\pm0.00}$ & \textbf{1.000}$_{\pm0.000}$ & 4.6$_{\pm33.3}$ & \textbf{100.0} & \textbf{100.0}$_{\pm0.0}$ & \textbf{0.00}$_{\pm0.00}$ & \textbf{0.982}$_{\pm0.069}$ & 4.5$_{\pm15.1}$ & \textbf{100.0} & \textbf{100.0}$_{\pm0.0}$ & \textbf{0.00}$_{\pm0.00}$ & 0.974$_{\pm0.071}$ & 0.4$_{\pm0.2}$\\
\midrule
18 & Diff.\ Policy$^{\dagger}$ & 0.0 & 0.0$_{\pm0.0}$ & 16.75$_{\pm4.10}$ & 0.999$_{\pm0.000}$ & 0.4 & 0.0 & 0.0$_{\pm0.0}$ & 16.75$_{\pm4.10}$ & 0.999$_{\pm0.000}$ & 0.4 & 0.0 & 0.0$_{\pm0.0}$ & 16.75$_{\pm4.10}$ & 0.999$_{\pm0.000}$ & 0.3\\
  & CD$^{\dagger}$ & 0.0 & 0.0$_{\pm0.0}$ & 19.75$_{\pm3.14}$ & 0.702$_{\pm0.038}$ & 1.3$_{\pm0.0}$ & 0.0 & 0.0$_{\pm0.0}$ & 19.75$_{\pm3.14}$ & 0.702$_{\pm0.038}$ & 0.6$_{\pm0.0}$ & 0.0 & 0.0$_{\pm0.0}$ & 19.75$_{\pm3.14}$ & 0.702$_{\pm0.038}$ & 1.3$_{\pm0.1}$\\
  & MMD-CBS$^{\dagger}$ & 0.0 & 0.0$_{\pm0.0}$ & --- & --- & 0.8$_{\pm0.4}$ & 0.0 & 0.0$_{\pm0.0}$ & --- & --- & 0.8$_{\pm0.4}$ & 0.0 & 0.0$_{\pm0.0}$ & --- & --- & 0.8$_{\pm0.4}$\\
  & MDOC$^{\dagger}$ & 0.0 & --- & --- & 0.000$_{\pm0.000}$ & 147.2$_{\pm8.4}$ & 0.0 & --- & --- & 0.000$_{\pm0.000}$ & 143.6$_{\pm7.6}$ & 0.0 & --- & --- & 0.000$_{\pm0.000}$ & 143.0$_{\pm8.6}$\\
  & PCD-LB & 0.0 & 0.0$_{\pm0.0}$ & 20.86$_{\pm4.27}$ & 0.353$_{\pm0.037}$ & 8.0$_{\pm1.0}$ & 0.0 & 0.0$_{\pm0.0}$ & 22.51$_{\pm4.22}$ & 0.356$_{\pm0.035}$ & 6.7$_{\pm0.7}$ & 0.0 & 0.0$_{\pm0.0}$ & 24.22$_{\pm4.20}$ & 0.359$_{\pm0.034}$ & 5.7$_{\pm0.5}$\\
  & PCD-SHD & 0.0 & 0.0$_{\pm0.0}$ & 10.38$_{\pm2.46}$ & 0.792$_{\pm0.031}$ & 4.3$_{\pm0.2}$ & 0.0 & 0.0$_{\pm0.0}$ & 10.52$_{\pm2.40}$ & 0.790$_{\pm0.031}$ & 2.0$_{\pm0.1}$ & 0.0 & 0.0$_{\pm0.0}$ & 10.74$_{\pm2.40}$ & 0.789$_{\pm0.030}$ & 4.1$_{\pm0.1}$\\
  & Final Proj. & 0.0 & 0.0$_{\pm0.0}$ & 2.43$_{\pm1.79}$ & 0.984$_{\pm0.007}$ & 14.1$_{\pm3.8}$ & 0.0 & 0.0$_{\pm0.0}$ & 2.01$_{\pm1.52}$ & \textbf{0.985}$_{\pm0.006}$ & 17.3$_{\pm0.7}$ & 0.0 & 0.0$_{\pm0.0}$ & 1.75$_{\pm1.40}$ & \textbf{0.986}$_{\pm0.006}$ & 9.2$_{\pm0.7}$\\
  & DiRecT & 5.0 & 0.1$_{\pm0.3}$ & 0.79$_{\pm0.60}$ & 0.936$_{\pm0.022}$ & 124.6$_{\pm10.6}$ & 5.0 & 0.1$_{\pm0.3}$ & 0.65$_{\pm0.52}$ & 0.937$_{\pm0.022}$ & 107.4$_{\pm20.5}$ & 0.0 & 0.0$_{\pm0.0}$ & 0.54$_{\pm0.52}$ & 0.937$_{\pm0.022}$ & 89.8$_{\pm1.3}$\\
  & GlobMP (Ours) & \textbf{100.0} & \textbf{100.0}$_{\pm0.0}$ & \textbf{0.00}$_{\pm0.00}$ & \textbf{1.000}$_{\pm0.000}$ & 5.0$_{\pm20.7}$ & \textbf{100.0} & \textbf{100.0}$_{\pm0.0}$ & \textbf{0.00}$_{\pm0.00}$ & 0.972$_{\pm0.090}$ & 4.8$_{\pm10.4}$ & \textbf{100.0} & \textbf{100.0}$_{\pm0.0}$ & \textbf{0.00}$_{\pm0.00}$ & 0.972$_{\pm0.069}$ & 0.5$_{\pm0.2}$\\
\midrule
21 & Diff.\ Policy$^{\dagger}$ & 0.0 & 0.0$_{\pm0.0}$ & 24.25$_{\pm6.19}$ & 0.999$_{\pm0.000}$ & 0.4 & 0.0 & 0.0$_{\pm0.0}$ & 24.25$_{\pm6.19}$ & 0.999$_{\pm0.000}$ & 0.3 & 0.0 & 0.0$_{\pm0.0}$ & 24.25$_{\pm6.19}$ & 0.999$_{\pm0.000}$ & 0.4\\
  & CD$^{\dagger}$ & 0.0 & 0.0$_{\pm0.0}$ & 29.53$_{\pm5.13}$ & 0.706$_{\pm0.022}$ & 1.3$_{\pm0.1}$ & 0.0 & 0.0$_{\pm0.0}$ & 29.53$_{\pm5.13}$ & 0.706$_{\pm0.022}$ & 1.3$_{\pm0.1}$ & 0.0 & 0.0$_{\pm0.0}$ & 29.53$_{\pm5.13}$ & 0.706$_{\pm0.022}$ & 1.3$_{\pm0.1}$\\
  & MMD-CBS$^{\dagger}$ & 0.0 & 0.0$_{\pm0.0}$ & --- & --- & 0.9$_{\pm0.3}$ & 0.0 & 0.0$_{\pm0.0}$ & --- & --- & 0.9$_{\pm0.3}$ & 0.0 & 0.0$_{\pm0.0}$ & --- & --- & 0.9$_{\pm0.3}$\\
  & MDOC$^{\dagger}$ & 0.0 & --- & --- & 0.000$_{\pm0.000}$ & 165.3$_{\pm8.7}$ & 0.0 & --- & --- & 0.000$_{\pm0.000}$ & 164.1$_{\pm8.1}$ & 0.0 & --- & --- & 0.000$_{\pm0.000}$ & 160.9$_{\pm8.6}$\\
  & PCD-LB & 0.0 & 0.0$_{\pm0.0}$ & 30.67$_{\pm3.84}$ & 0.356$_{\pm0.029}$ & 9.6$_{\pm1.4}$ & 0.0 & 0.0$_{\pm0.0}$ & 33.64$_{\pm3.88}$ & 0.357$_{\pm0.028}$ & 7.9$_{\pm1.1}$ & 0.0 & 0.0$_{\pm0.0}$ & 36.36$_{\pm3.94}$ & 0.359$_{\pm0.027}$ & 3.5$_{\pm0.8}$\\
  & PCD-SHD & 0.0 & 0.0$_{\pm0.0}$ & 14.48$_{\pm2.36}$ & 0.796$_{\pm0.024}$ & 4.9$_{\pm0.4}$ & 0.0 & 0.0$_{\pm0.0}$ & 14.69$_{\pm2.44}$ & 0.795$_{\pm0.024}$ & 4.7$_{\pm0.1}$ & 0.0 & 0.0$_{\pm0.0}$ & 14.90$_{\pm2.43}$ & 0.794$_{\pm0.023}$ & 4.6$_{\pm0.1}$\\
  & Final Proj. & 0.0 & 0.0$_{\pm0.0}$ & 4.31$_{\pm3.10}$ & 0.982$_{\pm0.007}$ & 23.1$_{\pm0.7}$ & 0.0 & 0.0$_{\pm0.0}$ & 3.68$_{\pm2.74}$ & \textbf{0.983}$_{\pm0.007}$ & 23.1$_{\pm0.7}$ & 0.0 & 0.0$_{\pm0.0}$ & 3.20$_{\pm2.45}$ & \textbf{0.984}$_{\pm0.006}$ & 16.5$_{\pm5.1}$\\
  & DiRecT & 0.0 & 0.0$_{\pm0.0}$ & 1.36$_{\pm0.97}$ & 0.931$_{\pm0.019}$ & 120.1$_{\pm1.0}$ & 0.0 & 0.0$_{\pm0.0}$ & 1.06$_{\pm0.78}$ & 0.932$_{\pm0.019}$ & 120.0$_{\pm1.1}$ & 0.0 & 0.0$_{\pm0.0}$ & 0.62$_{\pm0.57}$ & 0.933$_{\pm0.019}$ & 120.2$_{\pm0.7}$\\
  & GlobMP (Ours) & \textbf{100.0} & \textbf{100.0}$_{\pm0.0}$ & \textbf{0.00}$_{\pm0.00}$ & \textbf{1.000}$_{\pm0.000}$ & 5.1$_{\pm43.5}$ & \textbf{100.0} & \textbf{100.0}$_{\pm0.0}$ & \textbf{0.00}$_{\pm0.00}$ & 0.978$_{\pm0.081}$ & 5.0$_{\pm26.0}$ & \textbf{100.0} & \textbf{100.0}$_{\pm0.0}$ & \textbf{0.00}$_{\pm0.00}$ & 0.970$_{\pm0.071}$ & 0.9$_{\pm0.3}$\\
\bottomrule
\end{tabular}}
\end{center}
\end{table}

\begin{table}[p]
\caption{Multi-robot planning performance on \texttt{Conveyor} for $N_a\in\{3,6,\dots,21\}$ under the Low, Medium and High velocity caps. \emph{SR} and \emph{CS} measure trajectory feasibility, \emph{Coll} reports inter-robot collisions, \emph{DA} adherence to the environment-specific motion pattern, and \emph{Time} the planning time per configuration: for \emph{GlobMP (Ours)} the median over the configurations with the standard deviation as subscript, for the baselines the mean, so the two are not directly comparable. $^{\dagger}$ marks a method that applies no velocity limit, so its three blocks report the same run; MMD-CBS also returns one plan per configuration rather than $64$ candidates. A dash under \emph{Coll} marks a count that is not commensurable with the others, as MDOC and MMD-CBS count colliding time steps where every other row counts agent pairs with any collision; MDOC's \emph{CS} is omitted because its plans that do not cover every agent were counted as feasible. MMD-CBS plans under $60$\,s per configuration and \emph{GlobMP (Ours)} under $300$\,s per team. \emph{DA} on this map scores an end-to-end traversal of one conveyor lane, which the uniformly drawn start and goal pairs almost never require; it is therefore below $0.07$ for every method and does not separate them.}
\label{tab:hard-corridor}
\begin{center}\tiny\setlength{\tabcolsep}{2.2pt}
\resizebox{\linewidth}{!}{%
\begin{tabular}{ll ccccc ccccc ccccc}
\toprule
& & \multicolumn{5}{c}{Low ($v_{\max}=1.21$)} & \multicolumn{5}{c}{Medium ($v_{\max}=1.46$)} & \multicolumn{5}{c}{High ($v_{\max}=1.76$)}\\
\cmidrule(lr){3-7}\cmidrule(lr){8-12}\cmidrule(lr){13-17}
$N_a$ & Method & SR$\uparrow$ & CS$\uparrow$ & Coll$\downarrow$ & DA$\uparrow$ & Time$\downarrow$ & SR$\uparrow$ & CS$\uparrow$ & Coll$\downarrow$ & DA$\uparrow$ & Time$\downarrow$ & SR$\uparrow$ & CS$\uparrow$ & Coll$\downarrow$ & DA$\uparrow$ & Time$\downarrow$\\
\midrule
3 & Diff.\ Policy$^{\dagger}$ & 5.0 & 4.9$_{\pm21.5}$ & 0.37$_{\pm0.53}$ & 0.016$_{\pm0.070}$ & 0.6 & 5.0 & 4.9$_{\pm21.5}$ & 0.37$_{\pm0.53}$ & 0.016$_{\pm0.070}$ & 0.6 & 5.0 & 4.9$_{\pm21.5}$ & 0.37$_{\pm0.53}$ & 0.016$_{\pm0.070}$ & 0.6\\
  & CD$^{\dagger}$ & 45.0 & 17.4$_{\pm35.0}$ & 0.02$_{\pm0.04}$ & 0.001$_{\pm0.003}$ & 0.6$_{\pm0.1}$ & 45.0 & 18.3$_{\pm36.2}$ & 0.02$_{\pm0.04}$ & 0.001$_{\pm0.003}$ & 1.2$_{\pm0.1}$ & 45.0 & 19.2$_{\pm37.3}$ & 0.02$_{\pm0.04}$ & 0.001$_{\pm0.003}$ & 1.1$_{\pm0.1}$\\
  & MMD-CBS$^{\dagger}$ & 20.0 & 20.0$_{\pm40.0}$ & --- & 0.000$_{\pm0.000}$ & 1.3$_{\pm0.7}$ & 20.0 & 20.0$_{\pm40.0}$ & --- & 0.000$_{\pm0.000}$ & 1.3$_{\pm0.7}$ & 20.0 & 20.0$_{\pm40.0}$ & --- & 0.000$_{\pm0.000}$ & 1.3$_{\pm0.7}$\\
  & MDOC$^{\dagger}$ & 5.0 & --- & --- & 0.000$_{\pm0.000}$ & 40.1$_{\pm7.0}$ & 5.0 & --- & --- & 0.000$_{\pm0.000}$ & 36.4$_{\pm2.6}$ & 5.0 & --- & --- & 0.000$_{\pm0.000}$ & 36.3$_{\pm2.8}$\\
  & PCD-LB & 70.0 & 36.2$_{\pm42.5}$ & 0.01$_{\pm0.03}$ & 0.004$_{\pm0.016}$ & 3.0$_{\pm0.1}$ & 65.0 & 37.1$_{\pm43.6}$ & 0.01$_{\pm0.02}$ & 0.004$_{\pm0.016}$ & 3.1$_{\pm0.1}$ & 65.0 & 36.8$_{\pm43.5}$ & 0.01$_{\pm0.02}$ & 0.004$_{\pm0.017}$ & 3.1$_{\pm0.1}$\\
  & PCD-SHD & 10.0 & 9.5$_{\pm28.4}$ & 0.01$_{\pm0.03}$ & 0.003$_{\pm0.011}$ & 2.1$_{\pm0.7}$ & 10.0 & 9.5$_{\pm28.4}$ & 0.01$_{\pm0.02}$ & 0.002$_{\pm0.010}$ & 3.0$_{\pm0.1}$ & 10.0 & 9.6$_{\pm28.9}$ & 0.01$_{\pm0.02}$ & 0.002$_{\pm0.009}$ & 3.0$_{\pm0.1}$\\
  & Final Proj. & 10.0 & 10.0$_{\pm30.0}$ & 0.00$_{\pm0.01}$ & 0.000$_{\pm0.000}$ & 3.6$_{\pm0.5}$ & 10.0 & 10.0$_{\pm30.0}$ & 0.00$_{\pm0.00}$ & 0.000$_{\pm0.000}$ & 3.7$_{\pm0.7}$ & 10.0 & 10.0$_{\pm30.0}$ & \textbf{0.00}$_{\pm0.00}$ & 0.000$_{\pm0.000}$ & 3.7$_{\pm0.6}$\\
  & DiRecT & 10.0 & 9.9$_{\pm29.8}$ & \textbf{0.00}$_{\pm0.00}$ & 0.000$_{\pm0.000}$ & 24.3$_{\pm5.8}$ & 10.0 & 10.0$_{\pm30.0}$ & \textbf{0.00}$_{\pm0.00}$ & 0.000$_{\pm0.000}$ & 24.5$_{\pm5.7}$ & 10.0 & 10.0$_{\pm30.0}$ & \textbf{0.00}$_{\pm0.00}$ & 0.000$_{\pm0.000}$ & 26.4$_{\pm6.7}$\\
  & GlobMP (Ours) & \textbf{100.0} & \textbf{100.0}$_{\pm0.0}$ & \textbf{0.00}$_{\pm0.00}$ & \textbf{0.017}$_{\pm0.128}$ & 0.4$_{\pm0.2}$ & \textbf{100.0} & \textbf{100.0}$_{\pm0.0}$ & \textbf{0.00}$_{\pm0.00}$ & \textbf{0.033}$_{\pm0.180}$ & 0.5$_{\pm0.4}$ & \textbf{100.0} & \textbf{100.0}$_{\pm0.0}$ & \textbf{0.00}$_{\pm0.00}$ & \textbf{0.067}$_{\pm0.249}$ & 0.7$_{\pm0.6}$\\
\midrule
6 & Diff.\ Policy$^{\dagger}$ & 0.0 & 0.0$_{\pm0.0}$ & 1.22$_{\pm1.04}$ & 0.000$_{\pm0.000}$ & 0.6 & 0.0 & 0.0$_{\pm0.0}$ & 1.22$_{\pm1.04}$ & 0.000$_{\pm0.000}$ & 0.6 & 0.0 & 0.0$_{\pm0.0}$ & 1.22$_{\pm1.04}$ & 0.000$_{\pm0.000}$ & 0.6\\
  & CD$^{\dagger}$ & 5.0 & 0.1$_{\pm0.3}$ & 0.25$_{\pm0.29}$ & 0.008$_{\pm0.032}$ & 1.2$_{\pm0.1}$ & 5.0 & 0.1$_{\pm0.3}$ & 0.25$_{\pm0.29}$ & 0.008$_{\pm0.032}$ & 1.2$_{\pm0.1}$ & 5.0 & 0.2$_{\pm0.7}$ & 0.25$_{\pm0.29}$ & 0.008$_{\pm0.032}$ & 1.1$_{\pm0.1}$\\
  & MMD-CBS$^{\dagger}$ & 0.0 & 0.0$_{\pm0.0}$ & --- & --- & 1.5$_{\pm1.1}$ & 0.0 & 0.0$_{\pm0.0}$ & --- & --- & 1.5$_{\pm1.1}$ & 0.0 & 0.0$_{\pm0.0}$ & --- & --- & 1.5$_{\pm1.1}$\\
  & MDOC$^{\dagger}$ & 0.0 & --- & --- & 0.000$_{\pm0.000}$ & 62.5$_{\pm3.3}$ & 0.0 & --- & --- & 0.000$_{\pm0.000}$ & 82.6$_{\pm23.1}$ & 0.0 & --- & --- & 0.000$_{\pm0.000}$ & 85.2$_{\pm20.8}$\\
  & PCD-LB & 25.0 & 6.1$_{\pm14.7}$ & 0.85$_{\pm0.74}$ & 0.006$_{\pm0.013}$ & 2.8$_{\pm0.7}$ & 40.0 & 7.8$_{\pm16.5}$ & 0.72$_{\pm0.60}$ & 0.007$_{\pm0.013}$ & 3.1$_{\pm0.4}$ & 45.0 & 13.5$_{\pm20.5}$ & 0.44$_{\pm0.34}$ & 0.007$_{\pm0.014}$ & 3.3$_{\pm0.1}$\\
  & PCD-SHD & 0.0 & 0.0$_{\pm0.0}$ & 0.10$_{\pm0.18}$ & 0.001$_{\pm0.006}$ & 3.3$_{\pm0.1}$ & 0.0 & 0.0$_{\pm0.0}$ & 0.06$_{\pm0.09}$ & 0.001$_{\pm0.006}$ & 3.2$_{\pm0.1}$ & 0.0 & 0.0$_{\pm0.0}$ & 0.03$_{\pm0.05}$ & 0.001$_{\pm0.005}$ & 3.2$_{\pm0.1}$\\
  & Final Proj. & 0.0 & 0.0$_{\pm0.0}$ & 0.05$_{\pm0.14}$ & 0.000$_{\pm0.000}$ & 1.9$_{\pm0.9}$ & 0.0 & 0.0$_{\pm0.0}$ & 0.03$_{\pm0.10}$ & 0.000$_{\pm0.000}$ & 3.4$_{\pm0.7}$ & 0.0 & 0.0$_{\pm0.0}$ & 0.01$_{\pm0.03}$ & 0.000$_{\pm0.000}$ & 3.4$_{\pm0.6}$\\
  & DiRecT & 0.0 & 0.0$_{\pm0.0}$ & 0.00$_{\pm0.01}$ & 0.001$_{\pm0.005}$ & 23.4$_{\pm6.1}$ & 0.0 & 0.0$_{\pm0.0}$ & 0.01$_{\pm0.02}$ & 0.001$_{\pm0.002}$ & 24.0$_{\pm6.6}$ & 0.0 & 0.0$_{\pm0.0}$ & 0.01$_{\pm0.02}$ & 0.000$_{\pm0.001}$ & 25.0$_{\pm7.0}$\\
  & GlobMP (Ours) & \textbf{100.0} & \textbf{100.0}$_{\pm0.0}$ & \textbf{0.00}$_{\pm0.00}$ & \textbf{0.008}$_{\pm0.091}$ & 0.8$_{\pm0.8}$ & \textbf{100.0} & \textbf{100.0}$_{\pm0.0}$ & \textbf{0.00}$_{\pm0.00}$ & \textbf{0.008}$_{\pm0.091}$ & 1.2$_{\pm2.2}$ & \textbf{100.0} & \textbf{100.0}$_{\pm0.0}$ & \textbf{0.00}$_{\pm0.00}$ & \textbf{0.017}$_{\pm0.128}$ & 1.8$_{\pm10.5}$\\
\midrule
9 & Diff.\ Policy$^{\dagger}$ & 0.0 & 0.0$_{\pm0.0}$ & 3.16$_{\pm1.58}$ & 0.006$_{\pm0.024}$ & 0.6 & 0.0 & 0.0$_{\pm0.0}$ & 3.16$_{\pm1.58}$ & 0.006$_{\pm0.024}$ & 0.3 & 0.0 & 0.0$_{\pm0.0}$ & 3.16$_{\pm1.58}$ & 0.006$_{\pm0.024}$ & 0.3\\
  & CD$^{\dagger}$ & 0.0 & 0.0$_{\pm0.0}$ & 1.47$_{\pm0.91}$ & 0.000$_{\pm0.002}$ & 0.6$_{\pm0.1}$ & 0.0 & 0.0$_{\pm0.0}$ & 1.47$_{\pm0.91}$ & 0.000$_{\pm0.002}$ & 1.1$_{\pm0.1}$ & 0.0 & 0.0$_{\pm0.0}$ & 1.47$_{\pm0.91}$ & 0.000$_{\pm0.002}$ & 1.1$_{\pm0.1}$\\
  & MMD-CBS$^{\dagger}$ & 0.0 & 0.0$_{\pm0.0}$ & --- & --- & 1.8$_{\pm1.1}$ & 0.0 & 0.0$_{\pm0.0}$ & --- & --- & 1.8$_{\pm1.1}$ & 0.0 & 0.0$_{\pm0.0}$ & --- & --- & 1.8$_{\pm1.1}$\\
  & MDOC$^{\dagger}$ & 0.0 & --- & --- & 0.000$_{\pm0.000}$ & 113.3$_{\pm17.4}$ & 0.0 & --- & --- & 0.000$_{\pm0.000}$ & 97.2$_{\pm18.9}$ & 0.0 & --- & --- & 0.000$_{\pm0.000}$ & 74.7$_{\pm4.0}$\\
  & PCD-LB & 0.0 & 0.0$_{\pm0.0}$ & 4.20$_{\pm1.42}$ & 0.004$_{\pm0.008}$ & 3.5$_{\pm0.1}$ & 10.0 & 0.4$_{\pm1.4}$ & 3.09$_{\pm1.02}$ & 0.004$_{\pm0.010}$ & 3.5$_{\pm0.1}$ & 25.0 & 2.4$_{\pm5.4}$ & 1.54$_{\pm0.57}$ & 0.005$_{\pm0.011}$ & 3.5$_{\pm0.1}$\\
  & PCD-SHD & 0.0 & 0.0$_{\pm0.0}$ & 0.13$_{\pm0.16}$ & 0.003$_{\pm0.014}$ & 3.6$_{\pm0.1}$ & 0.0 & 0.0$_{\pm0.0}$ & 0.08$_{\pm0.10}$ & 0.003$_{\pm0.014}$ & 3.6$_{\pm0.1}$ & 0.0 & 0.0$_{\pm0.0}$ & 0.05$_{\pm0.06}$ & 0.003$_{\pm0.014}$ & 3.6$_{\pm0.1}$\\
  & Final Proj. & 0.0 & 0.0$_{\pm0.0}$ & 0.05$_{\pm0.10}$ & 0.005$_{\pm0.024}$ & 4.3$_{\pm0.7}$ & 0.0 & 0.0$_{\pm0.0}$ & 0.02$_{\pm0.04}$ & 0.005$_{\pm0.024}$ & 4.3$_{\pm0.7}$ & 0.0 & 0.0$_{\pm0.0}$ & \textbf{0.02}$_{\pm0.04}$ & 0.005$_{\pm0.024}$ & 4.3$_{\pm0.6}$\\
  & DiRecT & 0.0 & 0.0$_{\pm0.0}$ & 0.04$_{\pm0.06}$ & 0.007$_{\pm0.025}$ & 34.5$_{\pm9.7}$ & 0.0 & 0.0$_{\pm0.0}$ & 0.01$_{\pm0.02}$ & 0.006$_{\pm0.024}$ & 35.6$_{\pm10.5}$ & 0.0 & 0.0$_{\pm0.0}$ & 0.02$_{\pm0.03}$ & 0.006$_{\pm0.024}$ & 25.9$_{\pm1.3}$\\
  & GlobMP (Ours) & \textbf{100.0} & \textbf{100.0}$_{\pm0.0}$ & \textbf{0.00}$_{\pm0.00}$ & \textbf{0.028}$_{\pm0.164}$ & 2.0$_{\pm8.3}$ & \textbf{100.0} & \textbf{100.0}$_{\pm0.0}$ & \textbf{0.00}$_{\pm0.00}$ & \textbf{0.028}$_{\pm0.164}$ & 2.6$_{\pm2.7}$ & \textbf{95.0} & \textbf{95.0}$_{\pm21.8}$ & 0.10$_{\pm0.44}$ & \textbf{0.028}$_{\pm0.164}$ & 5.5$_{\pm64.2}$\\
\midrule
12 & Diff.\ Policy$^{\dagger}$ & 0.0 & 0.0$_{\pm0.0}$ & 6.33$_{\pm2.50}$ & 0.012$_{\pm0.029}$ & 0.6 & 0.0 & 0.0$_{\pm0.0}$ & 6.33$_{\pm2.50}$ & 0.012$_{\pm0.029}$ & 0.6 & 0.0 & 0.0$_{\pm0.0}$ & 6.33$_{\pm2.50}$ & 0.012$_{\pm0.029}$ & 0.6\\
  & CD$^{\dagger}$ & 0.0 & 0.0$_{\pm0.0}$ & 3.96$_{\pm1.28}$ & 0.018$_{\pm0.026}$ & 1.2$_{\pm0.1}$ & 0.0 & 0.0$_{\pm0.0}$ & 3.96$_{\pm1.28}$ & 0.018$_{\pm0.026}$ & 1.2$_{\pm0.1}$ & 0.0 & 0.0$_{\pm0.0}$ & 3.96$_{\pm1.28}$ & 0.018$_{\pm0.026}$ & 0.6$_{\pm0.0}$\\
  & MMD-CBS$^{\dagger}$ & 0.0 & 0.0$_{\pm0.0}$ & --- & --- & 1.9$_{\pm1.6}$ & 0.0 & 0.0$_{\pm0.0}$ & --- & --- & 1.9$_{\pm1.6}$ & 0.0 & 0.0$_{\pm0.0}$ & --- & --- & 1.9$_{\pm1.6}$\\
  & MDOC$^{\dagger}$ & 0.0 & --- & --- & 0.000$_{\pm0.000}$ & 107.9$_{\pm7.1}$ & 0.0 & --- & --- & 0.000$_{\pm0.000}$ & 120.6$_{\pm5.3}$ & 0.0 & --- & --- & 0.000$_{\pm0.000}$ & 127.3$_{\pm7.3}$\\
  & PCD-LB & 0.0 & 0.0$_{\pm0.0}$ & 9.25$_{\pm1.68}$ & 0.005$_{\pm0.007}$ & 3.7$_{\pm0.1}$ & 0.0 & 0.0$_{\pm0.0}$ & 7.14$_{\pm1.37}$ & 0.005$_{\pm0.007}$ & 1.9$_{\pm0.2}$ & 0.0 & 0.0$_{\pm0.0}$ & 3.76$_{\pm0.86}$ & 0.005$_{\pm0.006}$ & 3.7$_{\pm0.5}$\\
  & PCD-SHD & 0.0 & 0.0$_{\pm0.0}$ & 0.62$_{\pm0.51}$ & 0.013$_{\pm0.022}$ & 3.9$_{\pm0.1}$ & 0.0 & 0.0$_{\pm0.0}$ & 0.37$_{\pm0.32}$ & 0.012$_{\pm0.022}$ & 3.9$_{\pm0.1}$ & 0.0 & 0.0$_{\pm0.0}$ & 0.19$_{\pm0.18}$ & 0.012$_{\pm0.022}$ & 1.8$_{\pm0.2}$\\
  & Final Proj. & 0.0 & 0.0$_{\pm0.0}$ & 0.20$_{\pm0.29}$ & 0.010$_{\pm0.025}$ & 5.4$_{\pm0.7}$ & 0.0 & 0.0$_{\pm0.0}$ & \textbf{0.05}$_{\pm0.08}$ & 0.010$_{\pm0.025}$ & 5.4$_{\pm0.7}$ & 0.0 & 0.0$_{\pm0.0}$ & \textbf{0.06}$_{\pm0.11}$ & 0.010$_{\pm0.025}$ & 5.4$_{\pm0.7}$\\
  & DiRecT & 0.0 & 0.0$_{\pm0.0}$ & \textbf{0.06}$_{\pm0.08}$ & 0.012$_{\pm0.026}$ & 33.6$_{\pm1.4}$ & 0.0 & 0.0$_{\pm0.0}$ & 0.07$_{\pm0.09}$ & 0.013$_{\pm0.028}$ & 34.0$_{\pm1.9}$ & 0.0 & 0.0$_{\pm0.0}$ & 0.09$_{\pm0.16}$ & 0.014$_{\pm0.029}$ & 34.2$_{\pm1.9}$\\
  & GlobMP (Ours) & \textbf{90.0} & \textbf{90.0}$_{\pm30.0}$ & 0.85$_{\pm2.55}$ & \textbf{0.042}$_{\pm0.200}$ & 3.3$_{\pm91.3}$ & \textbf{95.0} & \textbf{95.0}$_{\pm21.8}$ & 0.45$_{\pm1.96}$ & \textbf{0.037}$_{\pm0.190}$ & 5.7$_{\pm72.9}$ & \textbf{90.0} & \textbf{90.0}$_{\pm30.0}$ & 0.60$_{\pm1.80}$ & \textbf{0.033}$_{\pm0.180}$ & 13.3$_{\pm79.3}$\\
\midrule
15 & Diff.\ Policy$^{\dagger}$ & 0.0 & 0.0$_{\pm0.0}$ & 9.77$_{\pm2.61}$ & 0.003$_{\pm0.013}$ & 0.3 & 0.0 & 0.0$_{\pm0.0}$ & 9.77$_{\pm2.61}$ & 0.003$_{\pm0.013}$ & 0.3 & 0.0 & 0.0$_{\pm0.0}$ & 9.77$_{\pm2.61}$ & 0.003$_{\pm0.013}$ & 0.3\\
  & CD$^{\dagger}$ & 0.0 & 0.0$_{\pm0.0}$ & 7.35$_{\pm1.96}$ & 0.003$_{\pm0.009}$ & 1.2$_{\pm0.1}$ & 0.0 & 0.0$_{\pm0.0}$ & 7.35$_{\pm1.96}$ & 0.003$_{\pm0.009}$ & 1.2$_{\pm0.1}$ & 0.0 & 0.0$_{\pm0.0}$ & 7.35$_{\pm1.96}$ & 0.003$_{\pm0.009}$ & 1.2$_{\pm0.1}$\\
  & MMD-CBS$^{\dagger}$ & 0.0 & 0.0$_{\pm0.0}$ & --- & --- & 1.9$_{\pm1.4}$ & 0.0 & 0.0$_{\pm0.0}$ & --- & --- & 1.9$_{\pm1.4}$ & 0.0 & 0.0$_{\pm0.0}$ & --- & --- & 1.9$_{\pm1.4}$\\
  & MDOC$^{\dagger}$ & 0.0 & --- & --- & 0.000$_{\pm0.000}$ & 155.4$_{\pm11.8}$ & 0.0 & --- & --- & 0.000$_{\pm0.000}$ & 172.5$_{\pm14.6}$ & 0.0 & --- & --- & 0.000$_{\pm0.000}$ & 174.3$_{\pm27.4}$\\
  & PCD-LB & 0.0 & 0.0$_{\pm0.0}$ & 20.74$_{\pm3.01}$ & 0.005$_{\pm0.006}$ & 4.1$_{\pm0.1}$ & 0.0 & 0.0$_{\pm0.0}$ & 17.02$_{\pm2.64}$ & 0.004$_{\pm0.006}$ & 4.1$_{\pm0.2}$ & 0.0 & 0.0$_{\pm0.0}$ & 11.84$_{\pm1.88}$ & 0.005$_{\pm0.006}$ & 4.1$_{\pm0.1}$\\
  & PCD-SHD & 0.0 & 0.0$_{\pm0.0}$ & 0.75$_{\pm0.40}$ & 0.003$_{\pm0.010}$ & 4.3$_{\pm0.1}$ & 0.0 & 0.0$_{\pm0.0}$ & 0.48$_{\pm0.28}$ & 0.003$_{\pm0.010}$ & 4.3$_{\pm0.1}$ & 0.0 & 0.0$_{\pm0.0}$ & 0.29$_{\pm0.19}$ & 0.003$_{\pm0.010}$ & 4.3$_{\pm0.1}$\\
  & Final Proj. & 0.0 & 0.0$_{\pm0.0}$ & 0.49$_{\pm0.42}$ & 0.003$_{\pm0.015}$ & 3.7$_{\pm0.7}$ & 0.0 & 0.0$_{\pm0.0}$ & 0.23$_{\pm0.25}$ & 0.003$_{\pm0.015}$ & 6.9$_{\pm0.7}$ & 0.0 & 0.0$_{\pm0.0}$ & \textbf{0.18}$_{\pm0.24}$ & 0.003$_{\pm0.015}$ & 6.0$_{\pm1.6}$\\
  & DiRecT & 0.0 & 0.0$_{\pm0.0}$ & \textbf{0.14}$_{\pm0.15}$ & 0.003$_{\pm0.013}$ & 44.0$_{\pm1.3}$ & 0.0 & 0.0$_{\pm0.0}$ & \textbf{0.16}$_{\pm0.12}$ & 0.003$_{\pm0.013}$ & 44.9$_{\pm2.2}$ & 0.0 & 0.0$_{\pm0.0}$ & \textbf{0.18}$_{\pm0.15}$ & 0.003$_{\pm0.013}$ & 47.2$_{\pm5.5}$\\
  & GlobMP (Ours) & \textbf{90.0} & \textbf{90.0}$_{\pm30.0}$ & 1.00$_{\pm3.15}$ & \textbf{0.027}$_{\pm0.161}$ & 5.5$_{\pm68.3}$ & \textbf{85.0} & \textbf{85.0}$_{\pm35.7}$ & 1.45$_{\pm3.58}$ & \textbf{0.023}$_{\pm0.151}$ & 11.2$_{\pm74.1}$ & \textbf{90.0} & \textbf{90.0}$_{\pm30.0}$ & 0.90$_{\pm2.70}$ & \textbf{0.023}$_{\pm0.151}$ & 22.8$_{\pm91.2}$\\
\midrule
18 & Diff.\ Policy$^{\dagger}$ & 0.0 & 0.0$_{\pm0.0}$ & 15.01$_{\pm3.59}$ & 0.000$_{\pm0.000}$ & 0.3 & 0.0 & 0.0$_{\pm0.0}$ & 15.01$_{\pm3.59}$ & 0.000$_{\pm0.000}$ & 0.3 & 0.0 & 0.0$_{\pm0.0}$ & 15.01$_{\pm3.59}$ & 0.000$_{\pm0.000}$ & 0.3\\
  & CD$^{\dagger}$ & 0.0 & 0.0$_{\pm0.0}$ & 12.92$_{\pm1.90}$ & 0.003$_{\pm0.006}$ & 1.2$_{\pm0.1}$ & 0.0 & 0.0$_{\pm0.0}$ & 12.92$_{\pm1.90}$ & 0.003$_{\pm0.006}$ & 1.2$_{\pm0.2}$ & 0.0 & 0.0$_{\pm0.0}$ & 12.92$_{\pm1.90}$ & 0.003$_{\pm0.006}$ & 1.2$_{\pm0.0}$\\
  & MMD-CBS$^{\dagger}$ & 0.0 & 0.0$_{\pm0.0}$ & --- & --- & 1.4$_{\pm0.9}$ & 0.0 & 0.0$_{\pm0.0}$ & --- & --- & 1.4$_{\pm0.9}$ & 0.0 & 0.0$_{\pm0.0}$ & --- & --- & 1.4$_{\pm0.9}$\\
  & MDOC$^{\dagger}$ & 0.0 & --- & --- & 0.000$_{\pm0.000}$ & 162.1$_{\pm8.6}$ & 0.0 & --- & --- & 0.000$_{\pm0.000}$ & 156.7$_{\pm7.6}$ & 0.0 & --- & --- & 0.000$_{\pm0.000}$ & 157.5$_{\pm7.0}$\\
  & PCD-LB & 0.0 & 0.0$_{\pm0.0}$ & 35.56$_{\pm3.24}$ & 0.002$_{\pm0.003}$ & 4.2$_{\pm0.2}$ & 0.0 & 0.0$_{\pm0.0}$ & 28.42$_{\pm2.79}$ & 0.002$_{\pm0.003}$ & 4.3$_{\pm0.1}$ & 0.0 & 0.0$_{\pm0.0}$ & 19.50$_{\pm1.96}$ & 0.003$_{\pm0.004}$ & 2.8$_{\pm0.9}$\\
  & PCD-SHD & 0.0 & 0.0$_{\pm0.0}$ & 1.18$_{\pm0.69}$ & 0.003$_{\pm0.004}$ & 4.4$_{\pm0.1}$ & 0.0 & 0.0$_{\pm0.0}$ & 0.80$_{\pm0.52}$ & 0.003$_{\pm0.004}$ & 4.4$_{\pm0.1}$ & 0.0 & 0.0$_{\pm0.0}$ & 0.55$_{\pm0.37}$ & 0.002$_{\pm0.004}$ & 4.4$_{\pm0.1}$\\
  & Final Proj. & 0.0 & 0.0$_{\pm0.0}$ & 0.84$_{\pm0.83}$ & 0.000$_{\pm0.000}$ & 8.7$_{\pm0.7}$ & 0.0 & 0.0$_{\pm0.0}$ & 0.33$_{\pm0.33}$ & 0.000$_{\pm0.000}$ & 7.3$_{\pm2.1}$ & 0.0 & 0.0$_{\pm0.0}$ & \textbf{0.23}$_{\pm0.23}$ & 0.000$_{\pm0.000}$ & 8.7$_{\pm0.7}$\\
  & DiRecT & 0.0 & 0.0$_{\pm0.0}$ & \textbf{0.20}$_{\pm0.15}$ & 0.003$_{\pm0.009}$ & 65.8$_{\pm10.7}$ & 0.0 & 0.0$_{\pm0.0}$ & \textbf{0.22}$_{\pm0.18}$ & 0.003$_{\pm0.008}$ & 64.5$_{\pm10.8}$ & 0.0 & 0.0$_{\pm0.0}$ & 0.26$_{\pm0.22}$ & 0.003$_{\pm0.008}$ & 64.4$_{\pm9.0}$\\
  & GlobMP (Ours) & \textbf{90.0} & \textbf{90.0}$_{\pm30.0}$ & 1.85$_{\pm5.55}$ & \textbf{0.047}$_{\pm0.212}$ & 12.4$_{\pm89.4}$ & \textbf{85.0} & \textbf{85.0}$_{\pm35.7}$ & 2.75$_{\pm6.77}$ & \textbf{0.044}$_{\pm0.206}$ & 14.2$_{\pm102.5}$ & \textbf{75.0} & \textbf{75.0}$_{\pm43.3}$ & 3.60$_{\pm6.61}$ & \textbf{0.039}$_{\pm0.193}$ & 36.0$_{\pm118.0}$\\
\midrule
21 & Diff.\ Policy$^{\dagger}$ & 0.0 & 0.0$_{\pm0.0}$ & 20.41$_{\pm2.03}$ & 0.002$_{\pm0.010}$ & 0.3 & 0.0 & 0.0$_{\pm0.0}$ & 20.41$_{\pm2.03}$ & 0.002$_{\pm0.010}$ & 0.3 & 0.0 & 0.0$_{\pm0.0}$ & 20.41$_{\pm2.03}$ & 0.002$_{\pm0.010}$ & 0.3\\
  & CD$^{\dagger}$ & 0.0 & 0.0$_{\pm0.0}$ & 20.28$_{\pm2.30}$ & 0.011$_{\pm0.016}$ & 1.3$_{\pm0.1}$ & 0.0 & 0.0$_{\pm0.0}$ & 20.28$_{\pm2.30}$ & 0.011$_{\pm0.016}$ & 1.2$_{\pm0.0}$ & 0.0 & 0.0$_{\pm0.0}$ & 20.28$_{\pm2.30}$ & 0.011$_{\pm0.016}$ & 1.3$_{\pm0.0}$\\
  & MMD-CBS$^{\dagger}$ & 0.0 & 0.0$_{\pm0.0}$ & --- & --- & 1.7$_{\pm1.4}$ & 0.0 & 0.0$_{\pm0.0}$ & --- & --- & 1.7$_{\pm1.4}$ & 0.0 & 0.0$_{\pm0.0}$ & --- & --- & 1.7$_{\pm1.4}$\\
  & MDOC$^{\dagger}$ & 0.0 & --- & --- & 0.000$_{\pm0.000}$ & 182.3$_{\pm8.9}$ & 0.0 & --- & --- & 0.000$_{\pm0.000}$ & 181.9$_{\pm9.0}$ & 0.0 & --- & --- & 0.000$_{\pm0.000}$ & 163.8$_{\pm8.1}$\\
  & PCD-LB & 0.0 & 0.0$_{\pm0.0}$ & 53.35$_{\pm2.98}$ & 0.002$_{\pm0.003}$ & 4.7$_{\pm0.1}$ & 0.0 & 0.0$_{\pm0.0}$ & 40.89$_{\pm3.07}$ & 0.002$_{\pm0.003}$ & 4.8$_{\pm0.1}$ & 0.0 & 0.0$_{\pm0.0}$ & 26.92$_{\pm2.55}$ & 0.002$_{\pm0.003}$ & 4.9$_{\pm0.1}$\\
  & PCD-SHD & 0.0 & 0.0$_{\pm0.0}$ & 1.60$_{\pm0.65}$ & 0.006$_{\pm0.011}$ & 5.0$_{\pm0.1}$ & 0.0 & 0.0$_{\pm0.0}$ & 1.09$_{\pm0.49}$ & 0.006$_{\pm0.011}$ & 3.0$_{\pm1.0}$ & 0.0 & 0.0$_{\pm0.0}$ & 0.68$_{\pm0.32}$ & 0.006$_{\pm0.011}$ & 5.0$_{\pm0.1}$\\
  & Final Proj. & 0.0 & 0.0$_{\pm0.0}$ & 0.98$_{\pm0.65}$ & 0.003$_{\pm0.010}$ & 11.5$_{\pm0.9}$ & 0.0 & 0.0$_{\pm0.0}$ & 0.43$_{\pm0.32}$ & 0.003$_{\pm0.010}$ & 10.8$_{\pm1.3}$ & 0.0 & 0.0$_{\pm0.0}$ & 0.28$_{\pm0.19}$ & 0.003$_{\pm0.010}$ & 11.4$_{\pm0.7}$\\
  & DiRecT & 0.0 & 0.0$_{\pm0.0}$ & \textbf{0.20}$_{\pm0.14}$ & 0.006$_{\pm0.013}$ & 89.0$_{\pm12.7}$ & 0.0 & 0.0$_{\pm0.0}$ & \textbf{0.21}$_{\pm0.11}$ & 0.007$_{\pm0.015}$ & 84.7$_{\pm5.9}$ & 0.0 & 0.0$_{\pm0.0}$ & \textbf{0.22}$_{\pm0.13}$ & 0.007$_{\pm0.015}$ & 89.1$_{\pm11.6}$\\
  & GlobMP (Ours) & \textbf{80.0} & \textbf{80.0}$_{\pm40.0}$ & 5.00$_{\pm10.27}$ & \textbf{0.038}$_{\pm0.191}$ & 16.4$_{\pm112.4}$ & \textbf{70.0} & \textbf{70.0}$_{\pm45.8}$ & 6.30$_{\pm9.77}$ & \textbf{0.026}$_{\pm0.160}$ & 41.4$_{\pm119.1}$ & \textbf{65.0} & \textbf{65.0}$_{\pm47.7}$ & 5.95$_{\pm9.13}$ & \textbf{0.026}$_{\pm0.160}$ & 108.4$_{\pm116.6}$\\
\bottomrule
\end{tabular}}
\end{center}
\end{table}

\begin{table}[p]
\caption{Multi-robot planning performance on \texttt{Highways} for $N_a\in\{3,6,\dots,21\}$ under the Low, Medium and High velocity caps. \emph{SR} and \emph{CS} measure trajectory feasibility, \emph{Coll} reports inter-robot collisions, \emph{DA} adherence to the environment-specific motion pattern, and \emph{Time} the planning time per configuration: for \emph{GlobMP (Ours)} the median over the configurations with the standard deviation as subscript, for the baselines the mean, so the two are not directly comparable. $^{\dagger}$ marks a method that applies no velocity limit, so its three blocks report the same run; MMD-CBS also returns one plan per configuration rather than $64$ candidates. A dash under \emph{Coll} marks a count that is not commensurable with the others, as MDOC and MMD-CBS count colliding time steps where every other row counts agent pairs with any collision; MDOC's \emph{CS} is omitted because its plans that do not cover every agent were counted as feasible. MMD-CBS plans under $60$\,s per configuration and \emph{GlobMP (Ours)} under $300$\,s per team.}
\label{tab:hard-room}
\begin{center}\tiny\setlength{\tabcolsep}{2.2pt}
\resizebox{\linewidth}{!}{%
\begin{tabular}{ll ccccc ccccc ccccc}
\toprule
& & \multicolumn{5}{c}{Low ($v_{\max}=0.647$)} & \multicolumn{5}{c}{Medium ($v_{\max}=0.783$)} & \multicolumn{5}{c}{High ($v_{\max}=0.878$)}\\
\cmidrule(lr){3-7}\cmidrule(lr){8-12}\cmidrule(lr){13-17}
$N_a$ & Method & SR$\uparrow$ & CS$\uparrow$ & Coll$\downarrow$ & DA$\uparrow$ & Time$\downarrow$ & SR$\uparrow$ & CS$\uparrow$ & Coll$\downarrow$ & DA$\uparrow$ & Time$\downarrow$ & SR$\uparrow$ & CS$\uparrow$ & Coll$\downarrow$ & DA$\uparrow$ & Time$\downarrow$\\
\midrule
3 & Diff.\ Policy$^{\dagger}$ & 10.0 & 10.0$_{\pm30.0}$ & 0.15$_{\pm0.35}$ & 0.567$_{\pm0.298}$ & 0.3 & 10.0 & 10.0$_{\pm30.0}$ & 0.15$_{\pm0.35}$ & 0.567$_{\pm0.298}$ & 0.3 & 10.0 & 10.0$_{\pm30.0}$ & 0.15$_{\pm0.35}$ & 0.567$_{\pm0.298}$ & 0.3\\
  & CD$^{\dagger}$ & 100.0 & 0.0$_{\pm0.0}$ & 0.25$_{\pm0.32}$ & 0.564$_{\pm0.289}$ & 1.1$_{\pm0.1}$ & 100.0 & 0.0$_{\pm0.0}$ & 0.25$_{\pm0.32}$ & 0.564$_{\pm0.289}$ & 1.2$_{\pm0.1}$ & 100.0 & 0.0$_{\pm0.0}$ & 0.25$_{\pm0.32}$ & 0.564$_{\pm0.289}$ & 1.2$_{\pm0.1}$\\
  & MMD-CBS$^{\dagger}$ & 90.0 & 45.0$_{\pm49.7}$ & --- & 0.567$_{\pm0.318}$ & 3.8$_{\pm3.2}$ & 90.0 & 55.0$_{\pm49.7}$ & --- & 0.567$_{\pm0.318}$ & 3.8$_{\pm3.2}$ & 90.0 & 55.0$_{\pm49.7}$ & --- & 0.567$_{\pm0.318}$ & 3.8$_{\pm3.2}$\\
  & MDOC$^{\dagger}$ & 15.0 & --- & --- & 0.083$_{\pm0.233}$ & 41.3$_{\pm5.6}$ & 15.0 & --- & --- & 0.083$_{\pm0.233}$ & 37.2$_{\pm3.4}$ & 15.0 & --- & --- & 0.083$_{\pm0.233}$ & 36.8$_{\pm3.3}$\\
  & PCD-LB & 0.0 & 0.0$_{\pm0.0}$ & 0.02$_{\pm0.02}$ & 0.565$_{\pm0.293}$ & 3.5$_{\pm0.4}$ & 0.0 & 0.0$_{\pm0.0}$ & 0.05$_{\pm0.11}$ & 0.564$_{\pm0.292}$ & 3.0$_{\pm0.2}$ & 0.0 & 0.0$_{\pm0.0}$ & 0.05$_{\pm0.13}$ & 0.564$_{\pm0.292}$ & 3.0$_{\pm0.1}$\\
  & PCD-SHD & \textbf{100.0} & 89.8$_{\pm15.8}$ & 0.09$_{\pm0.17}$ & 0.561$_{\pm0.281}$ & 2.8$_{\pm0.1}$ & \textbf{100.0} & 90.9$_{\pm13.8}$ & 0.08$_{\pm0.14}$ & 0.562$_{\pm0.279}$ & 2.8$_{\pm0.1}$ & \textbf{100.0} & 90.4$_{\pm13.7}$ & 0.08$_{\pm0.14}$ & 0.562$_{\pm0.279}$ & 2.8$_{\pm0.1}$\\
  & Final Proj. & 95.0 & 85.9$_{\pm33.0}$ & 0.18$_{\pm0.47}$ & 0.567$_{\pm0.298}$ & 3.8$_{\pm0.6}$ & \textbf{100.0} & 86.1$_{\pm32.4}$ & 0.18$_{\pm0.46}$ & 0.567$_{\pm0.298}$ & 3.9$_{\pm0.7}$ & \textbf{100.0} & 84.6$_{\pm32.6}$ & 0.20$_{\pm0.46}$ & 0.567$_{\pm0.298}$ & 3.9$_{\pm0.6}$\\
  & DiRecT & \textbf{100.0} & \textbf{100.0}$_{\pm0.0}$ & \textbf{0.00}$_{\pm0.00}$ & 0.568$_{\pm0.294}$ & 26.4$_{\pm6.4}$ & \textbf{100.0} & \textbf{100.0}$_{\pm0.0}$ & \textbf{0.00}$_{\pm0.00}$ & \textbf{0.568}$_{\pm0.294}$ & 27.6$_{\pm7.1}$ & \textbf{100.0} & \textbf{100.0}$_{\pm0.0}$ & \textbf{0.00}$_{\pm0.00}$ & \textbf{0.568}$_{\pm0.294}$ & 29.3$_{\pm7.7}$\\
  & GlobMP (Ours) & 95.0 & 95.0$_{\pm21.8}$ & \textbf{0.00}$_{\pm0.00}$ & \textbf{0.583}$_{\pm0.493}$ & 0.2$_{\pm15.3}$ & \textbf{100.0} & \textbf{100.0}$_{\pm0.0}$ & \textbf{0.00}$_{\pm0.00}$ & 0.567$_{\pm0.496}$ & 0.3$_{\pm0.2}$ & \textbf{100.0} & \textbf{100.0}$_{\pm0.0}$ & \textbf{0.00}$_{\pm0.00}$ & 0.550$_{\pm0.497}$ & 0.2$_{\pm0.2}$\\
\midrule
6 & Diff.\ Policy$^{\dagger}$ & 0.0 & 0.0$_{\pm0.0}$ & 1.39$_{\pm0.78}$ & 0.582$_{\pm0.187}$ & 0.3 & 0.0 & 0.0$_{\pm0.0}$ & 1.39$_{\pm0.78}$ & 0.582$_{\pm0.187}$ & 0.3 & 0.0 & 0.0$_{\pm0.0}$ & 1.39$_{\pm0.78}$ & 0.582$_{\pm0.187}$ & 0.3\\
  & CD$^{\dagger}$ & 30.0 & 0.0$_{\pm0.0}$ & 2.61$_{\pm0.80}$ & 0.555$_{\pm0.192}$ & 1.2$_{\pm0.1}$ & 30.0 & 0.0$_{\pm0.0}$ & 2.61$_{\pm0.80}$ & 0.555$_{\pm0.192}$ & 1.2$_{\pm0.1}$ & 30.0 & 0.0$_{\pm0.0}$ & 2.61$_{\pm0.80}$ & 0.555$_{\pm0.192}$ & 0.5$_{\pm0.0}$\\
  & MMD-CBS$^{\dagger}$ & 75.0 & 0.0$_{\pm0.0}$ & --- & 0.558$_{\pm0.219}$ & 32.0$_{\pm16.7}$ & 75.0 & 0.0$_{\pm0.0}$ & --- & 0.558$_{\pm0.219}$ & 32.0$_{\pm16.7}$ & 75.0 & 0.0$_{\pm0.0}$ & --- & 0.558$_{\pm0.219}$ & 32.0$_{\pm16.7}$\\
  & MDOC$^{\dagger}$ & 0.0 & --- & --- & 0.000$_{\pm0.000}$ & 63.5$_{\pm3.2}$ & 0.0 & --- & --- & 0.000$_{\pm0.000}$ & 121.4$_{\pm150.4}$ & 0.0 & --- & --- & 0.000$_{\pm0.000}$ & 72.0$_{\pm6.6}$\\
  & PCD-LB & 0.0 & 0.0$_{\pm0.0}$ & 0.26$_{\pm0.26}$ & 0.539$_{\pm0.205}$ & 3.6$_{\pm0.5}$ & 0.0 & 0.0$_{\pm0.0}$ & 0.19$_{\pm0.16}$ & 0.540$_{\pm0.204}$ & 1.6$_{\pm0.1}$ & 0.0 & 0.0$_{\pm0.0}$ & 0.15$_{\pm0.15}$ & 0.540$_{\pm0.203}$ & 3.1$_{\pm0.1}$\\
  & PCD-SHD & \textbf{100.0} & 52.6$_{\pm20.0}$ & 0.59$_{\pm0.34}$ & 0.561$_{\pm0.180}$ & 3.0$_{\pm0.1}$ & \textbf{100.0} & 52.8$_{\pm21.6}$ & 0.59$_{\pm0.35}$ & 0.562$_{\pm0.182}$ & 2.6$_{\pm0.7}$ & \textbf{100.0} & 53.3$_{\pm21.7}$ & 0.57$_{\pm0.36}$ & 0.563$_{\pm0.183}$ & 2.9$_{\pm0.5}$\\
  & Final Proj. & \textbf{100.0} & 69.9$_{\pm39.6}$ & 0.53$_{\pm0.90}$ & \textbf{0.582}$_{\pm0.187}$ & 2.5$_{\pm1.2}$ & \textbf{100.0} & 75.9$_{\pm33.4}$ & 0.33$_{\pm0.56}$ & \textbf{0.582}$_{\pm0.187}$ & 4.1$_{\pm0.6}$ & \textbf{100.0} & 79.1$_{\pm31.5}$ & 0.29$_{\pm0.51}$ & \textbf{0.582}$_{\pm0.187}$ & 4.1$_{\pm0.7}$\\
  & DiRecT & \textbf{100.0} & \textbf{100.0}$_{\pm0.0}$ & \textbf{0.00}$_{\pm0.00}$ & 0.570$_{\pm0.187}$ & 32.0$_{\pm8.6}$ & \textbf{100.0} & \textbf{100.0}$_{\pm0.0}$ & \textbf{0.00}$_{\pm0.00}$ & 0.570$_{\pm0.187}$ & 32.8$_{\pm9.3}$ & \textbf{100.0} & \textbf{100.0}$_{\pm0.0}$ & \textbf{0.00}$_{\pm0.00}$ & 0.570$_{\pm0.187}$ & 34.9$_{\pm9.7}$\\
  & GlobMP (Ours) & 80.0 & 80.0$_{\pm40.0}$ & 0.10$_{\pm0.44}$ & 0.575$_{\pm0.494}$ & 0.7$_{\pm68.6}$ & \textbf{100.0} & \textbf{100.0}$_{\pm0.0}$ & \textbf{0.00}$_{\pm0.00}$ & 0.567$_{\pm0.496}$ & 0.6$_{\pm10.0}$ & \textbf{100.0} & \textbf{100.0}$_{\pm0.0}$ & \textbf{0.00}$_{\pm0.00}$ & 0.567$_{\pm0.496}$ & 0.7$_{\pm0.7}$\\
\midrule
9 & Diff.\ Policy$^{\dagger}$ & 0.0 & 0.0$_{\pm0.0}$ & 4.58$_{\pm1.91}$ & 0.450$_{\pm0.172}$ & 0.3 & 0.0 & 0.0$_{\pm0.0}$ & 4.58$_{\pm1.91}$ & 0.450$_{\pm0.172}$ & 0.3 & 0.0 & 0.0$_{\pm0.0}$ & 4.58$_{\pm1.91}$ & 0.450$_{\pm0.172}$ & 0.3\\
  & CD$^{\dagger}$ & 0.0 & 0.0$_{\pm0.0}$ & 7.04$_{\pm1.79}$ & 0.463$_{\pm0.170}$ & 1.2$_{\pm0.1}$ & 0.0 & 0.0$_{\pm0.0}$ & 7.04$_{\pm1.79}$ & 0.463$_{\pm0.170}$ & 1.2$_{\pm0.1}$ & 0.0 & 0.0$_{\pm0.0}$ & 7.04$_{\pm1.79}$ & 0.463$_{\pm0.170}$ & 1.2$_{\pm0.1}$\\
  & MMD-CBS$^{\dagger}$ & 5.0 & 0.0$_{\pm0.0}$ & --- & 0.478$_{\pm0.186}$ & 59.0$_{\pm9.8}$ & 5.0 & 0.0$_{\pm0.0}$ & --- & 0.478$_{\pm0.186}$ & 59.0$_{\pm9.8}$ & 5.0 & 0.0$_{\pm0.0}$ & --- & 0.478$_{\pm0.186}$ & 59.0$_{\pm9.8}$\\
  & MDOC$^{\dagger}$ & 0.0 & --- & --- & 0.000$_{\pm0.000}$ & 121.0$_{\pm15.7}$ & 0.0 & --- & --- & 0.000$_{\pm0.000}$ & 91.0$_{\pm13.2}$ & 0.0 & --- & --- & 0.000$_{\pm0.000}$ & 82.7$_{\pm3.8}$\\
  & PCD-LB & 0.0 & 0.0$_{\pm0.0}$ & 1.29$_{\pm0.58}$ & \textbf{0.469}$_{\pm0.162}$ & 4.4$_{\pm0.4}$ & 0.0 & 0.0$_{\pm0.0}$ & 1.26$_{\pm0.55}$ & \textbf{0.469}$_{\pm0.161}$ & 3.4$_{\pm0.5}$ & 0.0 & 0.0$_{\pm0.0}$ & 1.28$_{\pm0.56}$ & 0.469$_{\pm0.160}$ & 3.5$_{\pm0.1}$\\
  & PCD-SHD & 80.0 & 9.9$_{\pm14.2}$ & 2.60$_{\pm0.95}$ & 0.460$_{\pm0.166}$ & 3.3$_{\pm0.1}$ & 90.0 & 10.8$_{\pm15.0}$ & 2.38$_{\pm0.89}$ & 0.461$_{\pm0.166}$ & 3.4$_{\pm0.1}$ & 90.0 & 10.9$_{\pm13.2}$ & 2.33$_{\pm0.93}$ & 0.462$_{\pm0.166}$ & 3.4$_{\pm0.1}$\\
  & Final Proj. & 70.0 & 41.5$_{\pm40.8}$ & 1.60$_{\pm1.72}$ & 0.452$_{\pm0.174}$ & 5.5$_{\pm0.9}$ & 80.0 & 44.0$_{\pm39.7}$ & 1.23$_{\pm1.35}$ & 0.452$_{\pm0.174}$ & 5.6$_{\pm0.7}$ & 80.0 & 41.5$_{\pm37.4}$ & 1.12$_{\pm1.13}$ & 0.451$_{\pm0.174}$ & 5.6$_{\pm0.7}$\\
  & DiRecT & \textbf{100.0} & \textbf{100.0}$_{\pm0.0}$ & \textbf{0.00}$_{\pm0.00}$ & 0.454$_{\pm0.176}$ & 49.0$_{\pm13.7}$ & \textbf{100.0} & \textbf{100.0}$_{\pm0.0}$ & \textbf{0.00}$_{\pm0.00}$ & 0.454$_{\pm0.175}$ & 36.4$_{\pm7.2}$ & \textbf{100.0} & \textbf{100.0}$_{\pm0.0}$ & \textbf{0.00}$_{\pm0.00}$ & 0.454$_{\pm0.176}$ & 35.2$_{\pm1.3}$\\
  & GlobMP (Ours) & 65.0 & 65.0$_{\pm47.7}$ & 1.00$_{\pm1.76}$ & 0.450$_{\pm0.497}$ & 2.5$_{\pm127.9}$ & \textbf{100.0} & \textbf{100.0}$_{\pm0.0}$ & \textbf{0.00}$_{\pm0.00}$ & 0.467$_{\pm0.499}$ & 1.1$_{\pm5.4}$ & \textbf{100.0} & \textbf{100.0}$_{\pm0.0}$ & \textbf{0.00}$_{\pm0.00}$ & \textbf{0.483}$_{\pm0.500}$ & 1.2$_{\pm13.1}$\\
\midrule
12 & Diff.\ Policy$^{\dagger}$ & 0.0 & 0.0$_{\pm0.0}$ & 8.75$_{\pm3.11}$ & 0.464$_{\pm0.154}$ & 0.3 & 0.0 & 0.0$_{\pm0.0}$ & 8.75$_{\pm3.11}$ & 0.464$_{\pm0.154}$ & 0.3 & 0.0 & 0.0$_{\pm0.0}$ & 8.75$_{\pm3.11}$ & 0.464$_{\pm0.154}$ & 0.3\\
  & CD$^{\dagger}$ & 0.0 & 0.0$_{\pm0.0}$ & 15.44$_{\pm3.23}$ & 0.466$_{\pm0.145}$ & 1.2$_{\pm0.1}$ & 0.0 & 0.0$_{\pm0.0}$ & 15.44$_{\pm3.23}$ & 0.466$_{\pm0.145}$ & 1.2$_{\pm0.1}$ & 0.0 & 0.0$_{\pm0.0}$ & 15.44$_{\pm3.23}$ & 0.466$_{\pm0.145}$ & 0.6$_{\pm0.1}$\\
  & MMD-CBS$^{\dagger}$ & 0.0 & 0.0$_{\pm0.0}$ & --- & 0.450$_{\pm0.150}$ & 61.2$_{\pm0.9}$ & 0.0 & 0.0$_{\pm0.0}$ & --- & 0.450$_{\pm0.150}$ & 61.2$_{\pm0.9}$ & 0.0 & 0.0$_{\pm0.0}$ & --- & 0.450$_{\pm0.150}$ & 61.2$_{\pm0.9}$\\
  & MDOC$^{\dagger}$ & 0.0 & --- & --- & 0.000$_{\pm0.000}$ & 100.8$_{\pm7.3}$ & 0.0 & --- & --- & 0.000$_{\pm0.000}$ & 124.5$_{\pm6.1}$ & 0.0 & --- & --- & 0.000$_{\pm0.000}$ & 122.5$_{\pm5.6}$\\
  & PCD-LB & 0.0 & 0.0$_{\pm0.0}$ & 5.02$_{\pm0.76}$ & 0.456$_{\pm0.140}$ & 4.4$_{\pm0.4}$ & 0.0 & 0.0$_{\pm0.0}$ & 6.04$_{\pm0.83}$ & 0.457$_{\pm0.140}$ & 3.7$_{\pm0.2}$ & 0.0 & 0.0$_{\pm0.0}$ & 6.68$_{\pm0.77}$ & 0.459$_{\pm0.140}$ & 3.7$_{\pm0.1}$\\
  & PCD-SHD & 20.0 & 0.6$_{\pm1.3}$ & 5.75$_{\pm1.94}$ & \textbf{0.468}$_{\pm0.142}$ & 3.6$_{\pm0.1}$ & 20.0 & 0.5$_{\pm1.1}$ & 5.46$_{\pm1.97}$ & 0.469$_{\pm0.142}$ & 3.6$_{\pm0.1}$ & 20.0 & 0.9$_{\pm2.0}$ & 5.23$_{\pm1.77}$ & \textbf{0.469}$_{\pm0.142}$ & 3.7$_{\pm0.1}$\\
  & Final Proj. & 50.0 & 29.1$_{\pm39.5}$ & 3.49$_{\pm3.31}$ & 0.463$_{\pm0.154}$ & 6.3$_{\pm1.4}$ & 55.0 & 29.2$_{\pm40.0}$ & 2.76$_{\pm2.80}$ & 0.463$_{\pm0.154}$ & 6.8$_{\pm0.7}$ & 65.0 & 30.2$_{\pm39.5}$ & 2.23$_{\pm2.24}$ & 0.463$_{\pm0.154}$ & 6.8$_{\pm0.7}$\\
  & DiRecT & \textbf{100.0} & \textbf{100.0}$_{\pm0.0}$ & \textbf{0.00}$_{\pm0.00}$ & 0.467$_{\pm0.148}$ & 44.1$_{\pm2.2}$ & \textbf{100.0} & 99.9$_{\pm0.3}$ & 0.00$_{\pm0.01}$ & 0.466$_{\pm0.148}$ & 43.4$_{\pm1.8}$ & \textbf{100.0} & 99.9$_{\pm0.3}$ & 0.00$_{\pm0.00}$ & 0.466$_{\pm0.149}$ & 43.7$_{\pm1.7}$\\
  & GlobMP (Ours) & 60.0 & 60.0$_{\pm49.0}$ & 3.45$_{\pm4.99}$ & 0.467$_{\pm0.499}$ & 25.2$_{\pm102.3}$ & \textbf{100.0} & \textbf{100.0}$_{\pm0.0}$ & \textbf{0.00}$_{\pm0.00}$ & \textbf{0.492}$_{\pm0.500}$ & 2.1$_{\pm24.5}$ & \textbf{100.0} & \textbf{100.0}$_{\pm0.0}$ & \textbf{0.00}$_{\pm0.00}$ & 0.458$_{\pm0.498}$ & 2.2$_{\pm12.3}$\\
\midrule
15 & Diff.\ Policy$^{\dagger}$ & 0.0 & 0.0$_{\pm0.0}$ & 14.06$_{\pm4.00}$ & 0.484$_{\pm0.109}$ & 0.3 & 0.0 & 0.0$_{\pm0.0}$ & 14.06$_{\pm4.00}$ & 0.484$_{\pm0.109}$ & 0.3 & 0.0 & 0.0$_{\pm0.0}$ & 14.06$_{\pm4.00}$ & 0.484$_{\pm0.109}$ & 0.4\\
  & CD$^{\dagger}$ & 0.0 & 0.0$_{\pm0.0}$ & 26.51$_{\pm3.99}$ & 0.481$_{\pm0.107}$ & 1.2$_{\pm0.0}$ & 0.0 & 0.0$_{\pm0.0}$ & 26.51$_{\pm3.99}$ & 0.481$_{\pm0.107}$ & 1.3$_{\pm0.0}$ & 0.0 & 0.0$_{\pm0.0}$ & 26.51$_{\pm3.99}$ & 0.481$_{\pm0.107}$ & 1.2$_{\pm0.1}$\\
  & MMD-CBS$^{\dagger}$ & 0.0 & 0.0$_{\pm0.0}$ & --- & 0.483$_{\pm0.111}$ & 63.0$_{\pm2.1}$ & 0.0 & 0.0$_{\pm0.0}$ & --- & 0.483$_{\pm0.111}$ & 63.0$_{\pm2.1}$ & 0.0 & 0.0$_{\pm0.0}$ & --- & 0.483$_{\pm0.111}$ & 63.0$_{\pm2.1}$\\
  & MDOC$^{\dagger}$ & 0.0 & --- & --- & 0.000$_{\pm0.000}$ & 141.9$_{\pm12.5}$ & 0.0 & --- & --- & 0.000$_{\pm0.000}$ & 168.6$_{\pm16.0}$ & 0.0 & --- & --- & 0.000$_{\pm0.000}$ & 177.0$_{\pm31.2}$\\
  & PCD-LB & 0.0 & 0.0$_{\pm0.0}$ & 11.99$_{\pm1.72}$ & 0.485$_{\pm0.110}$ & 5.1$_{\pm0.6}$ & 0.0 & 0.0$_{\pm0.0}$ & 14.72$_{\pm1.73}$ & \textbf{0.485}$_{\pm0.109}$ & 4.2$_{\pm0.3}$ & 0.0 & 0.0$_{\pm0.0}$ & 16.38$_{\pm1.74}$ & \textbf{0.486}$_{\pm0.109}$ & 2.8$_{\pm1.0}$\\
  & PCD-SHD & 0.0 & 0.0$_{\pm0.0}$ & 10.01$_{\pm2.01}$ & 0.480$_{\pm0.106}$ & 3.3$_{\pm0.9}$ & 0.0 & 0.0$_{\pm0.0}$ & 9.49$_{\pm1.89}$ & 0.481$_{\pm0.106}$ & 4.0$_{\pm0.1}$ & 0.0 & 0.0$_{\pm0.0}$ & 9.22$_{\pm1.82}$ & 0.481$_{\pm0.105}$ & 4.1$_{\pm0.1}$\\
  & Final Proj. & 30.0 & 11.0$_{\pm26.4}$ & 5.93$_{\pm4.23}$ & 0.485$_{\pm0.109}$ & 6.7$_{\pm2.3}$ & 35.0 & 11.1$_{\pm23.1}$ & 4.71$_{\pm3.34}$ & 0.485$_{\pm0.109}$ & 9.1$_{\pm0.8}$ & 35.0 & 10.9$_{\pm23.2}$ & 4.14$_{\pm2.88}$ & 0.485$_{\pm0.109}$ & 9.1$_{\pm0.7}$\\
  & DiRecT & \textbf{100.0} & \textbf{100.0}$_{\pm0.0}$ & \textbf{0.00}$_{\pm0.00}$ & 0.482$_{\pm0.109}$ & 60.0$_{\pm2.7}$ & \textbf{100.0} & 99.9$_{\pm0.3}$ & 0.00$_{\pm0.01}$ & 0.482$_{\pm0.109}$ & 69.0$_{\pm10.6}$ & \textbf{100.0} & 99.7$_{\pm1.4}$ & 0.01$_{\pm0.03}$ & 0.482$_{\pm0.109}$ & 70.2$_{\pm10.7}$\\
  & GlobMP (Ours) & 30.0 & 30.0$_{\pm45.8}$ & 9.00$_{\pm7.40}$ & \textbf{0.487}$_{\pm0.500}$ & 300.1$_{\pm136.1}$ & \textbf{100.0} & \textbf{100.0}$_{\pm0.0}$ & \textbf{0.00}$_{\pm0.00}$ & 0.467$_{\pm0.499}$ & 2.1$_{\pm74.3}$ & \textbf{100.0} & \textbf{100.0}$_{\pm0.0}$ & \textbf{0.00}$_{\pm0.00}$ & 0.467$_{\pm0.499}$ & 2.6$_{\pm61.0}$\\
\midrule
18 & Diff.\ Policy$^{\dagger}$ & 0.0 & 0.0$_{\pm0.0}$ & 19.84$_{\pm5.64}$ & 0.504$_{\pm0.119}$ & 0.3 & 0.0 & 0.0$_{\pm0.0}$ & 19.84$_{\pm5.64}$ & 0.504$_{\pm0.119}$ & 0.3 & 0.0 & 0.0$_{\pm0.0}$ & 19.84$_{\pm5.64}$ & 0.504$_{\pm0.119}$ & 0.3\\
  & CD$^{\dagger}$ & 0.0 & 0.0$_{\pm0.0}$ & 39.00$_{\pm3.90}$ & 0.508$_{\pm0.116}$ & 1.3$_{\pm0.1}$ & 0.0 & 0.0$_{\pm0.0}$ & 39.00$_{\pm3.90}$ & 0.508$_{\pm0.116}$ & 1.3$_{\pm0.1}$ & 0.0 & 0.0$_{\pm0.0}$ & 39.00$_{\pm3.90}$ & 0.508$_{\pm0.116}$ & 1.3$_{\pm0.1}$\\
  & MMD-CBS$^{\dagger}$ & 0.0 & 0.0$_{\pm0.0}$ & --- & 0.503$_{\pm0.121}$ & 64.8$_{\pm2.5}$ & 0.0 & 0.0$_{\pm0.0}$ & --- & 0.503$_{\pm0.121}$ & 64.8$_{\pm2.5}$ & 0.0 & 0.0$_{\pm0.0}$ & --- & 0.503$_{\pm0.121}$ & 64.8$_{\pm2.5}$\\
  & MDOC$^{\dagger}$ & 0.0 & --- & --- & 0.000$_{\pm0.000}$ & 163.3$_{\pm8.1}$ & 0.0 & --- & --- & 0.000$_{\pm0.000}$ & 159.7$_{\pm7.1}$ & 0.0 & --- & --- & 0.000$_{\pm0.000}$ & 160.6$_{\pm7.0}$\\
  & PCD-LB & 0.0 & 0.0$_{\pm0.0}$ & 20.39$_{\pm2.78}$ & \textbf{0.511}$_{\pm0.110}$ & 5.0$_{\pm0.5}$ & 0.0 & 0.0$_{\pm0.0}$ & 25.36$_{\pm2.89}$ & 0.511$_{\pm0.109}$ & 4.2$_{\pm0.2}$ & 0.0 & 0.0$_{\pm0.0}$ & 28.30$_{\pm2.87}$ & 0.511$_{\pm0.108}$ & 4.1$_{\pm0.1}$\\
  & PCD-SHD & 0.0 & 0.0$_{\pm0.0}$ & 15.17$_{\pm2.97}$ & 0.508$_{\pm0.114}$ & 4.1$_{\pm0.1}$ & 0.0 & 0.0$_{\pm0.0}$ & 14.31$_{\pm2.75}$ & 0.509$_{\pm0.114}$ & 4.1$_{\pm0.1}$ & 0.0 & 0.0$_{\pm0.0}$ & 14.20$_{\pm2.61}$ & 0.509$_{\pm0.113}$ & 4.2$_{\pm0.1}$\\
  & Final Proj. & 15.0 & 3.1$_{\pm12.2}$ & 10.49$_{\pm5.78}$ & 0.504$_{\pm0.118}$ & 7.5$_{\pm2.3}$ & 25.0 & 4.7$_{\pm17.3}$ & 8.44$_{\pm5.26}$ & 0.504$_{\pm0.118}$ & 11.5$_{\pm0.7}$ & 35.0 & 4.8$_{\pm14.9}$ & 7.71$_{\pm4.86}$ & 0.504$_{\pm0.118}$ & 10.0$_{\pm2.6}$\\
  & DiRecT & \textbf{100.0} & \textbf{100.0}$_{\pm0.0}$ & \textbf{0.00}$_{\pm0.00}$ & 0.507$_{\pm0.118}$ & 87.5$_{\pm9.0}$ & \textbf{100.0} & \textbf{99.8}$_{\pm0.7}$ & \textbf{0.01}$_{\pm0.02}$ & 0.507$_{\pm0.118}$ & 88.4$_{\pm12.7}$ & \textbf{100.0} & \textbf{99.9}$_{\pm0.3}$ & \textbf{0.00}$_{\pm0.01}$ & 0.507$_{\pm0.118}$ & 77.5$_{\pm8.1}$\\
  & GlobMP (Ours) & 0.0 & 0.0$_{\pm0.0}$ & 19.10$_{\pm5.20}$ & 0.503$_{\pm0.500}$ & 300.1$_{\pm0.3}$ & 90.0 & 90.0$_{\pm30.0}$ & 2.85$_{\pm8.79}$ & \textbf{0.525}$_{\pm0.499}$ & 9.2$_{\pm101.2}$ & 90.0 & 90.0$_{\pm30.0}$ & 3.55$_{\pm10.66}$ & \textbf{0.519}$_{\pm0.500}$ & 4.3$_{\pm101.1}$\\
\midrule
21 & Diff.\ Policy$^{\dagger}$ & 0.0 & 0.0$_{\pm0.0}$ & 23.25$_{\pm7.73}$ & 0.500$_{\pm0.121}$ & 0.3 & 0.0 & 0.0$_{\pm0.0}$ & 23.25$_{\pm7.73}$ & 0.500$_{\pm0.121}$ & 0.3 & 0.0 & 0.0$_{\pm0.0}$ & 23.25$_{\pm7.73}$ & 0.500$_{\pm0.121}$ & 0.3\\
  & CD$^{\dagger}$ & 0.0 & 0.0$_{\pm0.0}$ & 52.34$_{\pm7.72}$ & 0.500$_{\pm0.118}$ & 0.7$_{\pm0.1}$ & 0.0 & 0.0$_{\pm0.0}$ & 52.34$_{\pm7.72}$ & 0.500$_{\pm0.118}$ & 1.3$_{\pm0.1}$ & 0.0 & 0.0$_{\pm0.0}$ & 52.34$_{\pm7.72}$ & 0.500$_{\pm0.118}$ & 1.3$_{\pm0.0}$\\
  & MMD-CBS$^{\dagger}$ & 0.0 & 0.0$_{\pm0.0}$ & --- & 0.502$_{\pm0.122}$ & 66.3$_{\pm4.0}$ & 0.0 & 0.0$_{\pm0.0}$ & --- & 0.502$_{\pm0.122}$ & 66.3$_{\pm4.0}$ & 0.0 & 0.0$_{\pm0.0}$ & --- & 0.502$_{\pm0.122}$ & 66.3$_{\pm4.0}$\\
  & MDOC$^{\dagger}$ & 0.0 & --- & --- & 0.000$_{\pm0.000}$ & 165.8$_{\pm8.8}$ & 0.0 & --- & --- & 0.000$_{\pm0.000}$ & 183.6$_{\pm8.1}$ & 0.0 & --- & --- & 0.000$_{\pm0.000}$ & 180.1$_{\pm8.2}$\\
  & PCD-LB & 0.0 & 0.0$_{\pm0.0}$ & 33.40$_{\pm2.61}$ & 0.502$_{\pm0.113}$ & 5.5$_{\pm0.8}$ & 0.0 & 0.0$_{\pm0.0}$ & 40.60$_{\pm3.30}$ & 0.502$_{\pm0.113}$ & 4.7$_{\pm0.3}$ & 0.0 & 0.0$_{\pm0.0}$ & 44.22$_{\pm3.72}$ & 0.502$_{\pm0.113}$ & 4.7$_{\pm0.1}$\\
  & PCD-SHD & 0.0 & 0.0$_{\pm0.0}$ & 19.55$_{\pm3.93}$ & 0.500$_{\pm0.115}$ & 4.5$_{\pm0.1}$ & 0.0 & 0.0$_{\pm0.0}$ & 18.93$_{\pm3.44}$ & 0.500$_{\pm0.115}$ & 4.6$_{\pm0.1}$ & 0.0 & 0.0$_{\pm0.0}$ & 18.51$_{\pm3.52}$ & 0.500$_{\pm0.115}$ & 4.7$_{\pm0.1}$\\
  & Final Proj. & 25.0 & 3.2$_{\pm8.2}$ & 7.22$_{\pm6.39}$ & 0.501$_{\pm0.122}$ & 14.9$_{\pm0.7}$ & 40.0 & 5.4$_{\pm11.2}$ & 5.81$_{\pm5.73}$ & 0.501$_{\pm0.122}$ & 14.9$_{\pm0.8}$ & 45.0 & 6.6$_{\pm14.1}$ & 5.29$_{\pm5.39}$ & 0.501$_{\pm0.122}$ & 10.2$_{\pm3.4}$\\
  & DiRecT & \textbf{100.0} & \textbf{99.9}$_{\pm0.3}$ & \textbf{0.00}$_{\pm0.00}$ & 0.501$_{\pm0.123}$ & 99.2$_{\pm1.5}$ & \textbf{100.0} & \textbf{99.8}$_{\pm0.5}$ & \textbf{0.00}$_{\pm0.01}$ & 0.501$_{\pm0.123}$ & 97.0$_{\pm8.7}$ & \textbf{100.0} & \textbf{100.0}$_{\pm0.0}$ & \textbf{0.00}$_{\pm0.00}$ & 0.501$_{\pm0.123}$ & 77.5$_{\pm0.7}$\\
  & GlobMP (Ours) & 15.0 & 15.0$_{\pm35.7}$ & 20.90$_{\pm11.92}$ & \textbf{0.502}$_{\pm0.500}$ & 300.1$_{\pm68.7}$ & 55.0 & 55.0$_{\pm49.7}$ & 11.35$_{\pm14.57}$ & \textbf{0.519}$_{\pm0.500}$ & 35.2$_{\pm143.5}$ & 65.0 & 65.0$_{\pm47.7}$ & 10.75$_{\pm15.56}$ & \textbf{0.507}$_{\pm0.500}$ & 19.1$_{\pm136.7}$\\
\bottomrule
\end{tabular}}
\end{center}
\end{table}

\begin{table}[p]
\caption{Multi-robot planning performance on \texttt{Drop-Region} for $N_a\in\{3,6,\dots,21\}$ under the Low, Medium and High velocity caps. \emph{SR} and \emph{CS} measure trajectory feasibility, \emph{Coll} reports inter-robot collisions, \emph{DA} adherence to the environment-specific motion pattern, and \emph{Time} the planning time per configuration: for \emph{GlobMP (Ours)} the median over the configurations with the standard deviation as subscript, for the baselines the mean, so the two are not directly comparable. $^{\dagger}$ marks a method that applies no velocity limit, so its three blocks report the same run; MMD-CBS also returns one plan per configuration rather than $64$ candidates. A dash under \emph{Coll} marks a count that is not commensurable with the others, as MDOC and MMD-CBS count colliding time steps where every other row counts agent pairs with any collision; MDOC's \emph{CS} is omitted because its plans that do not cover every agent were counted as feasible. MMD-CBS plans under $60$\,s per configuration and \emph{GlobMP (Ours)} under $300$\,s per team.}
\label{tab:hard-shelf}
\begin{center}\tiny\setlength{\tabcolsep}{2.2pt}
\resizebox{\linewidth}{!}{%
\begin{tabular}{ll ccccc ccccc ccccc}
\toprule
& & \multicolumn{5}{c}{Low ($v_{\max}=0.928$)} & \multicolumn{5}{c}{Medium ($v_{\max}=1.13$)} & \multicolumn{5}{c}{High ($v_{\max}=1.34$)}\\
\cmidrule(lr){3-7}\cmidrule(lr){8-12}\cmidrule(lr){13-17}
$N_a$ & Method & SR$\uparrow$ & CS$\uparrow$ & Coll$\downarrow$ & DA$\uparrow$ & Time$\downarrow$ & SR$\uparrow$ & CS$\uparrow$ & Coll$\downarrow$ & DA$\uparrow$ & Time$\downarrow$ & SR$\uparrow$ & CS$\uparrow$ & Coll$\downarrow$ & DA$\uparrow$ & Time$\downarrow$\\
\midrule
3 & Diff.\ Policy$^{\dagger}$ & 0.0 & 0.0$_{\pm0.0}$ & 0.13$_{\pm0.34}$ & 0.139$_{\pm0.175}$ & 0.3 & 0.0 & 0.0$_{\pm0.0}$ & 0.13$_{\pm0.34}$ & 0.139$_{\pm0.175}$ & 0.3 & 0.0 & 0.0$_{\pm0.0}$ & 0.13$_{\pm0.34}$ & 0.139$_{\pm0.175}$ & 0.3\\
  & CD$^{\dagger}$ & 45.0 & 5.2$_{\pm21.7}$ & 0.11$_{\pm0.32}$ & 0.000$_{\pm0.001}$ & 0.5$_{\pm0.0}$ & 45.0 & 11.9$_{\pm24.6}$ & 0.11$_{\pm0.32}$ & 0.000$_{\pm0.001}$ & 1.0$_{\pm0.2}$ & 45.0 & 22.2$_{\pm34.5}$ & 0.11$_{\pm0.32}$ & 0.000$_{\pm0.001}$ & 1.2$_{\pm0.0}$\\
  & MMD-CBS$^{\dagger}$ & 30.0 & 25.0$_{\pm43.3}$ & --- & 0.000$_{\pm0.000}$ & 2.2$_{\pm1.5}$ & 30.0 & 30.0$_{\pm45.8}$ & --- & 0.000$_{\pm0.000}$ & 2.2$_{\pm1.5}$ & 30.0 & 30.0$_{\pm45.8}$ & --- & 0.000$_{\pm0.000}$ & 2.2$_{\pm1.5}$\\
  & MDOC$^{\dagger}$ & 25.0 & --- & --- & 0.000$_{\pm0.000}$ & 42.7$_{\pm6.5}$ & 25.0 & --- & --- & 0.000$_{\pm0.000}$ & 33.2$_{\pm2.7}$ & 25.0 & --- & --- & 0.000$_{\pm0.000}$ & 38.6$_{\pm3.8}$\\
  & PCD-LB & 75.0 & 55.9$_{\pm44.1}$ & 0.00$_{\pm0.01}$ & 0.000$_{\pm0.000}$ & 3.1$_{\pm0.1}$ & 75.0 & 56.8$_{\pm44.0}$ & 0.01$_{\pm0.02}$ & 0.000$_{\pm0.000}$ & 2.7$_{\pm0.6}$ & 80.0 & 57.0$_{\pm43.5}$ & 0.01$_{\pm0.03}$ & 0.000$_{\pm0.000}$ & 3.0$_{\pm0.1}$\\
  & PCD-SHD & 30.0 & 11.3$_{\pm28.3}$ & 0.09$_{\pm0.23}$ & 0.047$_{\pm0.094}$ & 1.2$_{\pm0.1}$ & 20.0 & 11.2$_{\pm28.4}$ & 0.08$_{\pm0.21}$ & 0.049$_{\pm0.094}$ & 1.2$_{\pm0.1}$ & 20.0 & 11.3$_{\pm28.5}$ & 0.06$_{\pm0.15}$ & 0.049$_{\pm0.092}$ & 1.2$_{\pm0.1}$\\
  & Final Proj. & 95.0 & 86.2$_{\pm26.3}$ & 0.11$_{\pm0.28}$ & \textbf{0.194}$_{\pm0.211}$ & 3.4$_{\pm0.5}$ & 95.0 & 86.2$_{\pm26.4}$ & 0.11$_{\pm0.28}$ & \textbf{0.204}$_{\pm0.215}$ & 3.6$_{\pm0.6}$ & \textbf{100.0} & 96.1$_{\pm5.1}$ & 0.01$_{\pm0.03}$ & \textbf{0.234}$_{\pm0.215}$ & 3.6$_{\pm0.6}$\\
  & DiRecT & \textbf{100.0} & 97.5$_{\pm8.1}$ & 0.03$_{\pm0.09}$ & 0.114$_{\pm0.149}$ & 24.6$_{\pm6.3}$ & \textbf{100.0} & 96.3$_{\pm12.5}$ & 0.04$_{\pm0.13}$ & 0.115$_{\pm0.150}$ & 26.2$_{\pm6.6}$ & \textbf{100.0} & \textbf{100.0}$_{\pm0.0}$ & \textbf{0.00}$_{\pm0.00}$ & 0.115$_{\pm0.150}$ & 26.2$_{\pm6.9}$\\
  & GlobMP (Ours) & \textbf{100.0} & \textbf{100.0}$_{\pm0.0}$ & \textbf{0.00}$_{\pm0.00}$ & 0.083$_{\pm0.276}$ & 0.3$_{\pm0.1}$ & \textbf{100.0} & \textbf{100.0}$_{\pm0.0}$ & \textbf{0.00}$_{\pm0.00}$ & 0.083$_{\pm0.276}$ & 0.5$_{\pm0.2}$ & \textbf{100.0} & \textbf{100.0}$_{\pm0.0}$ & \textbf{0.00}$_{\pm0.00}$ & 0.083$_{\pm0.276}$ & 0.4$_{\pm0.2}$\\
\midrule
6 & Diff.\ Policy$^{\dagger}$ & 0.0 & 0.0$_{\pm0.0}$ & 1.16$_{\pm0.91}$ & 0.148$_{\pm0.107}$ & 0.3 & 0.0 & 0.0$_{\pm0.0}$ & 1.16$_{\pm0.91}$ & 0.148$_{\pm0.107}$ & 0.3 & 0.0 & 0.0$_{\pm0.0}$ & 1.16$_{\pm0.91}$ & 0.148$_{\pm0.107}$ & 0.3\\
  & CD$^{\dagger}$ & 15.0 & 0.0$_{\pm0.0}$ & 0.91$_{\pm0.78}$ & 0.001$_{\pm0.002}$ & 1.2$_{\pm0.0}$ & 15.0 & 0.0$_{\pm0.0}$ & 0.91$_{\pm0.78}$ & 0.001$_{\pm0.002}$ & 1.2$_{\pm0.1}$ & 15.0 & 0.4$_{\pm1.0}$ & 0.91$_{\pm0.78}$ & 0.001$_{\pm0.002}$ & 1.1$_{\pm0.1}$\\
  & MMD-CBS$^{\dagger}$ & 0.0 & 0.0$_{\pm0.0}$ & --- & --- & 2.2$_{\pm1.3}$ & 0.0 & 0.0$_{\pm0.0}$ & --- & --- & 2.2$_{\pm1.3}$ & 0.0 & 0.0$_{\pm0.0}$ & --- & --- & 2.2$_{\pm1.3}$\\
  & MDOC$^{\dagger}$ & 0.0 & --- & --- & 0.000$_{\pm0.000}$ & 62.9$_{\pm3.7}$ & 0.0 & --- & --- & 0.000$_{\pm0.000}$ & 83.6$_{\pm24.8}$ & 0.0 & --- & --- & 0.000$_{\pm0.000}$ & 83.0$_{\pm17.9}$\\
  & PCD-LB & 5.0 & 0.1$_{\pm0.3}$ & 1.32$_{\pm0.80}$ & 0.000$_{\pm0.000}$ & 3.2$_{\pm0.2}$ & 10.0 & 0.4$_{\pm1.4}$ & 1.36$_{\pm0.81}$ & 0.000$_{\pm0.000}$ & 3.1$_{\pm0.1}$ & 20.0 & 0.9$_{\pm2.3}$ & 1.31$_{\pm0.79}$ & 0.000$_{\pm0.000}$ & 3.1$_{\pm0.1}$\\
  & PCD-SHD & 5.0 & 0.1$_{\pm0.3}$ & 0.24$_{\pm0.35}$ & 0.027$_{\pm0.047}$ & 1.3$_{\pm0.1}$ & 0.0 & 0.0$_{\pm0.0}$ & 0.21$_{\pm0.30}$ & 0.028$_{\pm0.049}$ & 1.3$_{\pm0.1}$ & 0.0 & 0.0$_{\pm0.0}$ & 0.17$_{\pm0.23}$ & 0.027$_{\pm0.049}$ & 1.4$_{\pm0.1}$\\
  & Final Proj. & \textbf{100.0} & 47.8$_{\pm35.8}$ & 0.56$_{\pm0.55}$ & \textbf{0.161}$_{\pm0.127}$ & 3.4$_{\pm0.6}$ & 90.0 & 45.7$_{\pm37.1}$ & 0.57$_{\pm0.50}$ & \textbf{0.168}$_{\pm0.127}$ & 3.4$_{\pm0.6}$ & \textbf{100.0} & 62.5$_{\pm32.4}$ & 0.33$_{\pm0.40}$ & \textbf{0.212}$_{\pm0.137}$ & 3.4$_{\pm0.7}$\\
  & DiRecT & \textbf{100.0} & 98.2$_{\pm6.1}$ & 0.02$_{\pm0.06}$ & 0.073$_{\pm0.106}$ & 24.9$_{\pm6.8}$ & \textbf{100.0} & 97.8$_{\pm6.2}$ & 0.02$_{\pm0.06}$ & 0.075$_{\pm0.106}$ & 25.7$_{\pm6.8}$ & \textbf{100.0} & 99.8$_{\pm0.7}$ & 0.00$_{\pm0.01}$ & 0.076$_{\pm0.107}$ & 26.7$_{\pm7.4}$\\
  & GlobMP (Ours) & \textbf{100.0} & \textbf{100.0}$_{\pm0.0}$ & \textbf{0.00}$_{\pm0.00}$ & 0.117$_{\pm0.321}$ & 0.6$_{\pm0.3}$ & \textbf{100.0} & \textbf{100.0}$_{\pm0.0}$ & \textbf{0.00}$_{\pm0.00}$ & 0.125$_{\pm0.331}$ & 0.9$_{\pm1.1}$ & \textbf{100.0} & \textbf{100.0}$_{\pm0.0}$ & \textbf{0.00}$_{\pm0.00}$ & 0.133$_{\pm0.340}$ & 1.1$_{\pm3.6}$\\
\midrule
9 & Diff.\ Policy$^{\dagger}$ & 0.0 & 0.0$_{\pm0.0}$ & 2.92$_{\pm1.22}$ & 0.124$_{\pm0.096}$ & 0.3 & 0.0 & 0.0$_{\pm0.0}$ & 2.92$_{\pm1.22}$ & 0.124$_{\pm0.096}$ & 0.3 & 0.0 & 0.0$_{\pm0.0}$ & 2.92$_{\pm1.22}$ & 0.124$_{\pm0.096}$ & 0.3\\
  & CD$^{\dagger}$ & 0.0 & 0.0$_{\pm0.0}$ & 3.40$_{\pm1.07}$ & 0.001$_{\pm0.002}$ & 1.2$_{\pm0.1}$ & 0.0 & 0.0$_{\pm0.0}$ & 3.40$_{\pm1.07}$ & 0.001$_{\pm0.002}$ & 1.2$_{\pm0.0}$ & 0.0 & 0.0$_{\pm0.0}$ & 3.40$_{\pm1.07}$ & 0.001$_{\pm0.002}$ & 1.2$_{\pm0.1}$\\
  & MMD-CBS$^{\dagger}$ & 0.0 & 0.0$_{\pm0.0}$ & --- & --- & 2.3$_{\pm1.9}$ & 0.0 & 0.0$_{\pm0.0}$ & --- & --- & 2.3$_{\pm1.9}$ & 0.0 & 0.0$_{\pm0.0}$ & --- & --- & 2.3$_{\pm1.9}$\\
  & MDOC$^{\dagger}$ & 0.0 & --- & --- & 0.000$_{\pm0.000}$ & 108.4$_{\pm21.4}$ & 0.0 & --- & --- & 0.000$_{\pm0.000}$ & 95.4$_{\pm14.3}$ & 0.0 & --- & --- & 0.000$_{\pm0.000}$ & 83.7$_{\pm3.8}$\\
  & PCD-LB & 0.0 & 0.0$_{\pm0.0}$ & 4.77$_{\pm1.42}$ & 0.000$_{\pm0.000}$ & 3.5$_{\pm0.2}$ & 0.0 & 0.0$_{\pm0.0}$ & 5.00$_{\pm1.45}$ & 0.000$_{\pm0.000}$ & 3.4$_{\pm0.1}$ & 0.0 & 0.0$_{\pm0.0}$ & 4.58$_{\pm1.31}$ & 0.000$_{\pm0.000}$ & 3.5$_{\pm0.1}$\\
  & PCD-SHD & 0.0 & 0.0$_{\pm0.0}$ & 1.06$_{\pm0.75}$ & 0.031$_{\pm0.043}$ & 1.5$_{\pm0.1}$ & 0.0 & 0.0$_{\pm0.0}$ & 0.98$_{\pm0.67}$ & 0.034$_{\pm0.044}$ & 1.5$_{\pm0.1}$ & 0.0 & 0.0$_{\pm0.0}$ & 0.78$_{\pm0.56}$ & 0.035$_{\pm0.043}$ & 1.5$_{\pm0.1}$\\
  & Final Proj. & 50.0 & 17.3$_{\pm31.9}$ & 1.92$_{\pm1.82}$ & \textbf{0.180}$_{\pm0.109}$ & 4.3$_{\pm0.7}$ & 40.0 & 13.0$_{\pm26.5}$ & 2.00$_{\pm1.72}$ & \textbf{0.209}$_{\pm0.118}$ & 3.9$_{\pm1.1}$ & 80.0 & 32.2$_{\pm32.6}$ & 1.22$_{\pm1.33}$ & \textbf{0.244}$_{\pm0.123}$ & 4.4$_{\pm0.7}$\\
  & DiRecT & 90.0 & 74.3$_{\pm35.6}$ & 0.28$_{\pm0.40}$ & 0.080$_{\pm0.084}$ & 37.2$_{\pm9.7}$ & 95.0 & 69.5$_{\pm37.5}$ & 0.34$_{\pm0.44}$ & 0.081$_{\pm0.085}$ & 30.0$_{\pm6.7}$ & \textbf{100.0} & \textbf{99.6}$_{\pm1.4}$ & \textbf{0.00}$_{\pm0.01}$ & 0.082$_{\pm0.086}$ & 31.1$_{\pm7.7}$\\
  & GlobMP (Ours) & \textbf{100.0} & \textbf{100.0}$_{\pm0.0}$ & \textbf{0.00}$_{\pm0.00}$ & 0.061$_{\pm0.240}$ & 1.3$_{\pm0.4}$ & \textbf{100.0} & \textbf{100.0}$_{\pm0.0}$ & \textbf{0.00}$_{\pm0.00}$ & 0.067$_{\pm0.249}$ & 2.0$_{\pm2.7}$ & 95.0 & 95.0$_{\pm21.8}$ & 0.35$_{\pm1.53}$ & 0.067$_{\pm0.249}$ & 2.1$_{\pm66.9}$\\
\midrule
12 & Diff.\ Policy$^{\dagger}$ & 0.0 & 0.0$_{\pm0.0}$ & 6.19$_{\pm2.67}$ & 0.142$_{\pm0.077}$ & 0.3 & 0.0 & 0.0$_{\pm0.0}$ & 6.19$_{\pm2.67}$ & 0.142$_{\pm0.077}$ & 0.3 & 0.0 & 0.0$_{\pm0.0}$ & 6.19$_{\pm2.67}$ & 0.142$_{\pm0.077}$ & 0.3\\
  & CD$^{\dagger}$ & 0.0 & 0.0$_{\pm0.0}$ & 7.84$_{\pm1.65}$ & 0.002$_{\pm0.005}$ & 1.2$_{\pm0.1}$ & 0.0 & 0.0$_{\pm0.0}$ & 7.84$_{\pm1.65}$ & 0.002$_{\pm0.005}$ & 0.6$_{\pm0.1}$ & 0.0 & 0.0$_{\pm0.0}$ & 7.84$_{\pm1.65}$ & 0.002$_{\pm0.005}$ & 1.2$_{\pm0.1}$\\
  & MMD-CBS$^{\dagger}$ & 0.0 & 0.0$_{\pm0.0}$ & --- & --- & 1.7$_{\pm0.9}$ & 0.0 & 0.0$_{\pm0.0}$ & --- & --- & 1.7$_{\pm0.9}$ & 0.0 & 0.0$_{\pm0.0}$ & --- & --- & 1.7$_{\pm0.9}$\\
  & MDOC$^{\dagger}$ & 0.0 & --- & --- & 0.000$_{\pm0.000}$ & 110.6$_{\pm10.0}$ & 0.0 & --- & --- & 0.000$_{\pm0.000}$ & 121.3$_{\pm5.9}$ & 0.0 & --- & --- & 0.000$_{\pm0.000}$ & 107.6$_{\pm5.6}$\\
  & PCD-LB & 0.0 & 0.0$_{\pm0.0}$ & 11.55$_{\pm2.61}$ & 0.000$_{\pm0.002}$ & 1.7$_{\pm0.1}$ & 0.0 & 0.0$_{\pm0.0}$ & 11.87$_{\pm2.46}$ & 0.000$_{\pm0.001}$ & 3.7$_{\pm0.1}$ & 0.0 & 0.0$_{\pm0.0}$ & 11.07$_{\pm2.25}$ & 0.000$_{\pm0.002}$ & 3.8$_{\pm0.1}$\\
  & PCD-SHD & 0.0 & 0.0$_{\pm0.0}$ & 2.48$_{\pm1.14}$ & 0.021$_{\pm0.022}$ & 1.6$_{\pm0.1}$ & 0.0 & 0.0$_{\pm0.0}$ & 2.31$_{\pm1.05}$ & 0.022$_{\pm0.023}$ & 1.6$_{\pm0.1}$ & 0.0 & 0.0$_{\pm0.0}$ & 1.88$_{\pm0.92}$ & 0.024$_{\pm0.023}$ & 1.6$_{\pm0.1}$\\
  & Final Proj. & 10.0 & 0.8$_{\pm3.1}$ & 4.75$_{\pm2.81}$ & \textbf{0.175}$_{\pm0.078}$ & 5.4$_{\pm0.7}$ & 5.0 & 0.7$_{\pm3.1}$ & 4.67$_{\pm2.64}$ & \textbf{0.206}$_{\pm0.089}$ & 5.4$_{\pm0.7}$ & 30.0 & 10.5$_{\pm24.4}$ & 2.49$_{\pm1.41}$ & \textbf{0.255}$_{\pm0.096}$ & 5.4$_{\pm0.7}$\\
  & DiRecT & 55.0 & 32.3$_{\pm37.5}$ & 1.09$_{\pm0.78}$ & 0.054$_{\pm0.048}$ & 41.0$_{\pm10.3}$ & 55.0 & 30.4$_{\pm37.4}$ & 1.21$_{\pm0.85}$ & 0.056$_{\pm0.050}$ & 45.0$_{\pm12.0}$ & \textbf{100.0} & \textbf{98.4}$_{\pm2.7}$ & \textbf{0.02}$_{\pm0.03}$ & 0.057$_{\pm0.051}$ & 47.4$_{\pm12.7}$\\
  & GlobMP (Ours) & \textbf{100.0} & \textbf{100.0}$_{\pm0.0}$ & \textbf{0.00}$_{\pm0.00}$ & 0.067$_{\pm0.249}$ & 2.0$_{\pm8.3}$ & \textbf{100.0} & \textbf{100.0}$_{\pm0.0}$ & \textbf{0.00}$_{\pm0.00}$ & 0.071$_{\pm0.257}$ & 3.2$_{\pm1.6}$ & 95.0 & 95.0$_{\pm21.8}$ & 0.50$_{\pm2.18}$ & 0.071$_{\pm0.257}$ & 4.0$_{\pm78.9}$\\
\midrule
15 & Diff.\ Policy$^{\dagger}$ & 0.0 & 0.0$_{\pm0.0}$ & 8.53$_{\pm2.25}$ & 0.183$_{\pm0.062}$ & 0.3 & 0.0 & 0.0$_{\pm0.0}$ & 8.53$_{\pm2.25}$ & 0.183$_{\pm0.062}$ & 0.3 & 0.0 & 0.0$_{\pm0.0}$ & 8.53$_{\pm2.25}$ & 0.183$_{\pm0.062}$ & 0.3\\
  & CD$^{\dagger}$ & 0.0 & 0.0$_{\pm0.0}$ & 13.91$_{\pm2.49}$ & 0.001$_{\pm0.002}$ & 1.2$_{\pm0.1}$ & 0.0 & 0.0$_{\pm0.0}$ & 13.91$_{\pm2.49}$ & 0.001$_{\pm0.002}$ & 1.2$_{\pm0.0}$ & 0.0 & 0.0$_{\pm0.0}$ & 13.91$_{\pm2.49}$ & 0.001$_{\pm0.002}$ & 1.2$_{\pm0.1}$\\
  & MMD-CBS$^{\dagger}$ & 0.0 & 0.0$_{\pm0.0}$ & --- & --- & 1.2$_{\pm0.8}$ & 0.0 & 0.0$_{\pm0.0}$ & --- & --- & 1.2$_{\pm0.8}$ & 0.0 & 0.0$_{\pm0.0}$ & --- & --- & 1.2$_{\pm0.8}$\\
  & MDOC$^{\dagger}$ & 0.0 & --- & --- & 0.000$_{\pm0.000}$ & 158.3$_{\pm11.1}$ & 0.0 & --- & --- & 0.000$_{\pm0.000}$ & 175.3$_{\pm13.5}$ & 0.0 & --- & --- & 0.000$_{\pm0.000}$ & 167.0$_{\pm72.6}$\\
  & PCD-LB & 0.0 & 0.0$_{\pm0.0}$ & 22.65$_{\pm2.75}$ & 0.000$_{\pm0.000}$ & 4.1$_{\pm0.2}$ & 0.0 & 0.0$_{\pm0.0}$ & 23.62$_{\pm2.63}$ & 0.000$_{\pm0.000}$ & 4.1$_{\pm0.1}$ & 0.0 & 0.0$_{\pm0.0}$ & 22.71$_{\pm2.48}$ & 0.000$_{\pm0.000}$ & 4.1$_{\pm0.1}$\\
  & PCD-SHD & 0.0 & 0.0$_{\pm0.0}$ & 2.78$_{\pm1.30}$ & 0.026$_{\pm0.028}$ & 1.7$_{\pm0.1}$ & 0.0 & 0.0$_{\pm0.0}$ & 2.60$_{\pm1.24}$ & 0.025$_{\pm0.027}$ & 1.8$_{\pm0.1}$ & 0.0 & 0.0$_{\pm0.0}$ & 2.11$_{\pm1.00}$ & 0.025$_{\pm0.027}$ & 1.8$_{\pm0.1}$\\
  & Final Proj. & 15.0 & 3.2$_{\pm12.9}$ & 5.40$_{\pm2.82}$ & \textbf{0.211}$_{\pm0.071}$ & 7.0$_{\pm0.7}$ & 15.0 & 2.0$_{\pm7.8}$ & 5.43$_{\pm2.64}$ & \textbf{0.254}$_{\pm0.074}$ & 4.7$_{\pm1.8}$ & 50.0 & 7.0$_{\pm18.1}$ & 2.96$_{\pm1.97}$ & \textbf{0.285}$_{\pm0.076}$ & 7.0$_{\pm0.7}$\\
  & DiRecT & 70.0 & 41.2$_{\pm40.4}$ & 1.16$_{\pm1.06}$ & 0.083$_{\pm0.037}$ & 57.0$_{\pm14.4}$ & 65.0 & 38.3$_{\pm40.2}$ & 1.28$_{\pm1.12}$ & 0.085$_{\pm0.036}$ & 51.8$_{\pm7.9}$ & \textbf{100.0} & 97.5$_{\pm7.6}$ & 0.03$_{\pm0.08}$ & 0.087$_{\pm0.036}$ & 52.8$_{\pm7.9}$\\
  & GlobMP (Ours) & \textbf{100.0} & \textbf{100.0}$_{\pm0.0}$ & \textbf{0.00}$_{\pm0.00}$ & 0.047$_{\pm0.211}$ & 2.3$_{\pm2.3}$ & \textbf{100.0} & \textbf{100.0}$_{\pm0.0}$ & \textbf{0.00}$_{\pm0.00}$ & 0.063$_{\pm0.244}$ & 4.3$_{\pm3.1}$ & \textbf{100.0} & \textbf{100.0}$_{\pm0.0}$ & \textbf{0.00}$_{\pm0.00}$ & 0.067$_{\pm0.249}$ & 11.8$_{\pm18.1}$\\
\midrule
18 & Diff.\ Policy$^{\dagger}$ & 0.0 & 0.0$_{\pm0.0}$ & 14.42$_{\pm4.49}$ & 0.124$_{\pm0.051}$ & 0.3 & 0.0 & 0.0$_{\pm0.0}$ & 14.42$_{\pm4.49}$ & 0.124$_{\pm0.051}$ & 0.3 & 0.0 & 0.0$_{\pm0.0}$ & 14.42$_{\pm4.49}$ & 0.124$_{\pm0.051}$ & 0.3\\
  & CD$^{\dagger}$ & 0.0 & 0.0$_{\pm0.0}$ & 23.43$_{\pm2.96}$ & 0.002$_{\pm0.004}$ & 1.2$_{\pm0.1}$ & 0.0 & 0.0$_{\pm0.0}$ & 23.43$_{\pm2.96}$ & 0.002$_{\pm0.004}$ & 1.3$_{\pm0.1}$ & 0.0 & 0.0$_{\pm0.0}$ & 23.43$_{\pm2.96}$ & 0.002$_{\pm0.004}$ & 1.2$_{\pm0.1}$\\
  & MMD-CBS$^{\dagger}$ & 0.0 & 0.0$_{\pm0.0}$ & --- & --- & 2.1$_{\pm1.2}$ & 0.0 & 0.0$_{\pm0.0}$ & --- & --- & 2.1$_{\pm1.2}$ & 0.0 & 0.0$_{\pm0.0}$ & --- & --- & 2.1$_{\pm1.2}$\\
  & MDOC$^{\dagger}$ & 0.0 & --- & --- & 0.000$_{\pm0.000}$ & 166.5$_{\pm9.9}$ & 0.0 & --- & --- & 0.000$_{\pm0.000}$ & 159.0$_{\pm7.7}$ & 0.0 & --- & --- & 0.000$_{\pm0.000}$ & 142.2$_{\pm7.0}$\\
  & PCD-LB & 0.0 & 0.0$_{\pm0.0}$ & 37.15$_{\pm5.84}$ & 0.000$_{\pm0.000}$ & 4.1$_{\pm0.1}$ & 0.0 & 0.0$_{\pm0.0}$ & 38.50$_{\pm5.83}$ & 0.000$_{\pm0.000}$ & 4.2$_{\pm0.1}$ & 0.0 & 0.0$_{\pm0.0}$ & 36.52$_{\pm5.40}$ & 0.000$_{\pm0.000}$ & 4.0$_{\pm0.6}$\\
  & PCD-SHD & 0.0 & 0.0$_{\pm0.0}$ & 4.46$_{\pm2.06}$ & 0.020$_{\pm0.023}$ & 1.8$_{\pm0.1}$ & 0.0 & 0.0$_{\pm0.0}$ & 4.29$_{\pm1.94}$ & 0.022$_{\pm0.026}$ & 1.8$_{\pm0.1}$ & 0.0 & 0.0$_{\pm0.0}$ & 3.56$_{\pm1.70}$ & 0.023$_{\pm0.026}$ & 1.8$_{\pm0.1}$\\
  & Final Proj. & 10.0 & 1.5$_{\pm4.8}$ & 8.88$_{\pm5.02}$ & \textbf{0.140}$_{\pm0.049}$ & 8.6$_{\pm0.8}$ & 10.0 & 0.9$_{\pm2.6}$ & 8.67$_{\pm4.89}$ & \textbf{0.166}$_{\pm0.052}$ & 7.4$_{\pm1.7}$ & 10.0 & 0.7$_{\pm2.2}$ & 5.66$_{\pm2.88}$ & \textbf{0.199}$_{\pm0.066}$ & 8.6$_{\pm0.7}$\\
  & DiRecT & 60.0 & 31.9$_{\pm37.5}$ & 2.05$_{\pm2.15}$ & 0.059$_{\pm0.043}$ & 68.0$_{\pm9.8}$ & 55.0 & 30.5$_{\pm37.9}$ & 2.19$_{\pm2.27}$ & 0.060$_{\pm0.042}$ & 64.9$_{\pm9.0}$ & \textbf{100.0} & \textbf{91.6}$_{\pm15.7}$ & \textbf{0.11}$_{\pm0.22}$ & 0.062$_{\pm0.042}$ & 64.4$_{\pm8.0}$\\
  & GlobMP (Ours) & \textbf{100.0} & \textbf{100.0}$_{\pm0.0}$ & \textbf{0.00}$_{\pm0.00}$ & 0.056$_{\pm0.229}$ & 3.6$_{\pm9.4}$ & \textbf{100.0} & \textbf{100.0}$_{\pm0.0}$ & \textbf{0.00}$_{\pm0.00}$ & 0.056$_{\pm0.229}$ & 9.5$_{\pm7.6}$ & 90.0 & 90.0$_{\pm30.0}$ & 2.15$_{\pm6.45}$ & 0.067$_{\pm0.249}$ & 24.2$_{\pm86.0}$\\
\midrule
21 & Diff.\ Policy$^{\dagger}$ & 0.0 & 0.0$_{\pm0.0}$ & 18.84$_{\pm3.94}$ & 0.161$_{\pm0.078}$ & 0.4 & 0.0 & 0.0$_{\pm0.0}$ & 18.84$_{\pm3.94}$ & 0.161$_{\pm0.078}$ & 0.3 & 0.0 & 0.0$_{\pm0.0}$ & 18.84$_{\pm3.94}$ & 0.161$_{\pm0.078}$ & 0.4\\
  & CD$^{\dagger}$ & 0.0 & 0.0$_{\pm0.0}$ & 33.21$_{\pm2.01}$ & 0.002$_{\pm0.003}$ & 0.8$_{\pm0.1}$ & 0.0 & 0.0$_{\pm0.0}$ & 33.21$_{\pm2.01}$ & 0.002$_{\pm0.003}$ & 1.3$_{\pm0.1}$ & 0.0 & 0.0$_{\pm0.0}$ & 33.21$_{\pm2.01}$ & 0.002$_{\pm0.003}$ & 1.3$_{\pm0.1}$\\
  & MMD-CBS$^{\dagger}$ & 0.0 & 0.0$_{\pm0.0}$ & --- & --- & 2.6$_{\pm1.7}$ & 0.0 & 0.0$_{\pm0.0}$ & --- & --- & 2.6$_{\pm1.7}$ & 0.0 & 0.0$_{\pm0.0}$ & --- & --- & 2.6$_{\pm1.7}$\\
  & MDOC$^{\dagger}$ & 0.0 & --- & --- & 0.000$_{\pm0.000}$ & 185.4$_{\pm9.2}$ & 0.0 & --- & --- & 0.000$_{\pm0.000}$ & 181.0$_{\pm8.7}$ & 0.0 & --- & --- & 0.000$_{\pm0.000}$ & 178.6$_{\pm8.8}$\\
  & PCD-LB & 0.0 & 0.0$_{\pm0.0}$ & 53.13$_{\pm5.24}$ & 0.000$_{\pm0.000}$ & 4.6$_{\pm0.2}$ & 0.0 & 0.0$_{\pm0.0}$ & 54.89$_{\pm5.28}$ & 0.000$_{\pm0.000}$ & 4.6$_{\pm0.1}$ & 0.0 & 0.0$_{\pm0.0}$ & 51.21$_{\pm4.97}$ & 0.000$_{\pm0.000}$ & 4.6$_{\pm0.1}$\\
  & PCD-SHD & 0.0 & 0.0$_{\pm0.0}$ & 6.27$_{\pm1.83}$ & 0.027$_{\pm0.023}$ & 2.0$_{\pm0.1}$ & 0.0 & 0.0$_{\pm0.0}$ & 5.99$_{\pm1.81}$ & 0.029$_{\pm0.024}$ & 2.0$_{\pm0.1}$ & 0.0 & 0.0$_{\pm0.0}$ & 4.79$_{\pm1.53}$ & 0.030$_{\pm0.024}$ & 2.0$_{\pm0.1}$\\
  & Final Proj. & 0.0 & 0.0$_{\pm0.0}$ & 14.01$_{\pm6.39}$ & \textbf{0.160}$_{\pm0.079}$ & 8.4$_{\pm2.8}$ & 0.0 & 0.0$_{\pm0.0}$ & 13.54$_{\pm6.03}$ & \textbf{0.194}$_{\pm0.084}$ & 11.9$_{\pm0.8}$ & 10.0 & 0.2$_{\pm0.7}$ & 8.83$_{\pm4.51}$ & \textbf{0.232}$_{\pm0.086}$ & 11.8$_{\pm0.7}$\\
  & DiRecT & 20.0 & 12.2$_{\pm26.9}$ & \textbf{3.57}$_{\pm2.00}$ & 0.062$_{\pm0.041}$ & 89.9$_{\pm13.2}$ & 20.0 & 9.7$_{\pm23.9}$ & 3.73$_{\pm2.03}$ & 0.066$_{\pm0.043}$ & 78.6$_{\pm1.2}$ & \textbf{100.0} & 88.0$_{\pm10.0}$ & \textbf{0.14}$_{\pm0.12}$ & 0.070$_{\pm0.046}$ & 79.9$_{\pm1.2}$\\
  & GlobMP (Ours) & \textbf{65.0} & 65.0$_{\pm47.7}$ & 5.45$_{\pm7.65}$ & 0.095$_{\pm0.294}$ & 15.2$_{\pm136.9}$ & \textbf{95.0} & \textbf{95.0}$_{\pm21.8}$ & \textbf{1.05}$_{\pm4.58}$ & 0.086$_{\pm0.280}$ & 15.1$_{\pm64.4}$ & 85.0 & 85.0$_{\pm35.7}$ & 3.95$_{\pm9.98}$ & 0.093$_{\pm0.290}$ & 27.1$_{\pm97.5}$\\
\bottomrule
\end{tabular}}
\end{center}
\end{table}

\begin{table}[p]
\caption{Multi-robot planning performance on \texttt{Corridor-Bay} for $N_a\in\{3,6,\dots,21\}$ under the Low, Medium and High velocity caps. \emph{SR} and \emph{CS} measure trajectory feasibility, \emph{Coll} reports inter-robot collisions, \emph{DA} adherence to the environment-specific motion pattern, and \emph{Time} the planning time per configuration: for \emph{GlobMP (Ours)} the median over the configurations with the standard deviation as subscript, for the baselines the mean, so the two are not directly comparable. $^{\dagger}$ marks a method that applies no velocity limit, so its three blocks report the same run; MMD-CBS also returns one plan per configuration rather than $64$ candidates. A dash under \emph{Coll} marks a count that is not commensurable with the others, as MDOC and MMD-CBS count colliding time steps where every other row counts agent pairs with any collision; MDOC's \emph{CS} is omitted because its plans that do not cover every agent were counted as feasible. MMD-CBS plans under $60$\,s per configuration and \emph{GlobMP (Ours)} under $300$\,s per team.}
\label{tab:hard-corridorbay}
\begin{center}\tiny\setlength{\tabcolsep}{2.2pt}
\resizebox{\linewidth}{!}{%
\begin{tabular}{ll ccccc ccccc ccccc}
\toprule
& & \multicolumn{5}{c}{Low ($v_{\max}=1.21$)} & \multicolumn{5}{c}{Medium ($v_{\max}=1.46$)} & \multicolumn{5}{c}{High ($v_{\max}=1.76$)}\\
\cmidrule(lr){3-7}\cmidrule(lr){8-12}\cmidrule(lr){13-17}
$N_a$ & Method & SR$\uparrow$ & CS$\uparrow$ & Coll$\downarrow$ & DA$\uparrow$ & Time$\downarrow$ & SR$\uparrow$ & CS$\uparrow$ & Coll$\downarrow$ & DA$\uparrow$ & Time$\downarrow$ & SR$\uparrow$ & CS$\uparrow$ & Coll$\downarrow$ & DA$\uparrow$ & Time$\downarrow$\\
\midrule
3 & Diff.\ Policy$^{\dagger}$ & 0.0 & 0.0$_{\pm0.0}$ & 0.31$_{\pm0.53}$ & 0.997$_{\pm0.002}$ & 0.3 & 0.0 & 0.0$_{\pm0.0}$ & 0.31$_{\pm0.53}$ & 0.997$_{\pm0.002}$ & 0.3 & 0.0 & 0.0$_{\pm0.0}$ & 0.31$_{\pm0.53}$ & 0.997$_{\pm0.002}$ & 0.3\\
  & CD$^{\dagger}$ & 85.0 & 12.5$_{\pm20.0}$ & 0.12$_{\pm0.26}$ & 0.683$_{\pm0.079}$ & 1.2$_{\pm0.0}$ & 85.0 & 16.3$_{\pm21.3}$ & 0.12$_{\pm0.26}$ & 0.683$_{\pm0.079}$ & 1.1$_{\pm0.0}$ & 85.0 & 18.0$_{\pm21.5}$ & 0.12$_{\pm0.26}$ & 0.683$_{\pm0.079}$ & 1.2$_{\pm0.0}$\\
  & MMD-CBS$^{\dagger}$ & 45.0 & 45.0$_{\pm49.7}$ & --- & 0.890$_{\pm0.068}$ & 5.7$_{\pm6.5}$ & 45.0 & 45.0$_{\pm49.7}$ & --- & 0.890$_{\pm0.068}$ & 5.7$_{\pm6.5}$ & 45.0 & 45.0$_{\pm49.7}$ & --- & 0.890$_{\pm0.068}$ & 5.7$_{\pm6.5}$\\
  & MDOC$^{\dagger}$ & 10.0 & --- & --- & 0.072$_{\pm0.215}$ & 36.1$_{\pm6.6}$ & 10.0 & --- & --- & 0.072$_{\pm0.215}$ & 36.4$_{\pm2.9}$ & 10.0 & --- & --- & 0.072$_{\pm0.215}$ & 37.0$_{\pm3.6}$\\
  & PCD-LB & 10.0 & 0.5$_{\pm1.7}$ & 0.02$_{\pm0.06}$ & --- & 3.0$_{\pm0.0}$ & 10.0 & 0.5$_{\pm2.1}$ & 0.01$_{\pm0.05}$ & --- & 3.1$_{\pm0.1}$ & 20.0 & 1.5$_{\pm3.8}$ & 0.01$_{\pm0.03}$ & --- & 3.1$_{\pm0.1}$\\
  & PCD-SHD & 45.0 & 25.9$_{\pm36.9}$ & 0.06$_{\pm0.09}$ & --- & 1.9$_{\pm0.8}$ & 50.0 & 26.9$_{\pm38.0}$ & 0.07$_{\pm0.12}$ & --- & 3.1$_{\pm0.1}$ & 40.0 & 26.9$_{\pm38.5}$ & 0.06$_{\pm0.09}$ & --- & 3.1$_{\pm0.1}$\\
  & Final Proj. & 80.0 & 53.4$_{\pm42.9}$ & 0.00$_{\pm0.00}$ & 0.912$_{\pm0.039}$ & 2.3$_{\pm0.9}$ & 85.0 & 51.9$_{\pm43.4}$ & 0.02$_{\pm0.08}$ & 0.911$_{\pm0.040}$ & 3.7$_{\pm0.6}$ & 85.0 & 51.9$_{\pm43.4}$ & \textbf{0.00}$_{\pm0.00}$ & 0.914$_{\pm0.039}$ & 3.7$_{\pm0.6}$\\
  & DiRecT & \textbf{95.0} & 87.2$_{\pm24.1}$ & \textbf{0.00}$_{\pm0.00}$ & \textbf{0.920}$_{\pm0.056}$ & 18.0$_{\pm0.6}$ & \textbf{95.0} & 82.0$_{\pm25.8}$ & \textbf{0.00}$_{\pm0.00}$ & \textbf{0.920}$_{\pm0.056}$ & 18.2$_{\pm0.7}$ & \textbf{95.0} & 80.0$_{\pm27.4}$ & \textbf{0.00}$_{\pm0.00}$ & \textbf{0.921}$_{\pm0.054}$ & 18.2$_{\pm0.7}$\\
  & GlobMP (Ours) & \textbf{95.0} & \textbf{95.0}$_{\pm21.8}$ & 0.05$_{\pm0.22}$ & 0.914$_{\pm0.070}$ & 0.7$_{\pm65.2}$ & \textbf{95.0} & \textbf{95.0}$_{\pm21.8}$ & 0.05$_{\pm0.22}$ & 0.918$_{\pm0.069}$ & 0.9$_{\pm65.2}$ & \textbf{95.0} & \textbf{95.0}$_{\pm21.8}$ & 0.05$_{\pm0.22}$ & 0.920$_{\pm0.068}$ & 0.7$_{\pm65.2}$\\
\midrule
6 & Diff.\ Policy$^{\dagger}$ & 0.0 & 0.0$_{\pm0.0}$ & 1.27$_{\pm1.05}$ & 0.996$_{\pm0.002}$ & 0.3 & 0.0 & 0.0$_{\pm0.0}$ & 1.27$_{\pm1.05}$ & 0.996$_{\pm0.002}$ & 0.3 & 0.0 & 0.0$_{\pm0.0}$ & 1.27$_{\pm1.05}$ & 0.996$_{\pm0.002}$ & 0.3\\
  & CD$^{\dagger}$ & 40.0 & 3.8$_{\pm12.3}$ & 0.70$_{\pm0.78}$ & 0.669$_{\pm0.032}$ & 1.2$_{\pm0.1}$ & 40.0 & 4.8$_{\pm12.6}$ & 0.70$_{\pm0.78}$ & 0.669$_{\pm0.032}$ & 0.5$_{\pm0.0}$ & 40.0 & 5.8$_{\pm13.6}$ & 0.70$_{\pm0.78}$ & 0.669$_{\pm0.032}$ & 1.2$_{\pm0.1}$\\
  & MMD-CBS$^{\dagger}$ & 50.0 & 35.0$_{\pm47.7}$ & --- & 0.880$_{\pm0.034}$ & 13.5$_{\pm13.0}$ & 50.0 & 50.0$_{\pm50.0}$ & --- & 0.880$_{\pm0.034}$ & 13.5$_{\pm13.0}$ & 50.0 & 50.0$_{\pm50.0}$ & --- & 0.880$_{\pm0.034}$ & 13.5$_{\pm13.0}$\\
  & MDOC$^{\dagger}$ & 0.0 & --- & --- & 0.000$_{\pm0.000}$ & 55.8$_{\pm3.5}$ & 0.0 & --- & --- & 0.000$_{\pm0.000}$ & 74.4$_{\pm18.6}$ & 0.0 & --- & --- & 0.000$_{\pm0.000}$ & 92.2$_{\pm23.6}$\\
  & PCD-LB & 0.0 & 0.0$_{\pm0.0}$ & 0.84$_{\pm0.63}$ & --- & 1.4$_{\pm0.2}$ & 0.0 & 0.0$_{\pm0.0}$ & 0.66$_{\pm0.47}$ & --- & 3.3$_{\pm0.1}$ & 0.0 & 0.0$_{\pm0.0}$ & 0.39$_{\pm0.30}$ & --- & 3.4$_{\pm0.1}$\\
  & PCD-SHD & 40.0 & 18.4$_{\pm34.0}$ & 0.44$_{\pm0.47}$ & --- & 3.2$_{\pm0.1}$ & 40.0 & 19.1$_{\pm33.2}$ & 0.39$_{\pm0.39}$ & --- & 3.2$_{\pm0.1}$ & 50.0 & 19.2$_{\pm33.2}$ & 0.32$_{\pm0.28}$ & --- & 3.2$_{\pm0.1}$\\
  & Final Proj. & 60.0 & 24.5$_{\pm37.2}$ & 0.02$_{\pm0.07}$ & 0.897$_{\pm0.042}$ & 3.8$_{\pm0.6}$ & 60.0 & 23.2$_{\pm36.9}$ & 0.01$_{\pm0.03}$ & 0.899$_{\pm0.039}$ & 2.0$_{\pm0.6}$ & 60.0 & 23.1$_{\pm37.0}$ & 0.01$_{\pm0.03}$ & 0.903$_{\pm0.037}$ & 3.8$_{\pm0.7}$\\
  & DiRecT & \textbf{100.0} & 65.8$_{\pm36.9}$ & 0.00$_{\pm0.01}$ & \textbf{0.915}$_{\pm0.037}$ & 18.3$_{\pm0.7}$ & \textbf{100.0} & 56.8$_{\pm38.4}$ & 0.00$_{\pm0.00}$ & 0.915$_{\pm0.036}$ & 18.3$_{\pm0.7}$ & \textbf{95.0} & 51.8$_{\pm37.4}$ & \textbf{0.00}$_{\pm0.00}$ & \textbf{0.917}$_{\pm0.035}$ & 18.3$_{\pm0.7}$\\
  & GlobMP (Ours) & \textbf{100.0} & \textbf{100.0}$_{\pm0.0}$ & \textbf{0.00}$_{\pm0.00}$ & 0.915$_{\pm0.071}$ & 0.4$_{\pm2.6}$ & \textbf{100.0} & \textbf{100.0}$_{\pm0.0}$ & \textbf{0.00}$_{\pm0.00}$ & \textbf{0.917}$_{\pm0.069}$ & 0.6$_{\pm1.5}$ & 85.0 & \textbf{85.0}$_{\pm35.7}$ & 0.20$_{\pm0.51}$ & 0.917$_{\pm0.071}$ & 1.7$_{\pm106.5}$\\
\midrule
9 & Diff.\ Policy$^{\dagger}$ & 0.0 & 0.0$_{\pm0.0}$ & 3.48$_{\pm1.58}$ & 0.997$_{\pm0.001}$ & 0.3 & 0.0 & 0.0$_{\pm0.0}$ & 3.48$_{\pm1.58}$ & 0.997$_{\pm0.001}$ & 0.6 & 0.0 & 0.0$_{\pm0.0}$ & 3.48$_{\pm1.58}$ & 0.997$_{\pm0.001}$ & 0.3\\
  & CD$^{\dagger}$ & 0.0 & 0.0$_{\pm0.0}$ & 2.86$_{\pm1.16}$ & 0.698$_{\pm0.034}$ & 1.2$_{\pm0.1}$ & 0.0 & 0.0$_{\pm0.0}$ & 2.86$_{\pm1.16}$ & 0.698$_{\pm0.034}$ & 1.2$_{\pm0.1}$ & 0.0 & 0.0$_{\pm0.0}$ & 2.86$_{\pm1.16}$ & 0.698$_{\pm0.034}$ & 1.2$_{\pm0.1}$\\
  & MMD-CBS$^{\dagger}$ & 5.0 & 0.0$_{\pm0.0}$ & --- & 0.901$_{\pm0.039}$ & 22.1$_{\pm32.2}$ & 5.0 & 0.0$_{\pm0.0}$ & --- & 0.901$_{\pm0.039}$ & 22.1$_{\pm32.2}$ & 5.0 & 5.0$_{\pm21.8}$ & --- & 0.901$_{\pm0.039}$ & 22.1$_{\pm32.2}$\\
  & MDOC$^{\dagger}$ & 0.0 & --- & --- & 0.000$_{\pm0.000}$ & 117.4$_{\pm19.9}$ & 0.0 & --- & --- & 0.000$_{\pm0.000}$ & 98.4$_{\pm14.9}$ & 0.0 & --- & --- & 0.000$_{\pm0.000}$ & 82.5$_{\pm4.3}$\\
  & PCD-LB & 0.0 & 0.0$_{\pm0.0}$ & 2.66$_{\pm1.28}$ & --- & 3.6$_{\pm0.1}$ & 0.0 & 0.0$_{\pm0.0}$ & 2.05$_{\pm0.97}$ & --- & 3.6$_{\pm0.1}$ & 0.0 & 0.0$_{\pm0.0}$ & 1.13$_{\pm0.50}$ & --- & 3.7$_{\pm0.1}$\\
  & PCD-SHD & 5.0 & 0.5$_{\pm2.4}$ & 1.54$_{\pm0.58}$ & --- & 3.6$_{\pm0.1}$ & 10.0 & 0.9$_{\pm3.4}$ & 1.19$_{\pm0.42}$ & --- & 3.6$_{\pm0.1}$ & 10.0 & 1.0$_{\pm3.8}$ & 1.04$_{\pm0.34}$ & --- & 3.7$_{\pm0.1}$\\
  & Final Proj. & 75.0 & 22.0$_{\pm28.3}$ & 0.03$_{\pm0.03}$ & 0.886$_{\pm0.024}$ & 4.9$_{\pm0.6}$ & 70.0 & 19.5$_{\pm27.8}$ & 0.03$_{\pm0.09}$ & 0.889$_{\pm0.025}$ & 4.9$_{\pm0.7}$ & 70.0 & 19.5$_{\pm27.8}$ & 0.00$_{\pm0.01}$ & 0.894$_{\pm0.024}$ & 2.6$_{\pm0.6}$\\
  & DiRecT & \textbf{95.0} & 76.7$_{\pm36.3}$ & \textbf{0.00}$_{\pm0.00}$ & \textbf{0.919}$_{\pm0.024}$ & 25.0$_{\pm0.8}$ & \textbf{90.0} & 67.1$_{\pm34.6}$ & \textbf{0.00}$_{\pm0.00}$ & \textbf{0.920}$_{\pm0.024}$ & 25.0$_{\pm0.7}$ & \textbf{90.0} & 60.2$_{\pm34.5}$ & \textbf{0.00}$_{\pm0.00}$ & \textbf{0.922}$_{\pm0.023}$ & 25.0$_{\pm0.9}$\\
  & GlobMP (Ours) & \textbf{95.0} & \textbf{95.0}$_{\pm21.8}$ & 0.35$_{\pm1.53}$ & 0.911$_{\pm0.064}$ & 1.8$_{\pm69.7}$ & 75.0 & \textbf{75.0}$_{\pm43.3}$ & 1.10$_{\pm2.07}$ & 0.909$_{\pm0.066}$ & 2.4$_{\pm128.5}$ & 75.0 & \textbf{75.0}$_{\pm43.3}$ & 0.85$_{\pm1.53}$ & 0.912$_{\pm0.064}$ & 12.5$_{\pm126.4}$\\
\midrule
12 & Diff.\ Policy$^{\dagger}$ & 0.0 & 0.0$_{\pm0.0}$ & 5.98$_{\pm1.89}$ & 0.997$_{\pm0.001}$ & 0.6 & 0.0 & 0.0$_{\pm0.0}$ & 5.98$_{\pm1.89}$ & 0.997$_{\pm0.001}$ & 0.3 & 0.0 & 0.0$_{\pm0.0}$ & 5.98$_{\pm1.89}$ & 0.997$_{\pm0.001}$ & 0.3\\
  & CD$^{\dagger}$ & 0.0 & 0.0$_{\pm0.0}$ & 5.62$_{\pm1.27}$ & 0.694$_{\pm0.024}$ & 1.2$_{\pm0.1}$ & 0.0 & 0.0$_{\pm0.0}$ & 5.62$_{\pm1.27}$ & 0.694$_{\pm0.024}$ & 1.2$_{\pm0.1}$ & 0.0 & 0.0$_{\pm0.0}$ & 5.62$_{\pm1.27}$ & 0.694$_{\pm0.024}$ & 0.6$_{\pm0.1}$\\
  & MMD-CBS$^{\dagger}$ & 0.0 & 0.0$_{\pm0.0}$ & --- & 0.921$_{\pm0.000}$ & 5.1$_{\pm12.7}$ & 0.0 & 0.0$_{\pm0.0}$ & --- & 0.921$_{\pm0.000}$ & 5.1$_{\pm12.7}$ & 0.0 & 0.0$_{\pm0.0}$ & --- & 0.921$_{\pm0.000}$ & 5.1$_{\pm12.7}$\\
  & MDOC$^{\dagger}$ & 0.0 & --- & --- & 0.000$_{\pm0.000}$ & 108.7$_{\pm7.5}$ & 0.0 & --- & --- & 0.000$_{\pm0.000}$ & 124.3$_{\pm6.0}$ & 0.0 & --- & --- & 0.000$_{\pm0.000}$ & 128.6$_{\pm6.3}$\\
  & PCD-LB & 0.0 & 0.0$_{\pm0.0}$ & 7.83$_{\pm1.80}$ & --- & 3.8$_{\pm0.1}$ & 0.0 & 0.0$_{\pm0.0}$ & 5.87$_{\pm1.42}$ & --- & 4.0$_{\pm0.1}$ & 0.0 & 0.0$_{\pm0.0}$ & 3.64$_{\pm0.99}$ & --- & 1.9$_{\pm0.2}$\\
  & PCD-SHD & 5.0 & 0.1$_{\pm0.3}$ & 2.98$_{\pm1.29}$ & --- & 3.9$_{\pm0.1}$ & 5.0 & 0.1$_{\pm0.3}$ & 2.48$_{\pm1.12}$ & --- & 2.0$_{\pm0.1}$ & 0.0 & 0.0$_{\pm0.0}$ & 2.15$_{\pm0.97}$ & --- & 4.0$_{\pm0.1}$\\
  & Final Proj. & 45.0 & 16.4$_{\pm29.4}$ & 0.09$_{\pm0.10}$ & 0.878$_{\pm0.036}$ & 6.1$_{\pm0.7}$ & 40.0 & 15.5$_{\pm30.9}$ & 0.04$_{\pm0.08}$ & 0.882$_{\pm0.035}$ & 6.1$_{\pm0.6}$ & 40.0 & 15.4$_{\pm31.0}$ & 0.00$_{\pm0.00}$ & 0.887$_{\pm0.033}$ & 6.1$_{\pm0.7}$\\
  & DiRecT & \textbf{95.0} & 63.2$_{\pm37.9}$ & \textbf{0.00}$_{\pm0.01}$ & \textbf{0.921}$_{\pm0.027}$ & 31.1$_{\pm0.7}$ & \textbf{95.0} & 54.1$_{\pm38.4}$ & \textbf{0.00}$_{\pm0.00}$ & \textbf{0.923}$_{\pm0.026}$ & 31.1$_{\pm0.7}$ & \textbf{95.0} & 49.1$_{\pm39.8}$ & \textbf{0.00}$_{\pm0.00}$ & \textbf{0.925}$_{\pm0.024}$ & 31.1$_{\pm0.7}$\\
  & GlobMP (Ours) & 85.0 & \textbf{85.0}$_{\pm35.7}$ & 1.15$_{\pm2.74}$ & 0.903$_{\pm0.067}$ & 4.5$_{\pm105.6}$ & 75.0 & \textbf{75.0}$_{\pm43.3}$ & 1.75$_{\pm3.05}$ & 0.905$_{\pm0.068}$ & 8.4$_{\pm127.3}$ & 70.0 & \textbf{70.0}$_{\pm45.8}$ & 1.75$_{\pm2.77}$ & 0.908$_{\pm0.066}$ & 17.1$_{\pm130.8}$\\
\midrule
15 & Diff.\ Policy$^{\dagger}$ & 0.0 & 0.0$_{\pm0.0}$ & 10.78$_{\pm2.90}$ & 0.996$_{\pm0.001}$ & 0.3 & 0.0 & 0.0$_{\pm0.0}$ & 10.78$_{\pm2.90}$ & 0.996$_{\pm0.001}$ & 0.3 & 0.0 & 0.0$_{\pm0.0}$ & 10.78$_{\pm2.90}$ & 0.996$_{\pm0.001}$ & 0.3\\
  & CD$^{\dagger}$ & 0.0 & 0.0$_{\pm0.0}$ & 10.72$_{\pm2.10}$ & 0.686$_{\pm0.021}$ & 1.2$_{\pm0.0}$ & 0.0 & 0.0$_{\pm0.0}$ & 10.72$_{\pm2.10}$ & 0.686$_{\pm0.021}$ & 1.2$_{\pm0.1}$ & 0.0 & 0.0$_{\pm0.0}$ & 10.72$_{\pm2.10}$ & 0.686$_{\pm0.021}$ & 1.2$_{\pm0.1}$\\
  & MMD-CBS$^{\dagger}$ & 0.0 & 0.0$_{\pm0.0}$ & --- & 0.844$_{\pm0.000}$ & 8.1$_{\pm17.6}$ & 0.0 & 0.0$_{\pm0.0}$ & --- & 0.844$_{\pm0.000}$ & 8.1$_{\pm17.6}$ & 0.0 & 0.0$_{\pm0.0}$ & --- & 0.844$_{\pm0.000}$ & 8.1$_{\pm17.6}$\\
  & MDOC$^{\dagger}$ & 0.0 & --- & --- & 0.000$_{\pm0.000}$ & 164.4$_{\pm13.3}$ & 0.0 & --- & --- & 0.000$_{\pm0.000}$ & 175.2$_{\pm19.9}$ & 0.0 & --- & --- & 0.000$_{\pm0.000}$ & 171.8$_{\pm26.1}$\\
  & PCD-LB & 0.0 & 0.0$_{\pm0.0}$ & 18.53$_{\pm1.79}$ & --- & 4.3$_{\pm0.1}$ & 0.0 & 0.0$_{\pm0.0}$ & 15.04$_{\pm1.50}$ & --- & 4.4$_{\pm0.1}$ & 0.0 & 0.0$_{\pm0.0}$ & 10.75$_{\pm1.19}$ & --- & 4.4$_{\pm0.1}$\\
  & PCD-SHD & 0.0 & 0.0$_{\pm0.0}$ & 5.32$_{\pm1.80}$ & --- & 4.3$_{\pm0.1}$ & 0.0 & 0.0$_{\pm0.0}$ & 4.44$_{\pm1.50}$ & --- & 4.3$_{\pm0.1}$ & 0.0 & 0.0$_{\pm0.0}$ & 3.73$_{\pm1.21}$ & --- & 4.4$_{\pm0.1}$\\
  & Final Proj. & 40.0 & 5.2$_{\pm13.1}$ & 0.22$_{\pm0.24}$ & 0.851$_{\pm0.030}$ & 4.3$_{\pm1.0}$ & 30.0 & 6.7$_{\pm17.7}$ & 0.08$_{\pm0.13}$ & 0.858$_{\pm0.030}$ & 7.8$_{\pm0.7}$ & 30.0 & 7.0$_{\pm18.3}$ & 0.02$_{\pm0.04}$ & 0.863$_{\pm0.028}$ & 7.8$_{\pm0.7}$\\
  & DiRecT & \textbf{100.0} & 60.2$_{\pm30.4}$ & \textbf{0.00}$_{\pm0.01}$ & \textbf{0.904}$_{\pm0.018}$ & 40.4$_{\pm0.7}$ & \textbf{95.0} & 39.7$_{\pm33.0}$ & \textbf{0.00}$_{\pm0.00}$ & \textbf{0.906}$_{\pm0.017}$ & 40.4$_{\pm0.7}$ & \textbf{90.0} & 33.6$_{\pm32.8}$ & \textbf{0.00}$_{\pm0.00}$ & \textbf{0.908}$_{\pm0.017}$ & 40.5$_{\pm0.9}$\\
  & GlobMP (Ours) & 75.0 & \textbf{75.0}$_{\pm43.3}$ & 2.85$_{\pm4.99}$ & 0.892$_{\pm0.074}$ & 3.9$_{\pm127.2}$ & 65.0 & \textbf{65.0}$_{\pm47.7}$ & 3.70$_{\pm5.45}$ & 0.895$_{\pm0.072}$ & 36.9$_{\pm133.7}$ & 65.0 & \textbf{65.0}$_{\pm47.7}$ & 3.15$_{\pm4.63}$ & 0.899$_{\pm0.072}$ & 93.9$_{\pm129.5}$\\
\midrule
18 & Diff.\ Policy$^{\dagger}$ & 0.0 & 0.0$_{\pm0.0}$ & 15.87$_{\pm4.43}$ & 0.996$_{\pm0.001}$ & 0.3 & 0.0 & 0.0$_{\pm0.0}$ & 15.87$_{\pm4.43}$ & 0.996$_{\pm0.001}$ & 0.3 & 0.0 & 0.0$_{\pm0.0}$ & 15.87$_{\pm4.43}$ & 0.996$_{\pm0.001}$ & 0.3\\
  & CD$^{\dagger}$ & 0.0 & 0.0$_{\pm0.0}$ & 16.96$_{\pm3.44}$ & 0.684$_{\pm0.022}$ & 1.3$_{\pm0.1}$ & 0.0 & 0.0$_{\pm0.0}$ & 16.96$_{\pm3.44}$ & 0.684$_{\pm0.022}$ & 1.3$_{\pm0.1}$ & 0.0 & 0.0$_{\pm0.0}$ & 16.96$_{\pm3.44}$ & 0.684$_{\pm0.022}$ & 1.3$_{\pm0.1}$\\
  & MMD-CBS$^{\dagger}$ & 0.0 & 0.0$_{\pm0.0}$ & --- & 0.898$_{\pm0.016}$ & 8.5$_{\pm18.7}$ & 0.0 & 0.0$_{\pm0.0}$ & --- & 0.898$_{\pm0.016}$ & 8.5$_{\pm18.7}$ & 0.0 & 0.0$_{\pm0.0}$ & --- & 0.898$_{\pm0.016}$ & 8.5$_{\pm18.7}$\\
  & MDOC$^{\dagger}$ & 0.0 & --- & --- & 0.000$_{\pm0.000}$ & 147.2$_{\pm7.4}$ & 0.0 & --- & --- & 0.000$_{\pm0.000}$ & 157.6$_{\pm7.0}$ & 0.0 & --- & --- & 0.000$_{\pm0.000}$ & 159.9$_{\pm7.2}$\\
  & PCD-LB & 0.0 & 0.0$_{\pm0.0}$ & 30.77$_{\pm4.62}$ & --- & 4.4$_{\pm0.1}$ & 0.0 & 0.0$_{\pm0.0}$ & 24.93$_{\pm3.85}$ & --- & 4.5$_{\pm0.1}$ & 0.0 & 0.0$_{\pm0.0}$ & 18.60$_{\pm2.71}$ & --- & 4.6$_{\pm0.1}$\\
  & PCD-SHD & 0.0 & 0.0$_{\pm0.0}$ & 8.11$_{\pm2.72}$ & --- & 4.4$_{\pm0.1}$ & 0.0 & 0.0$_{\pm0.0}$ & 6.79$_{\pm2.30}$ & --- & 4.4$_{\pm0.1}$ & 0.0 & 0.0$_{\pm0.0}$ & 5.77$_{\pm1.97}$ & --- & 4.5$_{\pm0.1}$\\
  & Final Proj. & 20.0 & 6.6$_{\pm18.2}$ & 0.38$_{\pm0.36}$ & 0.837$_{\pm0.040}$ & 10.2$_{\pm0.7}$ & 20.0 & 6.6$_{\pm15.8}$ & 0.09$_{\pm0.09}$ & 0.846$_{\pm0.038}$ & 8.8$_{\pm2.0}$ & 20.0 & 5.5$_{\pm13.1}$ & 0.02$_{\pm0.03}$ & 0.853$_{\pm0.038}$ & 10.2$_{\pm0.8}$\\
  & DiRecT & \textbf{90.0} & 43.4$_{\pm34.4}$ & \textbf{0.03}$_{\pm0.04}$ & \textbf{0.902}$_{\pm0.026}$ & 53.3$_{\pm0.9}$ & \textbf{90.0} & 30.7$_{\pm30.5}$ & \textbf{0.00}$_{\pm0.00}$ & \textbf{0.905}$_{\pm0.025}$ & 53.3$_{\pm0.9}$ & \textbf{85.0} & 23.4$_{\pm28.2}$ & \textbf{0.00}$_{\pm0.00}$ & \textbf{0.908}$_{\pm0.025}$ & 53.1$_{\pm0.7}$\\
  & GlobMP (Ours) & 65.0 & \textbf{65.0}$_{\pm47.7}$ & 6.05$_{\pm8.70}$ & 0.883$_{\pm0.071}$ & 28.6$_{\pm135.8}$ & 60.0 & \textbf{60.0}$_{\pm49.0}$ & 7.85$_{\pm10.09}$ & 0.892$_{\pm0.070}$ & 11.2$_{\pm143.1}$ & 40.0 & \textbf{40.0}$_{\pm49.0}$ & 9.95$_{\pm8.68}$ & 0.894$_{\pm0.068}$ & 300.0$_{\pm137.3}$\\
\midrule
21 & Diff.\ Policy$^{\dagger}$ & 0.0 & 0.0$_{\pm0.0}$ & 20.72$_{\pm5.04}$ & 0.997$_{\pm0.001}$ & 0.3 & 0.0 & 0.0$_{\pm0.0}$ & 20.72$_{\pm5.04}$ & 0.997$_{\pm0.001}$ & 0.3 & 0.0 & 0.0$_{\pm0.0}$ & 20.72$_{\pm5.04}$ & 0.997$_{\pm0.001}$ & 0.3\\
  & CD$^{\dagger}$ & 0.0 & 0.0$_{\pm0.0}$ & 24.06$_{\pm3.00}$ & 0.692$_{\pm0.016}$ & 1.3$_{\pm0.0}$ & 0.0 & 0.0$_{\pm0.0}$ & 24.06$_{\pm3.00}$ & 0.692$_{\pm0.016}$ & 0.8$_{\pm0.2}$ & 0.0 & 0.0$_{\pm0.0}$ & 24.06$_{\pm3.00}$ & 0.692$_{\pm0.016}$ & 1.3$_{\pm0.0}$\\
  & MMD-CBS$^{\dagger}$ & 0.0 & 0.0$_{\pm0.0}$ & --- & 0.954$_{\pm0.015}$ & 19.8$_{\pm41.4}$ & 0.0 & 0.0$_{\pm0.0}$ & --- & 0.954$_{\pm0.015}$ & 19.8$_{\pm41.4}$ & 0.0 & 0.0$_{\pm0.0}$ & --- & 0.954$_{\pm0.015}$ & 19.8$_{\pm41.4}$\\
  & MDOC$^{\dagger}$ & 0.0 & --- & --- & 0.000$_{\pm0.000}$ & 183.7$_{\pm8.7}$ & 0.0 & --- & --- & 0.000$_{\pm0.000}$ & 180.8$_{\pm8.7}$ & 0.0 & --- & --- & 0.000$_{\pm0.000}$ & 179.5$_{\pm8.6}$\\
  & PCD-LB & 0.0 & 0.0$_{\pm0.0}$ & 42.63$_{\pm4.11}$ & --- & 5.0$_{\pm0.2}$ & 0.0 & 0.0$_{\pm0.0}$ & 34.30$_{\pm3.31}$ & --- & 5.1$_{\pm0.2}$ & 0.0 & 0.0$_{\pm0.0}$ & 26.27$_{\pm2.54}$ & --- & 4.4$_{\pm1.1}$\\
  & PCD-SHD & 0.0 & 0.0$_{\pm0.0}$ & 10.66$_{\pm3.55}$ & --- & 4.8$_{\pm0.1}$ & 0.0 & 0.0$_{\pm0.0}$ & 9.02$_{\pm2.99}$ & --- & 4.9$_{\pm0.1}$ & 0.0 & 0.0$_{\pm0.0}$ & 7.59$_{\pm2.61}$ & --- & 2.5$_{\pm0.2}$\\
  & Final Proj. & 25.0 & 8.2$_{\pm21.8}$ & 0.70$_{\pm0.81}$ & 0.832$_{\pm0.033}$ & 13.0$_{\pm0.8}$ & 30.0 & 6.7$_{\pm19.6}$ & 0.23$_{\pm0.30}$ & 0.840$_{\pm0.032}$ & 9.6$_{\pm2.9}$ & 35.0 & 6.7$_{\pm19.5}$ & 0.07$_{\pm0.12}$ & 0.846$_{\pm0.032}$ & 13.1$_{\pm0.8}$\\
  & DiRecT & \textbf{100.0} & \textbf{52.2}$_{\pm29.1}$ & \textbf{0.04}$_{\pm0.06}$ & \textbf{0.906}$_{\pm0.024}$ & 68.5$_{\pm0.7}$ & \textbf{100.0} & 35.5$_{\pm29.5}$ & \textbf{0.00}$_{\pm0.00}$ & \textbf{0.910}$_{\pm0.022}$ & 68.5$_{\pm0.7}$ & \textbf{85.0} & 26.8$_{\pm29.7}$ & \textbf{0.00}$_{\pm0.00}$ & \textbf{0.913}$_{\pm0.021}$ & 68.5$_{\pm0.7}$\\
  & GlobMP (Ours) & 45.0 & 45.0$_{\pm49.7}$ & 12.90$_{\pm12.22}$ & 0.886$_{\pm0.077}$ & 300.0$_{\pm141.1}$ & 60.0 & \textbf{60.0}$_{\pm49.0}$ & 8.45$_{\pm10.77}$ & 0.900$_{\pm0.068}$ & 33.5$_{\pm137.9}$ & 70.0 & \textbf{70.0}$_{\pm45.8}$ & 5.60$_{\pm8.79}$ & 0.899$_{\pm0.071}$ & 83.2$_{\pm116.2}$\\
\bottomrule
\end{tabular}}
\end{center}
\end{table}

\begin{table}[p]
\caption{Multi-robot planning performance on \texttt{Cross-Room} for $N_a\in\{3,6,\dots,21\}$ under the Low, Medium and High velocity caps. \emph{SR} and \emph{CS} measure trajectory feasibility, \emph{Coll} reports inter-robot collisions, \emph{DA} adherence to the environment-specific motion pattern, and \emph{Time} the planning time per configuration: for \emph{GlobMP (Ours)} the median over the configurations with the standard deviation as subscript, for the baselines the mean, so the two are not directly comparable. $^{\dagger}$ marks a method that applies no velocity limit, so its three blocks report the same run; MMD-CBS also returns one plan per configuration rather than $64$ candidates. A dash under \emph{Coll} marks a count that is not commensurable with the others, as MDOC and MMD-CBS count colliding time steps where every other row counts agent pairs with any collision; MDOC's \emph{CS} is omitted because its plans that do not cover every agent were counted as feasible. MMD-CBS plans under $60$\,s per configuration and \emph{GlobMP (Ours)} under $300$\,s per team.}
\label{tab:hard-crossroom}
\begin{center}\tiny\setlength{\tabcolsep}{2.2pt}
\resizebox{\linewidth}{!}{%
\begin{tabular}{ll ccccc ccccc ccccc}
\toprule
& & \multicolumn{5}{c}{Low ($v_{\max}=0.928$)} & \multicolumn{5}{c}{Medium ($v_{\max}=1.13$)} & \multicolumn{5}{c}{High ($v_{\max}=1.34$)}\\
\cmidrule(lr){3-7}\cmidrule(lr){8-12}\cmidrule(lr){13-17}
$N_a$ & Method & SR$\uparrow$ & CS$\uparrow$ & Coll$\downarrow$ & DA$\uparrow$ & Time$\downarrow$ & SR$\uparrow$ & CS$\uparrow$ & Coll$\downarrow$ & DA$\uparrow$ & Time$\downarrow$ & SR$\uparrow$ & CS$\uparrow$ & Coll$\downarrow$ & DA$\uparrow$ & Time$\downarrow$\\
\midrule
3 & Diff.\ Policy$^{\dagger}$ & 0.0 & 0.0$_{\pm0.0}$ & 0.30$_{\pm0.55}$ & 0.985$_{\pm0.007}$ & 0.6 & 0.0 & 0.0$_{\pm0.0}$ & 0.30$_{\pm0.55}$ & 0.985$_{\pm0.007}$ & 0.5 & 0.0 & 0.0$_{\pm0.0}$ & 0.30$_{\pm0.55}$ & 0.985$_{\pm0.007}$ & 0.5\\
  & CD$^{\dagger}$ & 20.0 & 5.6$_{\pm15.6}$ & 0.02$_{\pm0.05}$ & 0.662$_{\pm0.070}$ & 1.1$_{\pm0.0}$ & 20.0 & 9.3$_{\pm25.5}$ & 0.02$_{\pm0.05}$ & 0.662$_{\pm0.070}$ & 1.1$_{\pm0.1}$ & 20.0 & 10.1$_{\pm27.3}$ & 0.02$_{\pm0.05}$ & 0.662$_{\pm0.070}$ & 1.1$_{\pm0.0}$\\
  & MMD-CBS$^{\dagger}$ & 10.0 & 10.0$_{\pm30.0}$ & --- & 0.887$_{\pm0.049}$ & 0.7$_{\pm0.4}$ & 10.0 & 10.0$_{\pm30.0}$ & --- & 0.887$_{\pm0.049}$ & 0.7$_{\pm0.4}$ & 10.0 & 10.0$_{\pm30.0}$ & --- & 0.887$_{\pm0.049}$ & 0.7$_{\pm0.4}$\\
  & MDOC$^{\dagger}$ & 0.0 & --- & --- & 0.000$_{\pm0.000}$ & 37.9$_{\pm2.4}$ & 0.0 & --- & --- & 0.000$_{\pm0.000}$ & 36.2$_{\pm2.5}$ & 0.0 & --- & --- & 0.000$_{\pm0.000}$ & 35.4$_{\pm2.6}$\\
  & PCD-LB & 55.0 & 16.6$_{\pm30.6}$ & 0.05$_{\pm0.17}$ & --- & 2.7$_{\pm0.2}$ & 45.0 & 15.6$_{\pm30.8}$ & 0.05$_{\pm0.15}$ & --- & 1.5$_{\pm0.1}$ & 45.0 & 16.3$_{\pm30.6}$ & 0.04$_{\pm0.14}$ & --- & 1.5$_{\pm0.1}$\\
  & PCD-SHD & 10.0 & 4.0$_{\pm16.3}$ & 0.00$_{\pm0.01}$ & --- & 2.0$_{\pm0.6}$ & 10.0 & 3.8$_{\pm15.3}$ & 0.00$_{\pm0.01}$ & --- & 1.6$_{\pm0.4}$ & 10.0 & 4.0$_{\pm16.3}$ & 0.00$_{\pm0.01}$ & --- & 1.5$_{\pm0.1}$\\
  & Final Proj. & 5.0 & 2.7$_{\pm11.9}$ & 0.01$_{\pm0.02}$ & \textbf{0.910}$_{\pm0.035}$ & 4.4$_{\pm0.9}$ & 5.0 & 2.7$_{\pm11.9}$ & 0.01$_{\pm0.01}$ & \textbf{0.907}$_{\pm0.035}$ & 4.4$_{\pm0.9}$ & 5.0 & 2.7$_{\pm11.9}$ & 0.00$_{\pm0.01}$ & \textbf{0.909}$_{\pm0.035}$ & 2.4$_{\pm1.0}$\\
  & DiRecT & 10.0 & 5.0$_{\pm21.1}$ & 0.03$_{\pm0.08}$ & 0.899$_{\pm0.038}$ & 33.1$_{\pm6.8}$ & 10.0 & 5.0$_{\pm21.1}$ & 0.00$_{\pm0.01}$ & 0.902$_{\pm0.038}$ & 30.4$_{\pm8.1}$ & 10.0 & 4.8$_{\pm20.8}$ & 0.00$_{\pm0.00}$ & 0.904$_{\pm0.038}$ & 29.1$_{\pm8.4}$\\
  & GlobMP (Ours) & \textbf{100.0} & \textbf{100.0}$_{\pm0.0}$ & \textbf{0.00}$_{\pm0.00}$ & 0.804$_{\pm0.142}$ & 0.4$_{\pm0.3}$ & \textbf{100.0} & \textbf{100.0}$_{\pm0.0}$ & \textbf{0.00}$_{\pm0.00}$ & 0.802$_{\pm0.133}$ & 0.4$_{\pm0.3}$ & \textbf{100.0} & \textbf{100.0}$_{\pm0.0}$ & \textbf{0.00}$_{\pm0.00}$ & 0.809$_{\pm0.143}$ & 0.6$_{\pm0.7}$\\
\midrule
6 & Diff.\ Policy$^{\dagger}$ & 0.0 & 0.0$_{\pm0.0}$ & 1.59$_{\pm0.98}$ & 0.985$_{\pm0.004}$ & 0.6 & 0.0 & 0.0$_{\pm0.0}$ & 1.59$_{\pm0.98}$ & 0.985$_{\pm0.004}$ & 0.6 & 0.0 & 0.0$_{\pm0.0}$ & 1.59$_{\pm0.98}$ & 0.985$_{\pm0.004}$ & 0.6\\
  & CD$^{\dagger}$ & 10.0 & 1.2$_{\pm5.4}$ & 0.40$_{\pm0.41}$ & 0.671$_{\pm0.041}$ & 1.1$_{\pm0.1}$ & 10.0 & 3.2$_{\pm14.0}$ & 0.40$_{\pm0.41}$ & 0.671$_{\pm0.041}$ & 0.5$_{\pm0.1}$ & 10.0 & 3.7$_{\pm16.0}$ & 0.40$_{\pm0.41}$ & 0.671$_{\pm0.041}$ & 1.1$_{\pm0.0}$\\
  & MMD-CBS$^{\dagger}$ & 5.0 & 0.0$_{\pm0.0}$ & --- & 0.849$_{\pm0.000}$ & 1.0$_{\pm1.1}$ & 5.0 & 0.0$_{\pm0.0}$ & --- & 0.849$_{\pm0.000}$ & 1.0$_{\pm1.1}$ & 5.0 & 0.0$_{\pm0.0}$ & --- & 0.849$_{\pm0.000}$ & 1.0$_{\pm1.1}$\\
  & MDOC$^{\dagger}$ & 0.0 & --- & --- & 0.000$_{\pm0.000}$ & 66.3$_{\pm3.8}$ & 0.0 & --- & --- & 0.000$_{\pm0.000}$ & 64.4$_{\pm6.6}$ & 0.0 & --- & --- & 0.000$_{\pm0.000}$ & 61.7$_{\pm6.2}$\\
  & PCD-LB & 25.0 & 8.6$_{\pm22.8}$ & 1.01$_{\pm0.48}$ & --- & 2.7$_{\pm0.7}$ & 20.0 & 7.7$_{\pm22.1}$ & 1.07$_{\pm0.46}$ & --- & 2.5$_{\pm0.5}$ & 40.0 & 7.7$_{\pm21.9}$ & 1.01$_{\pm0.46}$ & --- & 1.9$_{\pm0.1}$\\
  & PCD-SHD & 0.0 & 0.0$_{\pm0.0}$ & 0.17$_{\pm0.26}$ & --- & 2.1$_{\pm0.6}$ & 0.0 & 0.0$_{\pm0.0}$ & 0.14$_{\pm0.20}$ & --- & 2.1$_{\pm0.7}$ & 0.0 & 0.0$_{\pm0.0}$ & 0.12$_{\pm0.16}$ & --- & 1.8$_{\pm0.1}$\\
  & Final Proj. & 0.0 & 0.0$_{\pm0.0}$ & 0.08$_{\pm0.12}$ & 0.906$_{\pm0.029}$ & 3.9$_{\pm0.9}$ & 0.0 & 0.0$_{\pm0.0}$ & 0.07$_{\pm0.11}$ & 0.902$_{\pm0.030}$ & 3.9$_{\pm0.9}$ & 0.0 & 0.0$_{\pm0.0}$ & 0.06$_{\pm0.09}$ & 0.903$_{\pm0.031}$ & 3.9$_{\pm0.8}$\\
  & DiRecT & 0.0 & 0.0$_{\pm0.0}$ & 0.07$_{\pm0.14}$ & \textbf{0.908}$_{\pm0.031}$ & 24.5$_{\pm7.3}$ & 0.0 & 0.0$_{\pm0.0}$ & 0.03$_{\pm0.11}$ & \textbf{0.908}$_{\pm0.033}$ & 24.4$_{\pm7.3}$ & 0.0 & 0.0$_{\pm0.0}$ & 0.01$_{\pm0.02}$ & \textbf{0.909}$_{\pm0.034}$ & 27.1$_{\pm7.2}$\\
  & GlobMP (Ours) & \textbf{100.0} & \textbf{100.0}$_{\pm0.0}$ & \textbf{0.00}$_{\pm0.00}$ & 0.783$_{\pm0.133}$ & 1.2$_{\pm52.1}$ & \textbf{100.0} & \textbf{100.0}$_{\pm0.0}$ & \textbf{0.00}$_{\pm0.00}$ & 0.786$_{\pm0.136}$ & 1.5$_{\pm1.3}$ & \textbf{100.0} & \textbf{100.0}$_{\pm0.0}$ & \textbf{0.00}$_{\pm0.00}$ & 0.787$_{\pm0.133}$ & 2.7$_{\pm2.3}$\\
\midrule
9 & Diff.\ Policy$^{\dagger}$ & 0.0 & 0.0$_{\pm0.0}$ & 3.23$_{\pm1.34}$ & 0.985$_{\pm0.003}$ & 0.6 & 0.0 & 0.0$_{\pm0.0}$ & 3.23$_{\pm1.34}$ & 0.985$_{\pm0.003}$ & 0.6 & 0.0 & 0.0$_{\pm0.0}$ & 3.23$_{\pm1.34}$ & 0.985$_{\pm0.003}$ & 0.3\\
  & CD$^{\dagger}$ & 0.0 & 0.0$_{\pm0.0}$ & 1.42$_{\pm0.71}$ & 0.659$_{\pm0.030}$ & 1.1$_{\pm0.1}$ & 0.0 & 0.0$_{\pm0.0}$ & 1.42$_{\pm0.71}$ & 0.659$_{\pm0.030}$ & 1.1$_{\pm0.1}$ & 0.0 & 0.0$_{\pm0.0}$ & 1.42$_{\pm0.71}$ & 0.659$_{\pm0.030}$ & 1.1$_{\pm0.1}$\\
  & MMD-CBS$^{\dagger}$ & 0.0 & 0.0$_{\pm0.0}$ & --- & --- & 1.0$_{\pm0.8}$ & 0.0 & 0.0$_{\pm0.0}$ & --- & --- & 1.0$_{\pm0.8}$ & 0.0 & 0.0$_{\pm0.0}$ & --- & --- & 1.0$_{\pm0.8}$\\
  & MDOC$^{\dagger}$ & 0.0 & --- & --- & 0.000$_{\pm0.000}$ & 94.2$_{\pm4.1}$ & 0.0 & --- & --- & 0.000$_{\pm0.000}$ & 92.9$_{\pm3.9}$ & 0.0 & --- & --- & 0.000$_{\pm0.000}$ & 93.0$_{\pm6.6}$\\
  & PCD-LB & 0.0 & 0.0$_{\pm0.0}$ & 3.96$_{\pm1.41}$ & --- & 2.1$_{\pm0.4}$ & 0.0 & 0.0$_{\pm0.0}$ & 4.11$_{\pm1.38}$ & --- & 3.1$_{\pm0.7}$ & 0.0 & 0.0$_{\pm0.0}$ & 3.68$_{\pm1.18}$ & --- & 2.2$_{\pm0.1}$\\
  & PCD-SHD & 0.0 & 0.0$_{\pm0.0}$ & 0.28$_{\pm0.27}$ & --- & 2.0$_{\pm0.1}$ & 0.0 & 0.0$_{\pm0.0}$ & 0.26$_{\pm0.22}$ & --- & 2.1$_{\pm0.1}$ & 0.0 & 0.0$_{\pm0.0}$ & 0.22$_{\pm0.16}$ & --- & 2.0$_{\pm0.1}$\\
  & Final Proj. & 0.0 & 0.0$_{\pm0.0}$ & 0.14$_{\pm0.17}$ & 0.895$_{\pm0.032}$ & 2.8$_{\pm0.8}$ & 0.0 & 0.0$_{\pm0.0}$ & 0.11$_{\pm0.13}$ & 0.891$_{\pm0.034}$ & 5.1$_{\pm0.9}$ & 0.0 & 0.0$_{\pm0.0}$ & 0.08$_{\pm0.10}$ & 0.893$_{\pm0.034}$ & 5.1$_{\pm1.0}$\\
  & DiRecT & 0.0 & 0.0$_{\pm0.0}$ & 0.20$_{\pm0.26}$ & \textbf{0.902}$_{\pm0.027}$ & 35.4$_{\pm9.8}$ & 0.0 & 0.0$_{\pm0.0}$ & 0.06$_{\pm0.08}$ & \textbf{0.906}$_{\pm0.028}$ & 36.8$_{\pm9.8}$ & 0.0 & 0.0$_{\pm0.0}$ & 0.02$_{\pm0.03}$ & \textbf{0.908}$_{\pm0.027}$ & 38.5$_{\pm9.6}$\\
  & GlobMP (Ours) & \textbf{100.0} & \textbf{100.0}$_{\pm0.0}$ & \textbf{0.00}$_{\pm0.00}$ & 0.782$_{\pm0.133}$ & 2.3$_{\pm1.1}$ & \textbf{100.0} & \textbf{100.0}$_{\pm0.0}$ & \textbf{0.00}$_{\pm0.00}$ & 0.783$_{\pm0.124}$ & 3.3$_{\pm3.0}$ & \textbf{100.0} & \textbf{100.0}$_{\pm0.0}$ & \textbf{0.00}$_{\pm0.00}$ & 0.782$_{\pm0.132}$ & 7.3$_{\pm20.9}$\\
\midrule
12 & Diff.\ Policy$^{\dagger}$ & 0.0 & 0.0$_{\pm0.0}$ & 7.06$_{\pm2.70}$ & 0.985$_{\pm0.003}$ & 0.3 & 0.0 & 0.0$_{\pm0.0}$ & 7.06$_{\pm2.70}$ & 0.985$_{\pm0.003}$ & 0.6 & 0.0 & 0.0$_{\pm0.0}$ & 7.06$_{\pm2.70}$ & 0.985$_{\pm0.003}$ & 0.6\\
  & CD$^{\dagger}$ & 0.0 & 0.0$_{\pm0.0}$ & 4.08$_{\pm1.26}$ & 0.663$_{\pm0.020}$ & 1.2$_{\pm0.1}$ & 0.0 & 0.0$_{\pm0.0}$ & 4.08$_{\pm1.26}$ & 0.663$_{\pm0.020}$ & 1.2$_{\pm0.1}$ & 0.0 & 0.0$_{\pm0.0}$ & 4.08$_{\pm1.26}$ & 0.663$_{\pm0.020}$ & 0.6$_{\pm0.1}$\\
  & MMD-CBS$^{\dagger}$ & 0.0 & 0.0$_{\pm0.0}$ & --- & --- & 1.1$_{\pm1.2}$ & 0.0 & 0.0$_{\pm0.0}$ & --- & --- & 1.1$_{\pm1.2}$ & 0.0 & 0.0$_{\pm0.0}$ & --- & --- & 1.1$_{\pm1.2}$\\
  & MDOC$^{\dagger}$ & 0.0 & --- & --- & 0.000$_{\pm0.000}$ & 122.0$_{\pm5.3}$ & 0.0 & --- & --- & 0.000$_{\pm0.000}$ & 110.7$_{\pm5.9}$ & 0.0 & --- & --- & 0.000$_{\pm0.000}$ & 108.8$_{\pm4.8}$\\
  & PCD-LB & 0.0 & 0.0$_{\pm0.0}$ & 9.93$_{\pm1.54}$ & --- & 3.4$_{\pm0.9}$ & 0.0 & 0.0$_{\pm0.0}$ & 10.31$_{\pm1.59}$ & --- & 2.4$_{\pm0.1}$ & 0.0 & 0.0$_{\pm0.0}$ & 9.16$_{\pm1.46}$ & --- & 4.5$_{\pm0.1}$\\
  & PCD-SHD & 0.0 & 0.0$_{\pm0.0}$ & 0.60$_{\pm0.36}$ & --- & 4.2$_{\pm0.1}$ & 0.0 & 0.0$_{\pm0.0}$ & 0.53$_{\pm0.26}$ & --- & 2.4$_{\pm0.3}$ & 0.0 & 0.0$_{\pm0.0}$ & 0.46$_{\pm0.26}$ & --- & 4.1$_{\pm0.8}$\\
  & Final Proj. & 0.0 & 0.0$_{\pm0.0}$ & 0.41$_{\pm0.42}$ & 0.881$_{\pm0.026}$ & 5.4$_{\pm2.1}$ & 0.0 & 0.0$_{\pm0.0}$ & 0.33$_{\pm0.32}$ & 0.876$_{\pm0.028}$ & 7.5$_{\pm1.0}$ & 0.0 & 0.0$_{\pm0.0}$ & 0.33$_{\pm0.28}$ & 0.877$_{\pm0.029}$ & 7.5$_{\pm0.9}$\\
  & DiRecT & 0.0 & 0.0$_{\pm0.0}$ & 0.44$_{\pm0.48}$ & \textbf{0.898}$_{\pm0.020}$ & 49.1$_{\pm12.6}$ & 0.0 & 0.0$_{\pm0.0}$ & 0.10$_{\pm0.14}$ & \textbf{0.900}$_{\pm0.020}$ & 49.2$_{\pm12.6}$ & 0.0 & 0.0$_{\pm0.0}$ & 0.04$_{\pm0.05}$ & \textbf{0.902}$_{\pm0.020}$ & 49.8$_{\pm12.5}$\\
  & GlobMP (Ours) & \textbf{100.0} & \textbf{100.0}$_{\pm0.0}$ & \textbf{0.00}$_{\pm0.00}$ & 0.773$_{\pm0.140}$ & 3.1$_{\pm57.3}$ & \textbf{100.0} & \textbf{100.0}$_{\pm0.0}$ & \textbf{0.00}$_{\pm0.00}$ & 0.778$_{\pm0.135}$ & 5.8$_{\pm11.5}$ & \textbf{100.0} & \textbf{100.0}$_{\pm0.0}$ & \textbf{0.00}$_{\pm0.00}$ & 0.778$_{\pm0.140}$ & 9.8$_{\pm42.3}$\\
\midrule
15 & Diff.\ Policy$^{\dagger}$ & 0.0 & 0.0$_{\pm0.0}$ & 9.58$_{\pm2.09}$ & 0.985$_{\pm0.003}$ & 0.6 & 0.0 & 0.0$_{\pm0.0}$ & 9.58$_{\pm2.09}$ & 0.985$_{\pm0.003}$ & 0.6 & 0.0 & 0.0$_{\pm0.0}$ & 9.58$_{\pm2.09}$ & 0.985$_{\pm0.003}$ & 0.6\\
  & CD$^{\dagger}$ & 0.0 & 0.0$_{\pm0.0}$ & 6.69$_{\pm1.73}$ & 0.662$_{\pm0.016}$ & 1.2$_{\pm0.1}$ & 0.0 & 0.0$_{\pm0.0}$ & 6.69$_{\pm1.73}$ & 0.662$_{\pm0.016}$ & 1.2$_{\pm0.1}$ & 0.0 & 0.0$_{\pm0.0}$ & 6.69$_{\pm1.73}$ & 0.662$_{\pm0.016}$ & 1.2$_{\pm0.1}$\\
  & MMD-CBS$^{\dagger}$ & 0.0 & 0.0$_{\pm0.0}$ & --- & --- & 1.0$_{\pm0.9}$ & 0.0 & 0.0$_{\pm0.0}$ & --- & --- & 1.0$_{\pm0.9}$ & 0.0 & 0.0$_{\pm0.0}$ & --- & --- & 1.0$_{\pm0.9}$\\
  & MDOC$^{\dagger}$ & 0.0 & --- & --- & 0.000$_{\pm0.000}$ & 140.7$_{\pm6.7}$ & 0.0 & --- & --- & 0.000$_{\pm0.000}$ & 134.0$_{\pm6.8}$ & 0.0 & --- & --- & 0.000$_{\pm0.000}$ & 134.2$_{\pm6.1}$\\
  & PCD-LB & 0.0 & 0.0$_{\pm0.0}$ & 20.03$_{\pm2.29}$ & --- & 3.0$_{\pm0.6}$ & 0.0 & 0.0$_{\pm0.0}$ & 20.97$_{\pm2.25}$ & --- & 2.9$_{\pm0.1}$ & 0.0 & 0.0$_{\pm0.0}$ & 19.26$_{\pm1.97}$ & --- & 3.0$_{\pm0.4}$\\
  & PCD-SHD & 0.0 & 0.0$_{\pm0.0}$ & 0.84$_{\pm0.37}$ & --- & 3.9$_{\pm1.1}$ & 0.0 & 0.0$_{\pm0.0}$ & 0.78$_{\pm0.35}$ & --- & 2.4$_{\pm0.1}$ & 0.0 & 0.0$_{\pm0.0}$ & 0.68$_{\pm0.31}$ & --- & 2.4$_{\pm0.1}$\\
  & Final Proj. & 0.0 & 0.0$_{\pm0.0}$ & 0.59$_{\pm0.46}$ & 0.872$_{\pm0.019}$ & 7.9$_{\pm2.4}$ & 0.0 & 0.0$_{\pm0.0}$ & 0.51$_{\pm0.37}$ & 0.867$_{\pm0.020}$ & 11.2$_{\pm0.9}$ & 0.0 & 0.0$_{\pm0.0}$ & 0.44$_{\pm0.31}$ & 0.868$_{\pm0.020}$ & 8.8$_{\pm2.8}$\\
  & DiRecT & 0.0 & 0.0$_{\pm0.0}$ & 0.64$_{\pm0.50}$ & \textbf{0.892}$_{\pm0.014}$ & 68.2$_{\pm15.4}$ & 0.0 & 0.0$_{\pm0.0}$ & 0.20$_{\pm0.24}$ & \textbf{0.894}$_{\pm0.014}$ & 69.0$_{\pm15.5}$ & 0.0 & 0.0$_{\pm0.0}$ & \textbf{0.07}$_{\pm0.08}$ & \textbf{0.896}$_{\pm0.014}$ & 75.8$_{\pm12.1}$\\
  & GlobMP (Ours) & \textbf{100.0} & \textbf{100.0}$_{\pm0.0}$ & \textbf{0.00}$_{\pm0.00}$ & 0.774$_{\pm0.142}$ & 3.9$_{\pm23.9}$ & \textbf{100.0} & \textbf{100.0}$_{\pm0.0}$ & \textbf{0.00}$_{\pm0.00}$ & 0.772$_{\pm0.134}$ & 9.2$_{\pm44.7}$ & \textbf{95.0} & \textbf{95.0}$_{\pm21.8}$ & 0.80$_{\pm3.49}$ & 0.775$_{\pm0.139}$ & 21.9$_{\pm62.5}$\\
\midrule
18 & Diff.\ Policy$^{\dagger}$ & 0.0 & 0.0$_{\pm0.0}$ & 13.49$_{\pm2.49}$ & 0.985$_{\pm0.002}$ & 0.3 & 0.0 & 0.0$_{\pm0.0}$ & 13.49$_{\pm2.49}$ & 0.985$_{\pm0.002}$ & 0.6 & 0.0 & 0.0$_{\pm0.0}$ & 13.49$_{\pm2.49}$ & 0.985$_{\pm0.002}$ & 0.3\\
  & CD$^{\dagger}$ & 0.0 & 0.0$_{\pm0.0}$ & 11.98$_{\pm2.21}$ & 0.665$_{\pm0.016}$ & 1.2$_{\pm0.1}$ & 0.0 & 0.0$_{\pm0.0}$ & 11.98$_{\pm2.21}$ & 0.665$_{\pm0.016}$ & 1.2$_{\pm0.1}$ & 0.0 & 0.0$_{\pm0.0}$ & 11.98$_{\pm2.21}$ & 0.665$_{\pm0.016}$ & 1.2$_{\pm0.1}$\\
  & MMD-CBS$^{\dagger}$ & 0.0 & 0.0$_{\pm0.0}$ & --- & --- & 0.9$_{\pm0.6}$ & 0.0 & 0.0$_{\pm0.0}$ & --- & --- & 0.9$_{\pm0.6}$ & 0.0 & 0.0$_{\pm0.0}$ & --- & --- & 0.9$_{\pm0.6}$\\
  & MDOC$^{\dagger}$ & 0.0 & --- & --- & 0.000$_{\pm0.000}$ & 161.4$_{\pm7.5}$ & 0.0 & --- & --- & 0.000$_{\pm0.000}$ & 199.2$_{\pm11.4}$ & 0.0 & --- & --- & 0.000$_{\pm0.000}$ & 173.8$_{\pm8.0}$\\
  & PCD-LB & 0.0 & 0.0$_{\pm0.0}$ & 32.89$_{\pm2.64}$ & --- & 3.1$_{\pm0.2}$ & 0.0 & 0.0$_{\pm0.0}$ & 34.26$_{\pm2.78}$ & --- & 5.7$_{\pm0.5}$ & 0.0 & 0.0$_{\pm0.0}$ & 31.50$_{\pm2.87}$ & --- & 3.4$_{\pm0.1}$\\
  & PCD-SHD & 0.0 & 0.0$_{\pm0.0}$ & 1.07$_{\pm0.63}$ & --- & 3.2$_{\pm0.1}$ & 0.0 & 0.0$_{\pm0.0}$ & 0.99$_{\pm0.57}$ & --- & 2.9$_{\pm0.2}$ & 0.0 & 0.0$_{\pm0.0}$ & 0.88$_{\pm0.43}$ & --- & 3.3$_{\pm0.1}$\\
  & Final Proj. & 0.0 & 0.0$_{\pm0.0}$ & \textbf{0.56}$_{\pm0.38}$ & 0.869$_{\pm0.016}$ & 14.9$_{\pm0.9}$ & 0.0 & 0.0$_{\pm0.0}$ & 0.45$_{\pm0.33}$ & 0.865$_{\pm0.017}$ & 11.6$_{\pm3.7}$ & 0.0 & 0.0$_{\pm0.0}$ & 0.38$_{\pm0.26}$ & 0.866$_{\pm0.018}$ & 14.9$_{\pm0.9}$\\
  & DiRecT & 0.0 & 0.0$_{\pm0.0}$ & 1.02$_{\pm0.65}$ & \textbf{0.900}$_{\pm0.012}$ & 89.6$_{\pm24.7}$ & 0.0 & 0.0$_{\pm0.0}$ & \textbf{0.27}$_{\pm0.21}$ & \textbf{0.903}$_{\pm0.012}$ & 104.4$_{\pm3.8}$ & 0.0 & 0.0$_{\pm0.0}$ & 0.09$_{\pm0.06}$ & \textbf{0.905}$_{\pm0.011}$ & 58.6$_{\pm25.2}$\\
  & GlobMP (Ours) & \textbf{95.0} & \textbf{95.0}$_{\pm21.8}$ & 1.00$_{\pm4.36}$ & 0.789$_{\pm0.134}$ & 8.1$_{\pm96.3}$ & \textbf{95.0} & \textbf{95.0}$_{\pm21.8}$ & 1.00$_{\pm4.36}$ & 0.791$_{\pm0.130}$ & 11.1$_{\pm61.9}$ & \textbf{100.0} & \textbf{100.0}$_{\pm0.0}$ & \textbf{0.00}$_{\pm0.00}$ & 0.779$_{\pm0.130}$ & 60.5$_{\pm84.3}$\\
\midrule
21 & Diff.\ Policy$^{\dagger}$ & 0.0 & 0.0$_{\pm0.0}$ & 20.13$_{\pm3.78}$ & 0.985$_{\pm0.002}$ & 0.7 & 0.0 & 0.0$_{\pm0.0}$ & 20.13$_{\pm3.78}$ & 0.985$_{\pm0.002}$ & 0.7 & 0.0 & 0.0$_{\pm0.0}$ & 20.13$_{\pm3.78}$ & 0.985$_{\pm0.002}$ & 0.7\\
  & CD$^{\dagger}$ & 0.0 & 0.0$_{\pm0.0}$ & 18.02$_{\pm3.46}$ & 0.659$_{\pm0.017}$ & 1.2$_{\pm0.0}$ & 0.0 & 0.0$_{\pm0.0}$ & 18.02$_{\pm3.46}$ & 0.659$_{\pm0.017}$ & 0.6$_{\pm0.1}$ & 0.0 & 0.0$_{\pm0.0}$ & 18.02$_{\pm3.46}$ & 0.659$_{\pm0.017}$ & 1.3$_{\pm0.0}$\\
  & MMD-CBS$^{\dagger}$ & 0.0 & 0.0$_{\pm0.0}$ & --- & --- & 1.1$_{\pm0.9}$ & 0.0 & 0.0$_{\pm0.0}$ & --- & --- & 1.1$_{\pm0.9}$ & 0.0 & 0.0$_{\pm0.0}$ & --- & --- & 1.1$_{\pm0.9}$\\
  & MDOC$^{\dagger}$ & 0.0 & --- & --- & 0.000$_{\pm0.000}$ & 214.2$_{\pm10.8}$ & 0.0 & --- & --- & 0.000$_{\pm0.000}$ & 229.0$_{\pm40.8}$ & 0.0 & --- & --- & 0.000$_{\pm0.000}$ & 188.8$_{\pm17.4}$\\
  & PCD-LB & 0.0 & 0.0$_{\pm0.0}$ & 47.76$_{\pm4.94}$ & --- & 3.8$_{\pm0.3}$ & 0.0 & 0.0$_{\pm0.0}$ & 50.44$_{\pm4.95}$ & --- & 4.0$_{\pm0.4}$ & 0.0 & 0.0$_{\pm0.0}$ & 46.76$_{\pm4.28}$ & --- & 6.1$_{\pm1.6}$\\
  & PCD-SHD & 0.0 & 0.0$_{\pm0.0}$ & 2.04$_{\pm0.52}$ & --- & 3.9$_{\pm0.1}$ & 0.0 & 0.0$_{\pm0.0}$ & 1.97$_{\pm0.41}$ & --- & 3.3$_{\pm0.1}$ & 0.0 & 0.0$_{\pm0.0}$ & 1.75$_{\pm0.36}$ & --- & 3.4$_{\pm0.2}$\\
  & Final Proj. & 0.0 & 0.0$_{\pm0.0}$ & \textbf{1.34}$_{\pm1.05}$ & 0.842$_{\pm0.021}$ & 17.5$_{\pm4.4}$ & 0.0 & 0.0$_{\pm0.0}$ & 1.08$_{\pm0.87}$ & 0.837$_{\pm0.022}$ & 20.4$_{\pm1.0}$ & 0.0 & 0.0$_{\pm0.0}$ & 0.90$_{\pm0.71}$ & 0.837$_{\pm0.024}$ & 16.0$_{\pm4.8}$\\
  & DiRecT & 0.0 & 0.0$_{\pm0.0}$ & 1.52$_{\pm0.87}$ & \textbf{0.883}$_{\pm0.018}$ & 80.9$_{\pm35.1}$ & 0.0 & 0.0$_{\pm0.0}$ & \textbf{0.37}$_{\pm0.35}$ & \textbf{0.886}$_{\pm0.018}$ & 57.5$_{\pm0.7}$ & 0.0 & 0.0$_{\pm0.0}$ & \textbf{0.15}$_{\pm0.12}$ & \textbf{0.889}$_{\pm0.017}$ & 57.5$_{\pm0.7}$\\
  & GlobMP (Ours) & \textbf{85.0} & \textbf{85.0}$_{\pm35.7}$ & 5.30$_{\pm13.46}$ & 0.765$_{\pm0.146}$ & 14.4$_{\pm116.6}$ & \textbf{75.0} & \textbf{75.0}$_{\pm43.3}$ & 9.80$_{\pm17.31}$ & 0.753$_{\pm0.139}$ & 197.1$_{\pm116.2}$ & \textbf{90.0} & \textbf{90.0}$_{\pm30.0}$ & 4.35$_{\pm13.07}$ & 0.761$_{\pm0.138}$ & 98.2$_{\pm103.0}$\\
\bottomrule
\end{tabular}}
\end{center}
\end{table}

Figures~\ref{fig:hardsuite-best64}--\ref{fig:hardsuite-scaling} and Tables~\ref{tab:hard-empty}--\ref{tab:hard-crossroom} summarize the results. We interpret \emph{DA} together with feasibility, since adherence to the demonstrated motion pattern is not meaningful if the generated trajectory violates the constraints, consistent with the evaluation of \citet{direct2026}. At the Medium velocity cap, averaged across all maps and team sizes, \emph{GlobMP (Ours)} attains a mean \emph{SR} of $0.941$, ahead of DiRecT at $0.601$, Final~Proj.\ at $0.419$, PCD-SHD at $0.211$, CD at $0.132$, MMD-CBS at $0.121$, PCD-LB at $0.063$, MDOC at $0.025$ and Diffusion Policy at $0.020$.

Performance generally decreases as the number of robots increases, although the effect depends strongly on the environment. At the Medium cap, DiRecT maintains \emph{SR}${}=1.00$ on \texttt{Empty} and \texttt{Highways} at every team size and at least $0.90$ on \texttt{Basic}, but at most $0.50$ on \texttt{Dense}. \texttt{Conveyor} and \texttt{Cross-Room} separate the methods most sharply: the best baseline is PCD-LB, at $0.65$ and $0.40$ on \texttt{Conveyor} and $0.45$ and $0.20$ on \texttt{Cross-Room} for $N_a=3$ and $6$, and from $N_a=9$ no baseline exceeds $0.10$, while \emph{GlobMP (Ours)} stays at $1.00$ through $N_a=15$ on \texttt{Cross-Room} and ranges from $1.00$ to $0.70$ on \texttt{Conveyor}. \texttt{Drop-Region} and \texttt{Corridor-Bay} provide intermediate difficulty, with DiRecT attaining mean \emph{SR} values of $0.70$ and $0.95$, respectively; DiRecT exceeds GlobMP on \texttt{Corridor-Bay} from $N_a=9$ (at least $0.90$ against $0.60$ to $0.75$) and on \texttt{Highways} at $N_a\ge18$ ($1.00$ against $0.90$ and $0.55$), and edges it on \texttt{Empty} at $N_a=21$ ($1.00$ against $0.95$). Changing the velocity cap has a comparatively small effect (four of the nine rows, marked $^{\dagger}$, ignore it), except for DiRecT on \texttt{Drop-Region}, which reaches $1.00$ at every team size only under the High cap, and for GlobMP on \texttt{Highways}, whose Low-cap SR falls from $0.80$ at $N_a=6$ to $0.60$ at $N_a=12$ and whose cells from $N_a=15$ hit the planning budget ($0.30$, $0$ and $0.15$), suggesting that environment geometry and coupled collision avoidance dominate the difficulty of the harder instances.

The \emph{GlobMP (Ours)} rows use the same shared instances across all eight maps, all velocity caps, and all $168$ experimental cells, with a $300$\,s planning budget per team; MMD-CBS has $60$\,s per configuration. Planning times are measured with one thread per job, pinned to four CPU cores. Cells that reach the full $300$\,s budget were repeated on an idle machine and showed the same behaviour, so these cases reflect planning difficulty, not resource contention: on \texttt{Highways} at the Low cap for $N_a\ge15$ and on \texttt{Corridor-Bay} at $N_a=18$ (High) and $N_a=21$ (Low).

The suite's success rate is a best-of-$64$ quantity: an instance counts as solved when any of the
$64$ candidates a baseline samples is collision-free, whereas our method returns one trajectory.
Figure~\ref{fig:hardsuite-scaling} therefore scores every method with per-candidate data on the success of a single
returned plan, and Figure~\ref{fig:hardsuite-best64} shows the best-of-$64$ success rate the tables report.

\endgroup
\FloatBarrier

\subsection{Implementation details}
\label{app:impl}
This subsection lists how GlobMP is run in the experiments: the curve class
of each experiment, the settings of Algorithm~\ref{alg:proj}, the cone solve, and the checks of the implementation.

\paragraph{Curve class.}
Table~\ref{tab:curveparams} gives the degree, the cost order $k$ and the
boundary order $\ell$ of each experiment; the synthetic study uses $7$
pieces ($4$ in the scaling study). The synthetic
study uses the minimum-snap class of the baselines it is compared against
($k=\ell=4$); the multi-robot tasks prescribe only where each robot starts and
stops, so they fix position and velocity ($\ell=1$) and minimize
acceleration ($k=2$); Maze2D uses $k=\ell=2$. In every row $2(\ell+1)\ge k$, so
Assumption~\ref{assum:curv} holds, and every degree is at least $\max\{2k-1,\ell\}$, so the synthetic class is that of Theorem~\ref{thm:globalinf}, with rate $N^{-1}$ ($\ell=k$); Maze2D joins its pieces $C^{\ell}$, a subclass. Every row starts from heuristic seeds rather than the grid start and restores with Algorithm~\ref{alg:proj} rather than the exact repair \eqref{eq:inner-step}; the multi-robot rows add a speed limit, and Maze2D adds corridor blocks, each box binding one piece, and keep-out blocks outside \eqref{eq:Pinf}, so the theorem certifies no row's cost as optimal. No steering experiment (Section~\ref{sec:exp-tactile}, Appendix~\ref{app:avoiding}) is in the class either.

\begin{table}[h]
\caption{Curve class of each planning experiment.}
\label{tab:curveparams}
\begin{center}
\small
\begin{tabular}{llccc}
\toprule
Experiment & Section & degree $d$ & $k$ & $\ell$ \\
\midrule
Synthetic study & \S\ref{sec:exp-synthetic} & $9$ & $4$ & $4$ \\
Multi-robot suite & \S\ref{sec:exp-multirobot} & $9$ & $2$ & $1$ \\
Robotarium & \S\ref{sec:exp-mrmp} & $12$ & $2$ & $1$ \\
Maze2D & \S\ref{sec:exp-maze} & $9$ & $2$ & $2$ \\
\bottomrule
\end{tabular}
\end{center}
\end{table}

\paragraph{Piecewise curves.}
\label{app:piecewise}
With knots $0=t_0<\cdots<t_N=1$, $\Delta_i=t_i-t_{i-1}$ and
$u=(t-t_{i-1})/\Delta_i$ on piece $i$, the curve is
$\g\gamma(t)=\Gm^{(i)}\m b_d(u)$ and
$\g\gamma^{(s)}(t)=\Delta_i^{-s}\Gm^{(i)}\m b_d^{(s)}(u)$. The endpoint
conditions of order $s\le\ell$ and the $C^{k-1}$ (or smoother) continuity conditions at the $N-1$ junctions are linear in the stacked control points, so
$\Gm\m A=\m B$ gives the same orthonormal chart as for one piece, with
$d_f=N(d+1)-\operatorname{rank}(\m A)$. The cost matrix is
$\operatorname{blkdiag}(\Delta_i^{1-2k}\m G^{(k)})$ (with the equal knots used here, the implementation drops the constant factor), and the constraint
blocks are built piece by piece with the derivative scaling $\Delta_i^{-s}$. With these junctions the
curves lie in $\trajspace^k$, and a zero-cost perturbation is a polynomial of
degree below $k$, zero when $k\le2(\ell+1)$.

\paragraph{Restoration (Algorithm~\ref{alg:proj}).}
The provable step $1/L_{\mathrm{proj}}$ (Appendix~\ref{app:stepdisc}) is
too small to be practical, so the implementation calibrates the step: it
starts at $10^{-3}$, caps it at $2V/\|\nabla V\|^2$ for the current merit value $V$, halves it until the merit
\eqref{eq:restoration-merit} decreases, and then grows it by a factor $1.3$
(on Maze2D and the Robotarium, Barzilai--Borwein steps, halved until the merit is below its largest value over the last ten steps); blocks within their tolerance are left out of the merit and its gradient, so $V$ and $\nabla V$ above are over the remaining blocks and the bound of Appendix~\ref{app:stepdisc} does not cover these steps. In the implementation, line~\ref{alg1:stop} of Algorithm~\ref{alg:proj} uses one tolerance per block: a curve is accepted when every block satisfies
$\|\m e_j\|^2\le\mathrm{tol}_j$, with $\mathrm{tol}_j=(2r_j\delta+\delta^2)^2$ for a
ball, pair or moving-ball block whose radius is inflated by $\delta$, $\delta^2$ for a box, and $10^{-12}$ otherwise.\label{app:tol}

\paragraph{Cone solve.}
\label{app:cone}
Each projection onto $\mathcal K_j$ uses Moreau's decomposition and fully
corrective nonnegative matching pursuit (Algorithm~\ref{alg:pk}), with the
block rescaled to unit norm. The implementation maximizes the unnormalized
correlation $\langle\m v,\m b_{m_j}(t)\rangle$ (Algorithm~\ref{alg:pk} normalizes it, and the rate of Appendix~\ref{app:innerrate} is for that version), a polynomial in $t$, at the
real roots of its derivative and the two endpoints (by Bernstein branch-and-bound above degree $24$), then refits the
nonnegative weights and drops atoms whose weight reaches zero; it stops at the relative threshold $\varepsilon=10^{-6}$. The active-set weight solve keeps linearly independent atoms, so at most $m_j+1$ ($2d+1$ for a collision block of degree $2d$); the solve also stops after $2(m_j+1)+10$ rounds.

\paragraph{Finishing step.}
In the synthetic study, the minimum-snap curvature ratio \eqref{eq:albeta}, in the chart of $\mathcal S_7$ of Proposition~\ref{prop:globalC},
$\beta/\alpha\approx4\times10^7$ at $N=7$ makes first-order descent slow along flat
directions, so each main-comparison run is finished by Newton's method on the first-order
conditions and the active contact equations. An interior contact at a time $t_\ast$
contributes $g_j^{\Gm}(t_\ast)=0$ and $\partial_tg_j^{\Gm}(t_\ast)=0$; the unknowns are the
free control points, the contact times and the multipliers; contacts with negative multipliers are dropped and the system is solved again. This
step lowered the cost by a further factor of $2$ to $3$; the reported
curves pass the dense recheck, but the step is numerical and is not covered
by the theory. As in Maze2D (\S\ref{sec:exp-maze}), the step follows an SQP polish of the output of Algorithm~\ref{alg:pgd}, also outside the theory.

\paragraph{Several robots.}
\label{app:multiagent}
The free control points of all robots are stacked; the robots share $d$,
$k$ and $\ell$, and hence $\m A$, $\m N$ and the cost matrix, and each has
its own boundary data. There is one collision block per robot and
obstacle and one per pair of robots, the latter the obstacle block applied
to the difference curve $\Gm_a-\Gm_b$ at the separation radius, so
Algorithms~\ref{alg:pgd} and~\ref{alg:proj} are unchanged. Each block reads the control points of one or two robots only, so the gradient of Algorithm~\ref{alg:proj} costs a fixed amount per block.

\paragraph{Workspace bounds.}
\label{app:box}
A bound $\underline x_i\le\gamma_i(t)\le\overline x_i$ is two
nonnegativity conditions with Bernstein coefficients
$\Gm_{i,:}^{\!\top}-\underline x_i\m u$ and $\overline x_i\m u-\Gm_{i,:}^{\!\top}$. Since
$\m u^{\!\top}\m b_d\equiv1$, a degree-$d$ coefficient vector $\g\zeta$ is
raised to degree $2d$ by $S_d\big(\tfrac12(\g\zeta\m u^{\!\top}+\m u\g\zeta^{\!\top})\big)$,
so these blocks share the collision blocks' cone and are affine
in $\Gm$.

\paragraph{Checks of the implementation.}
Each of the following holds to machine precision on the benchmark
instances: $S_d$ reproduces the Bernstein coefficients of
$\m b_d^{\!\top}\m X\m b_d$; the chart satisfies $\Gm\m A=\m B$ and
$\m N^{\!\top}\m A=\m 0$; $\g\pi_j$ reproduces $\|\g\gamma(t)-\m c_j\|^2-r_j^2$;
the cost matches $\int_0^1\|\g\gamma^{(k)}\|^2dt$ up to the dropped factor; and the cone test is exact
in both directions, returning zero for collision-free curves and a positive
value, increasing with the penetration depth, for colliding ones.

\begingroup
\renewcommand{\floatpagefraction}{0.75}
\renewcommand{\textfraction}{0.01}

\FloatBarrier

\begin{figure}[p]
\centering
\includegraphics[width=\linewidth]{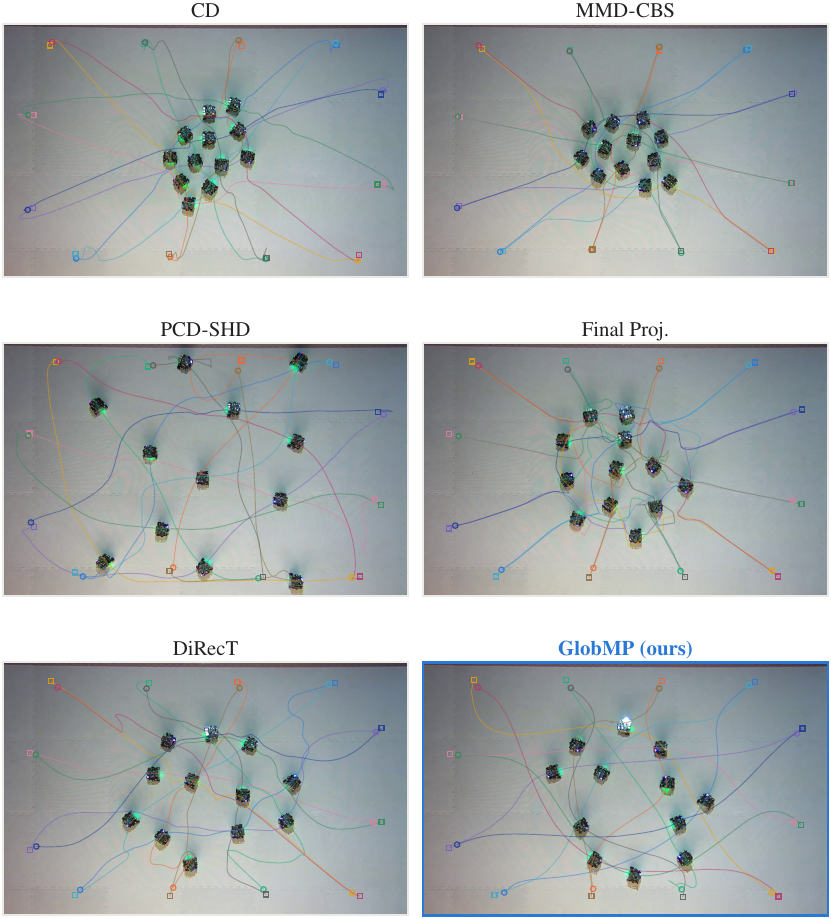}
\caption{Hardware execution on \texttt{Empty}: every method except PCD-LB, twelve GRITSBots on the
Georgia Tech Robotarium. The tracking controller and the virtual obstacles are those of
Appendix~\ref{app:expdetail}. Each panel reports one
execution of the ring instance by the method named above it: a single overhead
frame, with the trajectories recovered from the video drawn over it, each marked by a
circle at its start and a square at its terminal state. A robot is counted as being in
collision when its footprint overlaps an obstacle region. The frame shown is the midpoint of the
tracked window for an execution in which no robot overlaps a region at any time, and
otherwise the frame at which the number of overlapping robots is maximal; only the
midpoints are therefore proportional to the duration of an execution. \texttt{Empty}
carries no obstacles, so no execution can overlap a region and the
panels differ only in the routing of the exchange, with midpoints at $25$\,s for GlobMP,
$27$\,s for PCD-SHD, $28$\,s for Final~Proj., $29$\,s for DiRecT, $34$\,s for CD and
$57$\,s for MMD-CBS.}
\label{fig:rg-empty}
\end{figure}
\begin{figure}[p]
\centering
\includegraphics[width=\linewidth]{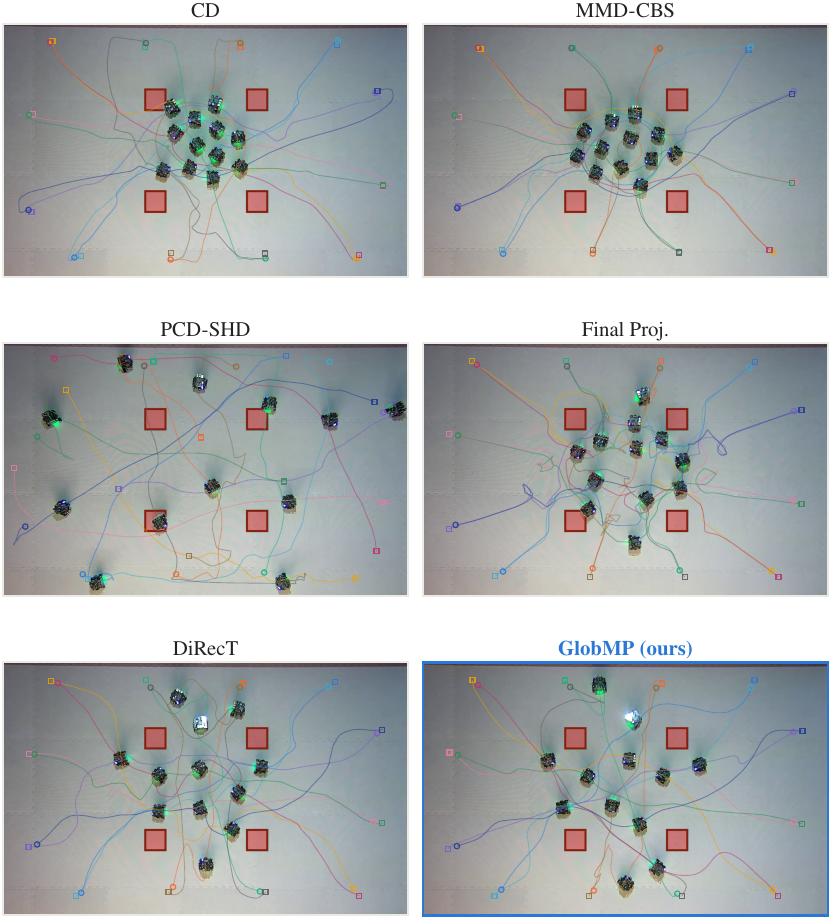}
\caption{Hardware execution on \texttt{Drop-Region}: the same methods, with the protocol of
Figure~\ref{fig:rg-empty}. Three of the six methods complete the exchange with no footprint
overlap: MMD-CBS (midpoint frame at $61$\,s), DiRecT ($31$\,s) and GlobMP ($31$\,s). For
the remaining three the panel shows the frame of maximal overlap: two of the twelve robots
overlap a keep-out region under Final~Proj.\ ($40$\,s), three under CD ($43$\,s) and five
under PCD-SHD ($23$\,s). PCD-SHD is additionally the only execution on this map whose
robots do not reach their goals; its terminal states lie well away from the goal ring. This map discriminates least
among the four maps with obstacles, and the discrimination that remains lies in when the
overlaps occur rather than in how many robots they involve.}
\label{fig:rg-shelf}
\end{figure}
\begin{figure}[p]
\centering
\includegraphics[width=\linewidth]{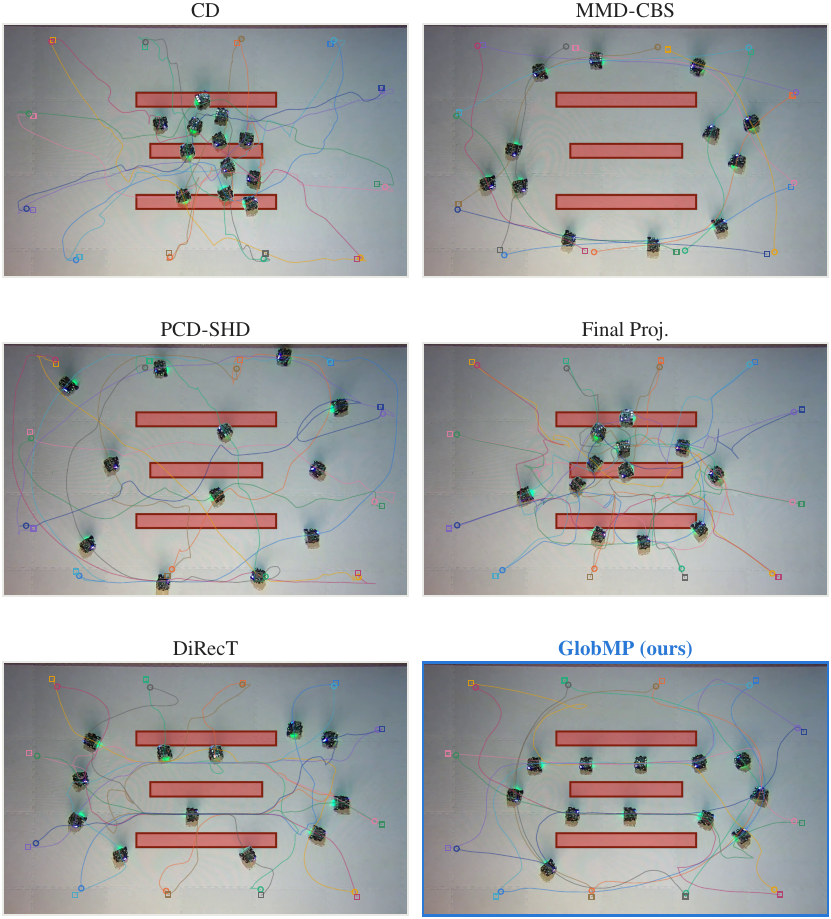}
\caption{Hardware execution on \texttt{Conveyor}: the same methods, with the protocol of
Figure~\ref{fig:rg-empty}. The three belts sit in the middle of the arena and leave two gaps between them,
so the exchange requires an ordering of the team as well as a route; the executions
differ principally in that ordering. GlobMP completes the instance with no footprint overlap
(midpoint frame at $37$\,s). PCD-SHD and DiRecT each leave five of the twelve robots
overlapping a belt at their frames of maximal overlap, $44$\,s and $65$\,s respectively,
the latter being the latest frame on this map. CD and Final~Proj.\ leave all twelve
overlapping, at $40$\,s and $48$\,s: both traverse the belts as a single block, so the
entire team is inside a region at one instant. MMD-CBS returned no plan on this map; its
robots therefore remain on the start ring, and its panel is marked rather than scored,
since counting overlaps for a team that never moved would record a success where no
trajectory was produced.}
\label{fig:rg-corridor}
\end{figure}
\begin{figure}[p]
\centering
\includegraphics[width=\linewidth]{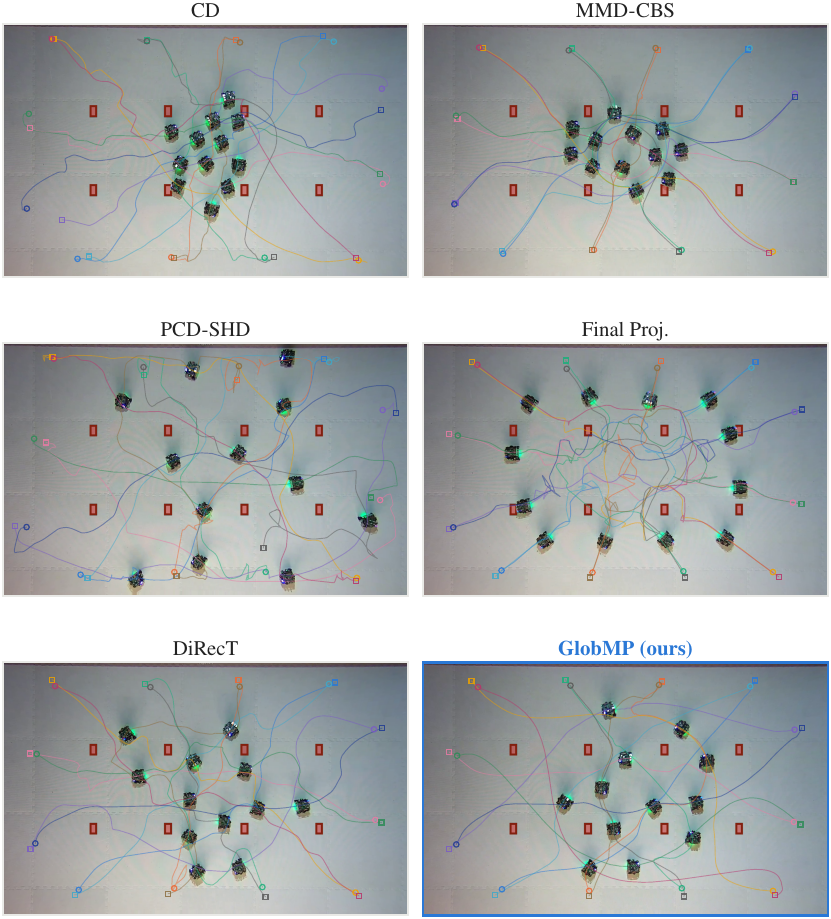}
\caption{Hardware execution on \texttt{Corridor-Bay}: the same methods, with the protocol of
Figure~\ref{fig:rg-empty}. As on \texttt{Drop-Region}, three methods complete the exchange with no footprint overlap:
PCD-SHD (midpoint frame at $45$\,s), DiRecT ($32$\,s) and GlobMP ($26$\,s). The remaining
three overlap the keep-out regions at their frames of maximal overlap: three of the twelve
robots under MMD-CBS ($42$\,s), four under Final~Proj.\ ($78$\,s), the latest frame of
maximal overlap in this gallery, and eight under CD ($49$\,s). Final~Proj.\ remains clear for the first minute and overlaps only as
the team disperses towards its goals, whereas CD overlaps from early in the execution.
MMD-CBS, which returns no plan on two of the other maps, produces one here and brings nine
of the twelve robots through without overlap.}
\label{fig:rg-corridorbay}
\end{figure}
\begin{figure}[p]
\centering
\includegraphics[width=\linewidth]{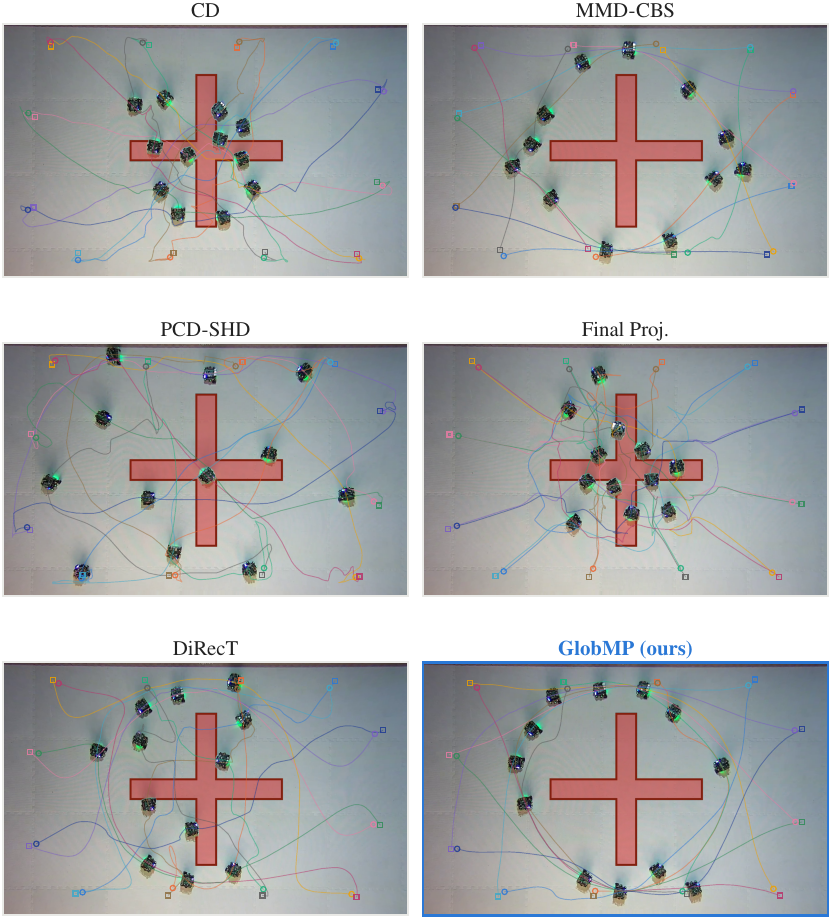}
\caption{Hardware execution on \texttt{Cross-Room}: the same methods, with the protocol of
Figure~\ref{fig:rg-empty}. The cross occupies the centre of the arena and leaves four
corner regions connected only along the walls, so the exchange must be routed around the
arms; a team that advances directly towards its goals enters one of them. GlobMP is the
only method that completes the instance with no footprint overlap (midpoint frame at
$26$\,s). At their frames of maximal overlap, six of the twelve robots overlap the cross
under DiRecT ($37$\,s) and seven under PCD-SHD ($41$\,s); both route part of the team
around the arms and lose the remainder. CD and Final~Proj.\ leave all twelve overlapping,
at $38$\,s and $34$\,s. MMD-CBS returned no plan on this map as well, so its panel again
shows the untouched start ring and is marked rather than scored. With \texttt{Conveyor},
this is the map on which the margin between GlobMP and the baselines is largest under
hardware execution.}
\label{fig:rg-crossroom}
\end{figure}
\endgroup

\end{document}